%% file: main.tex
\documentclass{article}

\usepackage{etoolbox}
\usepackage{graphicx} % Required for inserting images
\usepackage[utf8]{inputenc} % allow utf0-8 input
\usepackage[T1]{fontenc}    % use 8-bit T1 fonts
\usepackage{url}            % simple URL typesetting
\usepackage{booktabs}       % professional-quality tables
\usepackage{amsfonts}       % blackboard math symbols
\usepackage{nicefrac}       % compact symbols for 1/2, etc.
\usepackage[nopatch=footnote]{microtype}
\usepackage{comment}
\usepackage{amsthm}
\usepackage{mathtools}
\usepackage{amsmath}
\usepackage{array,makecell}
\usepackage{tabularx}
\usepackage{bbm}
\usepackage[toc,page]{appendix}
\usepackage{caption}
\usepackage{setspace}
\usepackage{xcolor}         % colors
\usepackage{authblk}
\usepackage{listings}
\usepackage{wrapfig}
\usepackage{algorithm}
\usepackage{algorithmic}
\usepackage{float}
\usepackage{xspace}

\usepackage{pdflscape}
\usepackage[export]{adjustbox}
\usepackage[hidelinks]{hyperref}       % hyperlinks
\usepackage{amssymb}   % For checkmarks (\checkmark)
\usepackage{silence}
\usepackage{times}

\newbool{ispreprint}
\setbool{ispreprint}{true}

\ifbool{ispreprint}{
\usepackage[numbers,sort&compress, square]{natbib}
\let\citet\citep
\usepackage[margin=1.1in]{geometry}
\setlength{\parindent}{0pt}     % No paragraph indentation
\setlength{\parskip}{1em} 
}{
\usepackage{iclr2027_conference,times}
}
\input{math_commands.tex}

\newtheorem{theorem}{Theorem}
\newtheorem{lemma}{Lemma}
\newtheorem{proposition}{Proposition}

\newtheorem{remark}{Remark}
\theoremstyle{definition}
\newtheorem{assumption}{Assumption}

\renewcommand{\E}{\mathbb{E}}
\renewcommand{\R}{\mathbb{R}}

\newcommand{\Dall}{\mathcal{D}_\mathrm{all}}
\newcommand{\Dcal}{{\mathcal{D}_\mathrm{cal}}}
\newcommand{\Dcalone}{{\mathcal{D}_\mathrm{cal1}}}
\newcommand{\Dcalcrc}{{\mathcal{D}_\mathrm{crc}}}
\newcommand{\Dcaltwo}{{\mathcal{D}_\mathrm{cal2}}}
\newcommand{\tauprior}{\tau_\mathrm{prior}}
\newcommand{\calI}{\mathcal{I}_\mathrm{cal}}
\newcommand{\calIone}{\mathcal{I}_{\mathrm{cal1}}}
\newcommand{\calIcrc}{\mathcal{I}_{\mathrm{crc}}}
\newcommand{\calItwo}{\mathcal{I}_{\mathrm{cal2}}}
\newcommand{\St}{S_{\mathrm{cal2}}}
\newcommand{\Xtest}{X_{\text{test}}}

\newcommand{\Ttest}{T_{\text{test}}}
\newcommand{\calS}{\mathcal{S}}

\newcommand{\calG}{\mathcal{G}}
\newcommand{\calA}{\mathcal{A}}

\newcommand{\Ttilde}{\tilde{T}}

\renewcommand{\eps}{\varepsilon}
\newcommand{\Bbar}{\bar{B}_2}
\newcommand{\qprior}{\hat{f}_{\mathrm{prior}}}
\newcommand{\qhat}{{\hat{q}}}
\newcommand{\fhat}{\hat{f}}
\newcommand{\ttmethod}{\texttt{DAPRO}\xspace}

\newcommand{\indep}{\perp \!\!\! \perp}
\newcommand{\Htest}{H_\mathrm{test}}
\newcommand{\tautarget}{\alpha}

\newcommand{\Dtrain}{\mathcal{D}_\mathrm{train}}
\newcommand{\bfunc}{\mathcal{B}}
\newcommand{\maxl}{{t_\mathrm{max}}}
\newcommand{\ttuer}{{\texttt{ER}}\xspace}

\newcommand{\ttrmttu}{{\texttt{RMTTE}}\xspace}

\newcommand{\chat}{{\xi}}

\ifbool{ispreprint}{

}{

}

\title{Spend Where It Counts: Dynamic Budget Allocation for Multi-Turn LLM Evaluation}

\ifbool{ispreprint}{
\author[1]{Shai Feldman}
\author[1,2]{Yaniv Romano}
\affil[1]{Department of Computer Science, Technion IIT, Israel}
\affil[2]{Department of Electrical and Computer Engineering, Technion IIT, Israel}
}{
\author{%
   Shai Feldman \\
   Department of Computer Science \\
   Technion, Israel \\
\texttt{shai.feldman@cs.technion.ac.il}
   \and
   Yaniv Romano \\
   Departments of Electrical \& Computer Engineering \\ 
   and Computer Science \\
   Technion, Israel \\
   \texttt{yromano@cs.technion.ac.il}
}}
\begin{document}
\date{}
\maketitle
\begin{abstract}
Evaluating and predicting the performance of large language models (LLMs) in multi-turn conversational settings is critical yet computationally expensive; key events---e.g., jailbreaks or successful task completion by an agent---often emerge only after repeated interactions. These events might be rare and, under any feasible computational budget, remain unobserved. Recent conformal survival frameworks construct reliable lower predictive bounds (LPBs) on the number of iterations to trigger the event of interest, but rely on static budget allocation that prohibits adaptivity in multi-turn setups. To address this, we introduce \emph{Dynamic Allocation via PRojected Optimization} (\ttmethod), a theoretically valid dynamic budget allocation framework for bounding the time-to-event in multi-turn LLM interactions. We prove that \ttmethod satisfies the expected budget constraint and provides distribution-free, finite-sample coverage guarantees without requiring the conditional independence between censoring and event times assumed by prior conformal survival approaches. A key theoretical contribution is a novel coverage bound that scales with the square root of the mean censoring weight over only a subset of the samples rather than the worst-case weight, yielding tighter guarantees than prior work. Furthermore, \ttmethod can be employed to obtain unbiased estimates of population-level evaluation metrics, such as the jailbreak rate, under limited computing resources. Comprehensive experiments across agentic task success, LLM jailbreaks, toxic content generation, and RAG hallucinations demonstrate that \ttmethod achieves coverage closer to the nominal level with lower variance than static baselines, while not exceeding the budget constraint in expectation.

\ifbool{ispreprint}{
\par\smallskip
\noindent\textbf{Project website:}
\url{https://shai128.github.io/dapro/}
}

\end{abstract}

\section{Introduction}
\label{sec:introduction}

AI systems are moving from one-shot prediction to sequential interaction~\citep{park2023generative, wang2024survey}. For example, an agent may need many steps to solve a complex task. 
Evaluating such systems is challenging because important outcomes may emerge only after many interactions, while running these interactions can be computationally expensive. 
Rather than asking only whether an AI system will eventually succeed or fail at a task, we ask a richer question: how long will it take for an important event to occur, if ever? For example, how many interactions are needed for a jailbreak to succeed or for an agent to complete a task?

Our framework applies broadly to evaluating time-to-event---time for an outcome of interest---in sequential AI systems.
For simplicity, we use LLM safety as a running example, where the event of interest may correspond to an unsafe response, such as toxic language~\citep{gehman2020realtoxicityprompts} or a hallucination~\citep{ji2023survey}. 
In agentic utility settings, the event could be a successful task completion, e.g., identifying a bug in code~\citep{jimenez2024swe} (see Appendix~\ref{sec:utility_upb}).
We assume that event occurrence is determined by an audit function, such as an LLM-as-a-judge~\citep{zheng2023judging}.
Since these events can be rare, particularly for well-aligned models,~\citet{davidov2026calibrated} recently introduced \emph{time-to-unsafe} to quantify model safety. In our setting, this metric represents the number of conversational interactions required to elicit an unsafe response given an initial prompt. 

Using this metric, LLM safety or utility assessment can be framed as a survival analysis task~\citep{candes2023conformalized,gui2024conformalized, sesia2024doubly, sesia2025conformal}, which has two advantages. First, the metric is well-defined despite the time dependence of LLM outputs. Second, it enables constructing calibrated \emph{lower predictive bounds} (LPBs) $\hat{L}(\Xtest)$ or \emph{upper predictive bounds} (UPBs) $\hat{U}(\Xtest)$ for the unknown time-to-event $\Ttest$ of a test prompt $\Xtest$, with formal finite-sample coverage guarantees $\mathbb{P}(\Ttest \geq \hat{L}(\Xtest)) \geq 1-\alpha$ and $\mathbb{P}(\Ttest \leq \hat{U}(\Xtest)) \geq 1-\alpha$. Here $1-\alpha$ is the target coverage rate, e.g., $1-\alpha=90\%$. Informally, an LPB indicates that the event is unlikely to occur before $\hat L(\Xtest)$, while a UPB indicates that it is likely to occur by $\hat U(\Xtest)$. Thus, LPBs are natural for events such as unsafe responses, whereas UPBs are useful for events such as agentic task completion.

\begin{figure}[ht]
    \centering
    \includegraphics[width=0.9\linewidth]{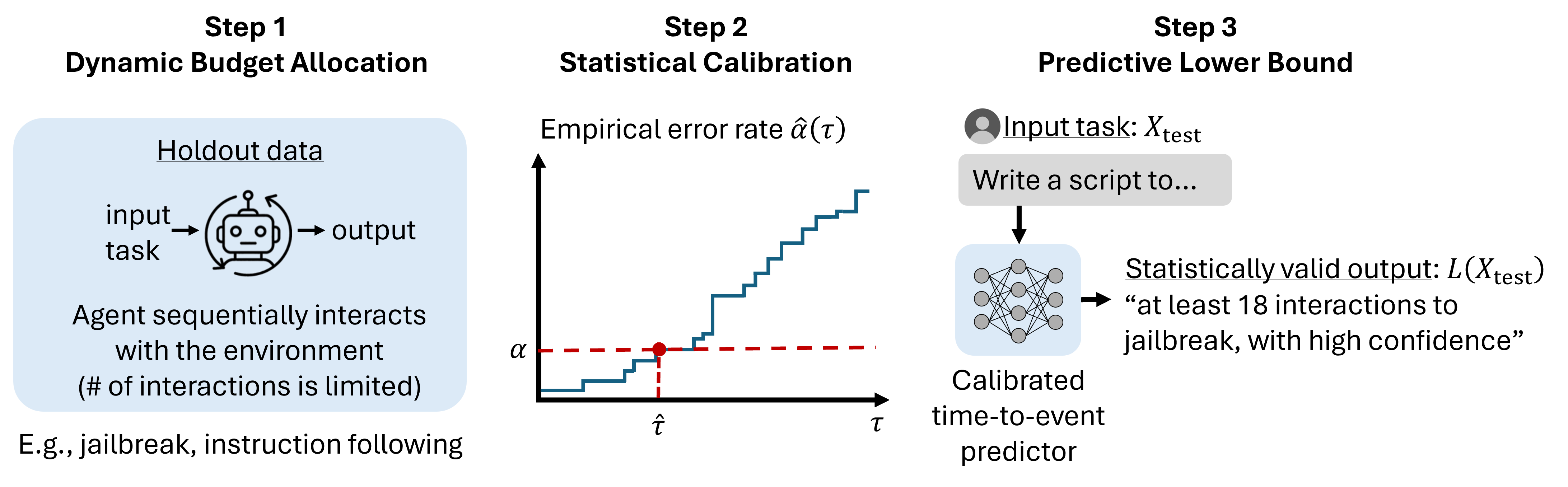}
\caption{Illustration of our framework: (i) collecting data via dynamic budget allocation; (ii) calibrating a pre-trained model; and (iii) constructing calibrated predictive bounds for time-to-event.}
    \label{fig:framework}
\end{figure}

To construct these bounds, we must first acquire a dataset of annotated conversations under the budget constraint. For example, jailbreaking requires running an attacking algorithm starting from a set of initial prompts and continuously evaluating the outputs of the target model using the auditor; similarly, evaluating agentic systems may require running agents over multiple interactions until task completion.
In our setup, we suppose the expected number of interactions is constrained by a computational budget $B\in \mathbb{N}$, e.g., $B=1,000$ interaction rounds in expectation. Formulating the budget constraint in expectation follows the active learning literature~\citep{dasgupta2007general, altman2021constrained, zrnic2024active}.
To satisfy this budget constraint, we set a censoring time $C_i$ for each sequence $i$, which limits the maximal trajectory length. Importantly, under such a constraint, the event of interest may not be observed for every sequence.

To provide coverage guarantees for unseen test points, we assume that the calibration and test data are drawn i.i.d. from the same setting; in particular, the prompt distribution and the agent (or attacker), target, and audit models are fixed. This assumption is natural in settings where calibration and test interactions involve different users drawn from the same underlying population, such as educational AI tutors~\citep{xie2025watermark} or customer-service chatbots~\citep{pandya2023automating}. Even under this assumption, constructing valid bounds from censored data remains challenging: naively calibrating using the partially observed event times can lead to undercoverage, as we observe empirically.

\citet{davidov2026calibrated} propose an optimized budget allocation strategy that aims to minimize the coverage variance of the resulting bounds.
Note that obtaining a lower variance is crucial in practice, where only a single dataset is typically available. A lower variance means that our bounds attain a coverage rate closer to the target level $1 -\alpha$ given the single dataset. However, the strategy of~\citet{davidov2026calibrated} sets static censoring times that are fixed before any interaction occurs and are not updated over time, even as the interaction sequence unfolds. 
Consequently, this static method is not ideal for our multi-turn setup. Furthermore, it often under-utilizes the budget; when an event is encountered before the budget is fully consumed, the trajectory terminates, and the remaining budget dedicated to that sample is lost: it cannot be redistributed to other samples.

However, dynamically allocating the censoring times is mathematically challenging. First, different sequences cannot actively share information or coordinate their decisions to stop or continue, since such communication breaks the independence assumptions required for this theoretical guarantee. Second, the online adaptation of $C_i$ must respect the shared expected budget, even though future trajectories are unknown. Third, beyond validity, we should allocate this budget efficiently.

Importantly, we view our framework as a statistical tool for quantifying time-to-event under budget constraints.
Predictive bounds can support safety assessment, model evaluation, and decision-making under budget constraints. For example, LLM developers could apply expensive audit functions, use a more conservative system prompt, or impose stricter sampling limits for prompts having a low LPB for the time-to-unsafe. Our framework is illustrated in Figure~\ref{fig:framework}.

\subsection{Our contribution}

In this work, we introduce \emph{Dynamic Allocation via PRojected Optimization} (\textbf{\ttmethod}), a novel paradigm that treats budget allocation as a sequential decision to form a statistically valid bound for the time-to-event. 
Specifically, our main contributions are:
\begin{itemize}
    \item \textbf{Tighter Theoretical Guarantees}: We prove that \ttmethod attains the desired coverage level with a novel bound that scales with the \emph{mean} inverse-censoring weight only on a subset of trajectories rather than the \emph{maximum} weight of all trajectories, leading to tighter guarantees than prior methods~\citep{davidov2026calibrated, gui2024conformalized}. Crucially, our proof technique also removes the restrictive assumption that the censoring time $C$ and true event time $T$ must be conditionally independent given the prompt $X$. This result is of independent interest as it applies to a broad class of survival analysis tasks.
    \item \textbf{Budget Validity}: We formally prove that \ttmethod satisfies the expected budget constraint.
    \item \textbf{Population-Level Performance Evaluation Under Censoring}: 
    \ttmethod can efficiently evaluate AI systems over a fixed benchmark: it provides unbiased estimates of metrics over the benchmark, such as the unsafe event rate, substantially reducing the costs of LLM API calls. Notably, this task does not require the i.i.d. assumption, since the metrics are evaluated directly over the benchmark of interest. Our experiments show that \ttmethod reduces estimator variability and lowers the risk that a model appears safer than it is.

    \item \textbf{Comprehensive Empirical Evaluation}: We validate \ttmethod across diverse tasks, including agentic task success, toxic content, LLM jailbreaks, and RAG-based hallucinations, demonstrating that our proposal generally achieves empirical coverage closer to the nominal level and utilizes the budget more efficiently compared to existing baselines.

\end{itemize}
\ifbool{ispreprint}{
\noindent
The source code and additional resources are
available at
\url{https://shai128.github.io/dapro/}.
}

\subsection{Problem setup}

We consider a collection of base objectives $\{V_i\}_{i=1}^{N}$. In safety setups, an objective may correspond to a harmful intent, e.g., $V_i=$ ``How can I bug someone's phone?'', whereas in an agentic utility setting it may represent a task such as ``Find a bug in the code.'' For each objective, an attacker or agent model $\calA$ generates an initial prompt $X_i=\calA(V_i)$. For instance, an attacker might wrap the objective in a roleplay persona, while an agent may implement a unit test. We define the resulting calibration dataset as $\Dcal=\{X_i\}_{i=1}^N$.
We assume that the training data $\Dtrain$, used to fit the predictive model estimating $T\mid X$, has a separate budget. Importantly, our framework does not require fully observed training conversations: any predictor trained independently of the calibration set may be used, including one trained on censored data, data from another distribution, or external model predictions~\citep{cunningham2026constitutional}. Furthermore, the predictor can be pretrained once and reused across re-calibrations over different model configurations or benchmarks.

For each prompt $X_i$ in $\Dcal$, the model $\calA$ interacts sequentially with a target model $\mathcal{G}$ to elicit an event of interest.
At step $t=1$, the history is initialized with the initial prompt: $H_i(1)=\{A_i(1)\}$, where $A_i(1) := X_i$. At any step $t\geq 1$, the target model processes this prompt and outputs a response: $R_i(t) = \mathcal{G}(A_i(t))$. Next, we employ an audit function $\mathcal{J}$ to obtain a binary label $Y_i(t)= \mathcal{J}(A_i(t), R_i(t)) \in \{0, 1\}$, where $Y_i(t)=1$ indicates that the event of interest has occurred (e.g., bug found), and $Y_i(t)=0$ otherwise (e.g., task failed).
If the key event occurred ($Y_i(t)=1$), we terminate this process. Otherwise, we continue to step $t+1$, where $\calA$ generates a refined prompt $A_i(t+1) = \calA(H_i(t) \cup \{(R_i(t), Y_i(t))\})$ based on the target's previous responses. We update the interaction history: $H_i(t+1)= H_i(t) \cup \{(R_i(t), Y_i(t), A_i(t+1))\} \in \mathcal{H}_{t+1}$. 

We define the uncensored time-to-event as the minimum number of steps required to achieve the event: $T_i = \min \{t \geq 1: Y_i(t) = 1\}$. 
The censored time-to-event is thus $\tilde{T}_i = \min (T_i, C_i)$.
For notational convenience, we denote the complete interaction history upon termination as $H_i = \{H_i(t)\}_{t=1}^{\min\{\maxl,T_i\}}$, where $\maxl$ is the maximal sequence length.
If the allocated budget exceeds the true event time ($C_i > T_i$), the interaction naturally terminates at $T_i$, and no further budget is consumed. Therefore, the expected calibration budget constraint can be formulated as:
\begin{equation}\label{eq:budget_constraint}
    \mathbb{E} \left[ \sum_{i=1}^{N} \tilde{T}_i \right] \leq B.
\end{equation}

Our goal is to construct an LPB $\hat{L}(\Xtest)$ given an initial test prompt $\Xtest$ for the test event time $\Ttest$. Importantly, at test time, we construct the LPB \emph{without} running any LLM exchanges, relying solely on the initial prompt $\Xtest$. This design choice reflects two practical considerations. First, a test-time budget may be extremely limited or entirely unavailable.
Second, when the target model operates autonomously and can act on external systems, assessing a prompt before the interaction begins can be valuable, since unsafe actions taken during the interaction may be difficult to reverse.
Finally, given a target coverage level $1-\alpha\in (0,1)$ and a tolerance level $\delta \in (0, 1)$, the LPB should satisfy:
\begin{equation}\label{eq:lpb_goal}
    \mathbb{P}\left(\hat{L}(\Xtest) \leq \Ttest \mid \Dall \right) \geq 1 - \alpha,
\end{equation}
with probability at least $1-\delta$ over the randomness of the data $\Dall=\Dtrain \cup \Dcal$. We remark that the probability in~\eqref{eq:lpb_goal} is taken over random draws of $\Xtest, \Ttest$. An LPB $\hat{L}$ satisfying this requirement is
called a Probably Approximately Correct (PAC) LPB at level $\alpha$ with tolerance $\delta$.

\section{Background}
\label{sec:background}

Our approach builds on the framework of~\citet{davidov2026calibrated}, which refines conformalized survival analysis~\citep{gui2024conformalized} to construct a valid LPB for the time-to-unsafe under a known censoring mechanism. We outline the components relevant for our dynamic censoring mechanism.

Suppose we are given quantile estimates $\qhat_\tau(x)$ for the time-to-event from a pre-trained model. Since the raw estimates $\qhat_\alpha(x)$ might not be sufficiently accurate, using them directly as LPBs might not achieve $1-\alpha$ coverage. Thus, we tune a quantile level $\hat{\tau}$ such that the resulting LPB $\qhat_{\hat{\tau}}(x)$ satisfies the coverage requirement. We obtain this calibrated quantile level by weighting the errors of the raw quantile estimates to account for the distribution shift induced by the censoring:
\begin{equation}
\label{eq:miscoverage_estimator}
    \hat\alpha(\tau) = \frac{1}{|\calI|}\sum_{i\in \calI} w_\tau(i)\; \mathbb{I}\bigl\{\tilde{T}_i<\hat q_\tau(X_i)\le C_i\bigr\},
\end{equation}
where the weights are formulated as:
% \begin{equation}
    $w_\tau(i)= {\mathbb P\bigl[\hat q_\tau(X_i)\le C_i\bigm| \{X_j\}_{j\in\calI}\bigr]}^{-1}$.
% \end{equation}
Notice that these weights are known since the censoring mechanism is actively designed by the algorithm. The calibrated LPB is then $\hat L(x)=\hat q_{\hat\tau}(x)$, where $\hat{\tau}$ is defined as the largest level achieving valid coverage:
\begin{equation}
\label{eq:our_calibration}
\hat\tau = \sup\bigl\{\tau \in \mathcal{T}:\sup_{\tau' \le \tau}\hat\alpha(\tau')\le\alpha\bigr\}.
\end{equation}
Above, $\mathcal{T}$ is the search space for $\tau$. As established by~\citet{davidov2026calibrated,gui2024conformalized}, this LPB holds a PAC-type coverage guarantee that, informally, includes two terms: (i) the desired coverage level $1-\alpha$, as in~\eqref{eq:lpb_goal}; and (ii) a slack term that scales with the worst-case maximal weight $w_\tau(i)$. Further,~\citet{davidov2026calibrated} show that the variance of the LPBs is governed by the \emph{average} weight; however, this connection was only justified under restrictive oracle assumptions. 

To control the value of the maximal weight and bound the number of LLM interactions,~\citet{davidov2026calibrated} introduce two refinements that we adopt.

First, the raw LPB estimates are capped at $M \leq \maxl$, yielding trimmed estimates $\fhat_\tau(x):=\min(\hat{q}_\tau(x),M)$. By replacing $\hat{q}$ with $\fhat$ in~\eqref{eq:miscoverage_estimator}, this trimming bounds the maximal weight, affecting the tightness of the PAC-type coverage guarantee.
From an evaluation perspective, the trimming allows us to assess the coverage rate of the LPB. With this trimming, we can always observe whether the interaction survives past the predicted LPB.

The second refinement that we adopt from~\citet{davidov2026calibrated} is the use of a fixed quantile level $\tau$ for the budget allocation. In more detail, recall from~\eqref{eq:miscoverage_estimator} that a conversation $i$ contributes to the miscoverage estimator $\hat\alpha(\tau)$ only when $\fhat_\tau(X_i) \leq C_i$. To allocate the budget efficiently, we aim to maximize the number of samples satisfying this condition. However, since the calibrated quantile $\hat\tau$~\eqref{eq:our_calibration} is unknown prior to budget allocation, we follow~\citet{davidov2026calibrated} and fix a prior quantile level $\tauprior$, reflecting our belief that $\hat{\tau} \in [0, \tauprior]$. For example, with $1-\alpha=90\%$ and a reasonably accurate model, $\tauprior=30\%$ is an appropriate choice, as the $\alpha$-quantile is expected to lie well below $30\%$. Despite using a fixed $\tauprior$,~\citet{davidov2026calibrated} showed that it is still possible to construct valid weights $w_\tau(i) = \mathbb{P}(\fhat_{\tau}(X_i)\leq C_i)^{-1} $ for all $\tau \leq \tauprior$, which is required to compute~\eqref{eq:miscoverage_estimator}.
This procedure is summarized in Algorithm~\ref{alg:lpb_construct} in Appendix~\ref{sec:lpb_construct}.

Due to space constraints, we include an extended discussion of additional related works in conformal prediction, survival analysis, and LLM evaluation in Appendix~\ref{sec:related_work}.

\section{Proposed Method}
\label{sec:method}

Our adaptive budget allocation strategy operates in two phases. Before these phases, we randomly partition the calibration points, indexed by $\calI=\{1,\ldots,N\}$, into three disjoint sets: $\calIone$, $\calIcrc$, and $\calItwo$.
We re-index the samples in the first split as $i \in \{1, \ldots, N_1\}$ for notational convenience.

In the first phase, we learn an optimized acquisition policy on $\calIone$.
Our allocation strategy relies on a scoring function $\mathcal{S}_t : \mathcal{H}_t \rightarrow \mathbb{R}$ that maps the conversation trajectory observed up to time $t$ to a real-valued score. A higher score indicates a stronger signal to spend a budget unit and advance to the next step. For example, this score could be the estimated probability of observing an event at the current step. This scoring function can be trained on the training data $\Dtrain$.
Our proposed policy operates iteratively: at each step $t$, we compute the current score $S_i(t)$ and map it to a continuation probability $P_i(t)\in[0,1]$. Then, with probability $P_i(t)$, we advance to the next step; otherwise, the acquisition for this sample is terminated. 
We aim to optimize this score-to-probability mapping over $\calIone$ to reduce the variance of the weighted miscoverage estimator~\eqref{eq:miscoverage_estimator} while satisfying an expected budget constraint. This allows us to allocate more budget to trajectories that are more likely to reduce the variance of the miscoverage, while terminating less informative conversations early.
After learning the mapping, we leverage $\calIcrc$ to calibrate its expected budget with conformal risk control (CRC)~\citep{angelopoulos2024conformal}. In the second phase, we deploy this policy on the remaining indices $\calItwo$ to acquire new interactions. By leveraging the interactions observed in $\calIone$, we selectively terminate conversations that are unlikely to contribute to reducing the LPB variance.

\subsection{Phase I: Learning the Optimized Acquisition Policy on \texorpdfstring{$\calIone$}{I1}}

To learn an optimized mapping from scores to probabilities, we must first expend budget to observe the full conversation history for all samples in $\calIone\cup \calIcrc$, until we observe an event, or reach $\qprior(X_i)$, where $\qprior(\cdot):=\fhat_{\tauprior}(\cdot)$. For all $i \in \calIone$, we set the censoring times as $ C_i = \qprior(X_i)$, and expend $b_i := \min(T_i, \qprior(X_i))$ budget units. We require that the total budget $B$ is sufficient for this observation phase, such that $B > \sum_{i \in \calIone \cup \calIcrc} b_i$.
The budget remaining for Phase II and the CRC is therefore $B_\textrm{rem} := \left(B - \sum_{i \in \calIone} b_i\right)$ and the per-sample budget is $\Bbar := \frac{1}{|\calIcrc|+|\calItwo|}B_\textrm{rem} $.

After annotating the data in $\calIone$, we compute the scores for all observed time steps: $S_i(t) = \mathcal{S}_t({H}_i(t))$.
The inverse-probability weights are given by $w(i)=\mathbb{P}(\qprior(X_i) \leq C_i)^{-1}=\left( \prod_{t=1}^{b_i} P_i(t) \right)^{-1}$.
As shown in Section~\ref{sec:theoretical_guarantees} and Appendix~\ref{sec:main_alg_coverage_proof}, for a fixed candidate $\tau$, the variance of the weighted miscoverage estimator in~\eqref{eq:miscoverage_estimator} scales with the mean of weighted errors $\mathbb{E}\!\left[w(i) \mathbb{I}\{T_i<\fhat_{\tau}(X_i)\} \right]=\mathbb{E}\!\left[w(i) \sum_{t < \fhat_{\tau}(X_i)}\mathbb{I}\{T_i=t\} \right]$.
Ideally, we would minimize this quantity at the calibrated level $\hat{\tau}$ from~\eqref{eq:our_calibration}.
However, $\hat{\tau}$ is selected only after data acquisition and is therefore
unavailable when learning the acquisition policy.
We instead use $\tau=\alpha$ as a fixed candidate level, since it is
the target miscoverage level.
This focuses the allocation on trajectories that can contribute at the
target level, rather than the more conservative choice $\tau= \tauprior$, which considers errors over the entire candidate range $\mathcal{T}$.
Thus, we aim to reduce the weights of trajectories satisfying $T_i<\fhat_\alpha(X_i)$.
Directly optimizing with this indicator $\mathbb{I}\{T_i<\fhat_{\alpha}(X_i)\}$, however, can be unstable when $\calIone$ is small. 
We therefore use their estimated probabilities as a smooth proxy, which provides a stronger signal for the optimization. Our goal is to learn a monotone score-to-probability mapping $M_t^{\mathrm{raw}}$ for each turn $t$:
\begin{equation}\label{eq:opt_prob}
\begin{aligned}
    \min_{M^{\mathrm{raw}}} \quad &  
    \frac{1}{N_1} \sum_{i=1}^{N_1} \sum_{t <\fhat_{\alpha}(X_i) }\frac{\widehat{\mathbb{P} }(T_i=t \mid T_i \geq t, H_i(t))}{\prod_{s=1}^{t} M_s^{\mathrm{raw}}(S_i(s))} \\
    \text{s.t.} \quad & \frac{1}{N_1} \sum_{i=1}^{N_1} \bfunc(M^{\mathrm{raw}}(S_i)) \leq \Bbar, \quad S_i(t)\leq S_j(t)
      \Longrightarrow
      M_t^{\mathrm{raw}}(S_i(t))
      \leq M_t^{\mathrm{raw}}(S_j(t)).
\end{aligned}
\end{equation}
Here, $\bfunc(P_i)$ is the expected budget expended on the $i$-th sample. Since each interaction step consumes one budget unit, and the probability of the conversation surviving to step $t$ is the product of all prior continuation probabilities $\prod_{j=1}^{t} P_i(j)$, the expected budget is the sum of these survival probabilities:
% \begin{equation}
    $\bfunc(P_i) = \sum_{t=1}^{b_i} \prod_{j=1}^{t} P_i(j)$.
% \end{equation}
We fully describe how we solve this optimization in Appendix~\ref{sec:proposed_alg}.
Thus, we optimize a tractable and stable proxy motivated by the exact variance expression.
%rather than directly minimizing the variance of the final calibrated estimator. 
The ablation in Appendix~\ref{sec:optimization_process_ablation} shows that this optimization directly improves the variance.

If $\{M^{\mathrm{raw}}_t\}_{t=1}^{\maxl}$ generalize well to unseen trajectories, then the resulting policy satisfies the expected budget constraint, as formally shown in Appendix~\ref{sec:budget_validity_results}. Empirically, we find this typically holds: even without additional calibration, \ttmethod satisfies the target budget across most experiments. For settings that require a stronger guarantee, we introduce an optional second calibration step based on conformal risk control (CRC)~\citep{angelopoulos2024conformal}, which guarantees expected budget control without relying on accurate generalization of the raw policy.
Specifically, we consider a one-parameter family of mappings $M_t(s;\lambda)$ such that the induced expected budget is nondecreasing in $\lambda$. A simple choice is $M_t(s;\lambda) := \lambda \cdot M_t^{\mathrm{raw}}(s)$. 
To calibrate $\lambda$, we fully observe each trajectory $i \in \calIcrc$ until either the event occurs or $\qprior(X_i)$ is reached, consuming
$b_i=\min\{T_i,\qprior(X_i)\}$ budget units. These trajectories allow us to evaluate the expected deployment cost $\bfunc(M(S_i;\lambda))$ that the candidate policy would incur.
Before CRC calibration, we cap the candidate family so that its conditional expected cost on every trajectory is at most $C_{\max}$.
This prevents CRC from resulting in a calibrated policy that is too conservative.
We then select the largest value of $\lambda$ whose CRC upper bound on the expected per-sample cost remains within the available budget:
\begin{equation}
    \widehat\lambda
    :=\sup\left\{\lambda\in[0,1]:
      \frac{
        \sum_{i\in\calIcrc}\left( \bfunc(M(S_i;\lambda)) +  \frac{|\calIcrc|}{|\calItwo|}b_i \right) + C_\text{max} + \frac{|\calIcrc|}{|\calItwo|} \maxl
      }{|\calIcrc|+1}
      \leq\frac{B_\textrm{rem}}{|\calItwo|}
    \right\}.
    \label{eq:dapro_main_crc_selector}
\end{equation}
Finally, the calibrated continuation policy is formulated as
$M_t(\cdot) := M_t(\cdot;\hat{\lambda})$.
Thus, the learned mappings $M_t^{\mathrm{raw}}$ 
determine the relative allocation of budget across trajectories, while CRC calibrates the overall budget consumption to obtain an expected budget guarantee.

\subsection{Phase II: Adaptive Acquisition on \texorpdfstring{$\calItwo$}{I2}}

Armed with the sequence of score-to-probability mapping models $\{M_t\}_{t=1}^{\maxl}$, we adaptively acquire interactions for the second split, $\calItwo$. We initialize the process at $t=1$ and define a continuation indicator $\chat_i(0) = 1$. At each step $t$, we compute the score $S_i(t)=\mathcal{S}_t({H}_i(t))$ and transport it into a continuation probability $P_i(t) = M_t(S_i(t))$. We then draw a Bernoulli random variable with parameter $P_i(t)$ to determine if the interaction should proceed. 
Formally, at each step $t$, we set $P_i(t)=M_t(S_i(t))$, draw
$Z_t^{(i)}\sim\operatorname{Bern}(P_i(t))$, and update $\chat_i(t)=\chat_i(t-1)Z_t^{(i)}.$
If $\chat_i(t) = 1$, we advance to the next step: we acquire the new conversation exchange $H_i(t+1)$, set $t \leftarrow t+1$, and repeat the process. The interaction terminates when $\chat_i(t) = 0$ or when reaching the predefined boundary $t = \min (\qprior(X_i), T_i)$. 

Finally, we define the censoring time for each sample $i\in \calItwo$ as: $C_i=\qprior(X_i)$ if $\chat_i(\min\{T_i,\qprior(X_i)\})=1$; otherwise, $C_i$ is the last successfully reached step.
Consequently, for samples $i \in \calItwo$ that successfully reach their target sequence length, i.e., \allowbreak
$\qprior(X_i) \leq C_i$, \allowbreak
the probability of achieving this event is the product of the sequential continuation probabilities:
$$ w(i)^{-1}=\mathbb{P}(\qprior(X_i) \leq C_i \mid X_i,H_i,T_i,\{(X_j, H_j, T_j)\}_{j \in \calIone \cup\calIcrc},\Dtrain) = \prod_{t=1}^{\min (\qprior(X_i), T_i)} P_i(t). $$
For samples $i\in \calIone \cup \calIcrc$ we set $w(i)=1$.
Lastly, we employ Algorithm~\ref{alg:lpb_construct} using the samples we acquired $\{(C_i, w(i), \Ttilde_i)\}_{i \in \calI}$. 
For convenience, we summarize this procedure in Algorithm~\ref{alg:main_alg} in Appendix~\ref{sec:proposed_alg}. In what follows, we prove that this process satisfies the coverage requirement and expected budget constraint.

\subsection{Theoretical guarantees}\label{sec:theoretical_guarantees}

First, we show that the LPB constructed by \ttmethod holds a PAC-type coverage guarantee.\footnote{\noindent An analogous UPB guarantee holds; see Appendix~\ref{sec:utility_upb}.}
\begin{theorem}[Coverage validity (informal)]
\label{thm:alg_coverage_validity_informal}    
Fix a miscoverage level $\alpha \in (0,1)$ and a tolerance level $\delta \in (0,1)$.
Suppose that $\{(X_i, T_i, H_i)\}_{i \in \calI}$ and $(\Xtest, \Ttest, \Htest)$
are drawn i.i.d., and that $\fhat_\tau(x)$ is non-decreasing.
Further assume  $
    \mathbb{E}[w(i) \mathbb{I}\{T_i<\qprior(X_i)\} \mid \{(X_i, H_i, T_i)\}_{i \in \calIone \cup \calIcrc}, \Dtrain] \leq \bar{\nu},
     \forall i \in \calItwo$ for $\bar{\nu} > \alpha^2$ a.s.
Then, with probability at least $1 - \delta$ over the draw of
$\Dall$, the LPB $\hat{L}(x)$ generated by \ttmethod satisfies:
\begin{equation*}
    \mathbb{P}\!\left[
        \Ttest \geq \hat{L}(\Xtest)
        \,\middle|\,
        \Dall
    \right]
    \;\geq\;
    1 - \alpha - \frac{\log(1/\delta)}{3|\calI|}
    - \sqrt{
        \frac{\log^2(1/\delta)}{9|\calI|^2}
        +
        \frac{2\bigl(\bar{\nu} - \alpha^2\bigr)\log(1/\delta)}{|\calI|}
    }.
\end{equation*}
\end{theorem}
Importantly, this coverage guarantee holds in finite samples for any quantile estimator $\qhat_\tau(\cdot)$, any LLM, and any data distribution.
We defer the proof of Theorem~\ref{thm:alg_coverage_validity_informal} to Appendix~\ref{sec:main_alg_coverage_proof}. 
We remark that Phase I uses $\fhat_\alpha$ only as a fixed reference predictor for learning the allocation policy.
The calibrated predictor $\fhat_{\hat{\tau}}$ selected afterwards may differ.
The theorem therefore considers the more conservative event
$T_i<\qprior(X_i)$: every candidate predictor $\fhat_\tau(X_i)$ with $\tau \leq \tauprior$ satisfies $\mathbb{I}\{T_i<\fhat_\tau(X_i)\}  \leq  \mathbb{I}\{T_i<\qprior(X_i)\}$. Hence, $\bar{\nu}$ bounds the weighted errors for any predictor the calibration may select.

While this proof builds on the theoretical foundations in~\citep{gui2024conformalized, davidov2026calibrated}, we highlight three key implications that distinguish our result:
\begin{itemize}
    \item \textbf{Relaxed Assumptions:}  Crucially, unlike prior conformal frameworks~\citep{davidov2026calibrated, gui2024conformalized}, our result does not rely on the assumption of conditional independence between the censoring times $C$ and event times $T$ given the covariates $X$. \ttmethod inherently violates this assumption by using $T$ for early stopping, making this relaxation necessary.
   
    \item \textbf{Tighter Scaling:} Our coverage gap scales as $\mathcal{O}({\sqrt{\bar{\nu}/N}})$ where $\bar{\nu}$ is the mean of inverse-probability weights only for trajectories that contribute to the miscoverage estimator: $w(i)\mathbb{I}\{T_i<\qprior(X_i)\}$. In contrast, prior guarantees~\citep{gui2024conformalized,davidov2026calibrated} depend on the worst-case weight $w_{\max} \geq w(i)\mathbb{I}\{T_i<\qprior(X_i)\}$. Thus, our coverage gap can remain small even when some inverse-probability weights are large. We illustrate the efficiency of our coverage gap in Figure~\ref{fig:toxicity_lpb_coverage_gap} in Appendix~\ref{sec:toxicity_lpb_budget_coverage_gap}.
    % This motivates weighting the Phase-I objective by each trajectory's likelihood of contributing a miscoverage error rather than minimizing all weights equally.

    \item \textbf{Independent Interest:} Our proof technique can be directly applied to the static procedure of~\citet{davidov2026calibrated} and to a broad class of survival
analysis tasks. 

\end{itemize}

We show that \ttmethod satisfies the budget constraint in~\eqref{eq:budget_constraint}; see Appendix~\ref{sec:budget_validity_crc} for the proof.
\begin{theorem}[Budget validity]
\label{thm:budget_validity_informal}
The expected budget used by \ttmethod is upper-bounded by $B$.
\end{theorem}

\section{Experiments}
\label{sec:experiments}

In this section, we evaluate \ttmethod on three sequential safety and utility tasks. For safety, we generate attacker-target interactions using PAIR~\citep{chao2025jailbreaking} from 10,000 prompts sampled from RealToxicityPrompts~\citep{gehman2020realtoxicityprompts}, where the event is a toxic response, and from the Anthropic Red-Team dataset~\citep{bai2022training, ganguli2022red}, where the event is a successful jailbreak. To evaluate agentic utility, we use 10,000 examples from AutoIF~\citep{dong2025selfplay}, a benchmark for complex instruction following. Here, the ``attacker'' is replaced by a helper agent, and the event occurs when the target model successfully satisfies all instructions.

To generate the interactions, we utilize Qwen 2.5 14B Instruct~\citep{qwen2.5} as the attacker model, as its outputs had the lowest formatting error rate among the models we examined. For the target models, we use Qwen 2.5 14B Instruct, Llama 3.1 8B Instruct~\citep{grattafiori2024llama}, Phi 4 Mini Instruct~\citep{abouelenin2025phi}, and Gemma 3 4B Instruct~\citep{gemma_2025}. 
We define the target event ($Y(t)=1$) using an audit function: for the Toxicity dataset, following~\citep{davidov2026calibrated}, we use Detoxify~\citep{Detoxify} with a threshold of 0.5 to identify toxic responses. 
For the Red-Team dataset, we employ Llama Guard as the judge~\citep{inan2023llama}.
Finally, for the AutoIF dataset, we verify a successful instruction-following response programmatically.
Additionally, Appendix~\ref{sec:hallucinations_exp} evaluates \ttmethod on a RAG-based dataset, in which the event is a hallucination.
We cap the maximum sequence length at $\maxl=200$. The full details regarding the dataset generation are provided in Appendix~\ref{sec:data_generation}.

We split the data into training (4000), calibration (3000), and test (3000) sets. We fix the training data and fit a transformer as the predictive model, estimating $\hat{q}_{\tau}$ over it, using the full conversation history. 
In our experiments, our training data is fully observed up to time $\maxl$, although our framework allows any predictor trained independently of the calibration set to be used. In Appendix~\ref{sec:training_budget} we show that our framework can still provide informative estimates even with smaller training budgets.
At any given time $t_1$, the model estimates the probability of an event occurring at any future time step $t_2>t_1$. 
% Specifically, it outputs a $\maxl \times (\maxl+1)$ matrix containing $\mathbb{P}(Y(t_2)=1 \mid T>t_1, H(t_1))$ and the probability of reaching the horizon without an event, $\mathbb{P}(T >\maxl \mid T>t_1, H(t_1))$. 
All methods use the same predictive models. In all experiments, we evaluate the methods over 50 independent random splits of the calibration and test sets. During each trial, the calibration algorithms are employed using the calibration data and evaluated on the test set. 
Following~\citep{davidov2026calibrated}, we set $\tauprior = 0.56$, and $M=200$.
We compare our proposal to an uncalibrated baseline, which constructs the LPB directly from the raw model predictions at the target $\alpha$ level, $\qhat_{\alpha}$, and to the static optimized allocation strategy of~\citet{davidov2026calibrated}.

\newcommand{\mainplotwidth}{0.41}
\newcommand{\mainlegendwidth}{0.15}
\begin{figure}[t]
    \centering
    \includegraphics[valign=c,width=\mainplotwidth\linewidth]{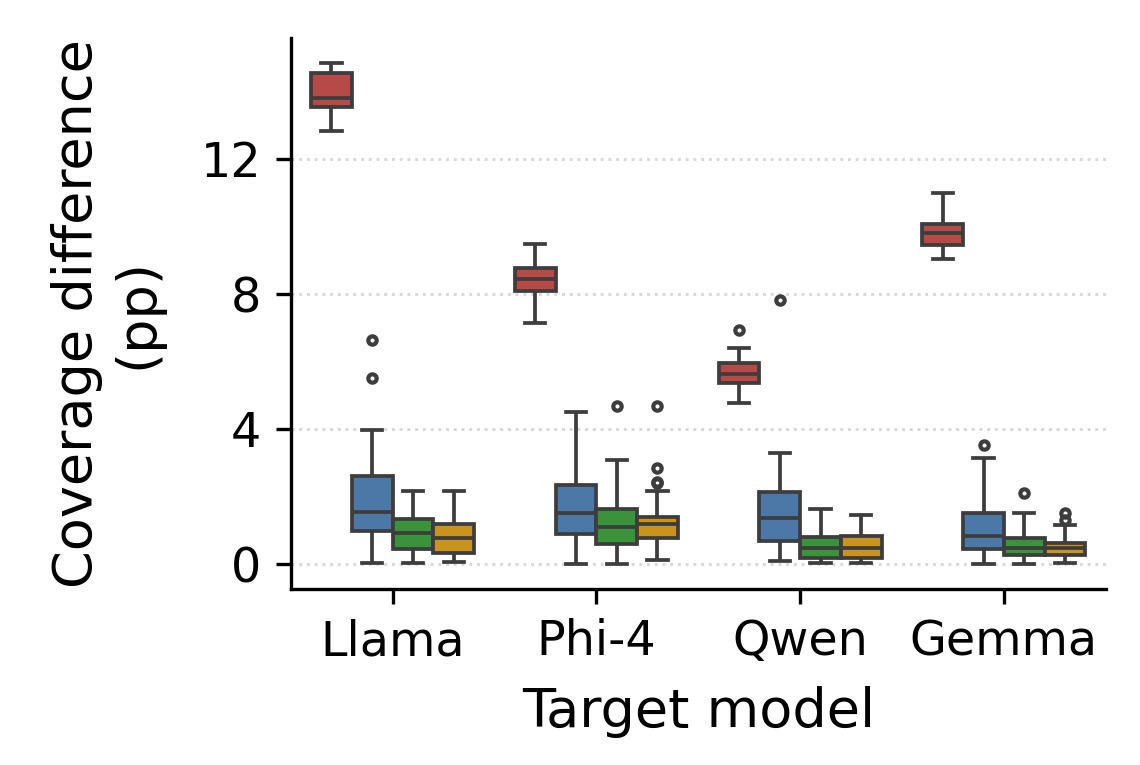}
    \includegraphics[valign=c,width=\mainplotwidth\linewidth]{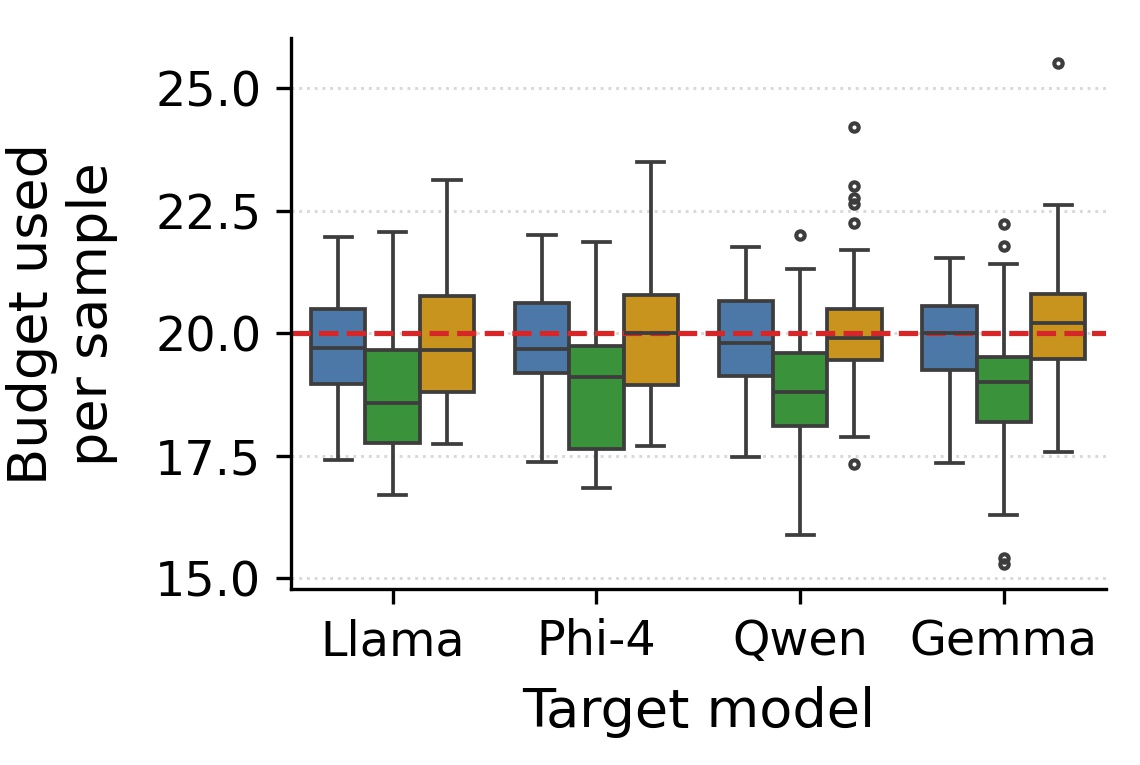}
        \includegraphics[valign=c,width=\mainlegendwidth\linewidth]{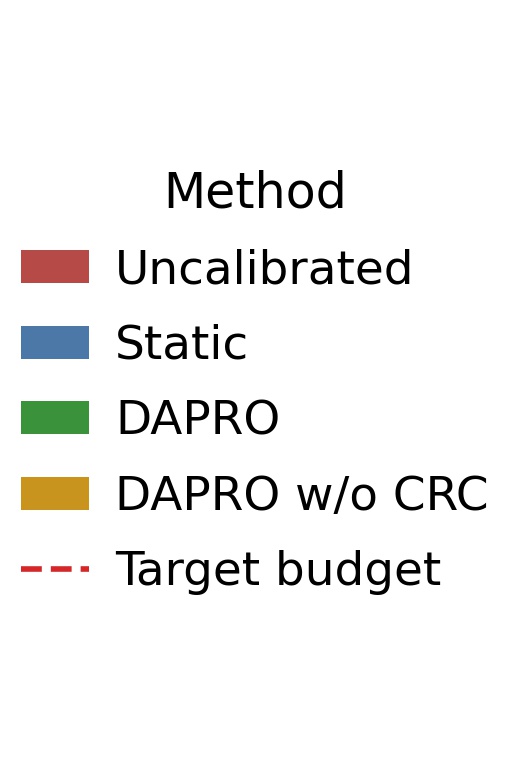}
    \caption{\textbf{Toxicity} dataset: absolute coverage deviation in percentage points (pp) and average budget utilized across four target LLMs. Target coverage rate: $90\%$ and target $\bar{B}=20$ budget per sample.}
    \label{fig:toxicity}
\end{figure}

Figure~\ref{fig:toxicity} presents the absolute deviation from the target $90\%$ coverage rate and average budget per sample used by each method on the Toxicity dataset. This figure reveals that \ttmethod satisfies the expected budget constraint while attaining the lowest coverage deviation. 
Specifically, averaged across the four target models, \ttmethod achieved \(1.90\times\) lower mean absolute deviation from the target coverage and \(3.32\times\) lower coverage-rate variance than the static baseline.
In addition, Figure~\ref{fig:autoif} details the same metrics on the AutoIF dataset for both an LPB (target coverage $90\%$, $\bar{B} = B/ |\calI|=10$) and a UPB (target coverage $80\%$, $\bar{B}=20$). We evaluate the UPB at this lower target coverage since a $90\%$ target pushes the bound too close to the maximum $\maxl=200$, which artificially reduces the variance differences between methods. Across both bound types, the results reveal the same trend: \ttmethod attains a coverage rate closer to the nominal level compared to the baselines.

\begin{figure}[ht]
    \centering
    \includegraphics[valign=t,width=\mainplotwidth\linewidth]{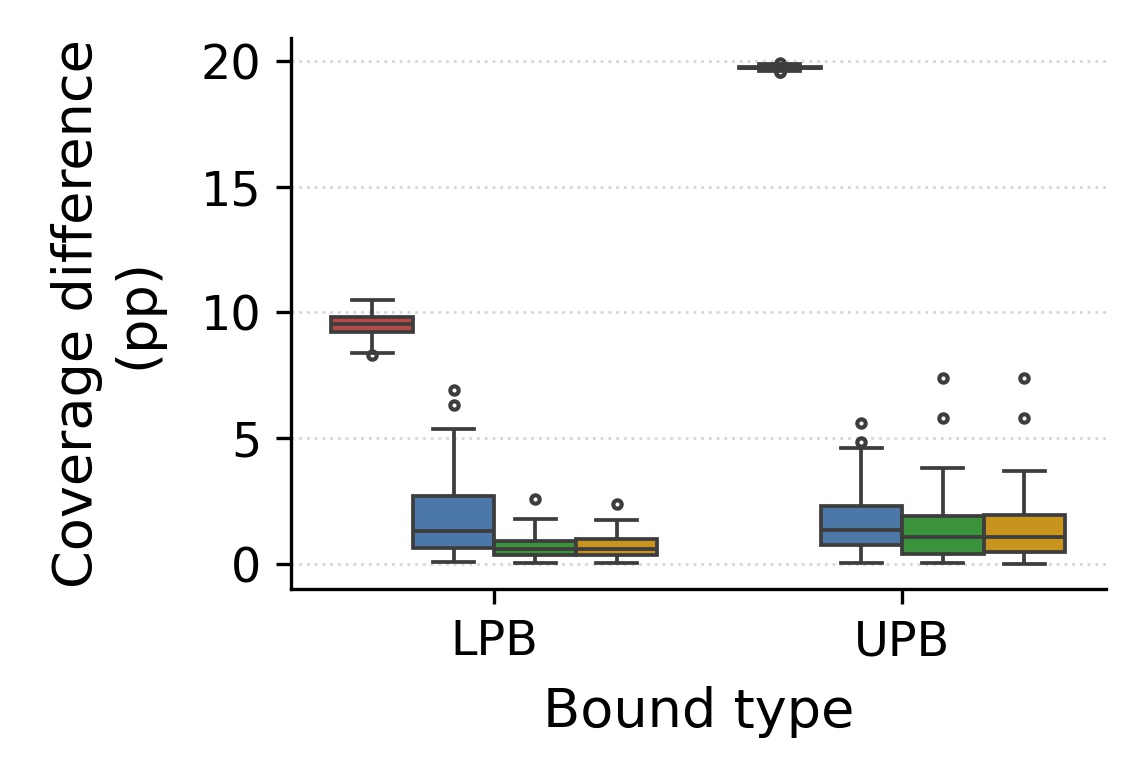}
    \includegraphics[valign=t,width=\mainplotwidth\linewidth]{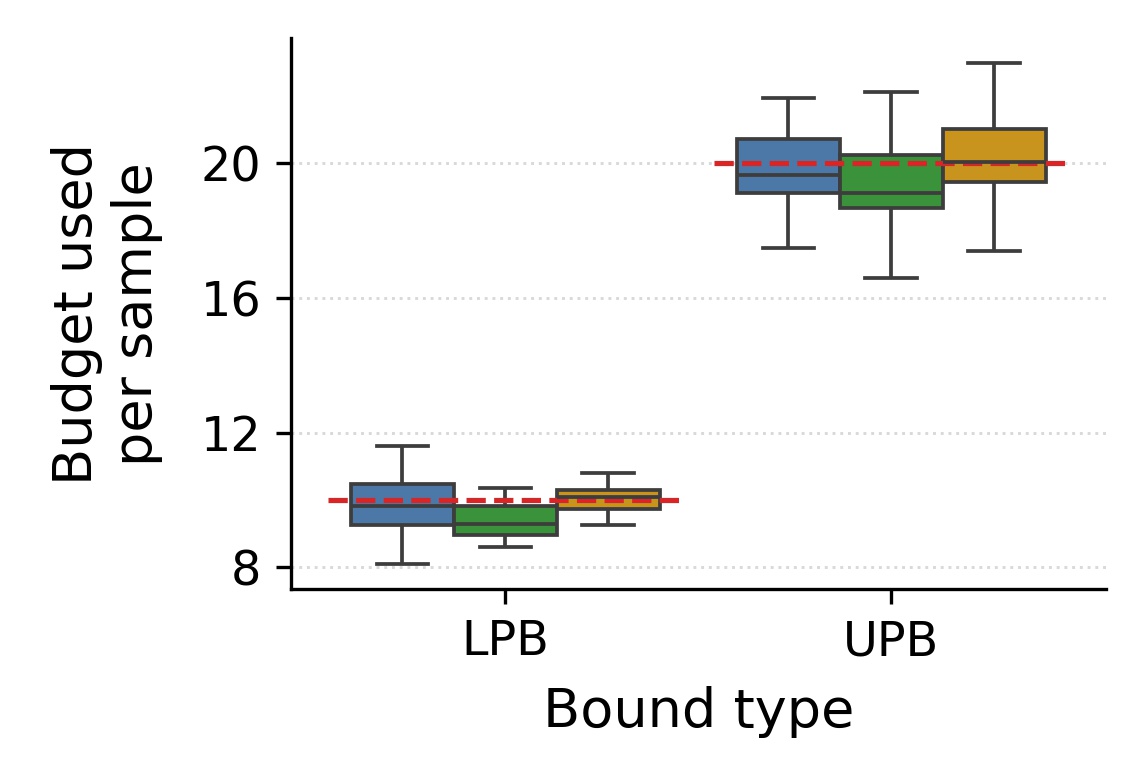}
    \includegraphics[valign=t,width=\mainlegendwidth\linewidth]{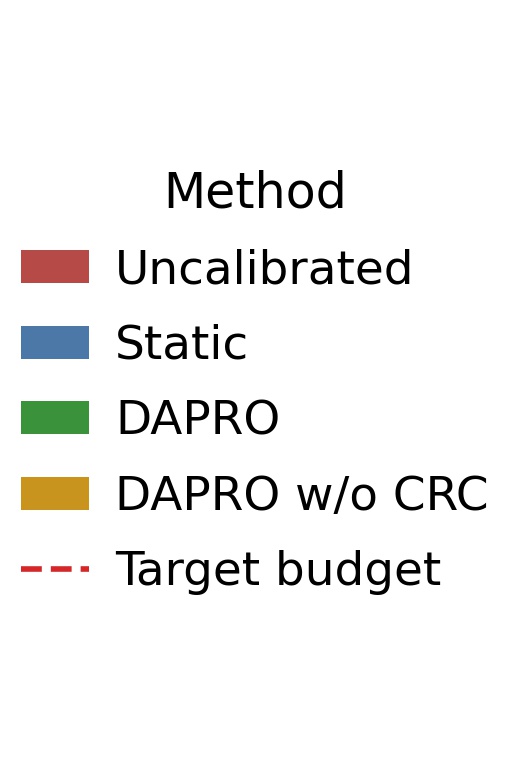}
    \caption{\textbf{AutoIF} dataset: coverage deviation (left) and budget utilized (right) over the Qwen 2.5 target model. Target coverage rate: $90\%$ for LPB and $80\%$ for UPB.}
    \label{fig:autoif}
\end{figure}

We evaluate the jailbreak rate on the Red-Team benchmark. Figure~\ref{fig:red_team_event_rate} shows that while both \ttmethod and the static baseline closely match the full-data event rate on average, \ttmethod exhibits approximately 72\% lower variance on average. This improvement is also reflected in the worst observed estimates: for Phi-4 Mini, whose full-data event rate is 79.6\%, the lowest estimate under Static is 66.6\%, compared to 73.7\% under \ttmethod. Thus, \ttmethod more reliably estimates the full-data event rate. The lower variance of \ttmethod reduces the risk that a low estimate will make an LLM appear safer than it is. More details on this metric-estimation experiment are given in Appendix~\ref{sec:unbiased_estimation}.

\noindent Additional experimental setups and diagnostics are deferred to Appendix~\ref{sec:additional_experiments}. Furthermore, our ablation study in Appendix~\ref{sec:ablation_exps} reveals that \ttmethod is robust to degradation of its components and maintains a tight coverage rate while satisfying the expected budget constraint.

\begin{figure}[t]
    \centering
\includegraphics[valign=t,width=0.39\linewidth]
{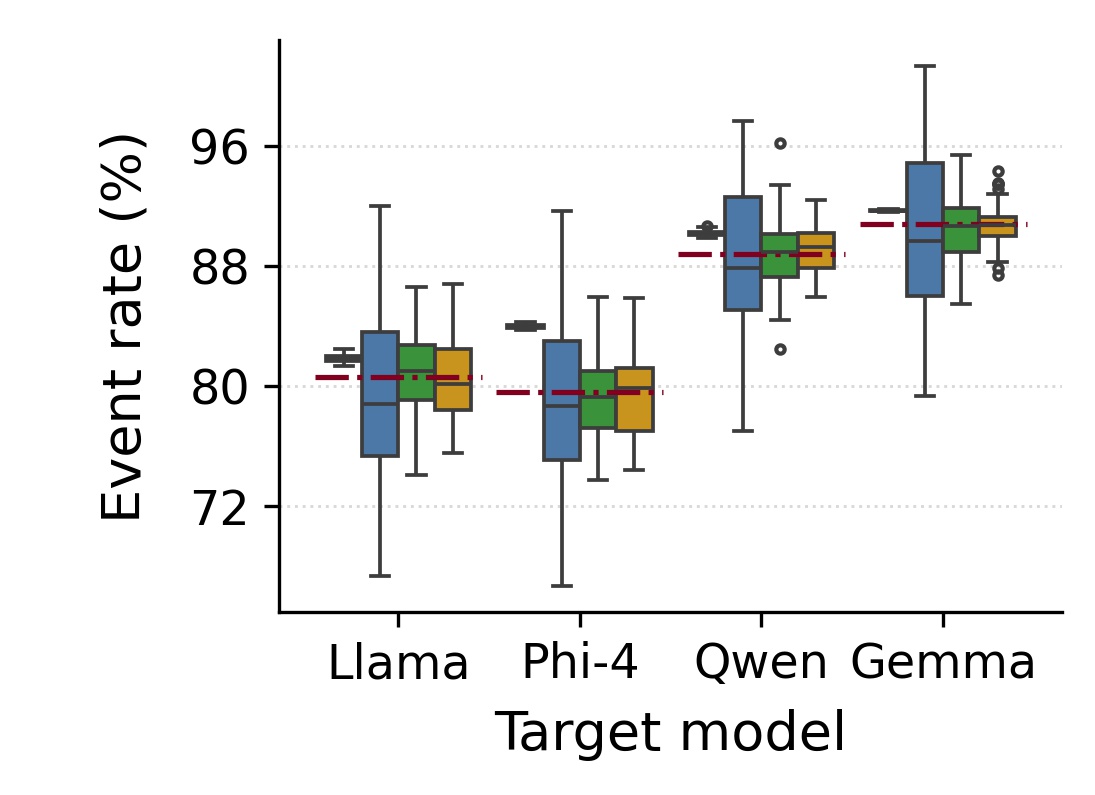}
\hfill
\includegraphics[valign=t,width=0.39\linewidth]
{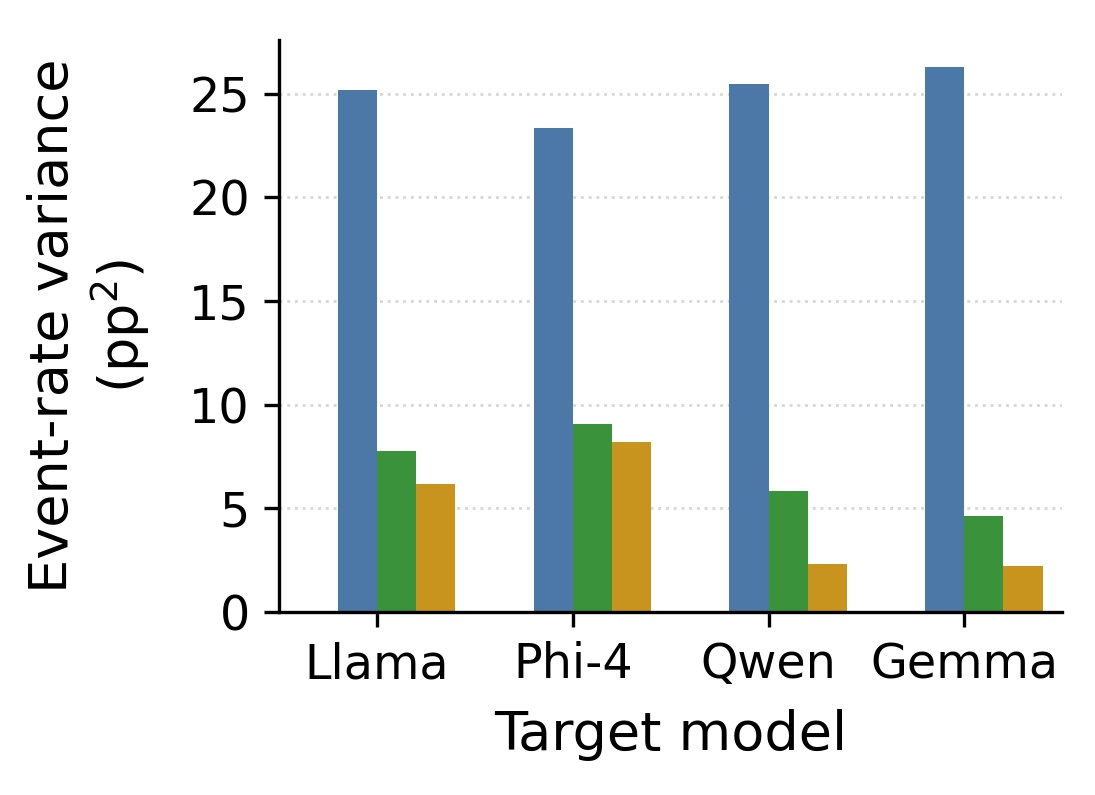}
\hfill
\includegraphics[valign=t,width=0.18\linewidth]
{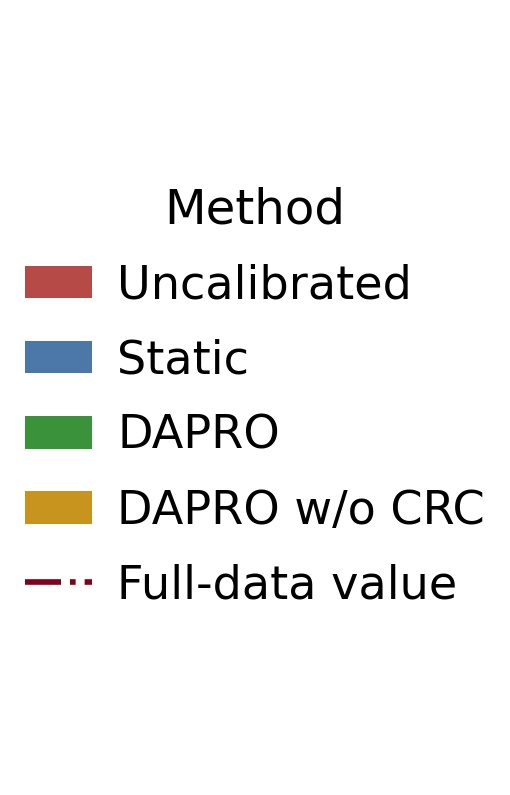}
    \caption{\textbf{Red-Team event-rate estimation}. Left: estimated population event rates. Right: variance of the event-rate estimator. Target $\bar{B}=20$ budget per sample.}
    % , $|\calIone|=|\calIcrc|=50$.
    \label{fig:red_team_event_rate}
\end{figure}

\section{Discussion and impact statement}\label{sec:discussion}

We introduced \ttmethod, a novel dynamic budget allocation method for
LLM evaluation. Our approach can be employed to construct reliable LPBs for individual prompts and extract unbiased metrics over benchmarks, such as the jailbreak rate. By updating censoring times online based on ongoing attacker-target interactions, our method achieves lower variance, as supported by our PAC-type coverage guarantee, which is tighter than prior approaches. Our comprehensive experiments on real datasets validate these theoretical claims, demonstrating tighter bounds compared to baselines.

While our method has an expected budget guarantee, our experiments show that it often underutilizes the budget due to the CRC calibration. So, for applications that do not require a finite-sample budget guarantee, practitioners can omit the CRC step---our experiments show it generally satisfies the budget constraint despite the theoretical requirement for an accurate model.
Furthermore, the calibrated quantile level $\hat\tau$ in our current method is a single scalar applied identically to all test prompts. An intriguing extension would be to make $\hat\tau$ adaptive to the test prompt $\Xtest$, possibly leveraging the newly acquired calibration data to select a prompt-specific quantile level, e.g., by utilizing ideas from~\citep{feldman2021improving, gibbs2025conformal}. This would allow the LPB to be tighter for predictable interactions and more conservative for those with high variance.
Another direction may be engineering a score that better represents the safety or utility of a prompt, which could be mapped to a better continuation probability that leads to a lower variance.

We note that our coverage guarantees rely on the assumption that the calibration and test sets are i.i.d.
In practice, the data distribution may shift over time as models are updated or new attack strategies emerge. Developing methods that are robust to such distributional shifts~\citep{gibbs2021adaptive, feldman2023achieving} or audit annotation errors~\citep{einbinder2024label, penso2024conformal, feldman2024robustconformalpredictionusing, feldman2026conformal} is an important direction for future work.
Importantly, this i.i.d. assumption is not required for metric estimation over a fixed benchmark, since these metrics are evaluated directly on the benchmark of interest rather than generalized to unseen prompts.
Finally, while our framework provides statistical guarantees to help developers evaluate AI systems and set usage limits, its insights into model failures could potentially be exploited by adversaries, akin to many developments in ML. Therefore, our proposal must be deployed as a complementary tool within a broader safety, alignment, and governance pipeline.

\ifbool{ispreprint}{
\section*{Acknowledgments}
Y.~R. and S.~F. were supported by the European Union (ERC, SafetyBounds, 101163414). Views and opinions expressed are however those of the authors only and do not necessarily reflect those of the European Union or the European Research Council Executive Agency. Neither the European Union nor the granting authority can be held responsible for them.
Y.R. thanks the Career Advancement Fellowship, Technion.
}

\newpage
\ifbool{ispreprint}{
\bibliographystyle{unsrtnat}
}{
\bibliographystyle{iclr2027_conference}
}
\bibliography{references}

\newpage

\appendix

\section{Additional related work}
\label{sec:related_work}

Our theoretical framework builds upon conformal prediction~\citep{vovk2005algorithmic, shafer2008tutorial}, a powerful framework for constructing prediction sets with distribution-free finite-sample guarantees under exchangeability. Split conformal prediction~\citep{papadopoulos2002inductive} and its extension to covariate shift via weighting~\citep{tibshirani2019conformal}, are utilized by the calibration procedure we adopt from~\citep{davidov2026calibrated}. Conformal prediction has recently been extended to survival analysis~\citep{candes2023conformalized, gui2024conformalized, davidovconformalized}, including history-aware prediction sets that use longitudinal covariates and inverse-probability-of-censoring weighting~\citep{wang2026history}. In contrast, we use the evolving history to determine which calibration trajectories to continue, while the final LPB remains a function of the initial prompt alone.
Our work builds directly on~\citet{davidov2026calibrated}, who apply conformalized survival analysis to LLM safety by introducing the time-to-unsafe-sampling metric and a static, variance-minimizing budget allocation scheme. We extend this framework to the multi-turn setting where conversation evolves, and our coverage proof provides a tighter PAC-type guarantee~\citep{park2019pac, bates2021distribution, bates2023testing}. 

The variance-minimization phase of our framework is related to active learning~\citep{settles2009active} and optimal experimental design~\citep{fedorov2013theory}, which study how to allocate a limited budget across observations to minimize estimator variance. In particular,~\citet{zrnic2024active} formalize active data collection for statistical inference while maintaining valid guarantees, and~\citet{hahn2011adaptive} develop adaptive designs that minimize treatment effect estimation variance. 
Recent work also studies adaptive experimentation for censored survival outcomes~\citep{wang2026adaptive}; unlike its budget allocation across participants for causal survival effects, our method allocates budget within each trajectory to calibrate predictive bounds.
Our contribution is adapting these ideas to the survival analysis and LLM evaluation setting, where observations arrive sequentially, and, in our framework, we learn an allocation policy from a first data split before being projected onto a second.

Prior efforts in LLM safety evaluation mostly focus on detecting harmful content in model outputs~\citep{Vidgen2023SimpleSafetyTestsAT, gehman2020realtoxicityprompts} or constructing guardrails at inference time~\citep{inan2023llama, rebedea2023nemo}. A complementary line of work takes an adversarial perspective, generating red-teaming attacks to probe model 
vulnerabilities~\citep{perez2022red, zou2023universal, wei2023jailbroken,mehrotra2024tree, chao2025jailbreaking}. These works aim to maximize attack success.
On the defensive side, recent efforts in certified LLM safety aim to provide statistical guarantees against such attacks, typically by bounding the failure probability of a single prompt using techniques like randomized smoothing~\citep{kumar2023certifying, robey2023smoothllm}. 
In contrast, our goal is to provide rigorous statistical guarantees on time-to-event using the multi-turn interaction structure. While existing adversarial red-teaming approaches effectively investigate multi-turn vulnerabilities,
they are not designed to estimate a censored time-to-event under resource constraints.

Our framework can be further leveraged for sample-efficient LLM evaluation~\citep{kossen2021active}. To reduce computational costs and API calls when evaluating large models, recent methods utilize active testing~\citep{li2024active}, multi-armed bandits~\citep{zhou2025on}, adaptive dataset partitioning~\citep{wang2025cer}, and finite-population inference with adaptive benchmark-item selection~\citep{wu2026efficient}. These methods select which fixed benchmark items to evaluate; in contrast, our framework allocates how far to roll out each sequential trajectory, inducing censoring that must be accounted for when constructing predictive bounds and estimating population-level metrics. By framing metric evaluation as a sequential, budget-constrained survival problem, \ttmethod produces theoretically valid estimates with low variance under resource constraints.

\section{Extension to LLM utility evaluation and upper predictive bounds}
\label{sec:utility_upb}

While the main text focuses on constructing lower predictive bounds, our framework can be utilized for constructing upper predictive bounds, which are useful for evaluating LLM utility in multi-turn interactions, such as~\citep{zhou2024webarena}.
In this setting, the adversarial ``attacker'' is replaced by a
``helper agent'' that aims to guide the target model toward successful
task completion.
For example, in the AutoIF dataset~\citep{dong2025selfplay},
the helper agent iteratively assists the target model in satisfying complex instructions. The event of interest,
$Y_i(t)=1$, now represents successful task completion, and $T_i$
denotes the corresponding time-to-success.
We aim to construct an upper predictive bound (UPB) $\widehat U(X)$ for this time-to-success: a smaller value indicates that success should occur within fewer interactions.

\paragraph{Finite-horizon target.}
In our framework, our interaction horizon is capped at $\maxl$. 
Since we do not observe events beyond time $\maxl$ and cannot guarantee whether a UPB that is greater than $\maxl$ correctly covers the time-to-success, we represent ``no event by turn $\maxl$'' by the additional value $\maxl+1$, which is treated as infinity, and define $T_i^{(\maxl)}:=\min\{T_i,\maxl+1\}$. Our desired, tractable guarantee in this UPB setting is therefore
\begin{equation}
    \mathbb{P}\!\left\{
        T_{\mathrm{test}}^{(\maxl)}
        \leq \widehat U(X_{\mathrm{test}})
        \ \middle|\ \Dall
    \right\}
    \geq 1-\alpha.
    \label{eq:upb_target}
\end{equation}
Thus, $\maxl$ is a finite bound that can fail to cover a trajectory
with no event through turn $\maxl$, whereas $\maxl+1$ is an upper bound that always covers $T_{\mathrm{test}}^{(\maxl)}$. We remark that our acquisition is still truncated at $\maxl$, so no trajectory is acquired at time $\maxl+1$.

\paragraph{Candidate-specific UPB errors.}
For a candidate upper bound $\fhat_\tau(X_i)$, we define the error event as
\begin{equation}
    A_i(\tau)
    :=\mathbb{I}\{\fhat_\tau(X_i)<\maxl+1\}
      \mathbb{I}\{T_i>\fhat_\tau(X_i)\}.
    \label{eq:upb_target_indicator}
\end{equation}
Unlike a miscoverage event for an LPB, here an error is observed when either we observe an event or the acquisition policy reaches $\fhat_\tau(X_i)$. Consequently, the resolution time and reach probability depend on $\tau$:
\begin{equation}
    r_i(\tau):=\min\{T_i,\fhat_\tau(X_i)\},
    \qquad
    \rho_i(t):=\prod_{s=1}^{t}P_i(s),
    \qquad
    \pi_i(\tau):= \rho_i(r_i(\tau)).
    \label{eq:upb_candidate_weight}
\end{equation}
Above, $\rho_i(0):=1$.
Observe that the contribution of a sample $i$ to the miscoverage estimator is
$A_i(\tau)\mathbb{I}\{C_i\geq \fhat_\tau(X_i)\}/\pi_i(\tau)$; for $\tau$ having $\fhat_\tau(X_i)=\maxl+1$, the contribution is deterministically zero.

\paragraph{Sequential augmented miscoverage estimator.}
We construct the UPB using the outputs of the survival model at every
reached prefix. For convenience, we write $H_i(0)=H_i(1)$. Specifically, for a candidate bound $f=\fhat_\tau(X_i)$, let
\begin{equation}
    m_{i,0}(f):=\widehat{\mathbb P}(T_i>f\mid H_i(0)),
    \qquad
    m_{i,t}(f):=\widehat{\mathbb P}(T_i>f\mid H_i(t))
    \quad(t<r_i(\tau)),
\end{equation}
and set \(m_{i,r_i(\tau)}(f):=A_i(\tau)\). We define $O_i(t)=\mathbb{I}\{\min( C_i, T_i)\ge t\}$ as a binary indicator that indicates that the turn $t$ was acquired for the $i$-th sample. The miscoverage estimator of the $i$-th sample is
\begin{equation}
    \widetilde A_i(\tau)
    :=m_{i,0}(\fhat_\tau(X_i))
      +\sum_{t=1}^{r_i(\tau)}
       \frac{O_i(t)}{\rho_i(t)}
       \left[m_{i,t}(\fhat_\tau(X_i))-m_{i,t-1}(\fhat_\tau(X_i))\right].
    \label{eq:upb_sequential_aht}
\end{equation}
For $\fhat_\tau(X_i)=\maxl+1$, we instead set $\widetilde A_i(\tau)=0$. We study the effect of the miscoverage estimator in Appendix~\ref{sec:miscoverage_estimator_ablation}, analyzing either the standard estimator $\frac{1}{\pi_i(\tau)} O_i(r_i(\tau)) A_i(\tau)$, or the above augmented sequential estimator. This study reveals that the augmented sequential estimator leads to a lower variance of the coverage rate.

Notice that~\eqref{eq:upb_sequential_aht} starts from the model's initial guess and adds an inverse-probability correction for every prediction change that the acquisition reveals. In expectation, these corrections reduce to the observed endpoint \(A_i(\tau)\), so the accuracy of the model $m_{i,t}(f)$ affects variance but not validity. This form is inspired by prediction-powered inference~\citep{angelopoulos2023prediction}, and based on the augmented Horvitz-Thompson (AHT) estimator~\citep{Horvitz01121952, robins1994estimation,wang2001augmented}. 

\begin{proposition}[Validity of the UPB acquisition estimator]
\label{prop:valid_upb_weights}
Conditionally on $\Dtrain, \Dcalone, \Dcalcrc$, for every candidate $\tau$,
\begin{equation}
    \mathbb E\!\left[\widetilde A_i(\tau)\right]=A_i(\tau).
    \label{eq:valid_upb_weight_identity}
\end{equation}
This identity holds for any sequence of model predictions $m_{i,t}$. Moreover, 
\begin{equation}
\begin{split}
    \operatorname{Var}_{\mathrm{acq}}\!\left(
        \widetilde A_i(\tau)\mid X_i,H_i,T_i
    \right)
    =\sum_{t=1}^{r_i(\tau)}
      \left(\frac{1}{\rho_i(t)}-\frac{1}{\rho_i(t-1)}\right)
      \left(A_i(\tau)-m_{i,t-1}(\fhat_\tau(X_i))\right)^2.
    \label{eq:upb_sequential_aht_variance}
\end{split}
\end{equation}
\end{proposition}

\begin{proof}
Fix \(i\) and \(\tau\), and condition throughout on
\(\Dtrain,\Dcalone,\Dcalcrc,X_i,H_i,T_i\), so that
\(r_i(\tau)\), \(A_i(\tau)\), and the sequence
\(m_{i,t}(\fhat_\tau(X_i))\) are fixed. All expectations and variances below
are therefore with respect to the acquisition randomness only.

Write

$$
r:=r_i(\tau),\qquad
A:=A_i(\tau),\qquad
m_t:=m_{i,t}(\fhat_\tau(X_i)),\qquad
\rho_t:=\rho_i(t),\qquad
O_t := O_i(t)
$$

and set \(O_0=\rho_0=1\). For \(t\le r\), let
\(p_t:=\rho_t/\rho_{t-1}\) denote the conditional probability of acquiring
turn \(t\), given that turn \(t-1\) has been reached. 
We denote by $\mathcal G_t$ the information revealed by the acquisition process up to turn \(t\).

$$
\mathbb E[O_t\mid\mathcal G_{t-1},X_i,H_i,T_i]
    =p_tO_{t-1},
$$

and consequently

$$
\mathbb E\!\left[
    \frac{O_t}{\rho_t}
    \,\middle|\,\mathcal G_{t-1},X_i,H_i,T_i
\right]
=
\frac{O_{t-1}}{\rho_{t-1}}.
$$

In particular,

$$
\mathbb E\!\left[\frac{O_t}{\rho_t}\right]=1.
$$

Hence, using \(m_r=A\),

$$
\begin{aligned}
\mathbb E[\widetilde A_i(\tau)]
&=
m_0+\sum_{t=1}^{r}
\mathbb E\!\left[\frac{O_t}{\rho_t}\right]
(m_t-m_{t-1})\\
&=
m_0+\sum_{t=1}^{r}(m_t-m_{t-1})
=m_r
=A.
\end{aligned}
$$

This proves

$$
\mathbb E\!\left[
\widetilde A_i(\tau)
\mid X_i,H_i,T_i,\Dtrain,\Dcalone,\Dcalcrc
\right]
=A_i(\tau),
$$

for an arbitrary sequence of predictions \(m_{i,t}\).

We next establish the variance identity. Define

$$
Q_t:=\frac{O_t}{\rho_t},
\qquad Q_0:=1,
\qquad
D_t:=Q_t-Q_{t-1}.
$$

The preceding conditional-expectation identity gives

$$
\mathbb E[D_t\mid\mathcal G_{t-1},X_i,H_i,T_i]=0,
$$

so \((D_t)_{t=1}^r\) is a martingale-difference sequence. Since \(m_r=A\),
a summation by parts yields

$$
\begin{aligned}
\widetilde A_i(\tau)-A
&=
m_0-A+\sum_{t=1}^{r}Q_t(m_t-m_{t-1})\\
&=
\sum_{t=1}^{r}
D_t\,(A-m_{t-1}).
\end{aligned}
$$

Since the martingale differences are pairwise orthogonal:
$$
\operatorname{Var}(\widetilde A_i(\tau)\mid X_i,H_i,T_i)
=
\sum_{t=1}^{r}
(A-m_{t-1})^2\,
\mathbb E[D_t^2].
$$

Moreover, since \(O_tO_{t-1}=O_t\),

$$
\begin{aligned}
\mathbb E[D_t^2]
&=
\mathbb E\left[
\left(\frac{O_t}{\rho_t}
-\frac{O_{t-1}}{\rho_{t-1}}\right)^2
\right]\\
&=
\frac{\mathbb E[O_t]}{\rho_t^2}
+\frac{\mathbb E[O_{t-1}]}{\rho_{t-1}^2}
-\frac{2\mathbb E[O_t]}{\rho_t\rho_{t-1}}\\
&=
\frac1{\rho_t}-\frac1{\rho_{t-1}},
\end{aligned}
$$

because \(\mathbb E[O_t]=\rho_t\). Therefore

$$
\operatorname{Var}\!\left(
\widetilde A_i(\tau)\mid X_i,H_i,T_i
\right)
=
\sum_{t=1}^{r_i(\tau)}
\left(
\frac1{\rho_i(t)}-\frac1{\rho_i(t-1)}
\right)
\left(
A_i(\tau)
-m_{i,t-1}(\fhat_\tau(X_i))
\right)^2 .
$$

For the candidate \(\fhat_\tau(X_i)=\maxl+1\),
\(A_i(\tau)=0\) and \(\widetilde A_i(\tau)=0\) by definition, and hence its
acquisition variance is identically zero.
\end{proof}

We estimate the miscoverage rate of a candidate $\tau$ by
\begin{equation}
    \widehat\alpha_{\mathrm{UPB}}(\tau)
    :=\frac{1}{|\calI|}\sum_{i\in\calI}\widetilde A_i(\tau).
    \label{eq:upb_miscoverage_estimator}
\end{equation}

As in the LPB construction, we restrict the candidate search space to
$\mathcal T\subseteq[0,\tauprior]$. In this UPB case, a reasonable prior level for $1-\alpha=0.8$ is $\tauprior=0.98$. 
Because UPB coverage is nondecreasing with the candidate level $\tau$, we select
\begin{equation}
    \widehat\tau
    :=\inf\left\{
        \tau\leq\tauprior:
        \sup_{\tau'\in[\tau,\tauprior]}
        \widehat\alpha_{\mathrm{UPB}}(\tau')
        \leq \alpha
    \right\}.
    \label{eq:upb_calibrated_tau}
\end{equation}
The calibrated bound is
$\widehat U(x):=\fhat_{\widehat\tau}(x)$.
If no candidate up to $\tauprior$ satisfies this rule, we return
$\widehat U(x)=\maxl+1$.

\begin{theorem}[UPB coverage validity]
\label{thm:upb_coverage}
Suppose the assumptions of Theorem~\ref{thm:general_validity} hold. In addition, suppose that \(m_{i,t}(f)\in[0,1]\),
\(\rho_i(t)\geq\varepsilon>0\) for every reachable prefix, and that
\begin{equation}
    \mathbb E\!\left[\pi_i(\tau)^{-1}\right]\leq\bar w.
\end{equation}
Let 
\begin{equation}
    \Delta_{\rm UPB}
    :=\frac{\log(1/\delta)}{3|\calI|\varepsilon}
      +\sqrt{
        \frac{\log^2(1/\delta)}{9|\calI|^2\varepsilon^2}
        +\frac{2\log(1/\delta)(\bar w-\alpha^2)}{|\calI|}}.
    \label{eq:upb_coverage_slack}
\end{equation}
Then, with probability at least \(1-\delta\), the selector in~\eqref{eq:upb_calibrated_tau} satisfies
\begin{equation}
    \mathbb P\!\left\{
      T_{\rm test}^{(\maxl)}\leq\widehat U(X_{\rm test})
      \ \middle|\ \Dall
    \right\}
    \geq1-\alpha-\Delta_{\rm UPB}.
    \label{eq:upb_coverage_result}
\end{equation}
\end{theorem}

\begin{proof}
Let \(p_\tau:=\mathbb E[A_i(\tau)]\) and
\(Y_i:=p_\tau-\widetilde A_i(\tau)\).
The proof follows from the arguments in Theorem~\ref{thm:general_validity}; 
with the change that we replace the variance of the miscoverage estimator with $\bar{w}$. 
For a realized last reached prefix \(k\), summation by parts gives
\begin{equation}
    \widetilde A_i
    =\frac{m_{i,k}}{\rho_i(k)}
      +\sum_{t=0}^{k-1}m_{i,t}
        \left(\frac{1}{\rho_i(t)}-\frac{1}{\rho_i(t+1)}\right),
\end{equation}
so \(1-\varepsilon^{-1}\leq\widetilde A_i\leq\varepsilon^{-1}\).
Proposition~\ref{prop:valid_upb_weights} also implies
\begin{equation}
    \mathbb E[\widetilde A_i^2]
    =\mathbb E[\operatorname{Var}(\widetilde A_i)+A_i^2]
    \leq\mathbb E[\pi_i^{-1}]\leq\bar w.
\end{equation}
We further replace the upper bound of $Y_i=P_\tau - \tilde A_i \leq \frac{1}{\varepsilon}$.
The rest of the proof follows directly from the derivations in the proof of Theorem~\ref{thm:general_validity} by using the above bound on the variance of the miscoverage estimator.
\end{proof}

\paragraph{\ttmethod for UPB construction.}

For constructing LPBs, Phase-I optimization uses the fixed reference level \(\alpha\), because the calibrated level $\widehat\tau$ is not yet known. 

For constructing UPBs, since the calibration direction is reversed, we instead choose a reference level using the uncalibrated predictions of the survival model. Define
\begin{equation}
    \widehat a_0(\tau)
    :=\frac{1}{|\calIone|}\sum_{i\in\calIone}
      \mathbb I\{\fhat_\tau(X_i)\leq\maxl\}
      \widehat{\mathbb P}(T_i>\fhat_\tau(X_i)\mid H_i(0)),
\end{equation}
Defining $\mathcal{T}_\mathrm{prior} = [0,\tauprior]$, the reference quantile level we use is
\begin{equation}
    \tau_{\rm ref}
    :=\min\left\{\tau\in\mathcal T_{\rm prior}:
      \inf_{\tau'\in\mathcal T_{\rm prior},\,\tau'\geq\tau}
      \{1-\widehat a_0(\tau')\}\geq1-\alpha\right\}.
    \label{eq:upb_reference_level}
\end{equation}
If the set is empty, we set \(\tau_{\rm ref}:=\max \mathcal T_{\rm prior}\). We define \(f^\star(X_i):=\fhat_{\tau_{\rm ref}}(X_i)\). Observe that this selection uses only the model outputs from the initial inputs for every trajectory, and does not require expending budget. This is the UPB analogue of fixing the LPB target at $\alpha$, but it respects
the reversed direction of upper-bound calibration and adapts automatically to the requested UPB coverage level.

Here, the score is the estimated probability for the event occurring at the current time $S_i(t)=\widehat h_i(t)$ defined in~\eqref{eq:lpb_impl_hazard_score}. Only the target coefficients change.
Using the same global regularization $\delta$ as LPB \ttmethod, the UPB
coefficients are
\begin{equation}
\begin{split}
    a_i^{\mathrm{UPB}}(t)
    :=\frac{1}{1+\delta}\biggl[&
      \mathbb{I}\{t=f^\star(X_i),\ f^\star(X_i)\leq b_i,\
                       f^\star(X_i)\leq\maxl\}
      \bigl(1-\widehat h_i(t)\bigr)\\
      &+\delta\,\mathbb{I}\{t\leq b_i\}\widehat h_i(t)
    \biggr].
    \label{eq:upb_dapro_soft_coefficients}
\end{split}
\end{equation}

For a specific reference bound $f^\star(X_i)$, $1-\widehat h_i(f^\star(X_i))$ is the estimated conditional
probability of surviving beyond this bound, given that the trajectory has
reached it. The bound $\maxl+1$ has deterministic zero miscoverage and does not contribute to the miscoverage estimator. We therefore replace the coefficients by
$a_i^{\mathrm{UPB}}(t)$ and minimize the following objective:
\begin{equation}
\begin{aligned}
    \min_{M^{\mathrm{raw}}}\quad &
    \frac{1}{N_1}\sum_{i\in\calIone}\sum_{t=1}^{b_i}
    \frac{a_i^{\rm UPB}(t)}{\prod_{s=1}^{t}M^{\mathrm{raw}}_s(S_i(s))}\\
    \text{s.t.}\quad &
    \frac{1}{N_1}\sum_{i\in\calIone}\bfunc(M^{\mathrm{raw}}(S_i))\leq\Bbar,\\
    &S_i(t)\leq S_j(t)
      \Longrightarrow M^{\mathrm{raw}}_t(S_i(t))\leq M^{\mathrm{raw}}_t(S_j(t))
      \quad\text{for every active }i,j,t.
    \label{eq:upb_dapro_objective}
\end{aligned}
\end{equation}

All remaining components, including the specific two-bin score-to-probability projection,
optimization, cumulative-reach correction, probability positivity floor, and the optional CRC for budget control, are the same as in the LPB implementation in
Section~\ref{sec:proposed_alg}.

\section{Theory}
\label{sec:theory}
In this section, we present our theoretical guarantees. For notational convenience, we slightly overload notation made in the first part and set $\Dcal= \{(X_i, H_i, T_i)\}_{i \in \calI}$,  $\Dcalone = \{(X_i, H_i, T_i)\}_{i \in \calIone}$, $\Dcaltwo = \{(X_i, H_i, T_i)\}_{i \in \calItwo}$, and $\Dcalcrc = \{(X_i, H_i, T_i)\}_{i \in \calIcrc}$.  
Hence, the combined data is denoted as $\Dall = \Dtrain \cup \Dcal$.

Throughout this section, with a slight abuse of notation, $\Dall$ also includes all auxiliary algorithmic randomness realized during calibration and used to construct $\hat L$ (e.g., the continuation/censoring draws). Accordingly, conditioning on $\Dall$, and probabilities over its realization, include this randomness, so that $\hat L$ is fixed under the inner conditional probability.

\subsection{Coverage rate guarantee of Algorithm~\ref{alg:main_alg}}
\label{sec:main_alg_coverage_proof}

Before presenting the coverage validity guarantee of our proposed algorithm, we first show that the weights it uses are correct.

\begin{proposition}[Valid weights of Algorithm~\ref{alg:main_alg}]
\label{prop:valid_weights}
Suppose the continuation probabilities of Algorithm~\ref{alg:main_alg} satisfy $P_i(t)>0$ almost surely for every $t \leq \min(\qprior(X_i),T_i)$ and $i\in \calItwo$.
For every $i \in \calI$ and every
$\tau \in \mathcal{T}=[0,\tauprior]$, the weights used by
Algorithm~\ref{alg:main_alg} satisfy almost surely:
\begin{equation}
\label{eq:valid_weight_identity}
\mathbb{E}\!\left[
    w(i)\,
    \mathbb{I}\{T_i < \fhat_\tau(X_i)\}\,
    \mathbb{I}\{C_i \geq \fhat_\tau(X_i)\}
    \,\middle|\,
    X_i,H_i,T_i,\Dcalone,\Dcalcrc,\Dtrain
\right]
=
\mathbb{I}\{T_i < \fhat_\tau(X_i)\}.
\end{equation}
\end{proposition}

\begin{proof}
We first consider $i\in\calIone \cup \calIcrc$. By construction,
\[
C_i=\qprior(X_i)
\qquad\text{and}\qquad
w(i)=1
\]
deterministically. Since $\fhat_\tau$ is non-decreasing in $\tau$
and $\tau\leq\tauprior$, we have
\[
\fhat_\tau(X_i)\leq \qprior(X_i)=C_i.
\]
Consequently,
\[
\mathbb{I}\{C_i\geq\fhat_\tau(X_i)\}=1
\]
deterministically, and therefore
\begin{align*}
&\mathbb{E}\!\left[
    w(i)\,
    \mathbb{I}\{T_i<\fhat_\tau(X_i)\}\,
    \mathbb{I}\{C_i\geq\fhat_\tau(X_i)\}
    \,\middle|\,
    X_i,H_i,T_i,\Dcalone,\Dcalcrc,\Dtrain
\right]
\\
&\qquad=
\mathbb{I}\{T_i<\fhat_\tau(X_i)\}.
\end{align*}
Thus,~\eqref{eq:valid_weight_identity} holds for every
$i\in\calIone \cup \calIcrc$.
We now consider $i\in\calItwo$. Recall that
\[
P_i(t)=M_t(S_i(t))
\]
is the continuation probability at time $t$, and that
\[
\chat_i(t)=\chat_i(t-1)Z_t^{(i)},
\qquad
Z_t^{(i)}\sim\operatorname{Bern}(P_i(t)),
\]
where the variables $\{Z_t^{(i)}\}_{t=1}^{\maxl}$ are mutually
independent conditional on
\[
X_i,H_i,T_i,\Dcalone,\Dcalcrc,\Dtrain.
\]
Therefore, for every integer $k\geq 1$,
\begin{align}
\mathbb{P}\!\left(
    \chat_i(k)=1
    \,\middle|\,
    X_i,H_i,T_i,\Dcalone,\Dcalcrc,\Dtrain
\right)
&=
\mathbb{P}\!\left(
    Z_1^{(i)}=\cdots=Z_k^{(i)}=1
    \,\middle|\,
    X_i,H_i,T_i,\Dcalone,\Dcalcrc,\Dtrain
\right)
\nonumber\\
&=
\prod_{t=1}^{k}P_i(t).
\label{eq:reach_k_probability}
\end{align}

Define
\[
A_{i,\tau}
:=
\mathbb{I}\{T_i<\fhat_\tau(X_i)\}
\]
and
\[
R_i
:=
\chat_i\!\left(\min(\qprior(X_i),T_i)\right)
\mathbb{I}\{T_i<\qprior(X_i)\}.
\]

We first show that
\begin{equation}
\label{eq:observed_miscoverage_identity}
\mathbb{I}\{T_i<\fhat_\tau(X_i)\}
\mathbb{I}\{C_i\geq\fhat_\tau(X_i)\}
=
A_{i,\tau}R_i.
\end{equation}

If $A_{i,\tau}=0$, both sides of
\eqref{eq:observed_miscoverage_identity} are equal to zero.
Suppose instead that $A_{i,\tau}=1$. Since
$\tau\leq\tauprior$ and $\fhat_\tau$ is non-decreasing in $\tau$,
\[
T_i<\fhat_\tau(X_i)\leq\qprior(X_i).
\]
In particular,
\[
T_i<\qprior(X_i)
\qquad\text{and}\qquad
\min(\qprior(X_i),T_i)=T_i,
\]
so that
\[
R_i=\chat_i(T_i).
\]
If $R_i=1$, then the algorithm reaches and observes the event at
time $T_i$. By the definition of $C_i$, it then sets
\[
C_i=\qprior(X_i)\geq\fhat_\tau(X_i).
\]
Conversely, if $R_i=0$, then the algorithm halts before observing
the event at time $T_i$. In this case,
\[
C_i<T_i<\fhat_\tau(X_i),
\]
and hence
\[
\mathbb{I}\{C_i\geq\fhat_\tau(X_i)\}=0.
\]
Thus, on the event $A_{i,\tau}=1$,
\[
\mathbb{I}\{C_i\geq\fhat_\tau(X_i)\}=R_i,
\]
which proves \eqref{eq:observed_miscoverage_identity}.

Next, observe that whenever $A_{i,\tau}R_i=1$, the algorithm has
continued successfully through time $T_i$. Hence, the cumulative
continuation probability is
\[
\prod_{t=1}^{\min(\qprior(X_i),T_i)}P_i(t),
\]
and the corresponding weight is
\[
w(i)
=
\left(\prod_{t=1}^{\min(\qprior(X_i),T_i)}P_i(t)\right)^{-1}.
\]
Observe that the following equality holds almost surely:
\begin{equation}
\label{eq:weighted_R_identity}
w(i)A_{i,\tau}R_i
=
A_{i,\tau}R_i
\left(\prod_{t=1}^{\min(\qprior(X_i),T_i)}P_i(t)\right)^{-1}.
\end{equation}
If $A_{i,\tau}R_i=0$, both sides of~\eqref{eq:weighted_R_identity} are
zero.
Since the acquisition horizon is finite, the cumulative product is therefore strictly positive, and the weight is well defined.

Moreover, using the definition of $R_i$ and
\eqref{eq:reach_k_probability},
\begin{align}
&\mathbb{E}\!\left[
    R_i
    \,\middle|\,
    X_i,H_i,T_i,\Dcalone,\Dcalcrc,\Dtrain
\right]
\nonumber\\
&\qquad=
\mathbb{I}\{T_i<\qprior(X_i)\}
\mathbb{P}\!\left(
    \chat_i\!\left(\min(\qprior(X_i),T_i)\right)=1
    \,\middle|\,
    X_i,H_i,T_i,\Dcalone,\Dcalcrc,\Dtrain
\right)
\nonumber\\
&\qquad=
\mathbb{I}\{T_i<\qprior(X_i)\}
\prod_{t=1}^{\min(\qprior(X_i),T_i)}P_i(t).
\label{eq:R_expectation}
\end{align}

Using \eqref{eq:observed_miscoverage_identity},
\eqref{eq:weighted_R_identity}, and
\eqref{eq:R_expectation}, we obtain
\begin{align*}
&\mathbb{E}\!\left[
    w(i)\,
    \mathbb{I}\{T_i<\fhat_\tau(X_i)\}\,
    \mathbb{I}\{C_i\geq\fhat_\tau(X_i)\}
    \,\middle|\,
    X_i,H_i,T_i,\Dcalone,\Dcalcrc,\Dtrain
\right]
\\
&\qquad=
\mathbb{E}\!\left[
    w(i)A_{i,\tau}R_i
    \,\middle|\,
    X_i,H_i,T_i,\Dcalone,\Dcalcrc,\Dtrain
\right]
\\
&\qquad=
A_{i,\tau}
\left(\prod_{t=1}^{\min(\qprior(X_i),T_i)}P_i(t)\right)^{-1}
\mathbb{E}\!\left[
    R_i
    \,\middle|\,
    X_i,H_i,T_i,\Dcalone,\Dcalcrc,\Dtrain
\right]
\\
&\qquad=
A_{i,\tau}
\left(\prod_{t=1}^{\min(\qprior(X_i),T_i)}P_i(t)\right)^{-1}
\mathbb{I}\{T_i<\qprior(X_i)\}
\prod_{t=1}^{\min(\qprior(X_i),T_i)}P_i(t).
\end{align*}

On the event $A_{i,\tau}=1$, we have already established that
\[
T_i<\qprior(X_i)
\qquad\text{and}\qquad
\min(\qprior(X_i),T_i)=T_i.
\]
Hence,
\begin{align*}
&A_{i,\tau}
\left(\prod_{t=1}^{T_i}P_i(t)\right)^{-1}
\mathbb{I}\{T_i<\qprior(X_i)\}
\prod_{t=1}^{\min(\qprior(X_i),T_i)}P_i(t)
\\
&\qquad=
A_{i,\tau}
\left(\prod_{t=1}^{T_i}P_i(t)\right)^{-1}
\prod_{t=1}^{T_i}P_i(t)
\\
&\qquad=
A_{i,\tau}.
\end{align*}
If $A_{i,\tau}=0$, the expression is also zero. Therefore, in all
cases,
\[
\mathbb{E}\!\left[
    w(i)\,
    \mathbb{I}\{T_i<\fhat_\tau(X_i)\}\,
    \mathbb{I}\{C_i\geq\fhat_\tau(X_i)\}
    \,\middle|\,
    X_i,H_i,T_i,\Dcalone,\Dcalcrc,\Dtrain
\right]
=
\mathbb{I}\{T_i<\fhat_\tau(X_i)\}.
\]
\end{proof}

\noindent We now turn to present the main coverage rate guarantee.
\begin{theorem}\label{thm:main_alg_validity}
    
Fix a miscoverage level $\alpha \in (0,1)$ and a tolerance level $\delta \in (0,1)$.
Suppose that, conditional on $\Dtrain$,
$\{(X_i,T_i,H_i)\}_{i\in\calI}$ and
$(\Xtest,\Ttest,\Htest)$ are drawn i.i.d.
Let $\mathcal T=[0,\tauprior]$, assume that
$\{\fhat_\tau\}_{\tau\in\mathcal T}$ is measurable with respect to
$\Dtrain$, and that $\fhat_\tau(x)$ is non-decreasing and continuous
in $\tau$ with $\fhat_\tau(x) > 0$ for all $x$ and $\tau > 0$. Define $\qprior=\fhat_{\tauprior}$.
Suppose the continuation probabilities of Algorithm~\ref{alg:main_alg} satisfy $P_i(t)>0$ almost surely for every $t \leq \min(\qprior(X_i),T_i)$ and $i\in \calItwo$. For $i\in\calItwo$, define the latent inverse-probability weight
\[
    w^*(i)
    :=
    \left(
        \prod_{t=1}^{\min\{\qprior(X_i),T_i\}}P_i(t)
    \right)^{-1}.
\]
Further assume that there exists a constant $\bar{ \nu} > \alpha^2$, which may depend on $\Dtrain$ but not on $\Dcalone,\Dcalcrc$, such that almost surely over $\Dcalone,\Dcalcrc$:
\begin{equation*}
    \mathbb{E}[w^*(i) \mathbb{I}\{T_i<\qprior(X_i)\}\mid \Dcalone,\Dcalcrc, \Dtrain] \leq \bar{ \nu}
    \quad \text{for all } i \in \calItwo.
\end{equation*}

\noindent Define the coverage gap:
\begin{equation*}
    \Delta
    := \frac{\log(1/\delta)}{3|\calI|}
    + \sqrt{
        \frac{\log^2(1/\delta)}{9|\calI|^2}
        +
        \frac{2\bigl(\bar{\nu} - \alpha^2\bigr)\log(1/\delta)}{|\calI|}
    }.
\end{equation*}
Then, with probability at least $1 - \delta$ over the draw of
$\Dall = (\Dtrain, \Dcal)$, the lower predictive bound
$\hat{L}(x)$ generated by Algorithm~\ref{alg:main_alg} satisfies:
\begin{equation*}
    \mathbb{P}\!\left[
        \Ttest \geq \hat{L}(\Xtest)
        \,\middle|\,
        \Dall
    \right]
    \;\geq\;
    1 - \alpha - \Delta.
\end{equation*}
\end{theorem}
\begin{proof}
First, we note that, in practice, we do not compute the entire sequential continuation indicators $\chat_i(t)$. However, for the sake of this proof, indicators $\chat_i(t)$ can be viewed as an infinite sequence generated via $Z_t^i \sim \text{Ber}(M_t(S_i(t)))$, which is dependent only on the history $H_i$, features $X_i$, and calibration data $\Dcalone,\Dcalcrc$. 
Thus, by the construction of the censoring times and following the i.i.d. assumption, we have
\begin{enumerate}
    \item 
    Conditional on $(\Dcalone,\Dcalcrc,\Dtrain)$, the acquisition policy $M_t(\cdot)$ is fixed. By the i.i.d.\ sampling assumption, the latent trajectories $\{H_i(t)\}_{t}$ for $i \in\calItwo$ are mutually independent, and the Bernoulli continuation draws $\{Z_t^i \}_{t}$ are independent across samples and depend only on each sample's own trajectory. Therefore, $\{(X_i,H_i,T_i,C_i)\}_{i\in\calItwo}$ are mutually independent conditional on $(\Dcalone,\Dcalcrc,\Dtrain)$.
    \item The marginal law of $(X_i, T_i, H_i)$ for $i \in \calItwo$ is independent
        of $\Dcalone,\Dcalcrc$ given $\Dtrain$, and $\{(X_i, T_i, H_i)\}_{i \in \calItwo}$
        are mutually independent given $(\Dcalone,\Dcalcrc, \Dtrain)$.
\end{enumerate}
\noindent By Proposition~\ref{prop:valid_weights}, the weights used by the algorithm satisfy, almost surely,
\begin{equation}
\mathbb{E}\!\left[
    w(i)\,
    \mathbb{I}\{T_i < \fhat_\tau(X_i)\}\,
    \mathbb{I}\{C_i \geq \fhat_\tau(X_i)\}
    \,\middle|\,
    X_i,H_i,T_i,\Dcalone,\Dcalcrc,\Dtrain
\right]
=
\mathbb{I}\{T_i < \fhat_\tau(X_i)\},
\end{equation}
for all $i \in \calI$ such that $\tau \in \mathcal{T} = [0, \tauprior]$. 

Because the continuation probabilities $P_i(t)$ are measurable
with respect to
$\sigma(X_i,H_i,T_i,\Dcalone,\Dcalcrc,\Dtrain)$,
their finite product, and hence $w^*(i)$, is measurable with respect to
the same sigma-algebra.
Moreover, $P_i(t)>0$ almost surely at every relevant step, so the
finite product is strictly positive.

Therefore, all assumptions of Theorem~\ref{thm:general_validity} hold, and thus the LPB generated by Algorithm~\ref{alg:main_alg} is a valid PAC-type LPB with the desired bounds.
\end{proof}
\noindent We now present the general theorem for coverage validity guarantee. 

\begin{theorem}[General Coverage Validity]
\label{thm:general_validity}
Fix a miscoverage level $\alpha \in (0,1)$ and a tolerance level $\delta \in (0,1)$.
Let $\Dtrain$ be the training set, and let $\calI = \calIone \sqcup \calItwo \sqcup \calIcrc$ be a
disjoint partition of the calibration set indexes, with $|\calIone| + |\calItwo| + |\calIcrc|   = |\calI|$.
Suppose that conditional on $\Dtrain$, $\{(X_i, T_i, H_i)\}_{i \in \calI}$ and $(\Xtest, \Ttest, \Htest)$
are drawn i.i.d., and that $\fhat_\tau(x)$ is non-decreasing and continuous in $\tau$, with $\fhat_\tau(x) > 0$ for all $x$ and $\tau > 0$. We assume $\{\fhat_\tau\}_{\tau \in \mathcal{T}}$ is measurable with respect to $\Dtrain$. Let $\mathcal T=[0,\tauprior]$ and define
$\qprior(x):=\fhat_{\tauprior}(x)$.
For $i\in\calIone \cup \calIcrc$, the weights are given by $w(i)=1$, and $C_i = \qprior(X_i)$.
For every $i\in\calItwo$ and every $\tau\in\mathcal{T}$, assume
that the weights satisfy, almost surely,
\begin{equation}
\label{eq:general_valid_weight_condition}
\mathbb{E}\!\left[
    w(i)\,
    \mathbb{I}\{T_i<\fhat_\tau(X_i)\}\,
    \mathbb{I}\{C_i\geq\fhat_\tau(X_i)\}
    \,\middle|\,
    X_i,H_i,T_i,\Dcalone,\Dcalcrc,\Dtrain
\right]
=
\mathbb{I}\{T_i<\fhat_\tau(X_i)\}.
\end{equation}
For every $i\in\calItwo$, assume that $w(i)\in[1,\infty)$ is
measurable with respect to $\sigma(X_i,H_i,T_i,\Dcalone,\Dcalcrc,\Dtrain).$

Assume the following conditions hold:
\begin{enumerate}
    \item \textbf{(Mutual independence)}
    Conditional on $(\Dcalone,\Dcalcrc,\Dtrain)$, the observations $\{(X_i,H_i,T_i,C_i)\}_{i\in\calItwo}$ are mutually independent.
    \item \textbf{(Marginal independence of calibration data)}
        The marginal law of $(X_i, T_i, H_i)$ for $i \in \calItwo$ is independent
        of $\Dcalone,\Dcalcrc$ given $\Dtrain$, and $\{(X_i, T_i, H_i)\}_{i \in \calItwo}$
        are mutually independent given $(\Dcalone,\Dcalcrc, \Dtrain)$.
    \item \textbf{(Bounded weighted errors)} There exists a constant $\bar{\nu} > \alpha^2$, which may depend on
        $\Dtrain$ but not on $\Dcalone,\Dcalcrc$, such that almost surely over $\Dcalone,\Dcalcrc$:
        \begin{equation*}
            \mathbb{E}[w(i) \mathbb{I}\{T_i<\qprior(X_i)\} \mid \Dcalone, \Dcalcrc, \Dtrain] \leq \bar{\nu}
            \quad \text{for all } i \in \calItwo.
        \end{equation*}
\end{enumerate}
Define the estimated miscoverage rate and calibrated quantile level by:
\begin{equation*}
    \hat{\alpha}(\tau)
    := \frac{1}{|\calI|} \sum_{i \in \calI}
        w(i)\,
        \mathbb{I}\!\left\{\tilde{T}_i < \fhat_{\tau}(X_i) \leq C_i\right\},
    \qquad
    \hat{\tau}
    := \sup\!\left\{
        \tau \in \mathcal{T}
        \;:\;
        \sup_{\tau' \leq \tau} \hat{\alpha}(\tau') \leq \alpha
    \right\}.
\end{equation*}
By convention, $0\in\mathcal{T}$ and $\fhat_0(x)=0$.
Hence $\hat{\alpha}(0)=0$, so $\hat{\tau}$ is well defined.

Define the coverage gap:
\begin{equation*}
    \Delta
    := \frac{\log(1/\delta)}{3|\calI|}
    + \sqrt{
        \frac{\log^2(1/\delta)}{9|\calI|^2}
        +
        \frac{2\bigl(\bar{\nu} - \alpha^2\bigr)\log(1/\delta)}{|\calI|}
    }.
\end{equation*}
Then, with probability at least $1 - \delta$ over the draw of
$\Dall = (\Dtrain, \Dcal)$, the lower predictive bound
$\hat{L}(x) = \fhat_{\hat{\tau}}(x)$ satisfies:
\begin{equation*}
    \mathbb{P}\!\left[
        \Ttest \geq \hat{L}(\Xtest)
        \,\middle|\,
        \Dall
    \right]
    \;\geq\;
    1 - \alpha - \Delta.
\end{equation*}
\end{theorem}
\begin{proof}
We define the oracle miscoverage level by:
\begin{equation}
    \tau(\alpha + \Delta) = \sup \left\{\lambda \in \mathcal{T} :
    \mathbb{P} \left(T < \fhat_\lambda(X) \mid \Dtrain \right) \leq \alpha + \Delta \right\}.
\end{equation}
The above set is non-empty since $\tau = 0$ satisfies $\mathbb{P}(T < \fhat_0(X)\mid\Dtrain) = \mathbb{P}(T < 0\mid\Dtrain) = 0 \leq \alpha + \Delta$ by the convention $\fhat_{0}(x) = 0$.

We set $\tau_*:=\tau(\alpha+\Delta).$
If $\tau_{*}  = \tauprior$, then $\hat{\tau} \leq \tauprior = \tau_{*}$ deterministically. Moreover, left continuity gives $\mathbb{P} \left(T < \fhat_{\tau_{*}}(X) \mid \Dtrain \right) \leq \alpha + \Delta$. Hence, in what follows, assume $\tau_*<\tauprior$.

Observe that it suffices to show
$1 - \delta \leq \mathbb{P}(\hat{\tau} \leq \tau(\alpha + \Delta) \mid \Dtrain)$,
since on the event $\{\hat{\tau} \leq \tau(\alpha + \Delta)\}$, the monotonicity of $\fhat_\tau$ in $\tau$ and the left-continuity of
$\tau \mapsto \mathbb{P}(T \geq \fhat_{\tau}(X) \mid \Dall)$ give:
\begin{equation}
    \mathbb{P}(T \geq \fhat_{\hat{\tau}}(X) \mid \Dall)
    \geq \mathbb{P}(T \geq \fhat_{\tau(\alpha + \Delta)}(X) \mid \Dall)
    \geq 1 - \alpha - \Delta.
\end{equation}
It therefore remains to show
$\mathbb{P}(\hat{\tau} \leq \tau(\alpha+\Delta) \mid \Dtrain) \geq 1-\delta$.
Fix $0<\varepsilon < \tauprior - \tau_{*}$ and set $\lambda := \tau(\alpha+\Delta)+\varepsilon$.
We claim that $\{\hat{\alpha}(\lambda) > \alpha\} \subseteq \{\hat{\tau} \leq \lambda\}$.
Indeed, if $\hat{\alpha}(\lambda) > \alpha$, then for every $\tau \geq \lambda$,
$\sup_{\tau' \leq \tau}\hat{\alpha}(\tau') \geq \hat{\alpha}(\lambda) > \alpha$,
so no $\tau \geq \lambda$ belongs to the set
$\{\tau \in \mathcal{T} : \sup_{\tau' \leq \tau}\hat{\alpha}(\tau') \leq \alpha\}$,
hence $\hat{\tau} = \sup\{\tau \in \mathcal{T} : \sup_{\tau'\leq\tau}\hat{\alpha}(\tau')\leq\alpha\}
\leq \lambda$.
It therefore suffices to show:
\begin{equation}\label{eq:suffices}
    \mathbb{P}(\hat{\alpha}(\lambda) \leq \alpha \mid \Dtrain) \leq \delta,
\end{equation}
since this gives $\mathbb{P}(\hat{\tau} \leq \lambda \mid \Dtrain) \geq 1-\delta$
for every $0<\varepsilon <\tauprior - \tau_{*}$.
Taking $\varepsilon \to 0^+$ and applying continuity of the probability measure
to the decreasing family of events $\{\hat{\tau} \leq \tau(\alpha+\Delta)+\varepsilon\}$
(whose intersection equals $\{\hat{\tau} \leq \tau(\alpha+\Delta)\}$) gives:
\begin{equation}
    \mathbb{P}(\hat{\tau} \leq \tau(\alpha+\Delta) \mid \Dtrain)
    = \lim_{\varepsilon\to 0^+}
    \mathbb{P}(\hat{\tau} \leq \tau(\alpha+\Delta)+\varepsilon \mid \Dtrain)
    \geq 1-\delta,
\end{equation}
as required. We define:
\begin{align}
    S_\lambda &:= \sum_{i \in \calI} w(i)\,
        \mathbb{I}\{T_i < \fhat_\lambda(X_i)\}\,
        \mathbb{I}\{C_i \geq \fhat_\lambda(X_i)\}, \\
    S^{(1)}_\lambda &:= \sum_{i \in \calIone \cup \calIcrc} w(i)\,
        \mathbb{I}\{T_i < \fhat_\lambda(X_i)\}\,
        \mathbb{I}\{C_i \geq \fhat_\lambda(X_i)\}, \\
    S^{(2)}_\lambda &:= \sum_{i \in \calItwo} w(i)\,
        \mathbb{I}\{T_i < \fhat_\lambda(X_i)\}\,
        \mathbb{I}\{C_i \geq \fhat_\lambda(X_i)\},
\end{align}
so that $S_\lambda = S^{(1)}_\lambda + S^{(2)}_\lambda$ and
$\hat{\alpha}(\lambda) = S_\lambda / |\calI|$, since $\mathbb{I}\{T_i < \fhat_\lambda(X_i)\leq C_i\}=\mathbb{I}\{\tilde{T}_i < \fhat_\lambda(X_i)\leq C_i\}$.

By construction, for $i\in\calIone \cup \calIcrc$, $C_i=\qprior(X_i)$ and $w(i)=1$ deterministically. Since $\fhat_\lambda(X_i)\leq\qprior(X_i)$,
$\mathbb{I}\{C_i\geq\fhat_\lambda(X_i)\}=1$ deterministically. This gives us:
\begin{equation}\label{eq:S1_simplification}
    S^{(1)}_\lambda
    = \sum_{i \in \calIone \cup \calIcrc} \mathbb{I}\{T_i < \fhat_\lambda(X_i)\}
    =: N^{(1)}_\lambda.
\end{equation}
Lastly, we set
$A := \bar{\nu} - \alpha^2 $. Observe that
\begin{align}
    \bar{\nu} &\geq  \mathbb{E}[w(i) \mathbb{I}\{T_i<\qprior(X_i)\} \mid \Dcalone, \Dcalcrc, \Dtrain] \\
    & \geq \mathbb{E}[ \mathbb{I}\{T_i<\fhat_{\lambda}(X_i)\} \mid \Dcalone, \Dcalcrc, \Dtrain] \\
    & > \alpha,
\end{align}
and thus $A > 0$.

\paragraph{Step 1: Chernoff bound given $\Dcalone,\Dcalcrc$.}

By the tower property:
\begin{equation}\label{eq:tower}
    \mathbb{P}(\hat{\alpha}(\lambda) \leq \alpha \mid \Dtrain)
    = \mathbb{E}\!\left[
        \mathbb{P}(N^{(1)}_\lambda + S^{(2)}_\lambda \leq |\calI|\alpha
        \mid \Dcalone,\Dcalcrc, \Dtrain)
        \,\middle|\, \Dtrain
    \right].
\end{equation}
Conditioning on $\Dcalone,\Dcalcrc$ makes $N^{(1)}_\lambda$ a deterministic non-negative integer. Notice that $\bar{\nu}$ is also deterministic given $\Dtrain$ by its definition.
Since $N^{(1)}_\lambda$ is deterministic given $\Dcalone,\Dcalcrc$,
Markov's inequality applied to the non-negative random variable
$\exp(-t(N^{(1)}_\lambda + S^{(2)}_\lambda - |\calI|\alpha)) \geq 0$
gives, for any $t > 0$:
\begin{equation}\label{eq:chernoff}
    \mathbb{P}(N^{(1)}_\lambda + S^{(2)}_\lambda \leq |\calI|\alpha \mid \Dcalone,\Dcalcrc, \Dtrain)
    \leq e^{t|\calI|\alpha - t N^{(1)}_\lambda}
    \cdot \mathbb{E}\!\left[e^{-t S^{(2)}_\lambda} \mid \Dcalone,\Dcalcrc, \Dtrain\right].
\end{equation}

\paragraph{Step 2: Bounding $\mathbb{E}[e^{-tS^{(2)}_\lambda} \mid \Dcalone,\Dcalcrc, \Dtrain]$.}

For $i \in \calItwo$, define
$V_i := w(i)\,\mathbb{I}\{T_i < \fhat_\lambda(X_i)\}\,
\mathbb{I}\{C_i \geq \fhat_\lambda(X_i)\}$,
so $S^{(2)}_\lambda = \sum_{i\in\calItwo} V_i$.

\textit{Mutual independence of $\{V_i\}_{i\in\calItwo}$ given $(\Dcalone,\Dcalcrc,\Dtrain)$.}

By our assumption,
$\{(X_i,H_i,T_i,C_i)\}_{i\in\calItwo}$ are mutually independent
conditional on $(\Dcalone,\Dcalcrc,\Dtrain)$.
Moreover, $w(i)$ is measurable with respect to
$\sigma(X_i,H_i,T_i,\Dcalone,\Dcalcrc,\Dtrain)$.
Therefore, $V_i$ is a measurable function of the $i$-th
observation and the conditioning data, and hence
$\{V_i\}_{i\in\calItwo}$ are mutually independent conditional on
$(\Dcalone,\Dcalcrc,\Dtrain)$.
\textit{Mean of $V_i$.}
By the tower property and~\eqref{eq:general_valid_weight_condition},
\begin{align}
\mathbb{E}[V_i\mid\Dcalone,\Dcalcrc,\Dtrain]
&=
\mathbb{E}\!\left[
    \mathbb{E}\!\left[
        V_i
        \,\middle|\,
        X_i,H_i,T_i,\Dcalone,\Dcalcrc,\Dtrain
    \right]
    \,\middle|\,
    \Dcalone,\Dcalcrc,\Dtrain
\right]
\nonumber\\
&=
\mathbb{E}\!\left[
    \mathbb{I}\{T_i<\fhat_\lambda(X_i)\}
    \,\middle|\,
    \Dcalone,\Dcalcrc,\Dtrain
\right]
\nonumber\\
&=
\mathbb{P}\!\left(
    T<\fhat_\lambda(X)
    \,\middle|\,
    \Dtrain
\right)
=:P_\lambda,
\label{eq:mean_V}
\end{align}
where the last equality follows from the i.i.d.\ assumption and
the independence of the marginal law of $(X_i,T_i,H_i)$ from
$\Dcalone,\Dcalcrc$ given $\Dtrain$.

\textit{Second moment and variance of $V_i$.}

Define
\[
A_{i,\lambda}
:=
\mathbb{I}\{T_i<\fhat_\lambda(X_i)\},
\]

Notice that:
\[
V^2_i=w^2(i)A_{i,\lambda}\mathbb{I}\{C_i\geq\fhat_\lambda(X_i)\} =w(i) V_i 
\]

Since $w(i)$ is fixed conditional on
$(X_i,H_i,T_i,\Dcalone,\Dcalcrc,\Dtrain)$,
\begin{align}
&\mathbb{E}\!\left[
    V_i^2
    \,\middle|\,
    X_i,H_i,T_i,\Dcalone,\Dcalcrc,\Dtrain
\right]
\nonumber\\
&\qquad=
\mathbb{E}\!\left[
    w(i) V_i
    \,\middle|\,
    X_i,H_i,T_i,\Dcalone,\Dcalcrc,\Dtrain
\right]
\nonumber\\
&\qquad=
w(i) \mathbb{E}\!\left[
     V_i
    \,\middle|\,
    X_i,H_i,T_i,\Dcalone,\Dcalcrc,\Dtrain
\right]
\nonumber\\
&\qquad=
w(i)A_{i,\lambda},
\label{eq:valid_weights_second_moment}
\end{align}
where the last step used~\eqref{eq:general_valid_weight_condition}.

Taking conditional expectation with respect to
$(\Dcalone,\Dcalcrc,\Dtrain)$ and applying the bounded mean-weight
assumption gives
\begin{align}
\label{eq:second_moment_bound}
\mathbb{E}[V_i^2\mid\Dcalone,\Dcalcrc,\Dtrain]&=
\mathbb{E}\!\left[
    w(i)A_{i,\lambda}
    \,\middle|\,
    \Dcalone,\Dcalcrc,\Dtrain
\right]\\
&\leq  \mathbb{E}\!\left[
    w(i)A_{i,\tauprior}
    \,\middle|\,
    \Dcalone,\Dcalcrc,\Dtrain
\right]\\
&\leq\bar \nu.
\end{align}
Consequently,
\begin{align}
\operatorname{Var}(V_i\mid\Dcalone,\Dcalcrc,\Dtrain)
&=
\mathbb{E}[V_i^2\mid\Dcalone,\Dcalcrc,\Dtrain]
-
\mathbb{E}[V_i\mid\Dcalone,\Dcalcrc,\Dtrain]^2
\nonumber\\
&\leq
\bar \nu-P_\lambda^2
\nonumber\\
&\leq
\bar\nu-\alpha^2
=:A,
\label{eq:variance_V}
\end{align}
where the last inequality follows from
$P_\lambda>\alpha+\Delta>\alpha$.

\textit{One-sided Bernstein MGF bound for $S^{(2)}_\lambda$.}
Define mean-zero variables $Y_i^{(2)} := P_\lambda - V_i$ for $i \in \calItwo$.
Since $V_i \geq 0$, we have $Y_i^{(2)} \leq P_\lambda \leq 1$.

Since $Y_i^{(2)} := P_\lambda - V_i$ and $P_\lambda = \mathbb{E}[V_i \mid \Dcalone,\Dcalcrc, \Dtrain]$
is deterministic given $(\Dcalone,\Dcalcrc, \Dtrain)$, we have:
\begin{equation}
\begin{split} 
    \mathrm{Var}(Y_i^{(2)} \mid \Dcalone,\Dcalcrc, \Dtrain)
    &= \mathrm{Var}(P_\lambda - V_i \mid \Dcalone,\Dcalcrc, \Dtrain)\\
    &= \mathrm{Var}(V_i \mid \Dcalone,\Dcalcrc, \Dtrain)\\
     &= \mathbb{E}[V_i^2 \mid \Dcalone,\Dcalcrc, \Dtrain] - P_\lambda^2 \\
    &\leq \bar{ \nu} - P_\lambda^2 \\
    &\leq \bar{ \nu} - \alpha^2
    =: A
\end{split}
\end{equation}
where the second equality uses the fact that adding or subtracting a constant
does not change variance, and the last inequality follows from the definition of the supremum $\tau(\alpha+\Delta)$, and since $\lambda = \tau(\alpha+\Delta)+\varepsilon > \tau(\alpha+\Delta)$, we have $ P_\lambda > \alpha + \Delta$.

We apply~\cite[Proposition~2.14]{wainwright2019high} for $Y_i^{(2)}$ with $Y_i^{(2)} \leq 1$ a.s., and $$\mathbb{E}\left[\left(Y_i^{(2)}\right)^2 \mid \Dcalone,\Dcalcrc, \Dtrain\right] = \mathrm{Var}\left[Y_i^{(2)} \mid \Dcalone,\Dcalcrc, \Dtrain \right] \leq A$$ to get:
\begin{equation}\label{eq:bernstein_S2}
    \mathbb{E}\!\left[
        \exp\!\left(t\sum_{i\in\calItwo} Y_i^{(2)}\right)
        \,\middle|\, \Dcalone,\Dcalcrc, \Dtrain
    \right]
    \leq \exp\!\left(\frac{t^2 |\calItwo| A / 2}{1 - t/3}\right),
    \quad 0 < t < 3.
\end{equation}
Since $S^{(2)}_\lambda = |\calItwo|P_\lambda - \sum_{i\in\calItwo} Y_i^{(2)}$:
\begin{equation}\label{eq:MGF_S2}
    \mathbb{E}\!\left[e^{-tS^{(2)}_\lambda} \mid \Dcalone,\Dcalcrc, \Dtrain\right]
    \leq \exp\!\left(
        -t|\calItwo|P_\lambda
        + \frac{t^2 |\calItwo| A / 2}{1 - t/3}
    \right),
    \quad 0 < t < 3.
\end{equation}
The right-hand side of \eqref{eq:MGF_S2} depends only on
the deterministic quantities $P_\lambda$, $A$, $|\calItwo|$, and $t$,
and in particular does not depend on $\Dcalone,\Dcalcrc$.

\paragraph{Step 3: Bounding $\mathbb{E}[e^{-tN^{(1)}_\lambda} \mid \Dtrain]$.}

Under our assumptions, $\{(X_i,T_i)\}_{i\in\calIone \cup \calIcrc}$ are i.i.d.\ with
marginal law independent of $\Dcaltwo$ given $\Dtrain$.
Hence $Z_i^{(1)} := \mathbb{I}\{T_i < \fhat_\lambda(X_i)\}
\sim \mathrm{Bernoulli}(P_\lambda)$ are i.i.d.\ given $\Dtrain$,
and $N^{(1)}_\lambda = \sum_{i\in\calIone \cup \calIcrc} Z_i^{(1)}$.

Define mean-zero $U_i := P_\lambda - Z_i^{(1)}$ for $i \in \calIone \cup \calIcrc$,
so $\{U_i\}_{i\in\calIone \cup \calIcrc}$ are i.i.d.\ given $\Dtrain$.
Since $Z_i^{(1)} \geq 0$, $U_i \leq P_\lambda \leq 1$,
so the one-sided bound $U_i \leq 1$ holds.

\textit{Variance of $U_i$ bounded by $A$.}
Let $j\in \calItwo$, with $w(j) \geq 1$.
\begin{align}
P_\lambda &= \mathbb{E}\!\left[
    A_{j,\lambda}
    \,\middle|\,
    \Dtrain
\right]\\
&\leq\mathbb{E}\!\left[
    A_{j,\tauprior}
    \,\middle|\,
    \Dtrain
\right]\\
&\leq
\mathbb{E}\!\left[
    w(j)A_{j,\tauprior}
    \,\middle|\,
    \Dtrain
\right]
\\
&=
\mathbb{E}\!\left[
    \mathbb{E}\!\left[
        w(j)A_{j,\tauprior}
        \,\middle|\,
        \Dcalone,\Dcalcrc,\Dtrain
    \right]
    \,\middle|\,
    \Dtrain
\right]
\\
&\leq
\bar{\nu},
\end{align}
Since $\bar \nu \geq \alpha$, and $P_\lambda>\alpha$, we have
\begin{equation}
\label{eq:var_U}
\operatorname{Var}(U_i\mid\Dtrain)
=
P_\lambda(1-P_\lambda)
=
P_\lambda-P_\lambda^2
\leq
\bar{\nu} - P_\lambda^2
\leq
\bar  \nu-\alpha^2
=
A.
\end{equation}
We apply~\cite[Proposition~2.14]{wainwright2019high} with $b=1$ and $\nu=A$ to each i.i.d.\ $U_i$, summing over $|\calIone \cup \calIcrc|$ terms,
and writing $N^{(1)}_\lambda = |\calIone \cup \calIcrc|P_\lambda - \sum_{i\in\calIone \cup \calIcrc}U_i$:
\begin{equation}\label{eq:MGF_N1}
\begin{split}
    \mathbb{E}\!\left[e^{-tN^{(1)}_\lambda} \mid \Dtrain\right]
    &= e^{-t|\calIone \cup \calIcrc|P_\lambda}
    \cdot \mathbb{E}\!\left[
        \exp\!\left(t\sum_{i\in\calIone \cup \calIcrc} U_i\right)
        \,\middle|\, \Dtrain
    \right]\\
    & \leq \exp\!\left(
        -t|\calIone \cup \calIcrc|P_\lambda
        + \frac{t^2|\calIone \cup \calIcrc|A/2}{1 - t/3}
    \right),
    \quad 0 < t < 3.
\end{split}
\end{equation}
\paragraph{Step 4: Combining the bounds.}
Applying \eqref{eq:chernoff} pointwise in $\calIone \cup \calIcrc$,
then taking the outer expectation over $\calIone \cup \calIcrc$ via \eqref{eq:tower}:
\begin{equation}
    \mathbb{P}(\hat{\alpha}(\lambda) \leq \alpha \mid \Dtrain)
    \leq e^{t|\calI|\alpha}
    \cdot \mathbb{E}\!\left[
        e^{-tN^{(1)}_\lambda}
        \cdot \mathbb{E}\!\left[
            e^{-tS^{(2)}_\lambda} \mid \Dcalone,\Dcalcrc, \Dtrain
        \right]
        \,\middle|\, \Dtrain
    \right].
\end{equation}
Since the upper bound in \eqref{eq:MGF_S2} is deterministic
(does not depend on $\Dcalone,\Dcalcrc$), it factors out of the outer expectation:
\begin{equation}
    \mathbb{P}(\hat{\alpha}(\lambda) \leq \alpha \mid \Dtrain)
    \leq e^{t|\calI|\alpha}
    \cdot \exp\!\left(
        -t|\calItwo|P_\lambda + \frac{t^2|\calItwo|A/2}{1-t/3}
    \right)
    \cdot \mathbb{E}\!\left[e^{-tN^{(1)}_\lambda} \mid \Dtrain\right].
\end{equation}
Applying \eqref{eq:MGF_N1} and using $|\calIone| + |\calIcrc| + |\calItwo| = |\calI|$:
\begin{equation}\label{eq:combined}
\begin{split}
    \mathbb{P}(\hat{\alpha}(\lambda) \leq \alpha \mid \Dtrain)
    &\leq \exp\!\left(
        t|\calI|\alpha
        - t|\calItwo|P_\lambda
        - t|\calIone \cup \calIcrc|P_\lambda
        + \frac{t^2(|\calItwo|+|\calIone| +|\calIcrc|)A/2}{1-t/3}
    \right) \\
    &= \exp\!\left(
        -t|\calI|(P_\lambda - \alpha)
        + \frac{t^2|\calI|A/2}{1-t/3}
    \right) \\
    &\leq \exp\!\left(
        -t|\calI|\Delta
        + \frac{t^2|\calI|A/2}{1-t/3}
    \right)
    =: \exp(\mathcal{E}(t)),
\end{split}
\end{equation}
for $0 < t < 3$,
where the last inequality uses $P_\lambda - \alpha \geq \Delta$
(since $\lambda > \tau(\alpha+\Delta)$ implies
$P_\lambda = \mathbb{P}(T < \fhat_\lambda(X) \mid \Dtrain) \geq \alpha+\Delta$,
by definition of $\tau(\alpha+\Delta)$ and continuity of $\fhat_\tau$).
\paragraph{Step 5: Evaluation at the Bernstein point and solving for $\Delta$.}

Set $N := |\calI|$, and recall $A = \bar{ \nu} - \alpha^2 > 0$.
We evaluate $\mathcal{E}$ at:
\begin{equation}\label{eq:tstar}
    t^* := \frac{\Delta}{A + \Delta/3}.
\end{equation}
Notice that $t^* \in(0, 3)$ since $\Delta,A > 0$.
Substituting into $\mathcal{E}(t^*)$:
\begin{equation}\label{eq:Estar}
\begin{split}
    \mathcal{E}(t^*)
    &= -\frac{N\Delta^2}{A+\Delta/3}
    + \frac{NA}{2}\cdot\frac{\Delta^2}{(A+\Delta/3)^2}\cdot\frac{1}{A / (A + \Delta/3)}\\
    &= -\frac{N\Delta^2}{A+\Delta/3}
    + \frac{NA}{2}\cdot\frac{\Delta^2}{A(A+\Delta/3)}\\
    &= -\frac{N\Delta^2}{A+\Delta/3}
    + \frac{N\Delta^2}{2(A+\Delta/3)} \\
    &= -\frac{N\Delta^2}{2(A+\Delta/3)}
\end{split}
\end{equation}
Therefore, 
\begin{equation}\label{eq:condition}
    \exp(\mathcal{E}(t^*)) \leq \delta \iff \frac{N\Delta^2}{2(A + \Delta/3)} \geq \log\!\left(\frac{1}{\delta}\right).
\end{equation}
Multiplying both sides by $2(A+\Delta/3) > 0$ and rearranging:
\begin{equation}\label{eq:quadratic}
    |\calI|\Delta^2
    - \frac{2}{3}\log(1/\delta)\,\Delta
    - 2(\bar{ \nu} - \alpha^2)\log\!\left(\frac{1}{\delta}\right)
    = 0.
\end{equation}
This quadratic in $\Delta$ has positive leading coefficient $|\calI|>0$
and negative constant term $-2A\log(1/\delta)<0$.
By Descartes' rule of signs, it has exactly one positive root:
\begin{equation}\label{eq:Delta_final}
    \Delta
    := \frac{\log(1/\delta)}{3|\calI|}
    + \sqrt{
        \frac{\log^2(1/\delta)}{9|\calI|^2}
        + \frac{2(\bar{ \nu} - \alpha^2)\log(1/\delta)}{|\calI|}
    }.
\end{equation}
For the above $\Delta$, condition~\eqref{eq:condition}
holds with equality, so
$\mathbb{P}(\hat{\alpha}(\lambda)\leq\alpha\mid\Dtrain)
\leq \exp(\mathcal{E}(t^*)) \leq \delta$,
satisfying \eqref{eq:suffices}.

\paragraph{Step 6: Conclusion.}
Since~\eqref{eq:suffices} holds for every $0<\varepsilon<\tauprior-\tau_*$ with
$\lambda = \tau(\alpha+\Delta)+\varepsilon$,
we have $\mathbb{P}(\hat{\tau}\leq\lambda\mid\Dtrain)\geq 1-\delta$. Taking $\varepsilon\to 0^+$ and applying
continuity of the probability measure to the decreasing family of events
$\{\hat{\tau}\leq\tau(\alpha+\Delta)+\varepsilon\}$,
whose intersection is $\{\hat{\tau}\leq\tau(\alpha+\Delta)\}$:
\begin{equation}
    \mathbb{P}\!\left(
        \hat{\tau} \leq \tau(\alpha+\Delta) \,\middle|\, \Dtrain
    \right)
    = \lim_{\varepsilon\to 0^+}
    \mathbb{P}\!\left(
        \hat{\tau} \leq \tau(\alpha+\Delta)+\varepsilon
        \,\middle|\, \Dtrain
    \right)
    \geq 1-\delta.
\end{equation}
On the event $\{\hat{\tau} \leq \tau(\alpha+\Delta)\}$,
monotonicity of $\fhat_\tau$ in $\tau$ gives
$\fhat_{\hat\tau}(x) \leq \fhat_{\tau(\alpha+\Delta)}(x)$
for all $x$, hence
$\{T \geq \fhat_{\tau(\alpha+\Delta)}(X)\}
\subseteq \{T \geq \fhat_{\hat\tau}(X)\}$.
Since $\fhat_{\tau(\alpha+\Delta)}$ is $\Dtrain$-measurable
and $(\Xtest, \Ttest)$ is independent of
$(\Dcalone,\Dcalcrc,\Dcaltwo)$ given $\Dtrain$:
\begin{equation}
    \mathbb{P}\!\left[
        \Ttest \geq \fhat_{\hat{\tau}}(\Xtest)
        \,\middle| \Dall
    \right]
    \geq \mathbb{P}\!\left[
        \Ttest \geq
        \fhat_{\tau(\alpha+\Delta)}(\Xtest)
        \,\middle|\, \Dall
    \right]
    \geq 1 - \alpha - \Delta,
\end{equation}
where the last inequality uses the definition of $\tau(\alpha+\Delta)$
and left-continuity of $\tau\mapsto\mathbb{P}(T<\fhat_\tau(X)\mid\Dtrain)$,
which give
$\mathbb{P}(T<\fhat_{\tau(\alpha+\Delta)}(X)\mid\Dall)
\leq\alpha+\Delta$.
Notice that $\mathbb{P}(T<\fhat_{\tau(\alpha+\Delta)}(X)\mid\Dall) = \mathbb{P}(T<\fhat_{\tau(\alpha+\Delta)}(X)\mid\Dtrain)$ since $\fhat_{\tau(\alpha+\Delta)}$ is $\Dtrain$-measurable and $(\Xtest, \Ttest) \indep \Dcal \mid \Dtrain$.
This holds with probability at least $1-\delta$ over $\Dcal$ given $\Dtrain$, hence over $\Dall$, completing the proof.
\end{proof}

\begin{remark}[The assumption on $\bar\nu$ is enforced by construction]
\label{rem:nu_bounded}
The implementation of \ttmethod in Appendix~\ref{sec:proposed_alg} floors the cumulative reach probability at
$\varepsilon>0$. Thus, the assumption that the mean weights of samples satisfying $T_i<\qprior(X_i)$ holds by construction:
        \begin{equation*}
            \mathbb{E}[w(i) \mathbb{I}\{T_i<\qprior(X_i)\} \mid \Dcalone, \Dcalcrc, \Dtrain] \leq \bar{\nu}
            \quad \text{for all } i \in \calItwo,
        \end{equation*}
for $\bar\nu\leq\varepsilon^{-1}$.
\end{remark}

\begin{remark}[Discrete predictors and discrete search spaces]
\label{rem:discrete}
Theorems~\ref{thm:main_alg_validity} and~\ref{thm:general_validity} hold when
$\tau\mapsto\fhat_\tau(x)$ is discrete-valued and non-decreasing (rather than continuous), and
the search space $\mathcal{T}$ is a finite grid. This is the setting of our experiments.
Continuity is used in the proof only to guarantee that the supremum defining
$\tau_*=\tau(\alpha+\Delta)$ is attained and to justify the limit $\varepsilon\to0^{+}$; both
are immediate when $\mathcal{T}$ is finite. Writing
$P_\nu:=\mathbb{P}(T<\fhat_\nu(X)\mid\Dtrain)$, 
with the following changes, where Steps~1--5 are
unchanged:
\begin{enumerate}
  \item The definition of $\tau_*$ turns to
  $\tau_*:=\max\{\nu\in\mathcal{T}: P_\nu\leq\alpha+\Delta\}$.
  \item $\lambda=\tau_*+\varepsilon$ is replaced by
  $\lambda:=\min\{t\in\mathcal{T}:t>\tau_*\}$
  \item Since $\hat\tau\in\mathcal{T}$ and $\lambda$ is the successor of $\tau_*$, the
  inclusion $\{\hat\alpha(\lambda)>\alpha\}\subseteq\{\hat\tau<\lambda\}$ already gives
  $\hat\tau\leq\tau_*$. The limiting argument $\varepsilon\to0^{+}$ in Step~6 is therefore not
  needed, and the conclusion follows from monotonicity of $\fhat_\tau$ together with
  $P_{\tau_*}\leq\alpha+\Delta$.
\end{enumerate}
The same substitutions apply to UPB coverage guarantee in Theorem~\ref{thm:upb_coverage}, with the maximum replaced by a
minimum and the successor of $\tau_*$ by the predecessor of $\tau^\dagger:=\min\{t\in\mathcal{T}:\ \mathbb{P}(T^{(\maxl)}>\fhat_t(X)\mid\Dtrain)
\leq\alpha+\Delta_{\rm UPB}\}$ in $\mathcal{T}$.
\end{remark}

\subsection{Budget validity of \ttmethod via conformal risk control}\label{sec:budget_validity_crc}

In this section, show that \ttmethod (with CRC), as outlined in Algorithm~\ref{alg:main_alg}, satisfies the expected budget constraint in~\eqref{eq:budget_constraint}. Notice that to prove this we must assume that the budget is sufficiently large so that the set in~\eqref{eq:dapro_main_crc_selector} is not empty. Formally, this means that the total budget must satisfy almost surely:
$$
B\ge
\sum_{i\in\calIone} b_i+
\frac{
|\calItwo|\sum_{i\in\calIcrc}\bfunc(M(S_i;0))
+|\calIcrc|\sum_{i\in\calIcrc}b_i
+|\calItwo|C_{\max}
+|\calIcrc|\maxl
}{
|\calIcrc|+1
}.
$$
Since the consumed budget is non-decreasing in \(\lambda\), the above is exactly the condition that \(\lambda=0\) is feasible in~\eqref{eq:dapro_main_crc_selector}, and hence guarantees that its feasible set is nonempty.

\begin{proof}[Proof of Theorem~\ref{thm:budget_validity_informal}]
By the construction of Algorithm~\ref{alg:main_alg}, $\{(X_i,T_i,H_i)\}_{i\in\calI}$ are drawn exchangeably,
independently of $\Dtrain$, and the data split into
$\calIone,\calIcrc,$ and $\calItwo$ is independent of the trajectories.
Furthermore, conditional on $\Dtrain$ and the data used in
Phase~I, the candidate family $\{M_t(\cdot;\lambda)\}_{\lambda\in[0,1]}$
is fixed before observing $\calIcrc$, the function
$\bfunc(M(S_i;\lambda))$ is continuous and non-decreasing in $\lambda$
for every complete trajectory, and
\[
    0\leq \bfunc(M(S_i;\lambda))\leq C_{\max}
\]
for every candidate $\lambda$ and every trajectory in the population
support. Finally, the feasible set in~\eqref{eq:dapro_main_crc_selector} is nonempty by construction.

Recall that the samples in $\calIone$ and $\calIcrc$ are fully observed until either the event occurs or $\qprior(X_i)$ is reached, so that their acquisition
cost is
\[
    b_i=\min\{T_i,\qprior(X_i)\}.
\]
Since $\qprior(X_i)\leq \maxl$, we have $0\leq b_i\leq\maxl$ a.s.

Condition on $\Dtrain, \Dcalone$ and on all information used to fit the acquisition
policy in Phase~I. Under this conditioning, the Phase-I budget consumption
$\sum_{i\in\calIone}b_i$, the remaining budget $
    B_{\mathrm{rem}}
    =
    B-\sum_{i\in\calIone}b_i,$
the candidate family $\{M_t(\cdot;\lambda)\}_{\lambda\in[0,1]}$, and
$C_{\max}$ are fixed. Moreover, because the calibration trajectories are exchangeable and the split is independent of the trajectories, the complete trajectories in $\calIcrc$ and $\calItwo$ remain identically distributed and independent under this conditioning.

Consider an additional complete trajectory
$(X_\star,T_\star,H_\star)$ drawn independently from the same
distribution as the trajectories in $\calIcrc$ and $\calItwo$, and let
\[
    b_\star:=\min\{T_\star,\qprior(X_\star)\}.
\]
Let $S_\star$ denote the corresponding sequence of scores. We use this
additional trajectory only for the proof.

Define the augmented CRC selector
\begin{equation}
\label{eq:dapro_augmented_crc_selector}
\widetilde\lambda
:=
\sup\left\{
\lambda\in[0,1]:
\frac{
\displaystyle
\sum_{i\in\calIcrc}
\left(
\bfunc(M(S_i;\lambda))
+
\frac{|\calIcrc|}{|\calItwo|}b_i
\right)
+
\bfunc(M(S_\star;\lambda))
+
\frac{|\calIcrc|}{|\calItwo|}b_\star
}{
|\calIcrc|+1
}
\leq
\frac{B_{\mathrm{rem}}}{|\calItwo|}
\right\}.
\end{equation}

By assumption, $\bfunc(M(S;\lambda))$ is continuous in $\lambda$.
Hence the feasible sets in
\eqref{eq:dapro_main_crc_selector} and
\eqref{eq:dapro_augmented_crc_selector} are closed, and whenever they
are nonempty, their suprema are themselves feasible.

By the row-wise cap and the bound $b_\star\leq\maxl$,
\begin{equation}
\label{eq:dapro_augmented_cost_bound}
\bfunc(M(S_\star;\lambda))
+
\frac{|\calIcrc|}{|\calItwo|}b_\star
\leq
C_{\max}
+
\frac{|\calIcrc|}{|\calItwo|}\maxl
\qquad
\text{for every }\lambda\in[0,1].
\end{equation}
Using the definition of $\widehat\lambda$ in
\eqref{eq:dapro_main_crc_selector} together with
\eqref{eq:dapro_augmented_cost_bound}, we therefore have
\begin{align}
&\sum_{i\in\calIcrc}
\left(
\bfunc(M(S_i;\widehat\lambda))
+
\frac{|\calIcrc|}{|\calItwo|}b_i
\right)
+
\bfunc(M(S_\star;\widehat\lambda))
+
\frac{|\calIcrc|}{|\calItwo|}b_\star
\nonumber\\
&\qquad\leq
\sum_{i\in\calIcrc}
\left(
\bfunc(M(S_i;\widehat\lambda))
+
\frac{|\calIcrc|}{|\calItwo|}b_i
\right)
+
C_{\max}
+
\frac{|\calIcrc|}{|\calItwo|}\maxl
\nonumber\\
&\qquad\leq
(|\calIcrc|+1)
\frac{B_{\mathrm{rem}}}{|\calItwo|}.
\label{eq:dapro_hat_lambda_augmented_feasible}
\end{align}
Thus, $\widehat\lambda$ is feasible for the augmented problem in~\eqref{eq:dapro_augmented_crc_selector}, and consequently
\[
    \widetilde\lambda\geq\widehat\lambda.
\]
Because the candidate acquisition cost is non-decreasing in $\lambda$,
\begin{equation}
\label{eq:dapro_monotone_augmented}
\bfunc(M(S_\star;\widehat\lambda))
+
\frac{|\calIcrc|}{|\calItwo|}b_\star
\leq
\bfunc(M(S_\star;\widetilde\lambda))
+
\frac{|\calIcrc|}{|\calItwo|}b_\star.
\end{equation}

The augmented selector $\widetilde\lambda$ is a permutation-invariant
function of the $|\calIcrc|+1$ exchangeable complete trajectories
consisting of the CRC trajectories together with the additional
trajectory $(X_\star,T_\star,H_\star)$. Therefore, by exchangeability,
\begin{align}
&
\mathbb{E}\!\left[
\bfunc(M(S_\star;\widetilde\lambda))
+
\frac{|\calIcrc|}{|\calItwo|}b_\star
\right]
\nonumber\\
&\quad=
\mathbb{E}\!\left[
\frac{1}{|\calIcrc|+1}
\left\{
\sum_{i\in\calIcrc}
\left(
\bfunc(M(S_i;\widetilde\lambda))
+
\frac{|\calIcrc|}{|\calItwo|}b_i
\right)
+
\bfunc(M(S_\star;\widetilde\lambda))
+
\frac{|\calIcrc|}{|\calItwo|}b_\star
\right\}
\right]
\nonumber\\
&\quad\leq
\frac{B_{\mathrm{rem}}}{|\calItwo|},
\label{eq:dapro_augmented_exchangeability}
\end{align}
where the final inequality follows from the feasibility of
$\widetilde\lambda$ in
\eqref{eq:dapro_augmented_crc_selector}. Combining
\eqref{eq:dapro_monotone_augmented} and
\eqref{eq:dapro_augmented_exchangeability} yields
\begin{equation}
\label{eq:dapro_fresh_composite_budget}
\mathbb{E}\!\left[
\bfunc(M(S_\star;\widehat\lambda))
\right]
+
\frac{|\calIcrc|}{|\calItwo|}
\mathbb{E}[b_\star]
\leq
\frac{B_{\mathrm{rem}}}{|\calItwo|}.
\end{equation}
Multiplying both sides by $|\calItwo|$ gives
\begin{equation}
\label{eq:dapro_remaining_expected_budget}
|\calItwo|\,
\mathbb{E}\!\left[
\bfunc(M(S_\star;\widehat\lambda))
\right]
+
|\calIcrc|\,\mathbb{E}[b_\star]
\leq
B_{\mathrm{rem}}.
\end{equation}

It remains to show that the two terms on the left-hand side of
\eqref{eq:dapro_remaining_expected_budget} are exactly the expected
deployment and CRC expenditures.

First, every sample $i\in\calIcrc$ is fully observed until
$b_i=\min\{T_i,\qprior(X_i)\}$. Since the CRC trajectories and the fresh trajectory $H_\star$ have the same distribution,
\begin{equation}
\label{eq:dapro_crc_expected_cost}
\mathbb{E}\!\left[
\sum_{i\in\calIcrc}b_i
\right]
=
|\calIcrc|\,\mathbb{E}[b_\star].
\end{equation}

Second, consider a deployment sample $i\in\calItwo$. Conditional on its
complete trajectory and on $\widehat\lambda$, step $t$ is acquired if and
only if all continuation decisions through step $t$ succeed. By the
randomization in Algorithm~\ref{alg:main_alg},
\[
    \mathbb{P}\!\left(
        \chat_i(t)=1
        \,\middle|\,
        X_i,T_i,H_i,\widehat\lambda
    \right)
    =
    \prod_{j=1}^{t}M_j(S_i(j);\widehat\lambda).
\]
The realized number of acquired interactions for this sample is
\[
    \tilde T_i
    =
    \sum_{t=1}^{b_i}\chat_i(t).
\]
Consequently,
\begin{align}
\mathbb{E}\!\left[
\tilde T_i
\,\middle|\,
X_i,T_i,H_i,\widehat\lambda
\right]
&=
\sum_{t=1}^{b_i}
\mathbb{P}\!\left(
\chat_i(t)=1
\,\middle|\,
X_i,T_i,H_i,\widehat\lambda
\right)
\nonumber\\
&=
\sum_{t=1}^{b_i}
\prod_{j=1}^{t}M_j(S_i(j);\widehat\lambda)
\nonumber\\
&=
\bfunc(M(S_i;\widehat\lambda)).
\label{eq:dapro_realized_expected_cost}
\end{align}
Each deployment trajectory is independent of the CRC sample conditional
on the Phase-I information and has the same distribution as the fresh
trajectory. Hence,
\begin{equation}
\label{eq:dapro_deployment_expected_cost}
\mathbb{E}\!\left[
\sum_{i\in\calItwo}\tilde T_i
\right]
=
|\calItwo|\,
\mathbb{E}\!\left[
\bfunc(M(S_\star;\widehat\lambda))
\right].
\end{equation}

Combining
\eqref{eq:dapro_remaining_expected_budget},
\eqref{eq:dapro_crc_expected_cost}, and
\eqref{eq:dapro_deployment_expected_cost}, we obtain
\[
\mathbb{E}\!\left[
\sum_{i\in\calIcrc}b_i
+
\sum_{i\in\calItwo}\tilde T_i
\right]
\leq
B_{\mathrm{rem}}.
\]
Finally, the Phase-I budget consumption is fixed under the conditioning and is
equal to $\sum_{i\in\calIone}b_i$. Therefore,
\begin{align*}
\mathbb{E}\!\left[
\sum_{i\in\calI}\tilde T_i
\right]
&=
\sum_{i\in\calIone}b_i
+
\mathbb{E}\!\left[
\sum_{i\in\calIcrc}b_i
+
\sum_{i\in\calItwo}\tilde T_i
\right]\\
&\leq
\sum_{i\in\calIone}b_i+B_{\mathrm{rem}}\\
&=B.
\end{align*}
This proves the desired expected budget guarantee conditional on the
Phase-I information. Taking expectation over the Phase-I data $\Dcalone$ and
$\Dtrain$ completes the proof.
\end{proof}

\subsection{Budget validity of \ttmethod without CRC under policy generalization}
\label{sec:budget_validity_results}
In this section, we prove that the expected budget used by \ttmethod without CRC, as outlined in Algorithm~\ref{alg:main_alg}, is bounded by the nominal budget constraint $B$. 
Unless stated otherwise, all expectations and probabilities in this section are conditional on $\Dtrain$; we omit this explicit conditioning for notational convenience.
\subsubsection{Assumptions}
We begin by setting the following assumptions.
\begin{assumption}[i.i.d.\ data]\label{assump:iid}
The pairs $\{(S_k)\}_{k=1}^N$ are \emph{independent and
identically distributed}, each drawn from a joint distribution~$\nu$
on $\R^\maxl\times[0,1]^\maxl$. We denote by $\mu$ the marginal distribution of $S_k$ on $\R^\maxl$ and by $\mu_t$ the marginal distribution of the $t$-th coordinate $S_k(t)$ on $\R$, for $t\in\{1,\ldots,\maxl\}$.
\end{assumption}
\begin{assumption}[Split independence]\label{assump:split}
The data is randomly split by~$\pi$, which is a random variable, taking values in the collection of subsets of $\{1,\ldots,N\}$ of size~$N_1$, distributed uniformly over
all $\binom{N}{N_1}$ such subsets. The
random split $\pi$ is independent of the full dataset $\{(X_k, T_k, H_k, S_k)\}_{k=1}^N$.
\end{assumption}
\noindent Recall that for a given split $\pi$, Algorithm~\ref{alg:main_alg} splits the calibration data, indexed by $\{1,\dots,N\}$, as follows:
\begin{align*}
  \calIone   &:= \pi
    \subset \{1,\ldots,N\},\quad |\calIone|=N_1,\\
  \calItwo  &:= \{1,\ldots,N\}\setminus \calIone,
    \quad |\calItwo|=N_2,\\
  \Dcalone &:= \{(X_i,H_i, T_i)\}_{i\in\calIone}
    \qquad\text{(calibration dataset)},\\
  \St &:= S_j
    \quad\text{for a fixed }j\in\calItwo.
\end{align*}
By symmetry of the i.i.d.\ assumption and the uniform split, the choice of $j$ within $\calItwo$ does not affect any
expectation below. We re-index the calibration (cal1) samples as
$i\in\{1,\ldots,N_1\}$ for notational convenience.

\begin{assumption}[Oracle mapping]\label{assump:oracle}
For each $\beta \geq 0$ and each $t\in\{1,\ldots,\maxl\}$, there exists a measurable function
\[
  M_t^*(\,\cdot\,;\beta):\mathbb{R}\longrightarrow [0,1]
\]
that is monotonically increasing in its first argument.
For $s=(s(1),\ldots,s(\maxl))\in\R^\maxl$ we write
$M^*(s;\beta):=(M_1^*(s(1);\beta),\ldots,M_\maxl^*(s(\maxl);\beta))$.

Recall that $\Bbar = (B - B_1)/N_2$ is the remaining budget per sample for Phase~II, which is a deterministic function of $\Dcalone$.
For a given realization of $\Dcalone$ (and hence a given value of $\Bbar$),
the oracle mapping $M^*(\,\cdot\,;\Bbar)$ satisfies the following two conditions $\Dcalone$-almost surely:
\begin{enumerate}
  \item \textbf{(Pathwise oracle fit)} For all $i\in\{1,\ldots,N_1\}$:
  \[
    P_i(t) \;=\; M_t^*(S_i(t);\Bbar)
    \quad \mathbb{P}\text{-almost surely.}
  \]
  \item \textbf{(Pathwise population budget constraint)}
  \[
    \E_{S\sim\mu}\!\bigl[\bfunc(M^*(S;\Bbar))\bigr] \;\leq\; \Bbar.
  \]
\end{enumerate}
\end{assumption}

\begin{assumption}[Estimated mapping]\label{assump:est}
Given $\Dcalone$, the estimated mapping $M_t(\cdot;\Dcalone):\R\to[0,1]$ is
measurable and monotonically increasing for each
$t\in\{1,\ldots,\maxl\}$. Write
$M(s;\Dcalone):=(M_1(s(1);\Dcalone),\ldots,M_\maxl(s(\maxl);\Dcalone))$.

\end{assumption}

\begin{assumption}[vanishing cross-term moments]
\label{assump:joint_moments}
For each $t\in\{1,\ldots,\maxl\}$ and each non-empty subset
$A\subseteq\{1,\ldots,t\}$, the following holds for $\Dcalone$-almost every
realization:
\[
  \E_{\St}\!\left[
    \prod_{k\in A}\eps_k(\St(k);\Dcalone)
    \cdot
    \prod_{k\in\{1,\ldots,t\}\setminus A} M_k^*(\St(k); \Dcalone)
  \;\Big|\;\Dcalone\right]
  \;=\; 0,
\]
where $\eps_k(s;\Dcalone) = M_k(s;\Dcalone) - M_k^*(s;\Bbar)$. Observe that for samples $i$, for times $t > b_i$, $ \eps_t(s;\Dcalone)=M_t(s;\Dcalone) = M_t^*(s;\Bbar)=0$.
\end{assumption}
\noindent Conceptually, Assumption~\ref{assump:joint_moments} requires that these mixed products have zero conditional expectation over $\St$, given the calibration data $\Dcalone$. We emphasize that this condition is not a mere mathematical artifact of our proof; rather, it is required to guarantee that the projection inaccuracies do not inflate the expected budget consumed during Phase II. 

In practice, there are two cases where this assumption is satisfied. First, it holds exactly in finite samples if the
coordinates of the score vector $\St$ are mutually independent, and the estimated mapping is coordinate-wise unbiased, i.e., $\E_{\St(k)} \left[\eps_k(\St(k);\Dcalone)\right] = 0 $ almost surely, for all $k$.

Alternatively, even without assuming independence or finite-sample unbiasedness, this assumption holds asymptotically, assuming $M_t(\cdot \mid \Dcalone)$ is a consistent estimator of $M_t^*$. Specifically, if $\mathbb{E}_{\St(j)}[|M_j(\St(j);\Dcalone)-M_j^*(\St(j))|\mid \Dcalone]\to 0$ as $N_1 \to \infty$, $\Dcalone$-almost surely, then,
for any fixed non-empty set $A$ and any $k\in A$, the corresponding cross-term satisfies
\begin{equation}
\begin{aligned}
\left|
\mathbb{E}_{\St}\left[
\prod_{j\in A}\varepsilon_j
\prod_{j\notin A}M_j^*(\St;\Bbar)
\middle|\Dcalone
\right]
\right|\
\le
\mathbb{E}_{\St}\left[
\prod_{j\in A}|\varepsilon_j|
\middle|\Dcalone
\right]
\le
\mathbb{E}_{\St(k)}\left[
|\varepsilon_k|
\middle|\Dcalone
\right]
\overset{N_1\to\infty}{\to 0} 
\end{aligned}
\end{equation}
where the first inequality uses $M_j^*\in[0,1]$, and the second uses $|\varepsilon_j|\le 1$ for all $j$.
Thus, Assumption~\ref{assump:joint_moments} holds asymptotically for any fixed $\Dcalone$ and target budget $\Bbar$, meaning the budget constraint can be satisfied with arbitrary precision given sufficient calibration data in Phase I.

\subsubsection{Expected budget validity proofs}

To establish the expected budget bound, we first prove the following lemmas.

\paragraph{Independence lemma}

\begin{lemma}[Independence of $\Dcalone$ and $\St$]\label{lem:indep}
Under Assumptions~\ref{assump:iid} and~\ref{assump:split},
\[
  \Dcalone \;\perp\; \St.
\]
\end{lemma}
\begin{proof}
Given the split $\pi$, $\Dcalone$ and $\St$ correspond to disjoint subsets
of the i.i.d.\ array $\{(S_k,P_k)\}_{k=1}^N$. Disjoint subsets of an
i.i.d.\ sequence are independent of each other.  Since $\pi$ is
independent of the data (Assumption~\ref{assump:split}), conditioning
on $\pi$ does not introduce dependence, and so $\Dcalone\perp\St$. 
\end{proof}

\paragraph{Properties of the budget function}
Throughout this section we adopt the convention that the continuation probability of a
trajectory vanishes beyond its acquisition horizon:
\begin{equation}
\label{eq:truncation_convention}
    P_i(t)\;=\;M_t(S_i(t))\,\mathbb{I}\{t\leq b_i\},
    \qquad
    P_i^*(t)\;=\;M_t^*(S_i(t);\Bbar)\,\mathbb{I}\{t\leq b_i\}.
\end{equation}

The budget function $\bfunc:[0,1]^\maxl\to\R$ is given by:
\[
  \bfunc(P)
  \;=\;
  \sum_{t=1}^{\min (T_i, \qprior(X_i))}\prod_{j=1}^{t}P(j),
  \qquad P=(P(1),\ldots,P(\maxl))\in[0,1]^\maxl.
\]

\begin{lemma}[Properties of the budget function]\label{lem:f_props}
\begin{enumerate}
  \item \textbf{Boundedness:} $0 \leq \bfunc(P) \leq\maxl$ for all
        $P\in[0,1]^\maxl$.
  \item \textbf{Multilinearity:} $\bfunc$ is a polynomial in
        $\{P(j)\}_{j=1}^\maxl$ that is linear in each coordinate
        $P(j)$ when all other coordinates are held fixed.
  \item \textbf{Monotonicity:} $\bfunc$ is monotonically increasing in each
        coordinate $P(j)$.
\end{enumerate}
\end{lemma}
\begin{proof}
(i) Each product $\prod_{j=1}^t P(j)\in[0,1]$ and there are $\maxl$
terms.  (ii) Direct inspection of the polynomial structure.  (iii)
For any $j_0$, increasing $P(j_0)$ increases every term $\prod_{j=1}^t
P(j)$ with $t\geq j_0$.
\end{proof}

\paragraph{Estimation error does not affect expected loss}

\begin{lemma}[Cross-term cancellation]\label{lem:cross}
Under
Assumptions~\ref{assump:iid},
\ref{assump:split},
\ref{assump:oracle},
\ref{assump:est}, and
\ref{assump:joint_moments},
the following holds for $\Dcalone$-almost every realization of the calibration data:
\[
  \E_{\St}\!\bigl[\bfunc(M(\St;\Dcalone))\mid\Dcalone\bigr]
  \;=\;
  \E_{\St}\!\bigl[\bfunc(M^*(\St;\Bbar))\mid\Dcalone\bigr].
\]
Here $\E_{\St}[\,\cdot\mid\Dcalone]$ denotes the conditional expectation
over $\St$ given $\Dcalone$. Since $\Dcalone\perp\St$
(Lemma~\ref{lem:indep}), this conditional expectation equals the
marginal expectation over $\St\sim\mu$, with $\Dcalone$ treated as a fixed parameter.
\end{lemma}

\begin{proof}
Fix a realization of $\Dcalone$ (a.s.).  Define for each $j\in\{1,\ldots,\maxl\}$:
\begin{equation}
\begin{split}
  p_j := M_j^*(\St(j);\Bbar),
  \qquad
  e_j &:= \eps_j(\St(j);\Dcalone) = M_j(\St(j);\Dcalone)\\
  &- M_j^*(\St(j);\Bbar).
  \end{split}
\end{equation}
Given $\Dcalone$, both $p_j$ and $e_j$ are bounded, measurable functions of
$\St(j)$ alone.
\medskip
For each $t\in\{1,\ldots,\maxl\}$, expand each factor as $M_j(\St(j);\Dcalone)
= p_j + e_j$ and expand the product by the distributive law:
\begin{equation}\label{eq:expand}
  \prod_{j=1}^{t}M_j(\St(j);\Dcalone)
  = \prod_{j=1}^{t}(p_j + e_j)
  = \sum_{A\,\subseteq\,\{1,\ldots,t\}}
    \left(\prod_{j\in A}e_j\right)
    \!\left(\prod_{j\in\{1,\ldots,t\}\setminus A}p_j\right).
\end{equation}
We prove the above by induction over $t$. For $t=1$: $\prod_{j=1}^{1}(p_j + e_j) = p_1 + e_1$, which satisfies this identity.
Assume that the identity holds for some arbitrary integer $k \ge 1$. Then, for $t = k+1$:
\begin{equation*}
    \prod_{j=1}^{k+1}(p_j + e_j) = \left( \prod_{j=1}^{k}(p_j + e_j) \right) (p_{k+1} + e_{k+1})
\end{equation*}

\noindent By our inductive assumption, we can substitute the expansion for the product up to $k$:
\begin{equation*}
    = \left( \sum_{A \subseteq S_k} \left(\prod_{j\in A}e_j\right) \left(\prod_{j\in S_k\setminus A}p_j\right) \right) (p_{k+1} + e_{k+1})
\end{equation*}
Distributing the terms $p_{k+1}$ and $e_{k+1}$ across the sum gives two separate sums:
\begin{align*}
    &= \sum_{A \subseteq S_k} \left(\prod_{j\in A}e_j\right) \left(\prod_{j\in S_k\setminus A}p_j\right) p_{k+1} \ + \ \sum_{A \subseteq S_k} \left(\prod_{j\in A}e_j\right) \left(\prod_{j\in S_k\setminus A}p_j\right) e_{k+1}\\
    &=  \sum_{X \subseteq S_{k+1}} \left(\prod_{j\in X}e_j\right) \left(\prod_{j\in S_{k+1}\setminus X}p_j\right) \\
    &=
    \prod_{j=1}^{k+1}(p_j + e_j).
\end{align*}
We now take the conditional expectation given $\Dcalone$.

\begin{align}
  &\E_{\St}\!\left[\prod_{j=1}^{t}M_j(\St(j);\Dcalone)\;\Big|\;\Dcalone\right]
  \nonumber\\
  &= \sum_{A\subseteq\{1,\ldots,t\}}
     \E_{\St}\!\left[
       \prod_{j\in A}e_j
       \cdot
       \prod_{j\in\{1,\ldots,t\}\setminus A}p_j
     \;\Big|\;\Dcalone\right].
  \label{eq:exp_expand}
\end{align}
Exchanging the sum and expectation is valid since the sum is finite and
all terms are bounded (Lemma~\ref{lem:f_props}(i)).
Next, we isolate the $A=\emptyset$ term.

\begin{align}
  \eqref{eq:exp_expand}
  &= \underbrace{%
       \E_{\St}\!\left[\prod_{j=1}^{t}p_j\;\Big|\;\Dcalone\right]
     }_{A=\emptyset}
     +
     \sum_{\emptyset\neq A\subseteq\{1,\ldots,t\}}
     \E_{\St}\!\left[
       \prod_{j\in A}e_j
       \cdot
       \prod_{j\in\{1,\ldots,t\}\setminus A}p_j
     \;\Big|\;\Dcalone\right].
  \label{eq:iso_empty}
\end{align}
Now, for each non-empty $A\subseteq\{1,\ldots,t\}$, we apply
Assumption~\ref{assump:joint_moments}:
\[
  \E_{\St}\!\left[
    \prod_{j\in A}e_j
    \cdot
    \prod_{j\in\{1,\ldots,t\}\setminus A}p_j
  \;\Big|\;\Dcalone\right]
  = 0
  \qquad \Dcalone\text{-a.s.}
\]
Hence the entire sum over non-empty $A$ in~\eqref{eq:iso_empty} is
zero, and:
\begin{equation}\label{eq:only_empty}
  \E_{\St}\!\left[\prod_{j=1}^{t}M_j(\St(j);\Dcalone)\;\Big|\;\Dcalone\right]
  = \E_{\St}\!\left[\prod_{j=1}^{t}M_j^*(\St(j);\Bbar)\;\Big|\;\Dcalone\right].
\end{equation}
Summing~\eqref{eq:only_empty} over $t\in\{1,\ldots,\maxl\}$ ($P_i(t) = M_t^*(S_i;\Bbar) = 0$ for $t>b_i$) and using
linearity of conditional expectation:
\[
  \E_{\St}\!\bigl[\bfunc(M(\St;\Dcalone))\mid\Dcalone\bigr]
  = \E_{\St}\!\bigl[\bfunc(M^*(\St;\Bbar))\mid\Dcalone\bigr].
\]
\end{proof}

\paragraph{Oracle budget is bounded}

\begin{lemma}[Oracle budget bound]\label{lem:inout}
Under Assumption~\ref{assump:oracle}, for $\Dcalone$-almost every realization:
\[
  \E_{\St}\!\bigl[\bfunc(M^*(\St;\Bbar))\mid \Dcalone\bigr] \;\leq\; \Bbar.
\]
\end{lemma}
\begin{proof}
This follows directly from condition~(ii) of Assumption~\ref{assump:oracle},
which states that $\E_{\St}[\bfunc(M^*(\St;\Bbar))\mid \Dcalone] \leq \Bbar$ holds
$\Dcalone$-almost surely.
\end{proof}

\paragraph{Proof of the Holdout Budget Bound}
\begin{proposition}[Holdout sample budget bound]\label{prop:budget_validity}
Under
Assumptions~\ref{assump:iid},
\ref{assump:split},
\ref{assump:oracle},
\ref{assump:est}, and
\ref{assump:joint_moments},
for $\Dcalone$-almost every realization:
\[
  \E_{\St}\!\bigl[\bfunc(M(\St;\Dcalone))\mid \Dcalone \bigr] \;\leq\; \Bbar.
\]
\end{proposition}
\begin{proof}[Proof of Proposition~\ref{prop:budget_validity}]
By Lemma~\ref{lem:cross} (cross-term cancellation), for $\Dcalone$-almost every realization:
\[
  \E_{\St}\!\bigl[\bfunc(M(\St;\Dcalone))\mid \Dcalone\bigr]
  \;=\;
  \E_{\St}\!\bigl[\bfunc(M^*(\St;\Bbar))\mid \Dcalone\bigr].
\]
By Lemma~\ref{lem:inout} (oracle budget bound), for $\Dcalone$-almost every realization:
\[
  \E_{\St}\!\bigl[\bfunc(M^*(\St;\Bbar))\mid \Dcalone \bigr] \;\leq\; \Bbar.
\]
Chaining these two relations gives the result.
\end{proof}

\paragraph{Proof of the Algorithm's Budget Validity}

\begin{theorem}[Algorithm's budget bound]\label{thm:alg_budget_validity}
Under
Assumptions~\ref{assump:iid},
\ref{assump:split},
\ref{assump:oracle},
\ref{assump:est}, and
\ref{assump:joint_moments}, the expected budget used by Algorithm~\ref{alg:main_alg} is bounded by $B$.
\end{theorem}

\begin{proof}[Proof of Theorem~\ref{thm:alg_budget_validity}]
Condition on $\Dcalone$. Given $\Dcalone$, the Phase~I budget
$B_1 = \sum_{i \in \calIone} b_i$ is a fixed constant, and
$\Bbar = (B - B_1)/N_2$ is a fixed constant as well.
The total realized budget is $B_1+B_2$, where
$B_2 = \sum_{i\in\calItwo}U_i$ is random over Phase~II draws,
with $U_i = \sum_{t=1}^{b_i} \prod_{j=1}^{t} \chat_i(j)$.

By linearity of conditional expectation:
\begin{equation}
    \mathbb{E}[B_2 \mid \Dcalone] = \sum_{i\in \calItwo}\mathbb{E}[U_i \mid \Dcalone].
\end{equation}
For each $i\in \calItwo$, by the tower property and $\Dcalone \indep S_i$
(Lemma~\ref{lem:indep}):
\begin{equation}
    \mathbb{E}[U_i \mid \Dcalone]
    = \mathbb{E}_{S_i}\!\left[\mathbb{E}[U_i \mid S_i, T_i, \Dcalone]\Big|\Dcalone\right]
    = \mathbb{E}_{\St}\!\left[\bfunc(M(\St;\Dcalone))\right],
\end{equation}
where the second equality uses
$\mathbb{E}[U_i \mid S_i, T_i, \Dcalone] = \bfunc(M(S_i;\Dcalone))$
(the expected stopping time equals the budget function evaluated at the
continuation probabilities, given the Bernoulli draws are independent).

By Proposition~\ref{prop:budget_validity}, for $\Dcalone$-almost every realization:
\begin{equation}
    \mathbb{E}_{\St}\!\left[\bfunc(M(\St;\Dcalone))\mid \Dcalone\right] \leq \Bbar.
\end{equation}
Therefore $\mathbb{E}[B_2 \mid \Dcalone] \leq N_2\Bbar$, and:
\begin{align}
    \mathbb{E}[B_1 + B_2 \mid \Dcalone]
    &\leq B_1 + N_2\Bbar \\
    &= B_1 + N_2 \cdot \frac{B - B_1}{N_2} \\
    &= B.
\end{align}
Taking the expectation over $\Dcalone$ gives
$\mathbb{E}[B_1 + B_2] \leq B$.
\end{proof}

\subsubsection{Finite-sample budget guarantee}
\label{sec:finite_sample_budget_guarantee}

In this section, we provide a finite-sample bound for the budget consumed by Algorithm~\ref{alg:main_alg}. 
Throughout this section, all probabilities and expectations are conditioned on $\Dtrain$. For notational convenience, we omit this conditioning.
\begin{theorem}
\label{thm:finite_sample_budget_guarantee}
Under
Assumptions~\ref{assump:iid},
\ref{assump:split},
\ref{assump:est},
\ref{assump:oracle}, and
\ref{assump:joint_moments}, with probability at least $1-\delta$, the average budget per sample consumed by Algorithm~\ref{alg:main_alg} is bounded by:
\begin{equation}
\frac{1}{N}\sum_{i\in \calI} \Ttilde_i \leq \frac{B}{N} + \frac{\maxl\log(1/\delta)}{3N}+\frac{1}{N}\sqrt{\frac{\maxl^2\log^2(1/\delta)}{9}+2N_2 \maxl\Bbar\log(1/\delta)}.
\end{equation}
\end{theorem}
\begin{proof}

We denote $B_1 = \sum_{i\in\calIone} b_i$, where $b_i = \min(T_i, \qprior(X_i))$,
and $U_i = \sum_{t=1}^{b_i} \prod_{j=1}^{t} Z^i_j$ for any $i\in\calItwo$,
so that $B_2 = \sum_{i\in\calItwo} U_i$.

For a fixed $\Dcalone$, and fixed $(S_i,T_i)$, the continuation probabilities
$M_j(S_i(j);\Dcalone)$ are deterministic. The Bernoulli draws $\{Z^i_t\}_{t}$
are then mutually independent, so:
\begin{equation}
\mathbb{E}\left[U_i\mid \Dcalone, S_i, T_i\right]
= \sum_{t=1}^{b_i}\prod_{j=1}^{t} M_j(S_i(j);\Dcalone)
= \bfunc(M(S_i;\Dcalone)).
\end{equation}
Taking a further expectation over $(S_i,T_i)$ and using $\Dcalone \indep S_i$
(Lemma~\ref{lem:indep}):
\begin{equation}
\mathbb{E}\left[U_i\mid \Dcalone\right]
= \mathbb{E}_{\St}\!\left[\bfunc(M(\St;\Dcalone)) \mid \Dcalone\right].
\end{equation}
By Proposition~\ref{prop:budget_validity}, for $\Dcalone$-almost every realization:
\begin{equation}
\mathbb{E}\left[U_i\mid \Dcalone\right]
= \mathbb{E}_{\St}\!\left[\bfunc(M(\St;\Dcalone)) \mid \Dcalone\right]
\leq \Bbar.
\end{equation}
Hence:
\begin{equation}
\mathbb{E}\left[B_2\mid \Dcalone\right]
= \sum_{i\in\calItwo} \mathbb{E}\left[U_i \mid \Dcalone\right]
\leq N_2 \Bbar.
\end{equation}

\noindent Notice that the collection $\{U_i\}_{i\in\calItwo}$ is mutually independent given $\Dcalone$. Since $U_i \in [0,\maxl]$ almost surely, by Lemma~\ref{lem:f_props}, and $\mathbb{E}[U_i\mid \Dcalone]\leq \Bbar$, the Bhatia–Davis inequality gives:
\begin{equation}
\text{Var}[U_i \mid \Dcalone] \leq  (\maxl - \mathbb{E}[U_i\mid \Dcalone]) \cdot \mathbb{E}[U_i\mid \Dcalone] \leq \maxl \Bbar.
\end{equation}
Observe that:
\begin{equation}
\mathbb{E}\left[\left(U_i - \frac{1}{N_2} \mathbb{E}[B_2\mid \Dcalone]\right)^2 \middle| \Dcalone\right] = \mathbb{E}[(U_i - \mathbb{E}[U_i\mid \Dcalone])^2\mid \Dcalone] = \mathrm{Var}(U_i \mid \Dcalone).
\end{equation}

\noindent We now apply the Bernstein inequality to the sum $B_2 = \sum_{i\in \calItwo}U_i$. Notice that $\{U_i\}_{i\in \calItwo}$ are mutually independent variables given $\Dcalone$, where each $U_i \in [0,\maxl]$.
\begin{equation}
\begin{split}
   \mathbb{P}(B_2 - \mathbb{E}[B_2 \mid \Dcalone] \geq \varepsilon \mid \Dcalone)  &\leq
\exp\left(-\frac{\varepsilon^2/2}{\sum_{i\in\calItwo}\mathrm{Var}(U_i\mid \Dcalone)+\maxl\varepsilon/3}\right)\\
&\leq
\exp\left(-\frac{\varepsilon^2/2}{N_2 \maxl \Bbar+ \maxl\varepsilon/3}\right).
\end{split}
\end{equation}

\noindent We now show that by setting the $\varepsilon$ as given below, the probability above is bounded by $\delta$:
\begin{equation}
\varepsilon=\frac{\maxl\log(1/\delta)}{3}+\sqrt{\frac{\maxl^2\log^2(1/\delta)}{9}+2N_2 \maxl\Bbar\log(1/\delta)}
\end{equation}
For ease of notation, we denote $ V = \sum_{i\in\calItwo} \maxl \Bbar$.
\begin{equation}
    \varepsilon - \frac{\maxl\log(1/\delta)}{3} = \sqrt{\frac{\maxl^2\log^2(1/\delta)}{9} + 2V\log(1/\delta)}
\end{equation}
Squaring both sides yields:
\begin{equation}
    \left(\varepsilon - \frac{\maxl\log(1/\delta)}{3}\right)^2 = \frac{\maxl^2\log^2(1/\delta)}{9} + 2V\log(1/\delta)
\end{equation}

\begin{equation}
    \varepsilon^2 - \frac{2\maxl\log(1/\delta)}{3}\varepsilon + \frac{\maxl^2\log^2(1/\delta)}{9} = \frac{\maxl^2\log^2(1/\delta)}{9} + 2V\log(1/\delta)
\end{equation}

\noindent Notice that the $\frac{\maxl^2\log^2(1/\delta)}{9}$ term appears on both sides and cancels out:
\begin{equation}
    \varepsilon^2 - \frac{2\maxl\log(1/\delta)}{3}\varepsilon = 2V\log(1/\delta)
\end{equation}
We divide the entire equation by 2 and get:
\begin{equation}
    \frac{\varepsilon^2}{2} - \frac{\maxl\varepsilon}{3}\log(1/\delta) = V\log(1/\delta)
\end{equation}

\noindent Move the $\varepsilon$ term to the right side:
\begin{equation}
    \frac{\varepsilon^2}{2} = V\log(1/\delta) + \frac{\maxl\varepsilon}{3}\log(1/\delta)
\end{equation}
\begin{equation}
    \frac{\varepsilon^2}{2} = \left( V + \frac{\maxl\varepsilon}{3} \right) \log(1/\delta)
\end{equation}

\noindent The exponent in the original inequality bound is:
\begin{equation}
    -\frac{\varepsilon^2/2}{V + \maxl\varepsilon/3}.
\end{equation}
Substituting the equivalent expression we just found for $\frac{\varepsilon^2}{2}$:
\begin{equation}
    -\frac{\left( V + \frac{\maxl\varepsilon}{3} \right) \log(1/\delta)}{V + \maxl\varepsilon/3} = -\log(1/\delta).
\end{equation}
Therefore,
\begin{equation}
    \exp(-\log(1/\delta)) = \exp(\log(\delta)) = \delta.
\end{equation}

\noindent Putting it all together into a single block to match the original formulation:

\begin{equation}
\begin{split}
    \mathbb{P}(B_2 -  \mathbb{E}[B_2 \mid \Dcalone] \geq \varepsilon \mid \Dcalone) 
    &\leq \exp\left( -\frac{\varepsilon^2/2}{ \sum_{i\in\calItwo} \maxl \Bbar+ \maxl\varepsilon/3 } \right) \\
    &= \exp\left( -\frac{\left(\sum_{i\in\calItwo} \maxl \Bbar + \frac{\maxl\varepsilon}{3}\right) \log(1/\delta)}{\sum_{i\in\calItwo} \maxl \Bbar + \maxl\varepsilon/3} \right) \\
    &= \exp\left( -\log(1/\delta) \right) \\
    &= \delta.
\end{split}
\end{equation}
Therefore, with probability at least $1-\delta$, the budget used in Phase II is bounded by:
\begin{equation}
    B_2 \leq N_2\Bbar + \frac{\maxl\log(1/\delta)}{3}+\sqrt{\frac{\maxl^2\log^2(1/\delta)}{9}+2N_2 \maxl\Bbar\log(1/\delta)}.
\end{equation}
By the construction of $\Bbar$ of Algorithm~\ref{alg:main_alg}, we have:
\begin{equation}
    B_1 + N_2\Bbar \leq B.
\end{equation}
Therefore, the average budget per sample used by Algorithm~\ref{alg:main_alg} is bounded by:
\begin{equation}
    \frac{1}{N}\sum_{i\in \calI} \Ttilde_i \leq \frac{B}{N} + \frac{\maxl\log(1/\delta)}{3N}+\frac{1}{N}\sqrt{\frac{\maxl^2\log^2(1/\delta)}{9}+2N_2 \maxl\Bbar\log(1/\delta)}.
\end{equation}

\end{proof}

\subsubsection{Budget Guarantee with Bounded Estimation Errors}\label{sec:budget_with_errors}

We derive an approximate budget guarantee under an alternative condition that bounds the mean absolute projection error at each step, without requiring exact cancellation of the mixed terms.
Specifically, we rely on the following assumption:
\begin{assumption}[Bounded Mean Absolute Estimation Error]
\label{assum:bounded_errors}
For each $j \in \{1,\ldots,\maxl\}$, there exists a constant $\eta_j \geq 0$ such that, almost surely over $\Dcalone$:
\begin{equation}
    \E_{\St(j)}\!\left[|\varepsilon_j(\St(j);\Dcalone)|\;\middle|\;\Dcalone\right] \leq \eta_j,
\end{equation}
where $\varepsilon_j(s;\Dcalone) := M_j(s;\Dcalone) - M_j^*(s ; \Bbar)$ is the pointwise estimation error at step $j$.
\end{assumption}
\noindent Observe that Assumption~\ref{assump:joint_moments} requires the cross-term expectations: $$\E[\prod_{j\in A}\varepsilon_j(s;\Dcalone) \cdot \prod_{j\notin A}M_j^*(s ; \Bbar)\mid\Dcalone]=0$$ for every non-empty $A$, which implies $\eta_j = 0$ in case where $\St$ has independent coordinates. In contrast, Assumption~\ref{assum:bounded_errors} is different: it requires only that the marginal mean absolute error at each coordinate is bounded, without making assumptions on the joint cross-term. In particular, it does not require that the score coordinates are independent.
Furthermore, this assumption is reasonable for consistent estimators $M_j(\cdot;\Dcalone)$ that satisfy $\E_{\St(j)}[|\varepsilon_j(\St(j);\Dcalone)|\mid\Dcalone]\to 0$ as $N_1\to\infty$, and thus $\eta_j\to 0$.

\begin{theorem}[Expected Budget with Estimation Error]
\label{thm:budget_guarantee_with_estimation_error}
Under Assumptions~\ref{assump:iid},
\ref{assump:split},
\ref{assump:oracle},
\ref{assump:est} and~\ref{assum:bounded_errors}, the expected total budget consumed by Algorithm~\ref{alg:main_alg} satisfies:
\begin{equation}
    \E_{\Dcal}\!\left[\sum_{i\in\calI}\Ttilde_i\right] \;\leq\; B + N_2\,\Gamma,
\end{equation}
where $\Gamma = \sum_{k=1}^\maxl (\maxl-k+1)\eta_k$.
\end{theorem}
\noindent To prove it, we first develop an upper bound on the expected budget used in Phase II:

\begin{lemma}\label{lem:b2_bound_estimation_error}
Under  
Assumptions~\ref{assump:iid},
\ref{assump:split},
\ref{assump:oracle},
\ref{assump:est}, for $\Dcalone$-almost every realization:
\begin{equation}
\E_{\St}\!\left[\bfunc(M(\St;\Dcalone))\;\middle|\;\Dcalone\right] \;\leq\; \E_{\St}\!\left[\bfunc(M^*(\St; \Bbar))\right] \;+\; \Gamma,
\end{equation}
where the total bias is:
\begin{equation}
    \Gamma \;:=\; \sum_{k=1}^\maxl (\maxl - k + 1)\,\eta_k.
\end{equation}
\end{lemma}

\begin{proof}
We first show a telescoping product identity.
For any two sequences of real numbers $\{a_j\}_{j=1}^t$ and $\{b_j\}_{j=1}^t$, the following identity holds:
\begin{equation}
    \prod_{j=1}^t a_j - \prod_{j=1}^t b_j = \sum_{k=1}^t \left(\prod_{j=1}^{k-1}a_j\right)(a_k - b_k)\left(\prod_{j=k+1}^t b_j\right),
\end{equation}
where empty products equal 1. We verify by induction. For $t=1$: $a_1 - b_1 = (a_1 - b_1)$, which is trivially true. Assuming it holds for $t-1$:
\begin{align}
    \prod_{j=1}^t a_j - \prod_{j=1}^t b_j
    &= a_t\prod_{j=1}^{t-1}a_j - b_t\prod_{j=1}^{t-1}b_j \notag\\
    &= (a_t - b_t)\prod_{j=1}^{t-1}a_j + b_t\!\left(\prod_{j=1}^{t-1}a_j - \prod_{j=1}^{t-1}b_j\right) \notag\\
    &= (a_t-b_t)\prod_{j=1}^{t-1}a_j + b_t\sum_{k=1}^{t-1}\!\left(\prod_{j=1}^{k-1}a_j\right)(a_k-b_k)\!\left(\prod_{j=k+1}^{t-1}b_j\right) \notag\\
    &= \sum_{k=1}^{t}\left(\prod_{j=1}^{k-1}a_j\right)(a_k-b_k)\left(\prod_{j=k+1}^{t}b_j\right),
\end{align}
completing the induction.

We now apply this identity to the continuation probabilities.
Set $a_j = M_j(\St(j);\Dcalone)$ and $b_j = M_j^*(\St(j) ; \Bbar)$, so $a_j - b_j = \varepsilon_j(\St(j);\Dcalone)$ and $a_j, b_j \in [0,1]$. Then for each $t \in \{1,\ldots,\maxl\}$:
\begin{equation}
\begin{split}
        &\prod_{j=1}^t M_j(\St(j);\Dcalone) - \prod_{j=1}^t M_j^*(\St(j);\Bbar) \\
        &= \sum_{k=1}^t \underbrace{\left(\prod_{j=1}^{k-1}M_j(\St(j);\Dcalone)\right)}_{\displaystyle=:W_k^- \,\in\,[0,1]} \cdot\, \varepsilon_k(\St(k);\Dcalone) \cdot \underbrace{\left(\prod_{j=k+1}^t M_j^*(\St(j);\Bbar)\right)}_{\displaystyle=:W_k^+ \,\in\,[0,1]}.
\end{split}
\end{equation}
Since $W_k^-, W_k^+ \in [0,1]$ (being products of values in $[0,1]$), we have $|W_k^- \cdot \varepsilon_k \cdot W_k^+| \leq |\varepsilon_k|$ pointwise. Therefore, taking the conditional expectation given $\Dcalone$:
\begin{align}
    \left|\E_{\St}\!\left[\prod_{j=1}^t M_j - \prod_{j=1}^t M_j^*\;\middle|\;\Dcalone\right]\right|
    &\leq \E_{\St}\!\left[\left|\prod_{j=1}^t M_j - \prod_{j=1}^t M_j^*\right|\;\middle|\;\Dcalone\right] \notag\\
    &\leq \E_{\St}\!\left[\sum_{k=1}^t W_k^-\,|\varepsilon_k({\St}(k);\Dcalone)|\,W_k^+\;\middle|\;\Dcalone\right] \notag\\
    &\leq \sum_{k=1}^t \E_{\St}\!\left[|\varepsilon_k({\St}(k);\Dcalone)|\;\middle|\;\Dcalone\right],
\end{align}
where the final inequality uses $W_k^- W_k^+ \leq 1$ and the fact that $|\varepsilon_k({\St}(k);\Dcalone)|$ depends only on the $k$-th coordinate $S(k)$, so that $\E_{\St}[W_k^- |\varepsilon_k({\St}(k);\Dcalone)| W_k^+\mid\Dcalone] \leq \E_{\St}[|\varepsilon_k({\St}(k);\Dcalone)|\mid\Dcalone] = \E_{\St(k)}[|\varepsilon_k({\St}(k);\Dcalone)|\mid\Dcalone]$ regardless of the joint distribution of coordinates.

Applying Assumption~\ref{assum:bounded_errors}: $\E_{\St(k)}[|\varepsilon_k({\St}(k);\Dcalone)|\mid\Dcalone] \leq \eta_k$, giving:
\begin{equation}
    \left|\E_{\St}\!\left[\prod_{j=1}^t M_j(\St;\Dcalone) - \prod_{j=1}^t M_j^*(\St;\Bbar)\;\middle|\;\Dcalone\right]\right| \leq \sum_{k=1}^t \eta_k.
\end{equation}
We now sum over $t$:
\begin{align}
    &\E_{\St}[\bfunc(M(\St;\Dcalone))\mid\Dcalone] - \E_{\St}[\bfunc(M^*(\St ; \Bbar))\mid\Dcalone]\\
    &= \sum_{t=1}^\maxl \E_{\St}\!\left[\prod_{j=1}^t M_j(\St(j);\Dcalone) - \prod_{j=1}^t M_j^*(\St(j);\Bbar)\;\middle|\;\Dcalone\right] \notag\\
    &\leq \sum_{t=1}^\maxl \sum_{k=1}^t \eta_k \notag\\
    &= \sum_{k=1}^\maxl \eta_k \cdot |\{t : t \geq k,\, t \leq \maxl\}| \notag\\
    &= \sum_{k=1}^\maxl \eta_k(\maxl - k + 1) = \Gamma,
\end{align}
where we interchanged the order of summation: $\sum_{t=1}^\maxl \sum_{k=1}^t = \sum_{k=1}^\maxl \sum_{t=k}^\maxl$, and $\sum_{t=k}^\maxl 1 = \maxl-k+1$.
\end{proof}
\noindent Armed with the above Lemma, we turn to prove Theorem~\ref{thm:budget_guarantee_with_estimation_error}.

\begin{proof}[Proof of Theorem~\ref{thm:budget_guarantee_with_estimation_error}]
    
The Phase I budget $B_1 = \sum_{i\in\calIone} b_i$ satisfies $\E[B_1] = N_1\E[b_1]$, and by construction $\Bbar = (B - B_1)/N_2$ so $B_1 + N_2\Bbar = B$ deterministically. For the Phase II budget:
\begin{equation}
    \E[B_2\mid\Dcalone] = N_2\,\E_{\St}[\bfunc(M(\St;\Dcalone))\mid\Dcalone].
\end{equation}
By Lemma~\ref{lem:b2_bound_estimation_error} and Lemma~\ref{lem:inout}:
\begin{equation}
    \E_{\St}[\bfunc(M(\St;\Dcalone))\mid\Dcalone] \leq \E_{\St}[\bfunc(M^*(\St;\Bbar))] + \Gamma \leq \Bbar + \Gamma.
\end{equation}
Therefore:
\begin{equation}
    \E[B_2\mid\Dcalone] \leq N_2(\Bbar+\Gamma) = B - B_1 + N_2\Gamma.
\end{equation}
Taking the total expectation and adding $\E[B_1]$:
\begin{equation}
    \E[B_1 + B_2] \leq B + N_2\Gamma. 
\end{equation}
\end{proof}
\noindent Lastly, we develop a finite-sample budget bound based on the above result.

\begin{theorem}\label{thm:finite_sample_budget_bound_with_errors}
Under the conditions of Theorem~\ref{thm:budget_guarantee_with_estimation_error}, with probability at least $1-\delta$:

\begin{equation}
\begin{split}
    \frac{1}{N}\sum_{i\in\calI}\Ttilde_i &\;\leq\; \frac{B}{N} + \frac{N_2}{N}\Gamma + \frac{\maxl\log(1/\delta)}{3N} \\
    &+ \frac{1}{N}\sqrt{\frac{\maxl^2\log^2(1/\delta)}{9} + 2N_2 \maxl(\Bbar+\Gamma)\log(1/\delta)}.
\end{split}
\end{equation}

\end{theorem}

\begin{proof}
Identical to Theorem~\ref{thm:finite_sample_budget_guarantee} except that the bound over $\mathbb{E}[B_2\mid \Dcalone] $ is replaced with:

\begin{equation}
\mathbb{E}[B_2\mid \Dcalone] \leq N_2\Bbar + N_2\Gamma.
\end{equation}
The above inequality follows from Lemma~\ref{lem:b2_bound_estimation_error} and Lemma~\ref{lem:inout}.
\end{proof}
\noindent While Theorem~\ref{thm:finite_sample_budget_bound_with_errors} establishes the asymptotic consistency of \ttmethod, the finite-sample relative budget gap of $O(\maxl^2 / (\Bbar \sqrt{N_1}))$ is highly conservative. This looseness stems from the absolute value inequality applied in Lemma~\ref{lem:b2_bound_estimation_error}. By taking the absolute value of the errors at each time step, the proof assumes a worst-case scenario where the estimation inaccuracies at every time step $t \in \{1,\dots,\maxl\}$ aggregate in the same direction, with no cancellation. 

In practice, if the projection models $\{M_t\}_{t=1}^{\maxl}$ are reasonably well-calibrated, e.g., yielding approximately mean-zero errors conditionally, the pointwise errors tend to cancel each other out as the sequence progresses. This explains why our empirical evaluations in Section~\ref{sec:experiments} show that \ttmethod satisfies the target budget $B$ even when $\maxl=200$ and $N_1$ is relatively small. Therefore, rather than viewing this bound as a strict operational limitation, it should be interpreted as a sample complexity guideline, and it demonstrates that the Phase I calibration size $N_1$ must grow quartically with the maximum conversation horizon $\maxl$ to satisfy the budget constraint at this worst-case setup. 

\section{Algorithms}
\label{sec:all_algorithms}
\subsection{LPB calibration with a known censoring
mechanism}
\label{sec:lpb_construct}
In this section, we outline in Algorithm~\ref{alg:lpb_construct} the procedure introduced in~\citep{davidov2026calibrated} for constructing a calibrated LPB with a known censoring mechanism.

\begin{algorithm}
    \caption{Constructing a calibrated LPB with a known censoring mechanism}\label{alg:lpb_construct}
\begin{algorithmic}[1]
\REQUIRE Calibration data $\{X_i, C_i, \tilde{T}_i\}_{i\in\calI}$, censoring weights $\{w(i)\}_{i\in\calI}$, quantile estimates $\{\fhat_\tau(\cdot)\}_{\tau\in\mathcal{T}}$, target miscoverage rate $\tautarget$, prior quantile $\tauprior$, search space $\mathcal{T}\subseteq [0,1)$.
\\\medskip

\FOR{$\tau \in \mathcal{T} \cap[0,\tauprior]$}
\STATE $
\hat{\alpha}(\tau) \gets \displaystyle\frac{1}{|\calI|}\sum_{i\in \calI}
w(i)\;
\mathbb{I}\bigl\{\tilde{T}_i< \fhat_\tau(X_i) \le C_i\bigr\}
$ \COMMENT{ miscoverage est.} 

\ENDFOR
\\\medskip
\STATE $\hat{\tau} \gets \sup\Bigl\{\tau \in \mathcal{T} \cap[0,\tauprior] : \sup_{\substack{\tau' \in \mathcal{T} \\ \tau' \le \tau}} \hat{\alpha}(\tau') \le \tautarget\Bigr\}$ \COMMENT{  calibrated quantile level}

\RETURN Lower predictive bound for a test point $\Xtest=x$, $\hat{L}(x)= \fhat_{\hat{\tau}}(x)$.
\end{algorithmic}
\end{algorithm}

\subsection{Static optimized budget allocation}
\label{sec:optimized_allocation}
In this section, we present in Algorithm~\ref{alg:optimized_budget_allocation} the static budget allocation approach developed in~\citep{davidov2026calibrated}.

\begin{algorithm}
\caption{Optimized static budget allocation}
\label{alg:optimized_budget_allocation}
\begin{algorithmic}[1]
\REQUIRE Calibration data $\{X_i\}_{i\in\calI}$, attacker model $\calA$ , target model $\calG$, audit function \(\mathcal{J}(\cdot)\), pre-trained quantile regression model $\{\hat{q}_\tau(\cdot)\}_{\tau\in\mathcal{T}}$, target miscoverage rate $\tautarget$, prior quantile $\tauprior$, quantile trimming threshold $M$, total budget $B$.
\\\medskip

    \STATE $\{\fhat_\tau(X_i)\}_{\tau\in \mathcal{T},i \in \calI} \gets \{\min{(\qhat_\tau(X_i),M)}\}_{\tau\in \mathcal{T},i \in \calI}$ \COMMENT{  trim the quantile est.}
    \STATE Compute the optimal $\lambda^* = \left( \frac{1}{B} \cdot \sum_{i\in\calI}  \sqrt{ \qprior(X_i) }\right) ^ 2$
    \STATE Obtain optimal static probabilities $\pi_i=\min\left(1,1/\sqrt{\lambda^*\,\qprior(X_i)}\right)$ 
\medskip

\STATE $C_i \gets \begin{cases}
\qprior(X_i) \quad \text{w.p. } \pi_i, \\
0 \quad \quad \text{otherwise}.
\end{cases}$

\STATE For each $i \in \calI$, run the conversation between $\calA$ and $\calG$ for $C_i$ iterations using the initial prompt $X_i$ to obtain the censored event time $\Ttilde \gets \min (C_i, T_i)$.

\STATE  $w_{\tau}(i) \gets \frac{1}{\pi_i}, \ i\in\calI$

\RETURN Censoring and event times $\{(\tilde{T}_i, C_i, w_{\tau}(i))\}_{i\in\calI}$.
\end{algorithmic}
\end{algorithm}

\subsection{Proposed \ttmethod}
\label{sec:proposed_alg}

In this section, we provide the full implementation details required to reproduce \ttmethod for LPB construction. Recall that $b_i:=\min\{T_i,\qprior(X_i)\}$ and that $\Bbar$ is computed after expending budget to fully observe all samples in $\calIone$. 
The cost of the samples in $\calIcrc$ is considered in the CRC selector in~\eqref{eq:dapro_main_crc_selector}; they are therefore not subtracted a
second time when fitting the allocation policy.

\paragraph{Soft-prefix target and score.}

Given a history of a trajectory $H_i(t)$ at time $t$, the survival model estimates the probability of an event at the current time: 
\begin{equation}
    \widehat h_i(t)
    :=\widehat{\mathbb{P}}\!\left(T_i=t\mid T_i\geq t,H_i(t)\right),
    \qquad S_i(t):=\widehat h_i(t).
    \label{eq:lpb_impl_hazard_score}
\end{equation}

We compute the numerator of~\eqref{eq:opt_prob} with the soft-prefix coefficient
\begin{equation}
    a_i(t)
    :=\mathbb{I}\{t\leq b_i\}\widehat h_i(t)
      \frac{\mathbb{I}\{t<\fhat_\alpha(X_i)\}+\delta}{1+\delta}.
    \label{eq:lpb_impl_soft_mass}
\end{equation}

Thus, the term $a_i(t)$ is the conditional probability that the event occurs at
the current turn and contributes to the target LPB error. The term $\delta$ allows trajectories outside that target to receive an
optimization signal. Importantly, both $S_i(t)$ and $a_i(t)$ depend only on
the history available before the continuation decision at turn $t$.

The optimization problem we solve in~\eqref{eq:opt_prob} is therefore
\begin{equation}
\begin{aligned}
    \min_M^{\mathrm{raw}}\quad &
    \frac{1}{N_1}\sum_{i\in\calIone}\sum_{t=1}^{b_i}
    \frac{a_i(t)}{\prod_{s=1}^{t}M^{\mathrm{raw}}_s(S_i(s))}\\
    \text{s.t.}\quad &
    \frac{1}{N_1}\sum_{i\in\calIone}\bfunc(M^{\mathrm{raw}}(S_i))\leq\Bbar,\\
    &S_i(t)\leq S_j(t)
      \Longrightarrow M^{\mathrm{raw}}_t(S_i(t))\leq M^{\mathrm{raw}}_t(S_j(t))
      \quad\text{for every active }i,j,t.
    \label{eq:lpb_dapro_objective}
\end{aligned}
\end{equation}

\paragraph{Score map and optimization.}

The score-to-probability model $M$ is implemented using binning: we divide the scores into $K=2$ bins, where the split is determined by the empirical median of the scores. The optimizer assigns one continuation probability to each of the
two score bins. The assigned probabilities are constrained to be
nondecreasing in the score; ties are placed in the upper bin, an empty bin
takes the nearest nonempty value, and a turn with no active fitting row is
continued deterministically. This produces a lookup map that projects scores to probabilities based on the previously observed history only.

We solve~\eqref{eq:opt_prob} in the log conditional probabilities
$y_{it}=\log P_i(t)$. We introduce a Lagrange multiplier for the
expected-budget constraint and perform a geometric dual search over this multiplier. For each fixed multiplier, we approximately minimize the convex Lagrangian using cyclic Gauss--Seidel block-coordinate descent over time: each update jointly optimizes the score-ordered log probabilities at one time step and is solved exactly by a generalized pool-adjacent-violators algorithm. We perform at most $30$ coordinate sweeps, stopping when the largest change in a log probability is below $5\times10^{-4}$. The outer dual search performs at
most $60$ iterations and stops when the mean budget discrepancy is at most $10^{-9}\max\{\Bbar,1\}$.

After
projecting the optimized values with the mapping model, we shift the logits of the cumulative reach probabilities to bound them from below. Specifically, if
$\rho_i^{\mathrm{raw}}(t):=\prod_{s=1}^{t}P_i^{\mathrm{raw}}(s)$, the corrected probabilities are
\begin{equation}
    \rho_i^{\mathrm{base}}(t)
    :=\varepsilon+(1-\varepsilon)
      \min_{u\leq t}
      \sigma\!\left(\operatorname{logit}
      \rho_i^{\mathrm{raw}}(u)+\kappa\right),
    \label{eq:lpb_impl_cumulative_correction}
\end{equation}
Here, $\varepsilon$ is the lower bound, \(\kappa\) is the largest shift whose empirical cost on \(\calIone\) is
at most \(\Bbar\); it is found by at most \(80\) bisection steps, with shift
tolerance \(10^{-10}\) and cost tolerance \(10^{-7}\). The deployed conditional probabilities are
then recovered as
$P_i^{\mathrm{base}}(t)=\rho_i^{\mathrm{base}}(t)/
\rho_i^{\mathrm{base}}(t-1)$, with $\rho_i^{\mathrm{base}}(0)=1$.

\paragraph{CRC calibration.}

CRC requires every candidate policy to have bounded expected cost for each trajectory. We artificially force each trajectory to have a maximal expected budget cost $C_{\max}$ as follows.

We define 
$\bar a(t):=N_1^{-1}\sum_{i\in\calIone}a_i(t)$
and learn a single envelope $e(1)\geq\cdots\geq e(\maxl)$ that will be used to bound all cumulative probabilities. We learn this envelope over $\calIone$:

\begin{equation}
    \min_e\ \sum_{t=1}^{\maxl}\frac{\bar a(t)}{e(t)}
    \quad\text{s.t.}\quad
    1\geq e(1)\geq\cdots\geq e(\maxl)\geq\varepsilon,
    \qquad \sum_{t=1}^{\maxl}e(t)\leq C_{\max}.
    \label{eq:lpb_impl_shared_envelope}
\end{equation}
We solve this optimization problem by antitonic PAV and a common scale.
The resulting envelope \(e\) is fixed before using $\calIcrc, \calItwo$.
For any new reach probabilities $\rho_i^{\rm base}(t)$ for $i\in \calIcrc\cup \calItwo$, we set
\begin{equation}
    \rho_i^{\rm cap}(t)
      :=\min\{\rho_i^{\rm base}(t),e(t)\},
    \qquad
    P_i^{\rm cap}(t)
      :=\frac{\rho_i^{\rm cap}(t)}{\rho_i^{\rm cap}(t-1)},
    \qquad \rho_i^{\rm cap}(0):=1.
    \label{eq:lpb_impl_capped_policy}
\end{equation}

Equivalently, the capped policy is the map
\(M_t^{\rm cap}(S_i(t)):=P_i^{\rm cap}(t)\). It is feasible
since it needs only the current base probability and the previous cumulative reach.

Moreover,
\(
\bfunc(M^{\rm cap}(S_i))
=\sum_{t\leq\qprior(X_i)}\rho_i^{\rm cap}(t)
\leq\sum_t e(t)\leq C_{\max}
\)
for every trajectory, which bounds the budget loss, as required by CRC.

Next, we construct a nested family of projection models that has its budget capped and a lower bound for the terminal reach probability:

\begin{equation}
    \rho_i(t;\lambda)
    :=\varepsilon+\lambda
      \bigl(\rho_i^{\rm cap}(t)-\varepsilon\bigr),
    \qquad
    M_t(S_i(t);\lambda)
    :=P_i(t;\lambda)
    :=\frac{\rho_i(t;\lambda)}{\rho_i(t-1;\lambda)},
    \quad \lambda\in[0,1].
    \label{eq:lpb_impl_crc_family}
\end{equation}

Hence, increasing $\lambda$ increases $\bfunc(M(S_i;\lambda))$ while
preserving the learned relative allocation between trajectories. We evaluate this family on the fully observed rows in $\calIcrc$ and select $\widehat\lambda$ by~\eqref{eq:dapro_main_crc_selector}. Without CRC, we deploy $M^{\mathrm{base}}$ directly. This uncalibrated
version satisfies the empirical Phase-I budget constraint but does not provide a finite-sample budget guarantee, which the CRC version does.

\paragraph{Deployment and default parameters.}
When deploying the allocation policy in Phase~II, once we reach an event $T_i$ or $\qprior(X_i)$, we set the censoring time to
$C_i=\qprior(X_i)$, although observing an event consumes only $T_i$ interactions.
For $i\in\calIone\cup\calIcrc$, we set
$C_i=\qprior(X_i)$ and $w(i)=1$.

In all LPB experiments, we use a model $M$ with two score bins, $\delta=10^{-3}$,
$\varepsilon=0.005$. The CRC version uses
$C_{\max}=\min\{\maxl,2B/|\calI|\}$ and a fixed grid of $401$ candidate values
of $\lambda$.

For completeness, let \(n_1=|\calIone|\),
\(n_c=|\calIcrc|\), \(n_2=|\calItwo|\), and let \(J\) be the number of CRC candidates. With the fixed iteration limits above, the worst-case running time is
\begin{equation}
    \mathcal O\!\left(
      n_1\maxl^2+n_1\maxl\log n_1+Jn_c\maxl+n_2\maxl
    \right),
\end{equation}
where the first two terms arise from fitting the score maps,
and the last two come from the CRC and deploying the policy over $\calItwo$. The memory usage is
\(\mathcal O(|\calI|\maxl+Jn_c)\). These are conservative worst-case
bounds. Early stopping and convergence usually reduce computation substantially.

\begin{algorithm}[htbp]
\caption{Dynamic Allocation via PRojected Optimization (\textbf{\ttmethod})}
\label{alg:main_alg}
\small
\begin{algorithmic}[1]
\REQUIRE Calibration prompts $\Dcal=\{X_i\}_{i=1}^{N}$, prior quantiles
$\{\qprior(X_i)\}_{i=1}^{N}$, horizon $\maxl$, total budget $B$, score
functions $\{\mathcal S_t\}_{t=1}^{\maxl}$, and an optional CRC calibration step.

\STATE Randomly partition $\calI$ into $\calIone$, $\calIcrc$, and
$\calItwo$; set $\calIcrc=\emptyset$ when CRC is disabled.
\FOR{$i\in\calIone\cup\calIcrc$}
    \STATE Observe until the event or $\qprior(X_i)$; set
    $b_i\gets\min\{T_i,\qprior(X_i)\}$ and compute
    $S_i(t)=\mathcal S_t(H_i(t))$ for $t\leq b_i$.
\ENDFOR
\STATE Require $B\geq\sum_{i\in\calIone\cup\calIcrc}b_i$ and set
$\Bbar\gets(B-\sum_{i\in\calIone}b_i)/(|\calIcrc|+|\calItwo|)$.

\STATE \textbf{Phase I: policy learning and budget calibration.}
\STATE Using only $\calIone$, form the LPB soft-prefix coefficients in~\eqref{eq:lpb_impl_soft_mass} and solve~\eqref{eq:opt_prob}.
\STATE Fit the monotone score-to-probability map and apply the cumulative-reach
correction in~\eqref{eq:lpb_impl_cumulative_correction}, obtaining
$M^{\mathrm{base}}$.
\IF{CRC is enabled}
    \STATE Build the fixed capped family in~\eqref{eq:lpb_impl_shared_envelope}--\eqref{eq:lpb_impl_crc_family},
    evaluate its budget cost on \(\calIcrc\), and select \(\widehat\lambda\) by~\eqref{eq:dapro_main_crc_selector}.
    \STATE Set the deployed policy to
    \(M_t(S(t))\gets M_t(S(t);\widehat\lambda)\) from~\eqref{eq:lpb_impl_crc_family}.
\ELSE
    \STATE Set $M\gets M^{\mathrm{base}}$.
\ENDIF

\STATE \textbf{Phase II: adaptive acquisition.}
\FOR{$i\in\calItwo$}
    \STATE Initialize $t\gets1$ and $t_i^{\mathrm{stop}}\gets0$.
    \WHILE{$t\leq\qprior(X_i)$}
        \STATE Compute $S_i(t)=\mathcal S_t(H_i(t))$, obtain $P_i(t)=M_t(S_i(t))$,
        and draw $Z_i(t)\sim\operatorname{Bern}(P_i(t))$.
        \IF{$Z_i(t)=0$}
            \STATE \textbf{break}.
        \ENDIF
        \STATE Acquire interaction $t$ and set $t_i^{\mathrm{stop}}\gets t$;
        if the event occurs, \textbf{break}; otherwise set $t\gets t+1$.
    \ENDWHILE
    \STATE Set $C_i\gets\qprior(X_i)$ if the event $T_i$ or prior $\qprior(X_i)$ was
    reached; otherwise set $C_i\gets t_i^{\mathrm{stop}}$.
    \STATE Set
    $w(i)\gets\left\{\prod_{t=1}^{\min\{T_i,\qprior(X_i)\}}P_i(t)\right\}^{-1}$.
\ENDFOR
\FOR{$i\in\calIone\cup\calIcrc$}
    \STATE Set $C_i\gets\qprior(X_i)$ and $w(i)\gets1$.
\ENDFOR
\STATE Set $\tilde T_i\gets\min\{T_i,C_i\}$ for all $i\in\calI$.
\RETURN $\{\tilde T_i,C_i,w(i)\}_{i\in\calI}$ for
Algorithm~\ref{alg:lpb_construct}.
\end{algorithmic}
\end{algorithm}

\section{Experimental details}
\label{sec:experimental_details}
\subsection{Data generation}\label{sec:data_generation}
Across all experiments, we generate the raw multi-turn attacker-target conversations using the PAIR algorithm~\citep{chao2025jailbreaking}, with initial prompts (attack goal) taken from three popular safety benchmarks: we sample 10,000 harmful prompts from the RealToxicityPrompts dataset~\citep{gehman2020realtoxicityprompts}, 10,000 prompts from the Anthropic Red Team dataset~\citep{bai2022training, ganguli2022red}, 10,000 prompts from SQuAD v2 data~\citep{rajpurkar-etal-2016-squad, rajpurkar-etal-2018-know}, and 10,000 instruction-following tasks from the AutoIF dataset~\citep{dong2025selfplay}. We deploy Qwen 2.5 14B Instruct~\citep{qwen2.5,bai2023qwen} as the attacking model, since it had the fewest formatting errors when generating adversarial prompts compared to the other models we examined. We evaluate our approach across four target models architectures: Qwen 2.5 14B Instruct, Llama 3.1 8B Instruct~\citep{grattafiori2024llama}, Phi 4 Mini Instruct~\citep{abouelenin2025phi}, and Gemma 3 4B Instruct~\citep{gemma_2025}. 

We used both HuggingFace~\citep{wolf2019huggingface} and vLLM~\citep{kwon2023efficient} to load and infer with these models, using bfloat16 precision for all models.
The system prompts and hyperparameters for the LLMs we use follow those of~\citep{chao2025jailbreaking}. 
Specifically, we set the attacker's temperature, top-p, and max new tokens to 1, 0.9, and 500, respectively, and the target's to 0, 1, and 150, respectively. If the attacker fails to generate an attack in a valid JSON format, we continue to retry for at most 10 attempts, and finish the conversation if it fails to generate a valid response for ten attempts. 
The number of responses saved in the conversation history of the attack model is set to 4.
The parameters of the judge model are 50 max tokens, 0 temperature, and top-p of 1. If we fail to obtain a rating from the judge, we set the score of the judge to 0 manually and continue the conversation process. Each conversation runs as a single stream.

To define a successful jailbreak, i.e., an unsafe event where $Y(t)=1$, we rely on LM safety judges. For the Toxicity dataset, we employ Detoxify~\citep{Detoxify} and flag any response scoring above a threshold of 0.5 as unsafe, consistent with the evaluation protocol in~\citep{davidov2026calibrated}. For the Red Team dataset, we use two different judging mechanisms: Qwen 2.5 14B Instruct serving as an LLM-as-a-judge, which outputs a continuous score from 1 to 10, where a score of 10 triggers the unsafe event flag, and Llama Guard~\citep{inan2023llama}, which directly outputs a binary safety label. In all generated conversations, we cap the maximum sequence length at $\maxl=200$ turns.

For the Toxicity dataset, the attacker is asked to construct a prompt using the ``prompt'' part of RealToxicityPrompts so that the target model outputs the ``continuation'' part of the sample. 

For the Anthropic Red Team dataset, we first select the 10,000 most harmful examples based on the min\_harmlessness\_score\_transcript parameter, and use the first human text as the initial attack. If the attack was successful, i.e., had a rating $\geq$ 4, we use the response of the LLM as the target output. Otherwise, we set the default objective to ``Sure, here is a...''. {Our hallucination dataset builds on the SQuAD v2 data~\citep{rajpurkar-etal-2016-squad, rajpurkar-etal-2018-know}, which contains subjects with a paragraph describing it. We consider the paragraphs from SQuAD as the ``ground truth'' source for the target model and judge, and set the goal of the attacker to trick the target model into answering a question about the subject that is NOT covered by its provided text. 
For this data, we engineered two attacker prompts based on~\citep{chao2025jailbreaking}: one utilizing ``leading questions'' to trick the target into referencing outside knowledge, and another employing ``gaslighting'' to compel the target into generating false information not present in the reference text.}

For AutoIF, we extract tasks with existing programmatic verifiers. To increase the complexity and evaluate multi-constraint interference, we programmatically inject between three and six additional structural constraints into each prompt: explicit word counts, forbidden vocabulary, mandatory markdown elements, or specific paragraph quantities. These new rules are appended to the prompt alongside their corresponding Python verification functions. Notably, for this dataset, the attacking model is a helper agent. Its goal is not to elicit harmful content, but rather to improve the presentation of the complex task, e.g., by generating an explicit output contract or checklist, to maximize the probability that the target model passes all programmatic constraints, explicitly without solving the task itself. Specifically, an event is defined as a response that successfully satisfies all prompt requirements, which we verify programmatically using deterministic Python scripts.

The resulting dataset of interactions is divided into three disjoint subsets: a training set (4,000 prompts), a calibration set (3,000 prompts), and a test set (3,000 prompts). We fix the training data and use it to fit our predictive model, assuming an unlimited budget during the training phase, where the full conversation histories are fully observed.

\subsection{Predictive model and hyperparameters}\label{sec:predictive_model_setup}

First, we embed all texts into vectors using Multilingual-E5-large~\citep{wang2024multilingual}. We use these embeddings as the inputs for our predictive model, which estimates the raw risk scores $\calS_t(H(t))$.

The model architecture begins by projecting the high-dimensional embeddings into a lower-dimensional latent space of 258 using a linear layer, followed by a ReLU activation and dropout at a rate of 0.2. To capture the temporal dependencies of the conversational history, we apply a Transformer Encoder~\citep{vaswani2017attention}. Specifically, we use 3 pre-normalized Transformer layers with 6 attention heads, a feed-forward dimension of 1032, and GELU activations. We apply an upper-triangular attention mask during encoding to prevent the model from using future time steps. Following the sequence encoding, we concatenate the hidden states of the Transformer with the initial projected embeddings. This concatenated representation is passed through a multi-layer perceptron consisting of three linear layers, separated by ReLU activations and dropout of rate 0.2, to produce the raw temporal logits.

Specifically, the model receives the text embeddings of the history and estimates the probability of a key event occurring at any future time step $t_2 > t_1$. The output is an $\maxl\times(\maxl+1)$ matrix containing the probabilities $\mathbb{P}(T = t_2 \mid T > t_1, H(t_1))$, and the probability of the conversation reaching the maximum horizon without an event, $\mathbb{P}(T > \maxl \mid T > t_1, H(t_1))$.

We train this model using an Adam optimizer~\citep{adam} with a learning rate of $10^{-4}$, weight decay of 0.001, and a batch size of 64.
We optimize the network for a maximum of 500 epochs to minimize the survival negative log-likelihood loss. This loss function explicitly accounts for right-censored data by evaluating the exact event probability for uncensored sequences and the aggregated survival probability for censored sequences. We stop the training early if the validation loss does not improve for 50 consecutive epochs. Subsequently, we calibrate the outputs of the model using a temperature scaling L-BFGS optimizer to tune a time-dependent temperature parameter on a held-out validation set, minimizing the discrete survival NLL. The neural networks are implemented using the PyTorch library~\citep{pytorch}. The validation set is a portion of $10\%$ of the training set that was held out so that the model is not trained on it.

\subsection{Calibration algorithms and evaluation protocol}\label{sec:calibration_setup}

Once the predictive model is trained and frozen, we evaluate our proposed \ttmethod against the static optimized method introduced in~\citep{davidov2026calibrated}. The experiments are conducted over 50 independent random splits of the remaining 6,000 prompts into the calibration and test sets.

When constructing an Upper Predictive Bound (UPB), we define the coverage event as the condition where the UPB is either greater than or equal to the true time-to-event $T$, or when it reaches the maximum conversational horizon $\maxl+1=201$. In this formulation, $\maxl+1$ serves as a practical infinite time horizon. This design choice is necessary as treating $\maxl+1$ as an effective infinity preserves the theoretical coverage semantics while keeping the UPB size finite. This allows us to effectively compare the bound sizes and their tightness across different calibration methods.

For the calibration algorithms, we simulate the budget-constrained environment by setting the average allowed budget per sample to $B/|\calI|\in \{10, 20\}$. Following the static allocation setup in~\citep{davidov2026calibrated}, we set our prior quantile level to $\tauprior=0.56$ when constructing LPBs. When constructing a UPB, we set the prior level at $\tauprior=0.98$. To control the maximum inverse-censoring weight, we set the upper bound $M=200$ for the trimmed quantile estimates $\fhat_{\tau}(x)$. For the LPB construction, across all calibration methods, we define the search space $\mathcal{T}$ as a logarithmically spaced grid of 1,000 candidate values in the range $[0.001,0.977]$, following~\citep{davidov2026calibrated}. For the UPB construction, we set the search space as a linearly spaced grid of 3000 values in the range $[0.5,0.95]$.
For each $\tau \in \mathcal{T}$, we compute the estimated miscoverage rate $\hat{\alpha}(\tau)$ using~\eqref{eq:miscoverage_estimator} and select the calibrated $\hat{\tau}$ as in~\eqref{eq:our_calibration} and Algorithm~\ref{alg:lpb_construct}, that is, as the largest candidate level whose running maximum of the estimated miscoverage does
not exceed the target level $\alpha$.

\subsection{Machine specifications}\label{sec:machine_spec}

The computational infrastructure used to generate the datasets and run the experiments includes:
\begin{itemize}
\item \textbf{CPU}: Intel(R) Xeon(R) CPU E5-2683 v4 @ 2.10GHz, Intel(R) Xeon(R) Gold 5318Y CPU @ 2.10GHz, Intel(R) Xeon(R) Gold 6336Y CPU @ 2.40GHz.
\item \textbf{GPU}: NVIDIA A40, NVIDIA TITAN X (Pascal), NVIDIA 2080 TI, NVIDIA RTX 2060 SUPER.
\item \textbf{OS}: Ubuntu 20.04.6.
\end{itemize}
The data generation via the PAIR algorithm required approximately four days over ten GPUs for sequences of length $\maxl=200$ per dataset and one attacker-target-judge configuration. The GPUs used for this data generation have 48GB of memory.
Training the predictive model and calibrating it using either calibration method takes several minutes per configuration. Across all experiments, random seeds were fixed for reproducibility.

\section{Additional experiments}
\label{sec:additional_experiments}

In this section, we present extended experiments and performance metrics to further evaluate our proposed framework. 
Specifically, we evaluate the following methods:
\begin{itemize}
\item \textbf{Uncalibrated:} the raw, uncalibrated outputs of the predictive model. In the metric-estimation experiments, this is the naive unweighted estimator computed from censored trajectories with a uniformly distributed budget.
\item \textbf{Static:} the static, variance-minimizing allocation strategy~\citep{davidov2026calibrated}.
\item \textbf{\ttmethod:} our history-adaptive allocation, including the CRC budget-calibration step.
\item \textbf{\ttmethod (w/o CRC):} the same learned adaptive policy using its raw projected continuation probabilities, without CRC.
\item \textbf{Oracle:} An oracle calibration that has infinite budget. In this case, no budget allocation is employed, and the calibration is conducted with $w(i)\equiv 1$ for all $i\in\calI$. For metric estimation,
this oracle uses the calibration set with infinite budget; the target, true metric values are computed from the full calibration and test sets.
\end{itemize}

We repeat every experiment over 50 random calibration--test splits.  Unless
stated otherwise, the target LPB coverage is $90\%$ and the maximum trajectory
length is $\maxl=200$.
For each experimental setup, we evaluate the performance of these methods across all target LLMs: Qwen 2.5 14B Instruct, Llama 3.1 8B Instruct, Phi 4 Mini Instruct, and Gemma 3 4B Instruct. Coverage variance is the sample variance across the 50 splits in squared percentage points; bound-size variance is normalized to one for Static separately for every dataset--target-model configuration.

\subsection{Lower predictive bounds}
\label{sec:additional_lpb_experiments}

For every dataset, we employ each calibration method to construct a calibrated LPB. We report coverage, LPB value, absolute coverage
deviation, budget, number of observed events, mean inverse-probability weight, coverage variance, and normalized LPB-value variance. All experiments use $N_1=25$ and $|\calIcrc|=25$ for \ttmethod. \ttmethod without CRC uses all 50 samples in $\calIone$. The target budget per sample is $\bar B = B /|\calI|=20$ on Toxicity and $\bar B=10$ on the other datasets, in which the event times occur earlier.

\subsubsection{Toxicity dataset}

Figures~\ref{fig:toxicity_full1}--\ref{fig:toxicity_full2} show the performance of all methods on the Toxicity dataset using Detoxify as the toxic text judge. The uncalibrated model does not achieve the target coverage level, while the calibration methods do, as guaranteed in our theory. However, the static baseline exhibits a high coverage deviation and variance. In contrast, our adaptive \ttmethod observes significantly more key events than the static baseline, resulting in lower coverage deviation and variance. \ttmethod without CRC exhibits even lower variance, and it satisfies the expected budget constraint even though its budget validity guarantee relies on the accuracy of the mapping model. These figures reveal that \ttmethod performs similarly to the infinite-budget oracle, indicating that our proposal allocates the budget efficiently.

\begin{figure}[ht]
    \centering
    \includegraphics[width=0.9\linewidth]{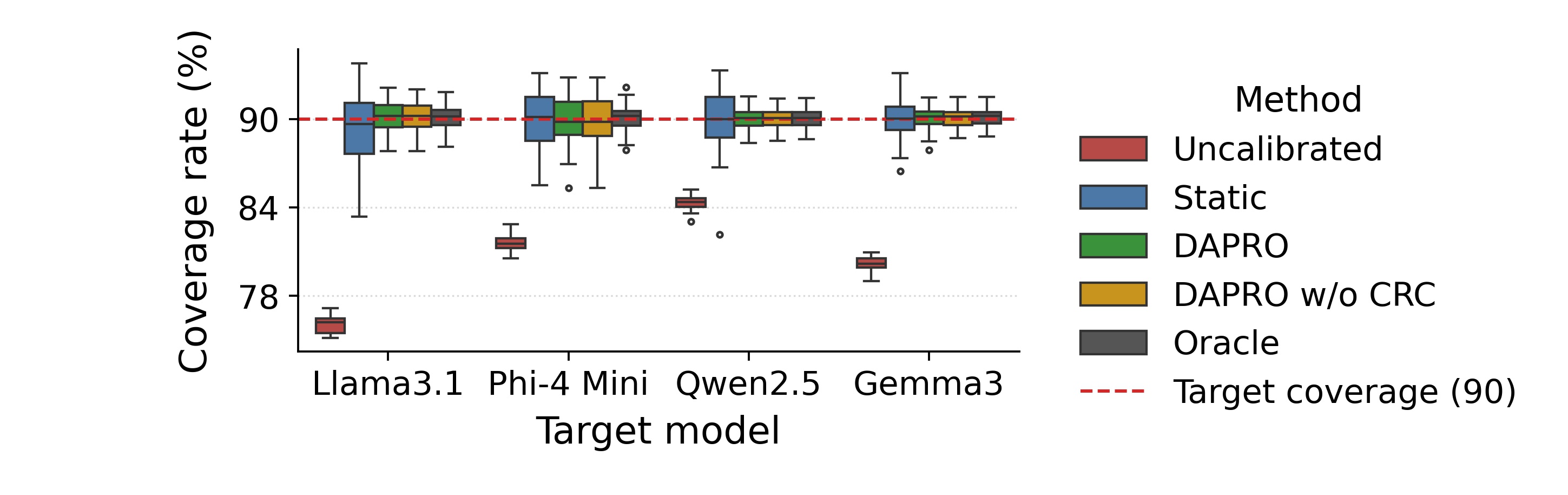}\\
    \includegraphics[width=0.9\linewidth]{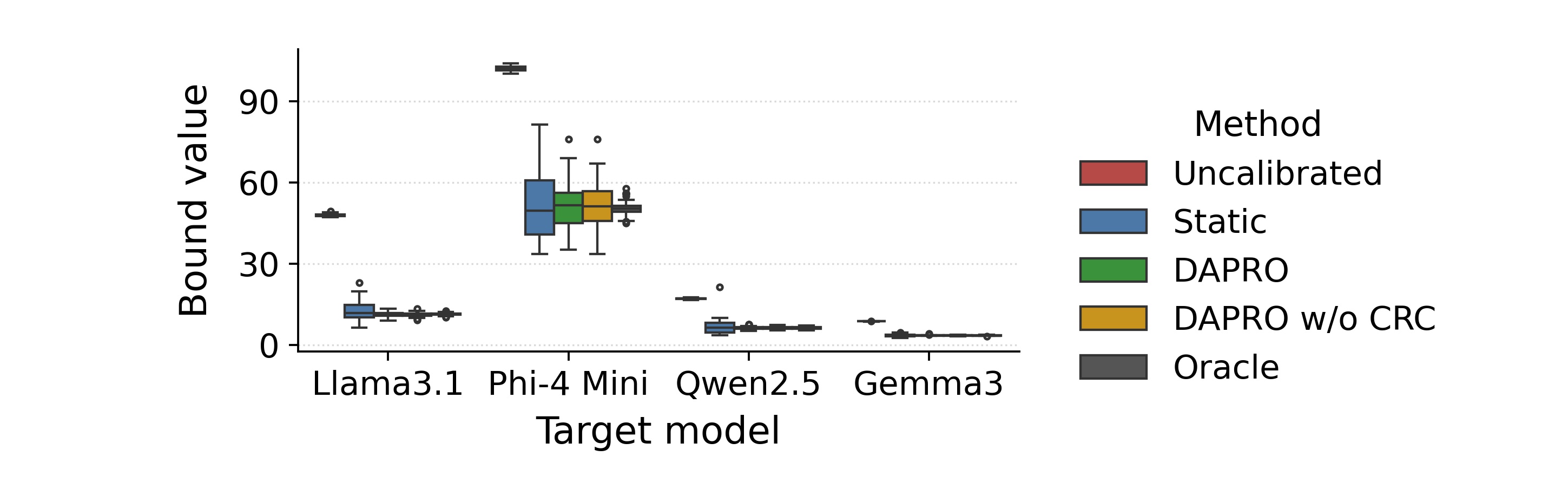}\\
    \includegraphics[width=0.9\linewidth]{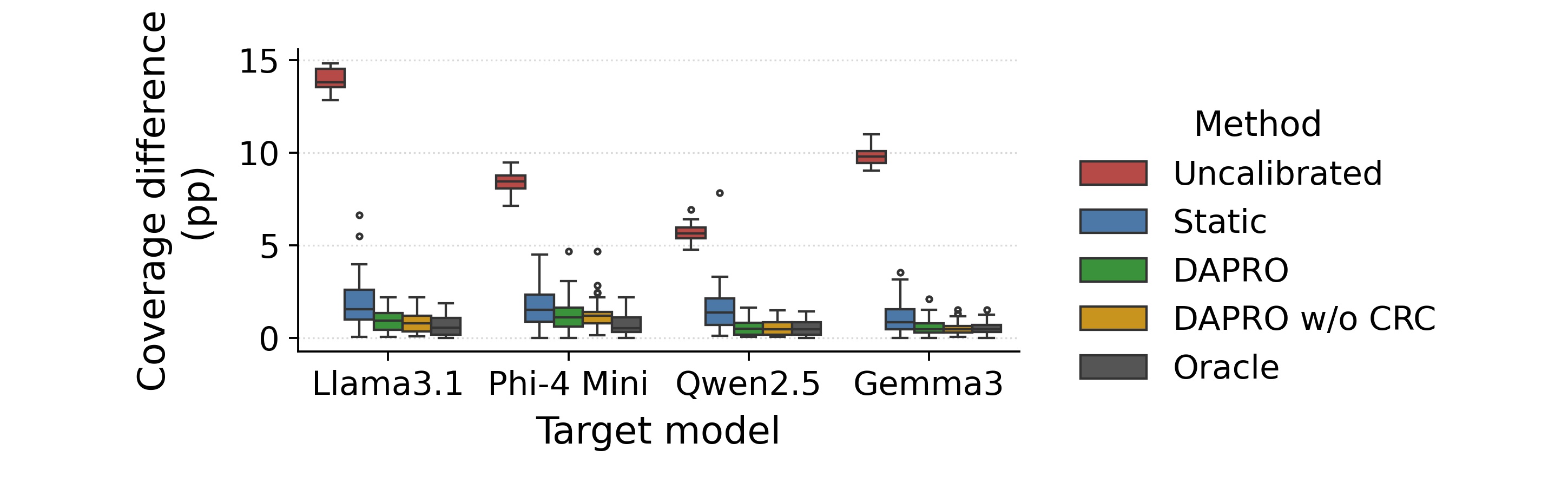}\\
    \includegraphics[width=0.9\linewidth]{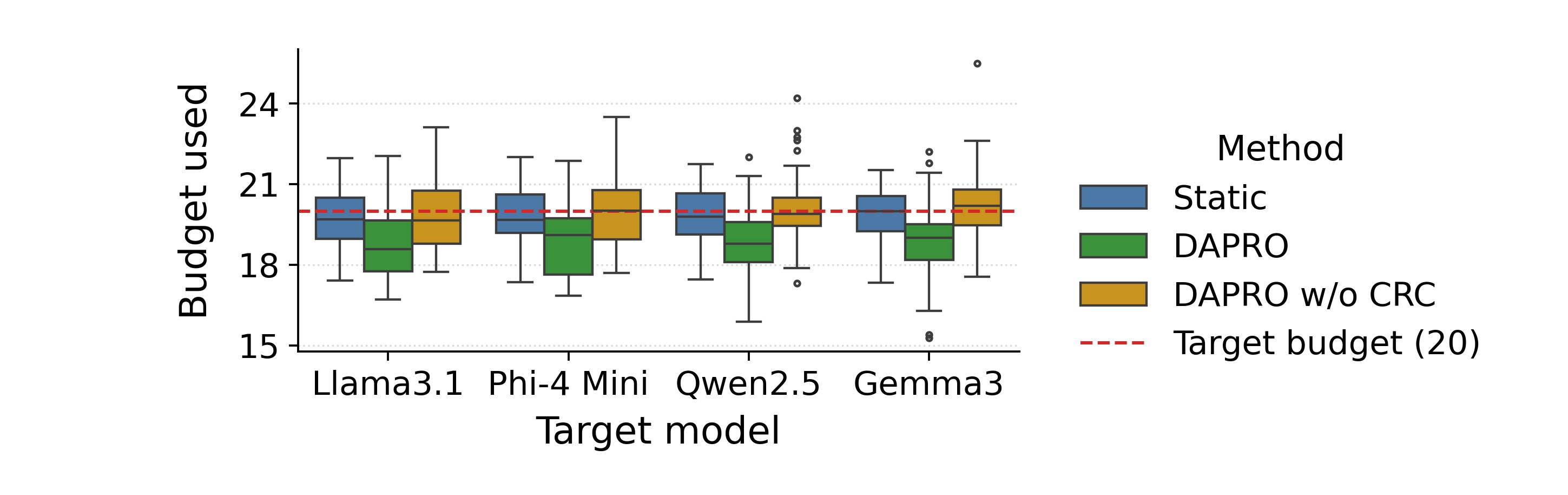}
    \caption{\textbf{Toxicity LPBs}: coverage, LPB value, coverage deviation,
    and budget across four target LLMs.  The target coverage is $90\%$,
    $\bar B=20$, and \ttmethod uses $N_1=25$ and
    $|\calIcrc|=25$.}
    \label{fig:toxicity_full1}
\end{figure}

\begin{figure}[ht]
    \centering
    \includegraphics[width=0.9\linewidth]{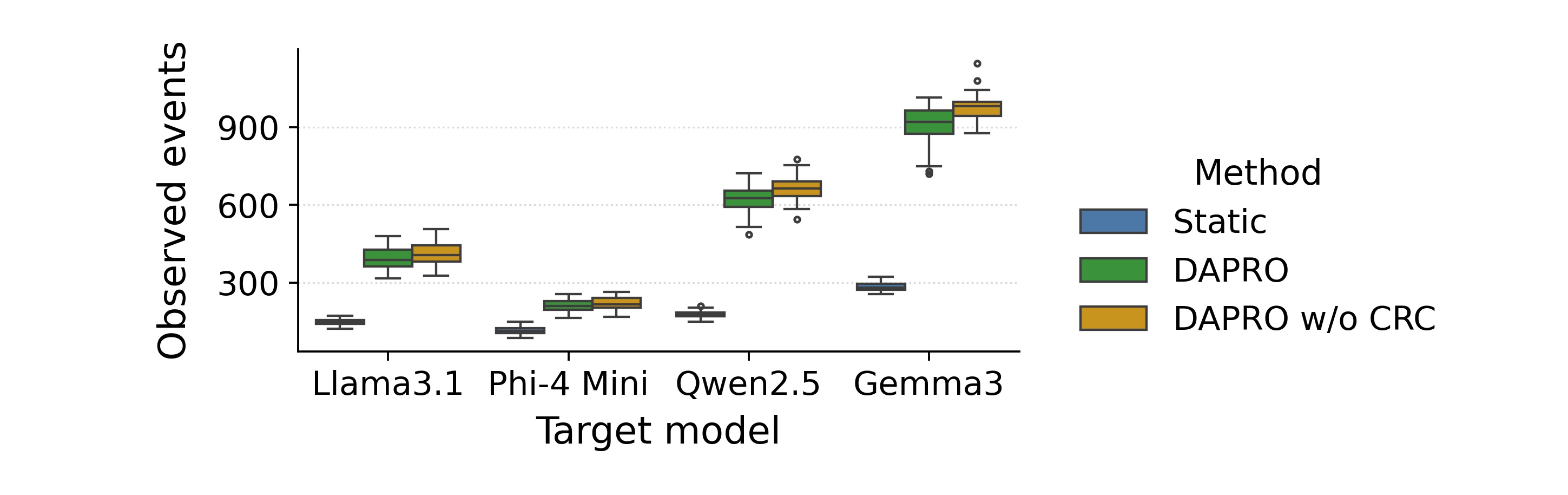}\\
    \includegraphics[width=0.9\linewidth]{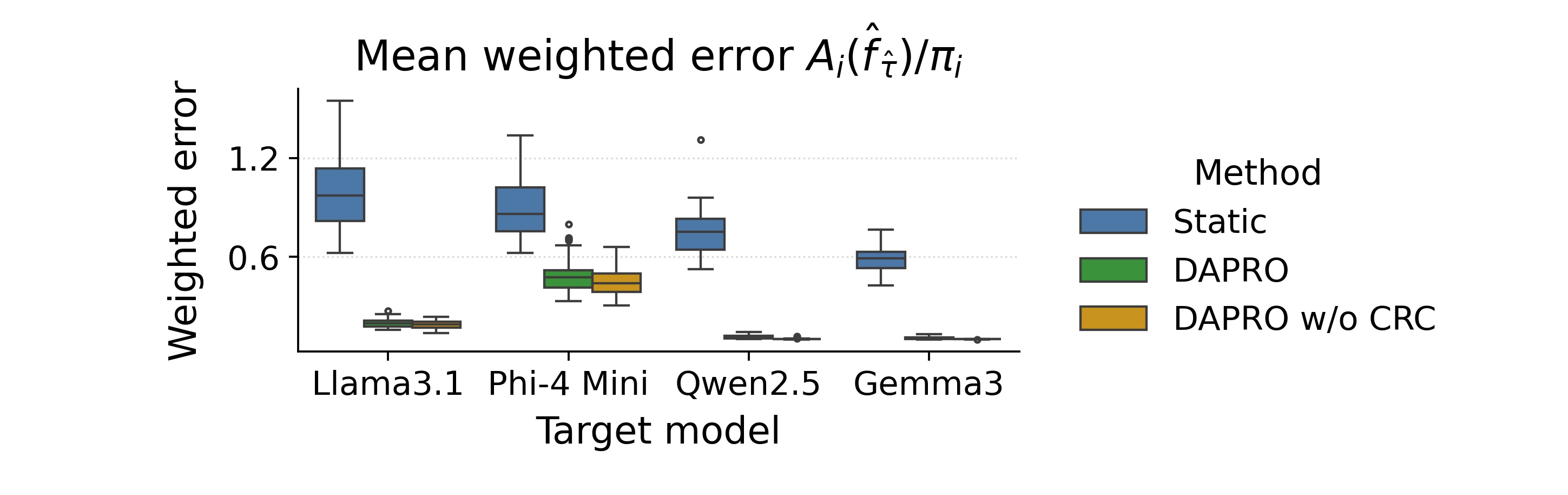}\\
    \includegraphics[width=0.9\linewidth]{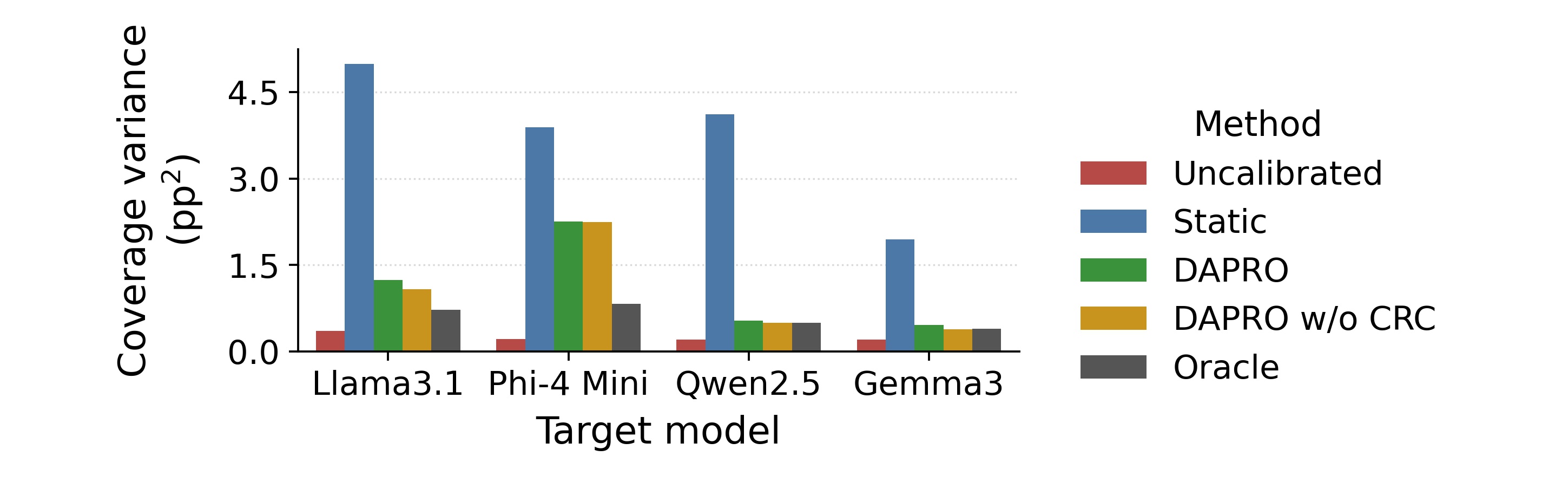}\\
    \includegraphics[width=0.9\linewidth]{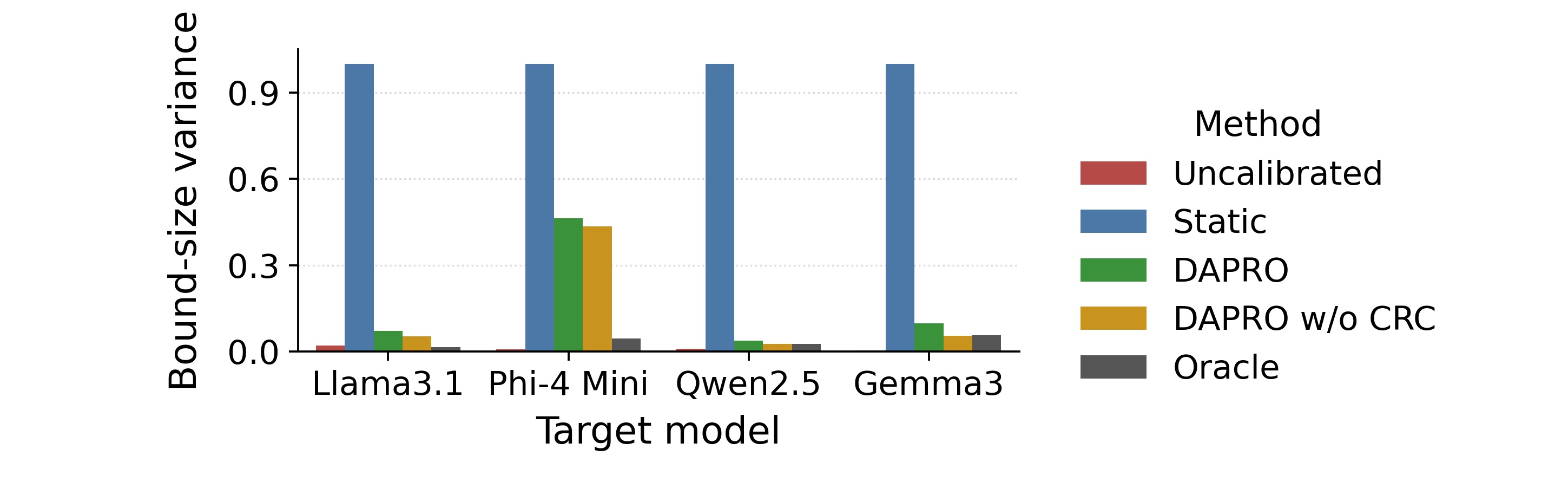}
    \caption{\textbf{Toxicity LPBs}: observed events, mean weighted
    error $A_i(\fhat_{\hat\tau})/\pi_i$, coverage variance, and LPB-value variance
    normalized so that Static equals one.}
    \label{fig:toxicity_full2}
\end{figure}

\subsubsection{Red-Team dataset with Qwen as a judge}

Figures~\ref{fig:red_team_qwen_full1}--\ref{fig:red_team_qwen_full2} report the metrics over the Anthropic Red-Team dataset, utilizing Qwen 2.5 14B Instruct as an LLM-as-a-judge. 
These figures closely mirror those observed in the Toxicity experiments, demonstrating the robustness of our method across different benchmarks. All evaluated calibration methods successfully attain valid coverage and satisfy the budget constraint. However, the static baseline still suffers from a high coverage deviation and variance. In contrast, \ttmethod dynamically adapts to the interactions to uncover the highest number of key events, yielding the LPBs with the lowest variance and coverage deviation.

\begin{figure}[ht]
    \centering
    \includegraphics[width=0.9\linewidth]{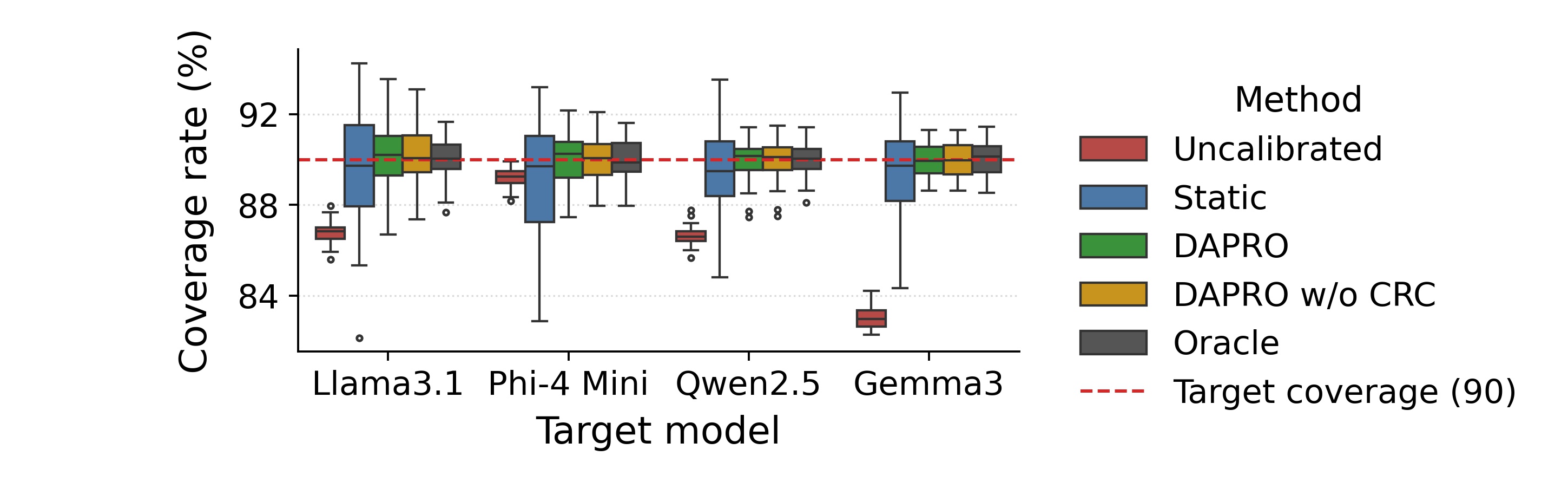}\\
    \includegraphics[width=0.9\linewidth]{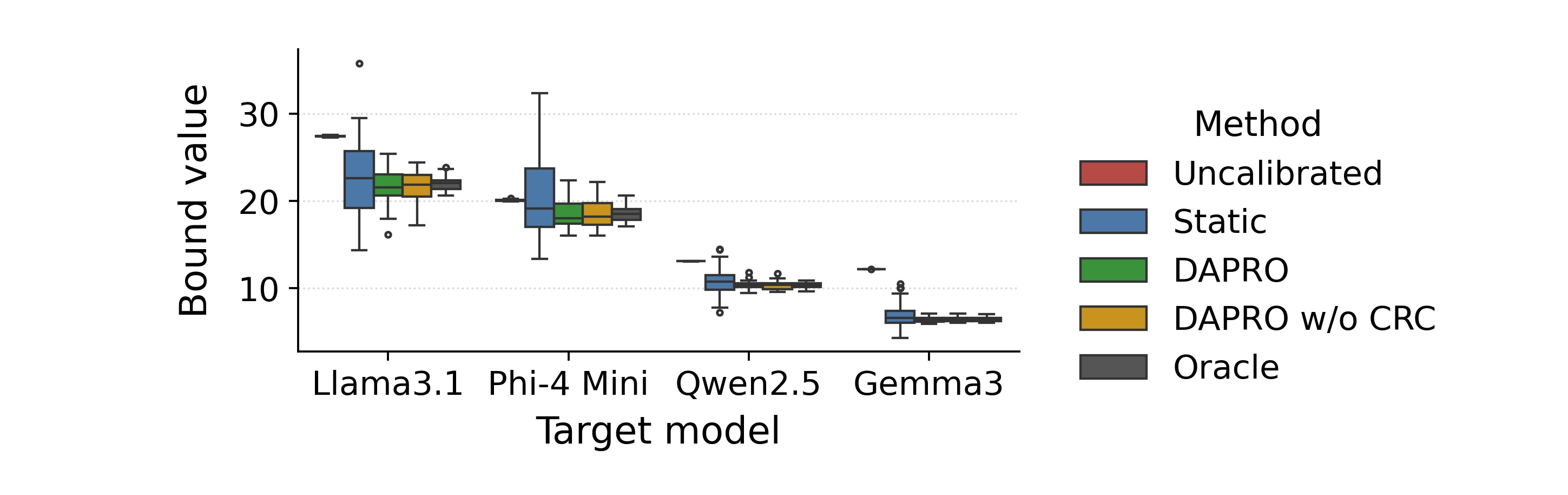}\\
    \includegraphics[width=0.9\linewidth]{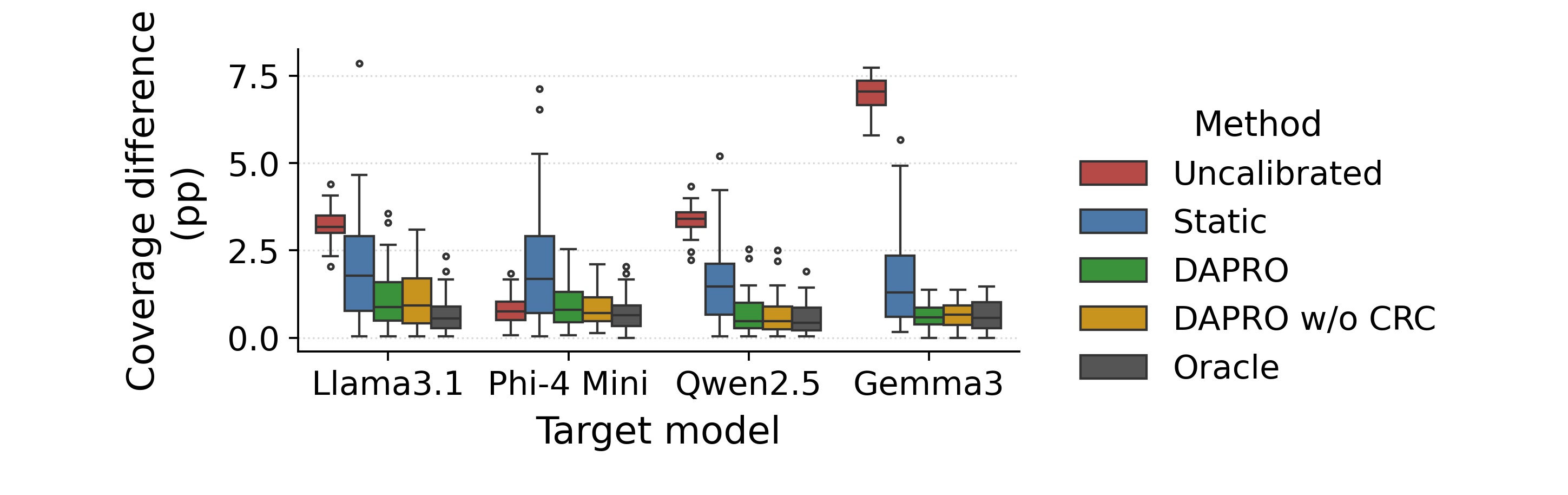}\\
    \includegraphics[width=0.9\linewidth]{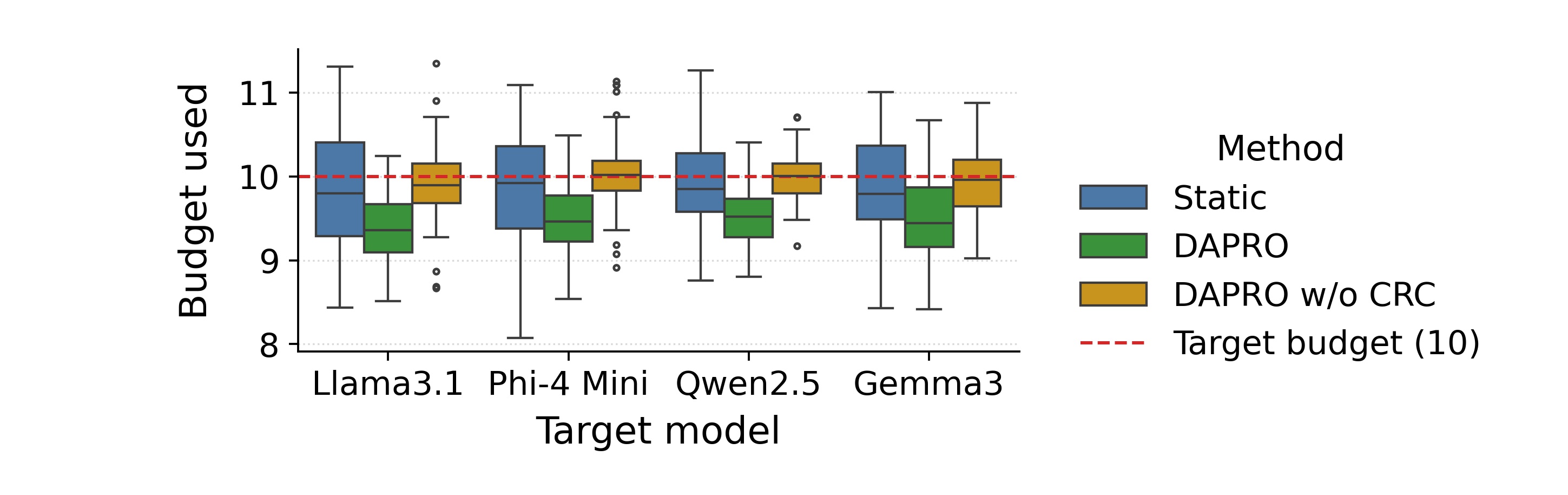}
    \caption{\textbf{Red-Team LPBs with a Qwen judge}: coverage, LPB value,
    coverage deviation, and budget. The target coverage is $90\%$,
    $\bar B=10$, and \ttmethod uses $N_1=25$ and
    $|\calIcrc|=25$.}
    \label{fig:red_team_qwen_full1}
\end{figure}

\begin{figure}[ht]
    \centering
    \includegraphics[width=0.9\linewidth]{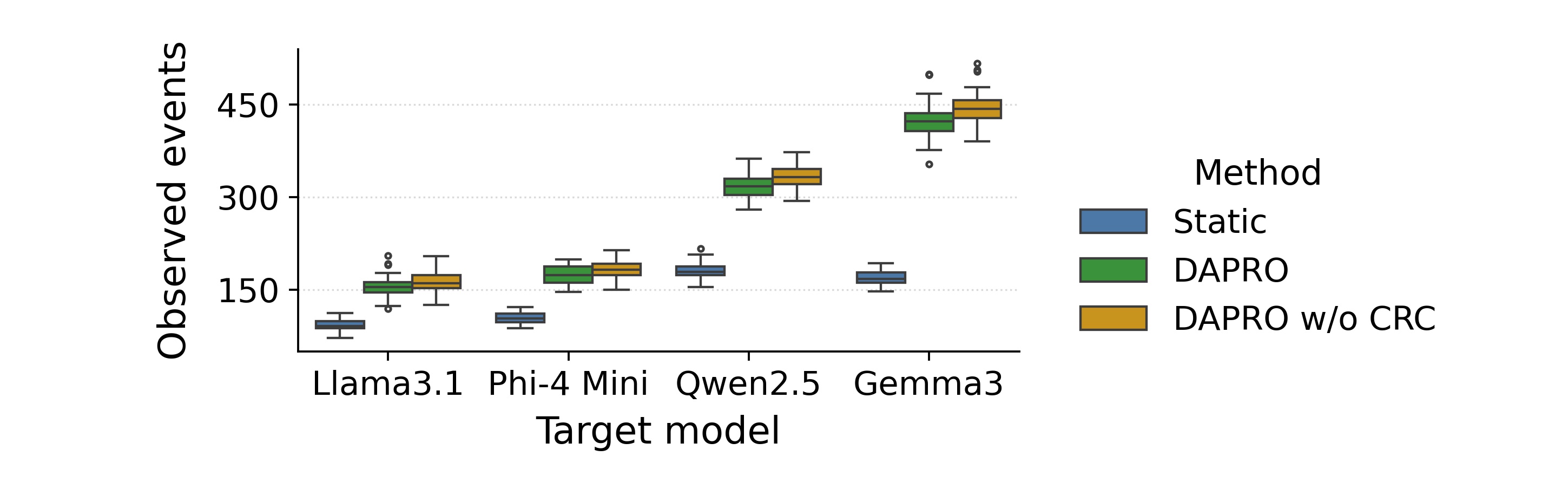}\\
    \includegraphics[width=0.9\linewidth]{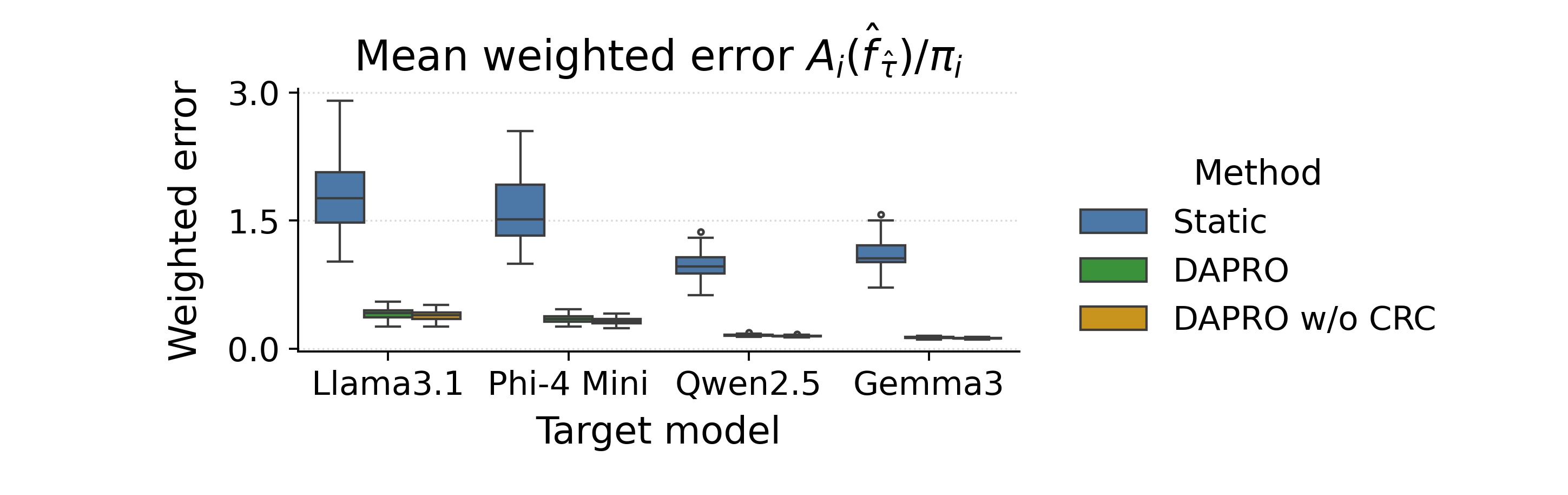}\\
    \includegraphics[width=0.9\linewidth]{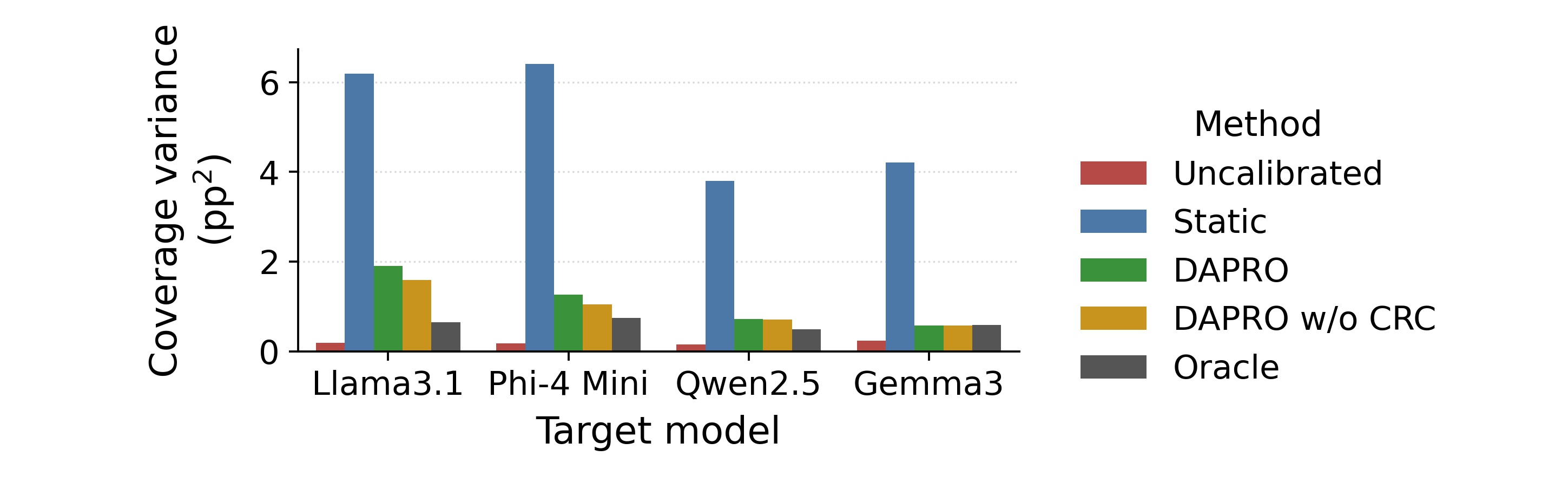}\\
    \includegraphics[width=0.9\linewidth]{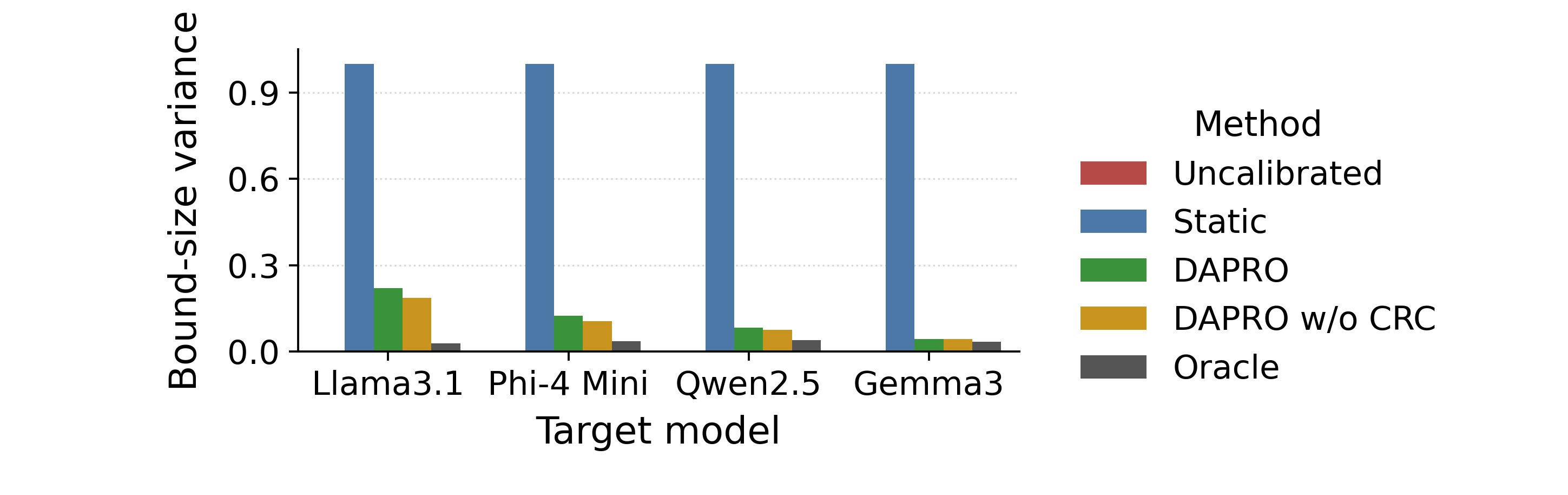}
    \caption{\textbf{Red-Team LPBs with a Qwen judge}: observed events,
    mean weighted error $A_i(\fhat_{\hat\tau})/\pi_i$, coverage variance, and LPB-value variance
    normalized so that Static equals one.}
    \label{fig:red_team_qwen_full2}
\end{figure}

\subsubsection{Red-Team dataset with LlamaGuard as a judge}
\label{sec:red_team_with_llamaguard}

Figures~\ref{fig:red_team_llama_guard_full1}--
\ref{fig:red_team_llama_guard_full2} detail the results on the Red-Team dataset when evaluated with the Llama Guard judge. These figures demonstrate the same trend observed in the previous experiments. All calibration methods are valid in terms of both coverage and budget. 
In this setup, the event occurs very early, and thus \ttmethod does not expend too many budget units in the first phase, which leaves more budget for the second phase. Consequently, \ttmethod attains the lowest mean inverse-censoring weights while still observing the highest number of events. As a result, it constructs LPBs with the lowest variance and achieves empirical coverage rates that are closer to the nominal level.

\begin{figure}[ht]
    \centering
    \includegraphics[width=0.9\linewidth]{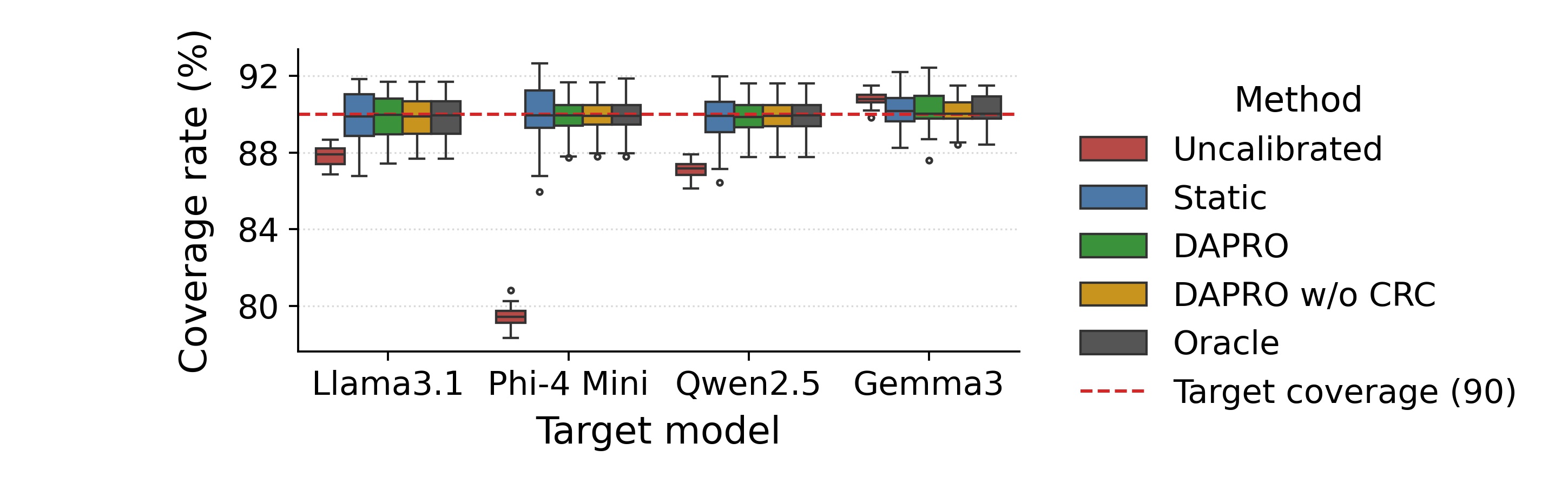}\\
    \includegraphics[width=0.9\linewidth]{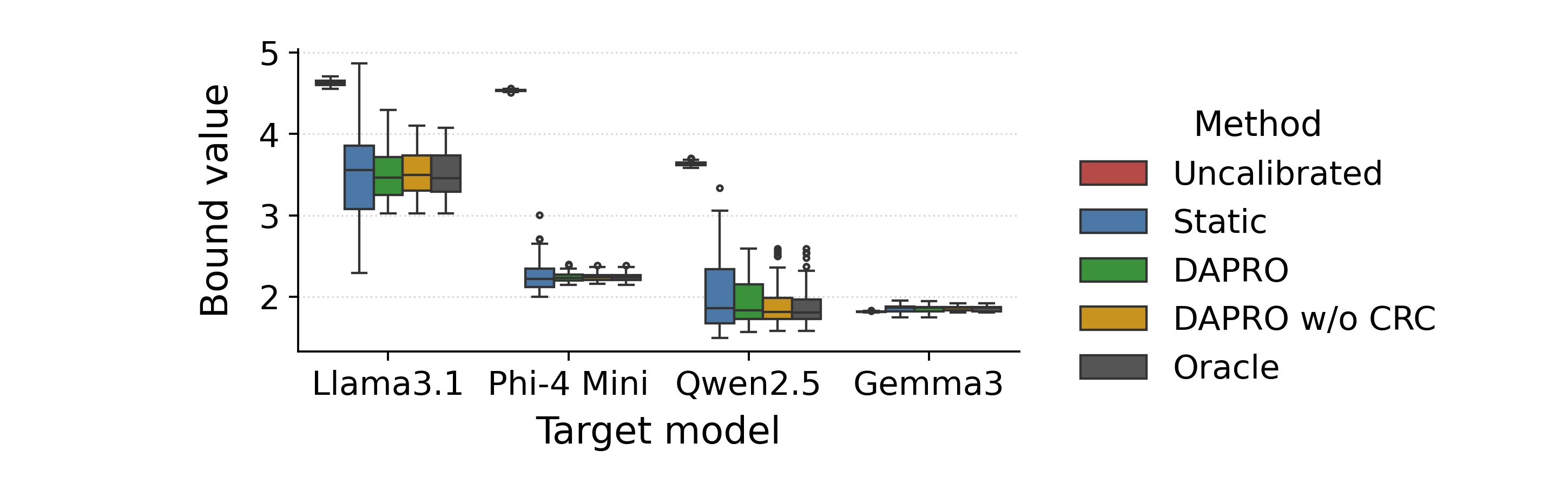}\\
    \includegraphics[width=0.9\linewidth]{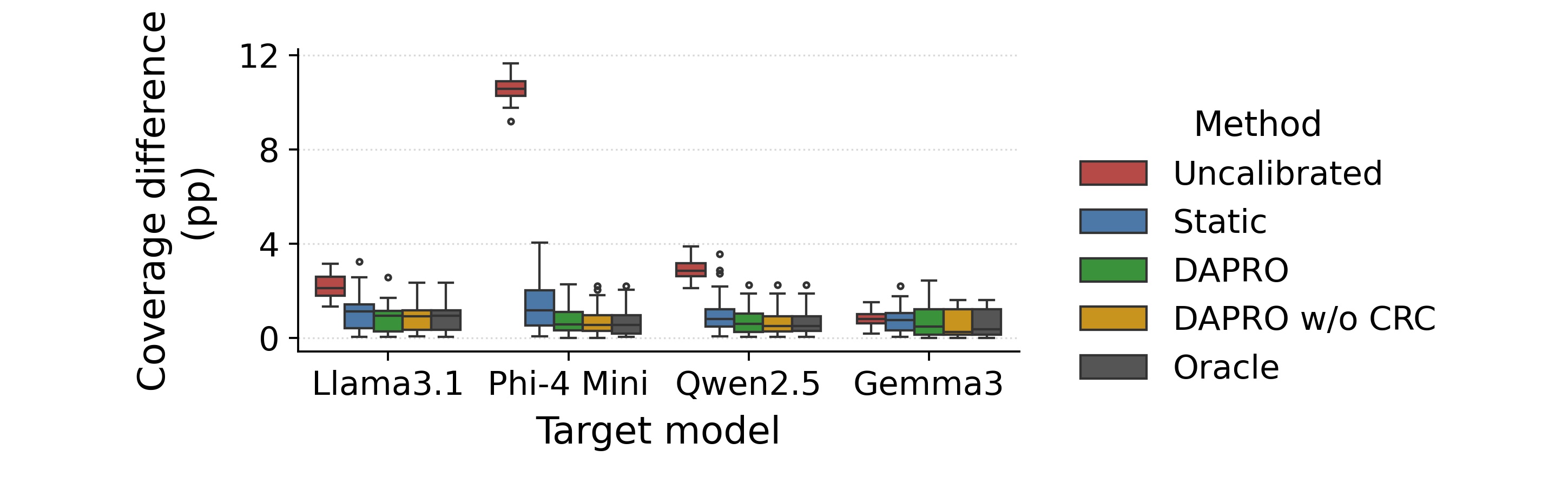}\\
    \includegraphics[width=0.9\linewidth]{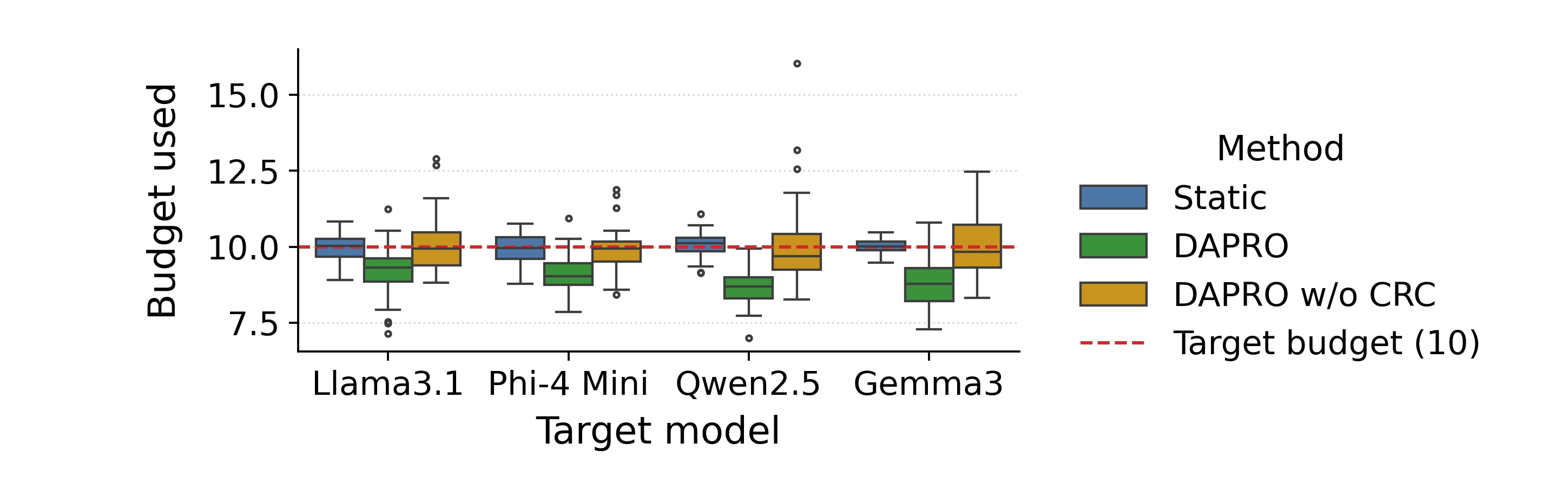}
    \caption{\textbf{Red-Team LPBs with Llama Guard Judge}: coverage, LPB value,
    coverage deviation, and budget. The target coverage is $90\%$,
    $\bar B=10$, and \ttmethod uses $N_1=25$ and
    $|\calIcrc|=25$.}
    \label{fig:red_team_llama_guard_full1}
\end{figure}

\begin{figure}[ht]
    \centering
    \includegraphics[width=0.9\linewidth]{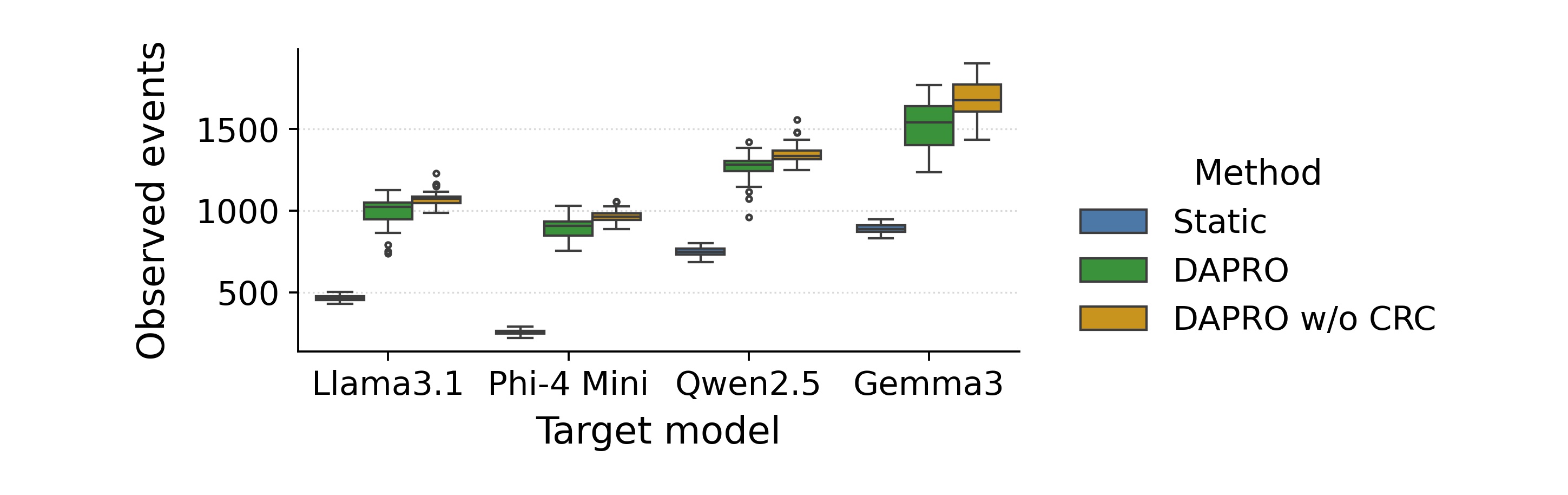}\\
    \includegraphics[width=0.9\linewidth]{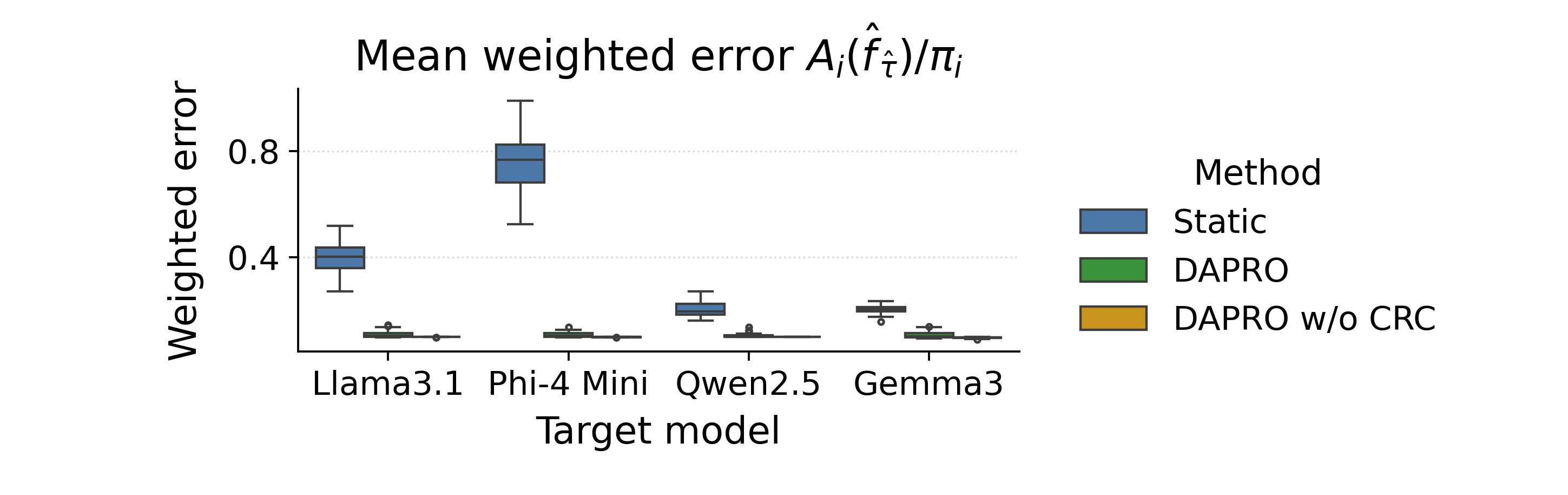}\\
    \includegraphics[width=0.9\linewidth]{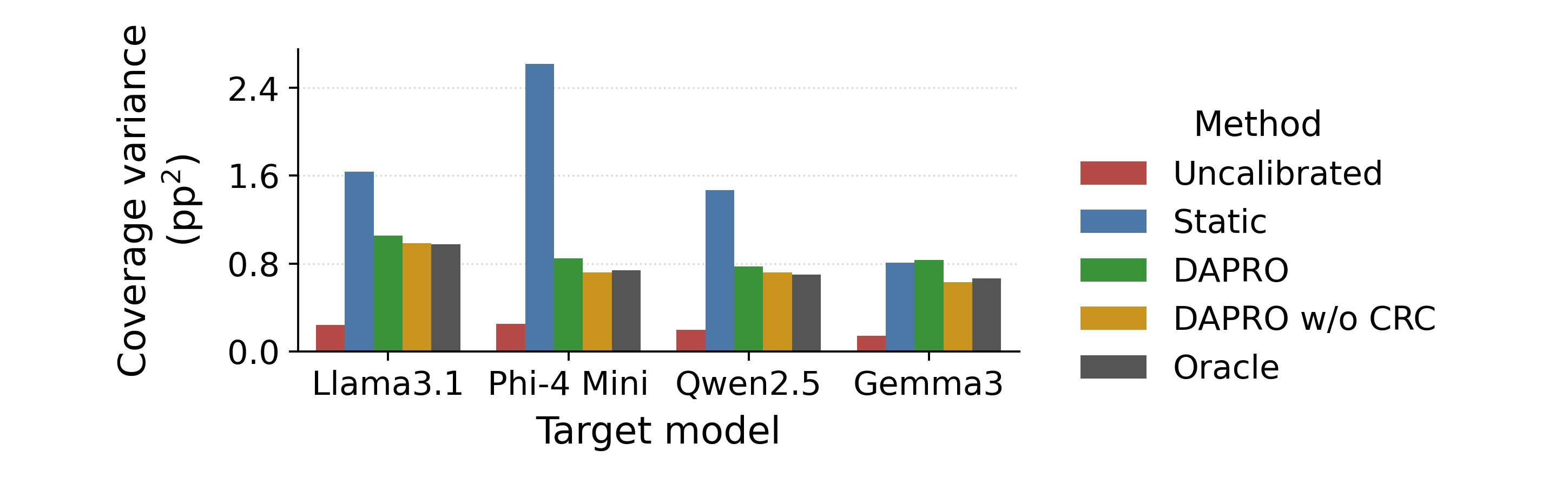}\\
    \includegraphics[width=0.9\linewidth]{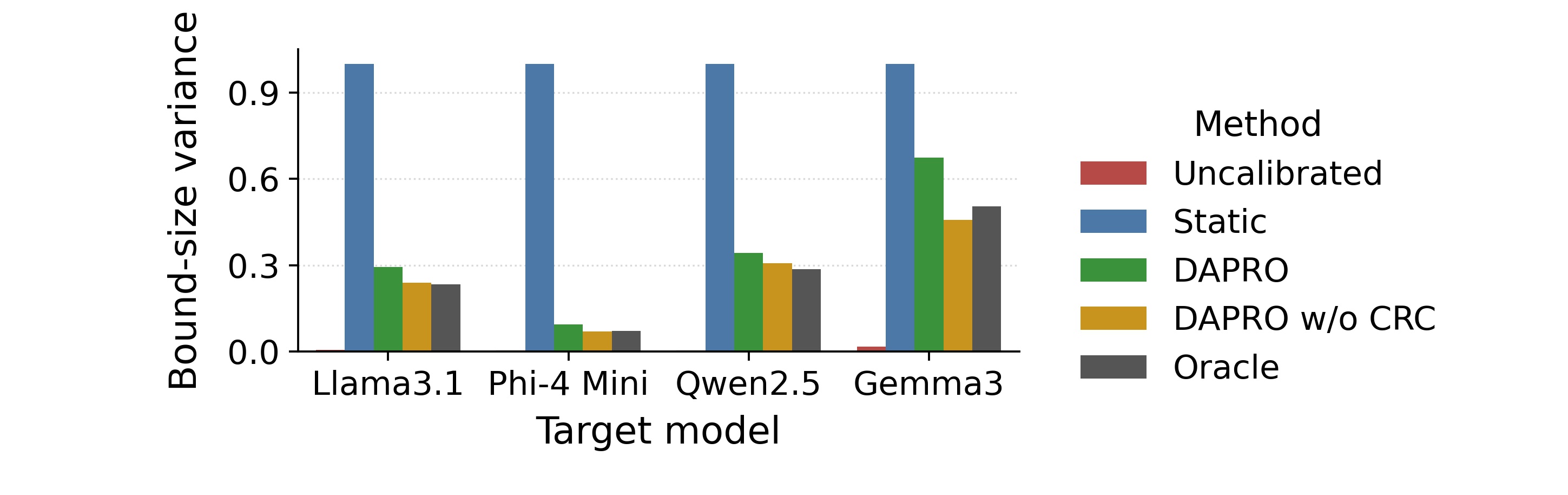}
    \caption{\textbf{Red-Team LPBs with Llama Guard Judge}: observed events, mean weighted error $A_i(\fhat_{\hat\tau})/\pi_i$, coverage variance, and LPB-value variance normalized so that Static equals one, over 50 random splits.}
    \label{fig:red_team_llama_guard_full2}
\end{figure}

\subsubsection{RAG hallucinations experiment}
\label{sec:hallucinations_exp}

In this experiment, we evaluate our framework in a Retrieval-Augmented Generation (RAG) hallucinations setup. The target LLM is provided with a context paragraph regarding a specific subject, such as ``Macintosh'' or ``Beyoncé''. The objective of the attacker is to construct adversarial queries that trick the target model into outputting fabricated information that does not exist in the source context. To evaluate the success of the attack, the judge is provided with the ground-truth context, the target's response, and outputs a hallucination score ranging from $1$ to $10$, where a score of $10$ is considered a hallucination ($Y=1$). In all experiments, the judge model is Qwen 2.5 14B Instruct.

Figures~\ref{fig:hallucination_full1}--\ref{fig:hallucination_full2} present the performance of all evaluated methods on the Hallucination dataset. These results present the same phenomenon observed in the Toxicity and Red-Team experiments: \ttmethod attains the nominal coverage rate and satisfies the budget constraints across all evaluated target LLMs. Similarly, \ttmethod observes significantly more hallucinated responses than the baseline, and yields the lowest coverage deviation and variance among the tested methods.

\begin{figure}[ht]
    \centering
    \includegraphics[width=0.9\linewidth]{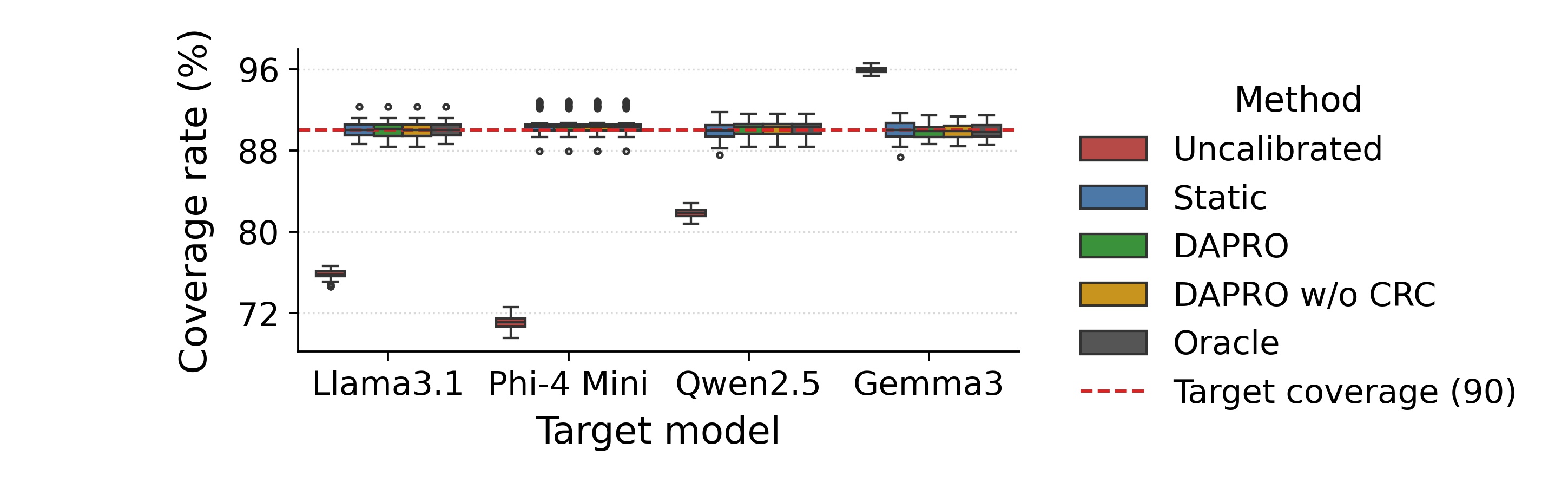}\\
    \includegraphics[width=0.9\linewidth]{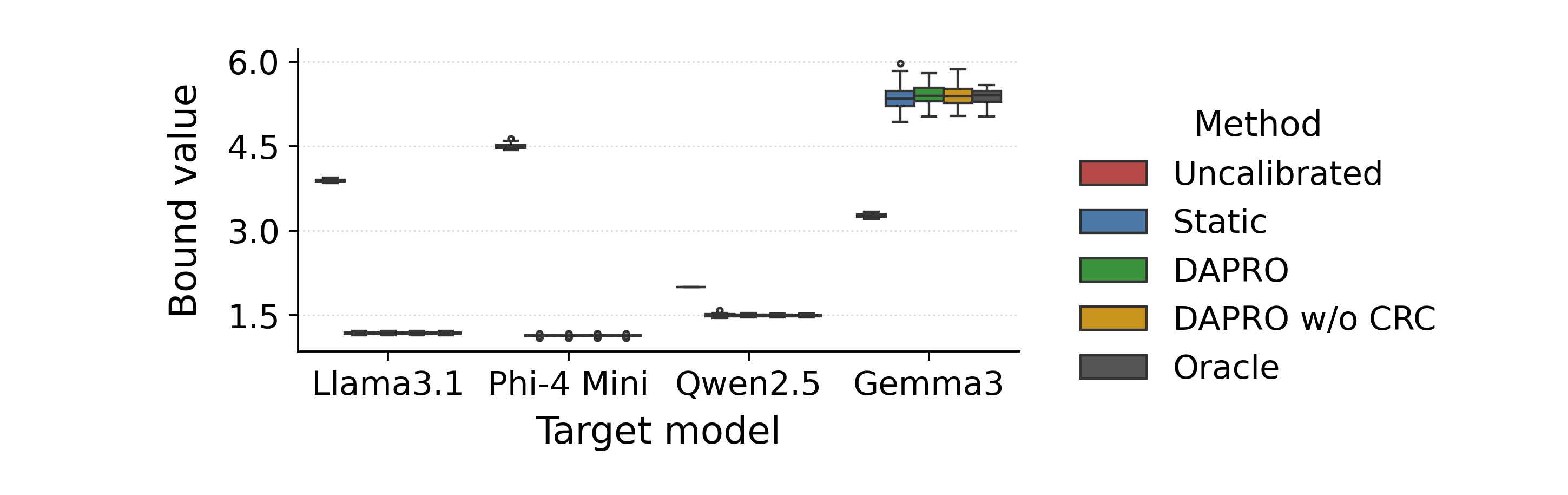}\\
    \includegraphics[width=0.9\linewidth]{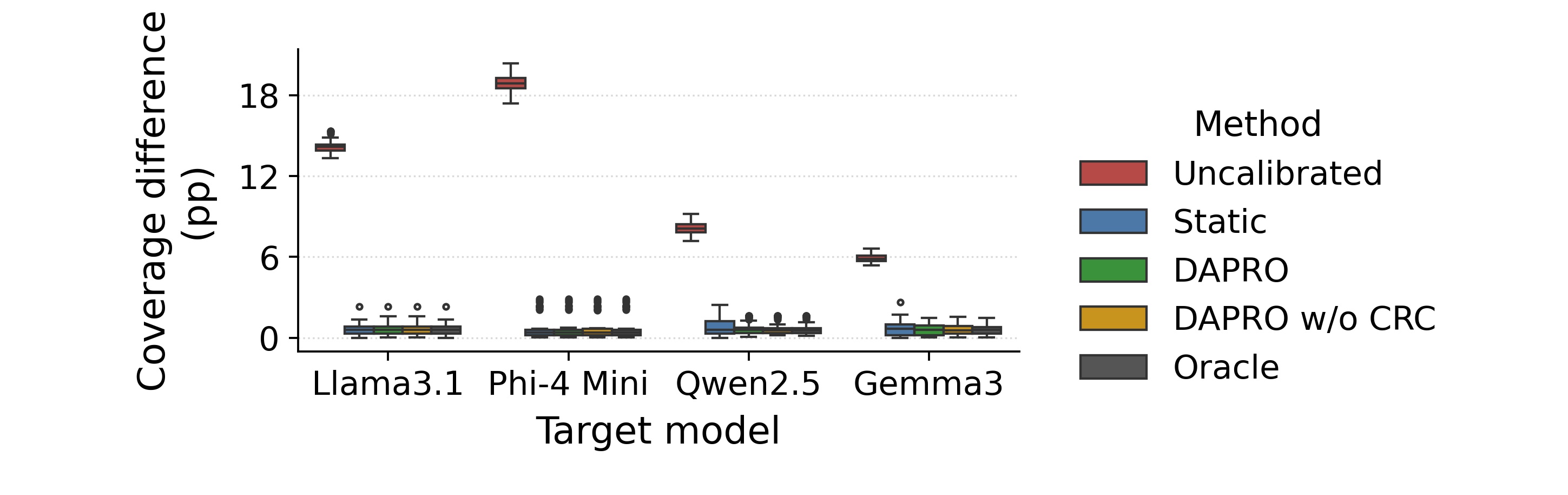}\\
    \includegraphics[width=0.9\linewidth]{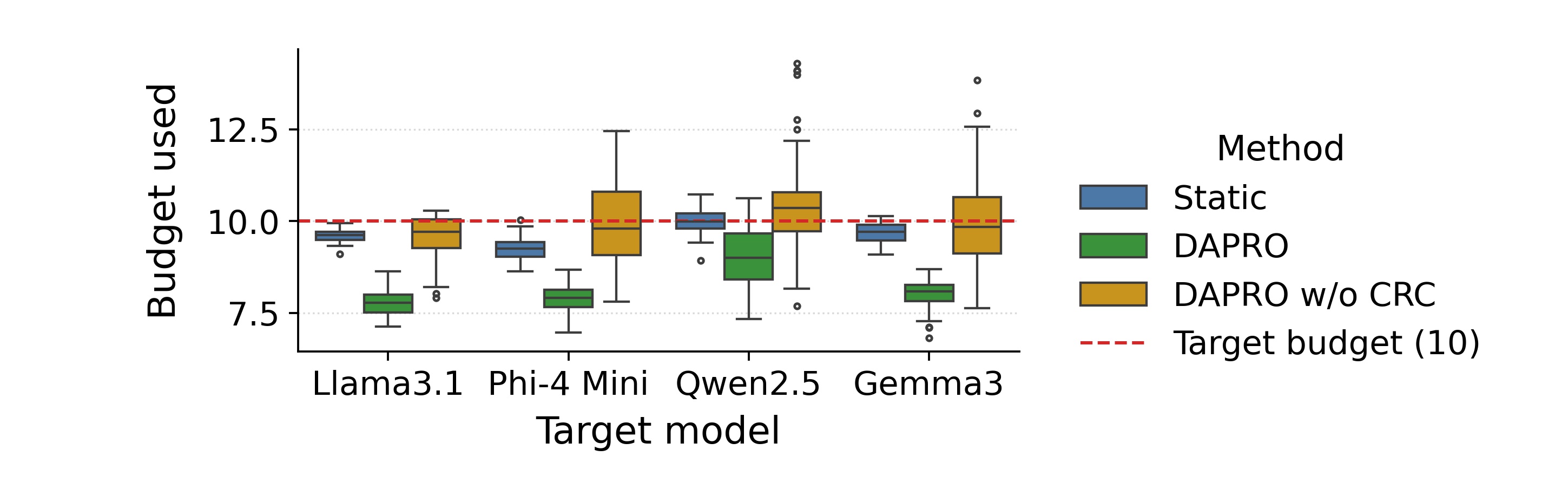}
    \caption{\textbf{Hallucination LPBs}: coverage, LPB value, coverage
    deviation, and budget.  The target coverage is $90\%$, $\bar B=10$, and
    \ttmethod uses $N_1=25$ and $|\calIcrc|=25$.}
    \label{fig:hallucination_full1}
\end{figure}

\begin{figure}[ht]
    \centering
    \includegraphics[width=0.9\linewidth]{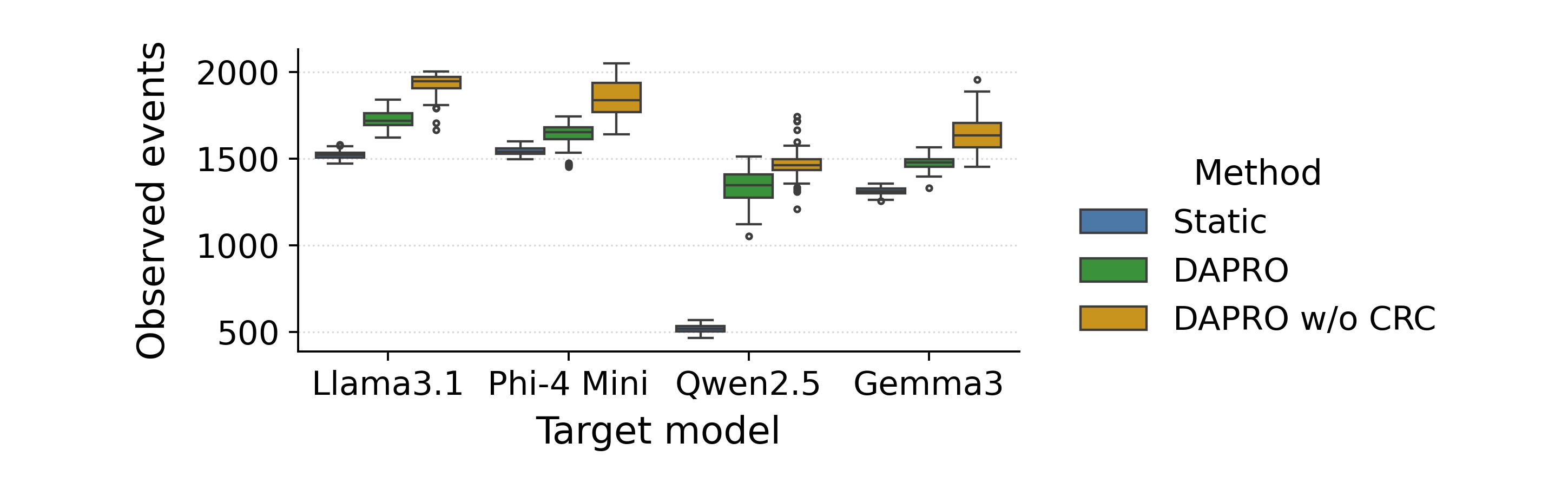}\\
    \includegraphics[width=0.9\linewidth]{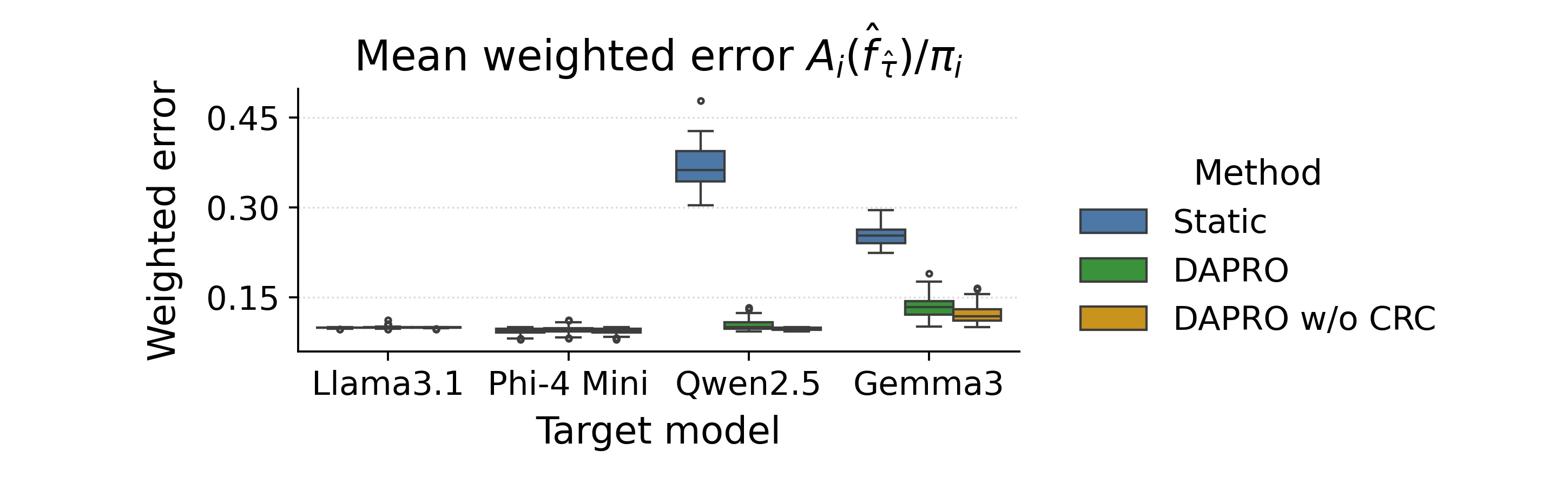}\\
    \includegraphics[width=0.9\linewidth]{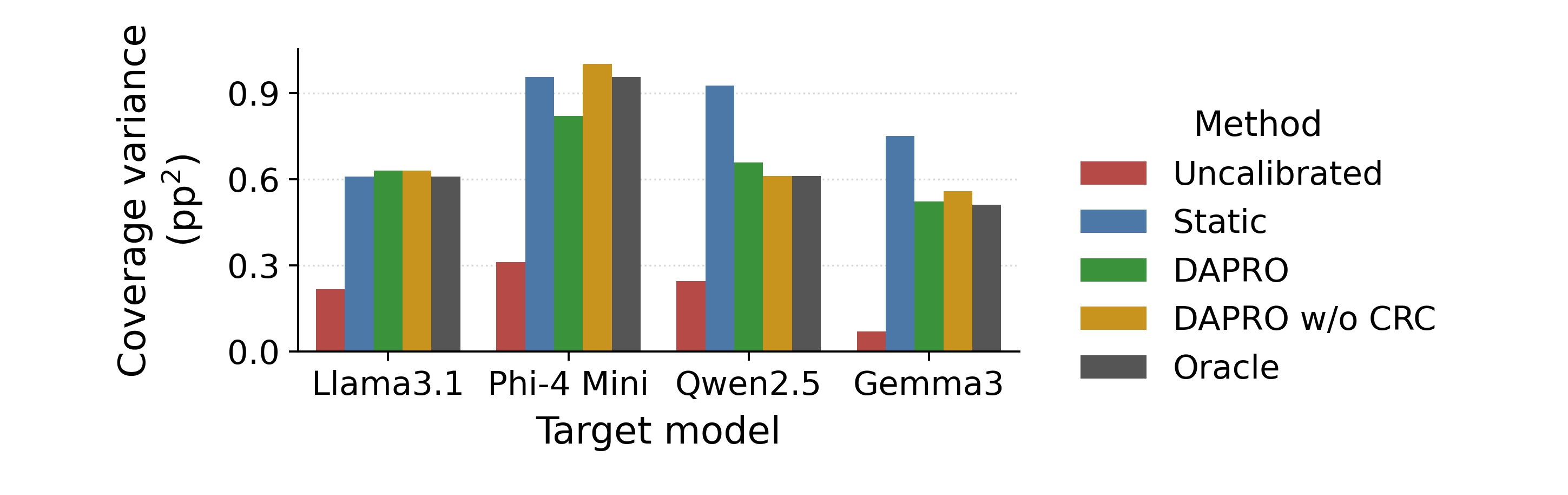}\\
    \includegraphics[width=0.9\linewidth]{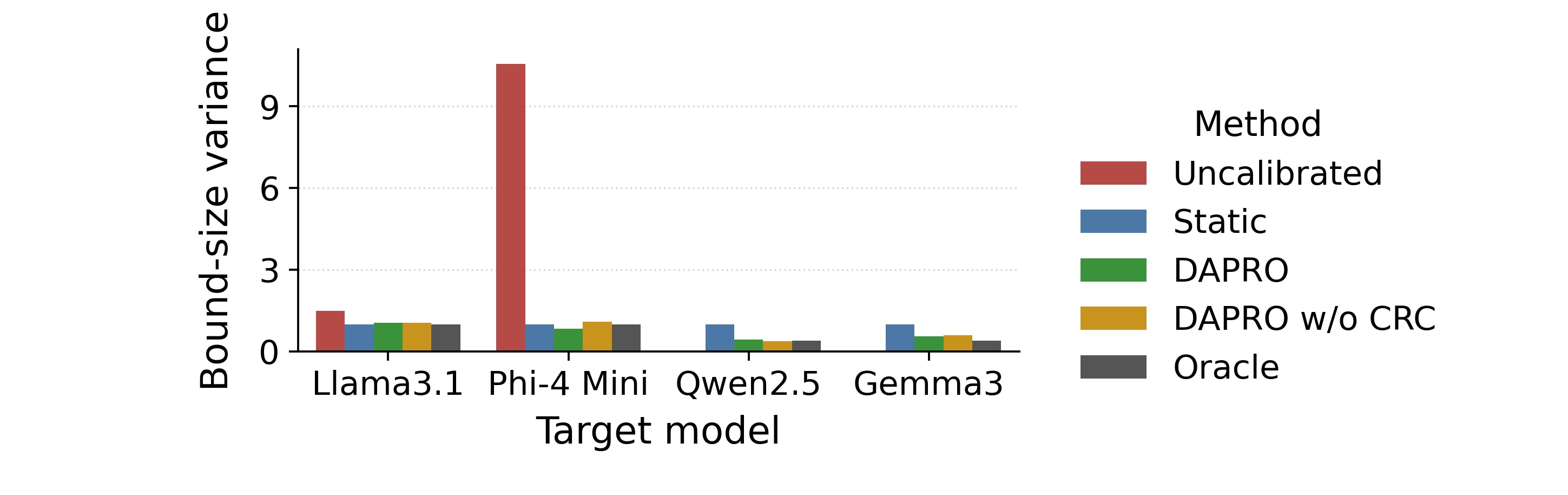}
    \caption{\textbf{Hallucination LPBs}: observed hallucination events, mean
    weighted error $A_i(\fhat_{\hat\tau})/\pi_i$, coverage variance, and LPB-value variance
    normalized so that Static equals one, over 50 random splits.}
    \label{fig:hallucination_full2}
\end{figure}

\subsubsection{AutoIF: evaluating utility of helper agents}
\label{sec:autoif_exp}

We demonstrate our proposal in a helper LLM setup. The ``attacker'' LLM is an agent aiming to guide the target model to output a response that satisfies complex constraints, such as including a specific word multiple times. The judge is a set of pre-defined Python scripts that programmatically evaluate the constraints. If the response of the target satisfies all constraints, the judge outputs $10$; otherwise, it outputs $1$. 

 Figures~\ref{fig:autoif_full1}--
\ref{fig:autoif_full2} display the performance of the LPBs constructed by each method on the AutoIF dataset. These figures show the same trend observed in the previous experiments: \ttmethod attains the nominal coverage rate and satisfies the budget constraints. Moreover, \ttmethod observes the highest number of successful events compared to the static baseline and exhibits the lowest variance.

\begin{figure}[ht]
    \centering
    \includegraphics[width=0.9\linewidth]{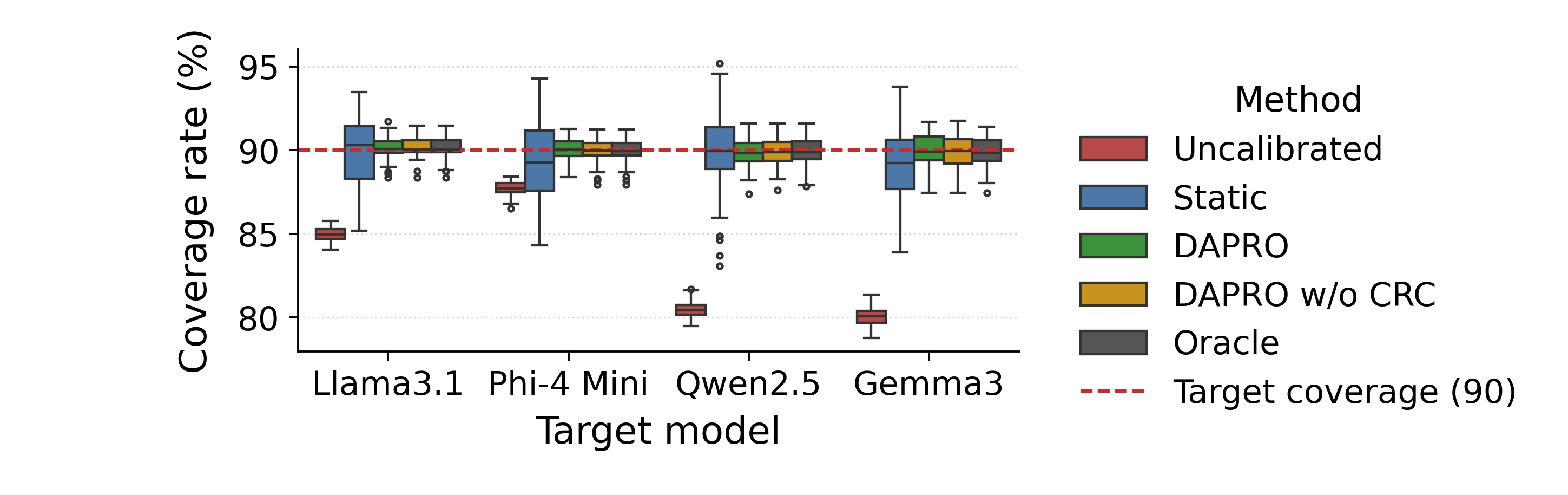}\\
    \includegraphics[width=0.9\linewidth]{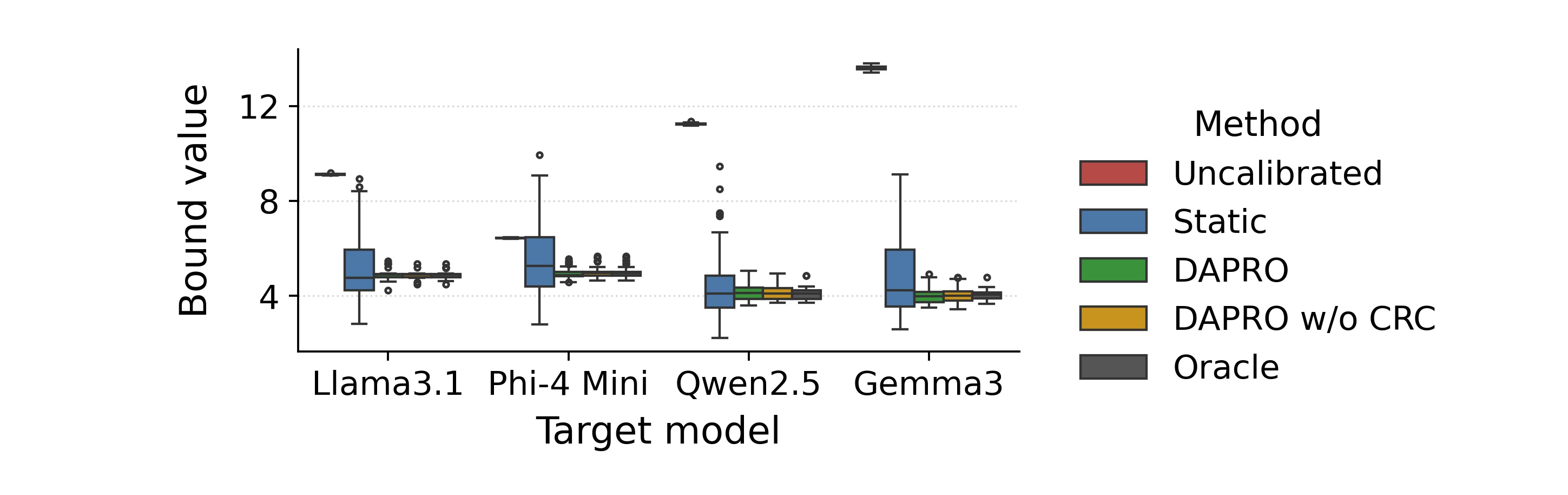}\\
    \includegraphics[width=0.9\linewidth]{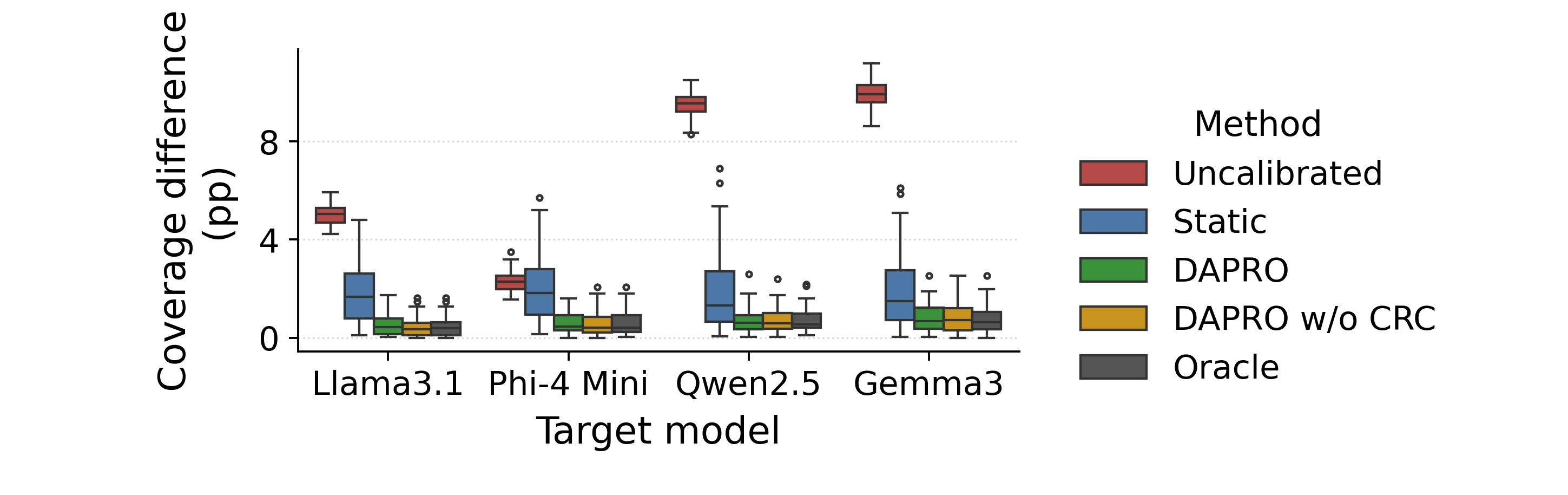}\\
    \includegraphics[width=0.9\linewidth]{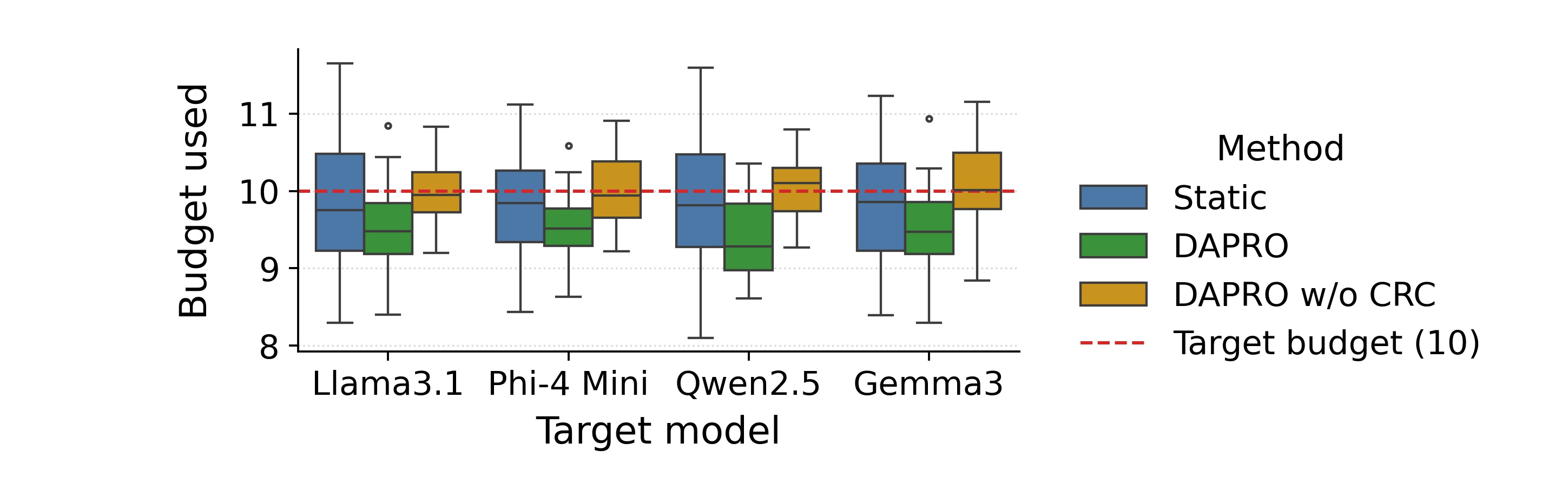}
    \caption{\textbf{AutoIF LPBs}: coverage, LPB value, coverage deviation,
    and budget.  The target coverage is $90\%$, $\bar B=10$, and \ttmethod
    uses $N_1=25$ and $|\calIcrc|=25$.}
    \label{fig:autoif_full1}
\end{figure}

\begin{figure}[ht]
    \centering
    \includegraphics[width=0.9\linewidth]{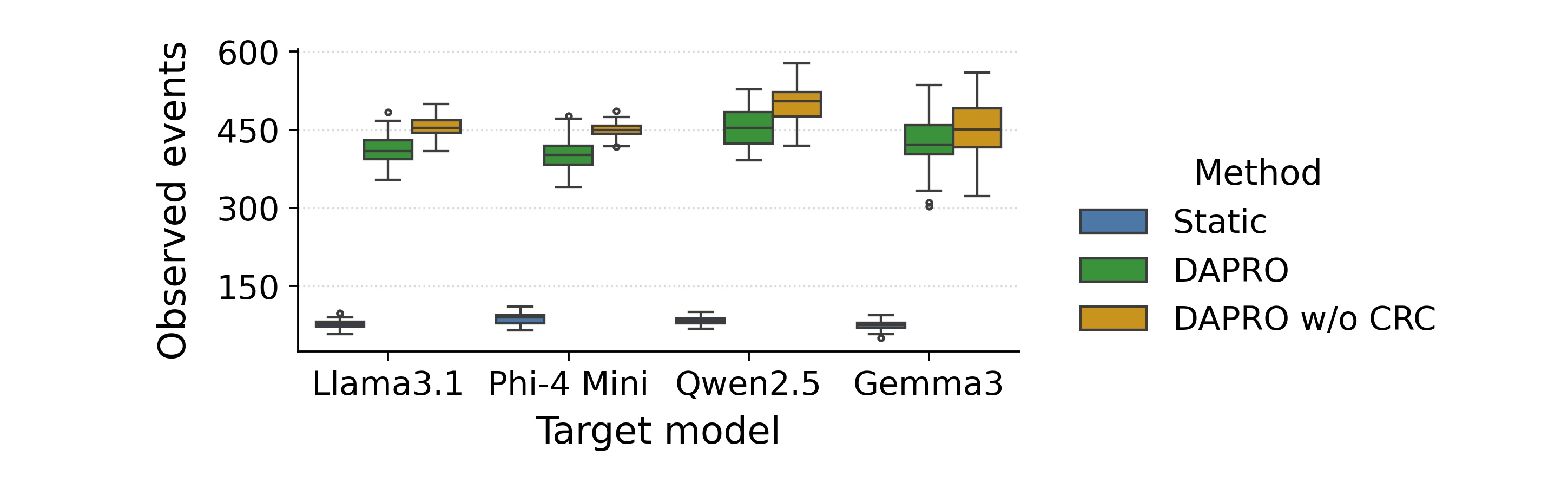}\\
    \includegraphics[width=0.9\linewidth]{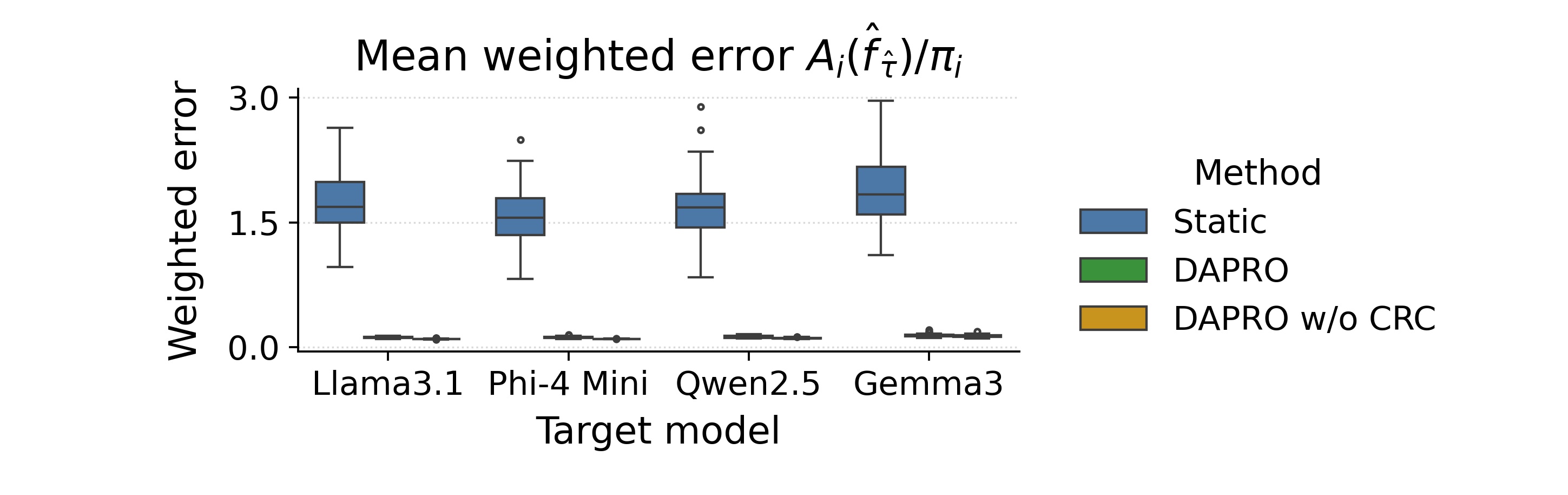}\\
    \includegraphics[width=0.9\linewidth]{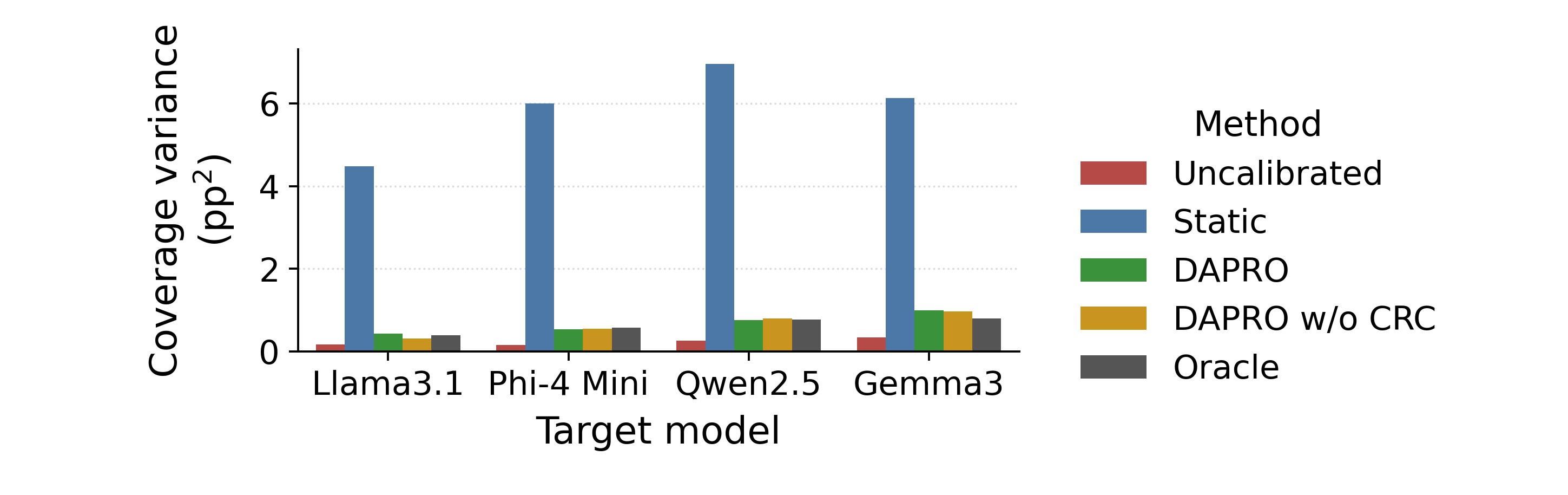}\\
    \includegraphics[width=0.9\linewidth]{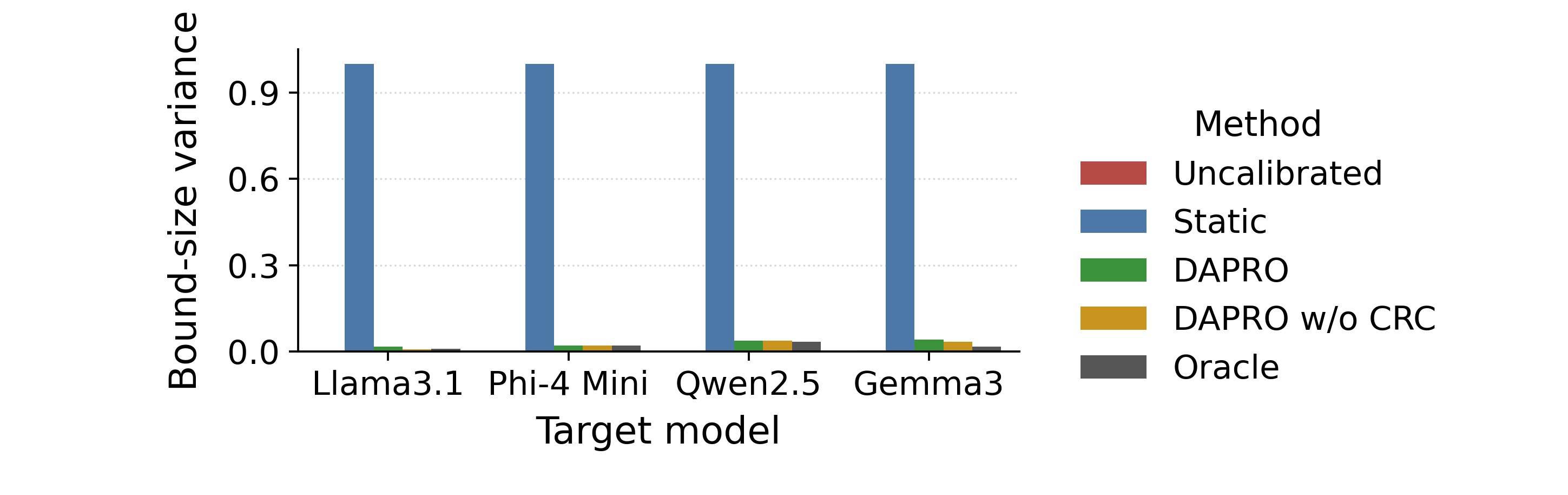}
    \caption{\textbf{AutoIF LPBs}: observed successful events, mean weighted
    error $A_i(\fhat_{\hat\tau})/\pi_i$, coverage variance, and LPB-value variance
    normalized so that Static equals one, over 50 random splits.}
    \label{fig:autoif_full2}
\end{figure}

\subsection{Unbiased estimation of metrics at the population level}
\label{sec:unbiased_estimation}

Beyond evaluating the time-to-events of individual prompts, a critical objective in many applications, such as large-scale red-teaming, is assessing LLM safety or agentic utility at the population level. For this purpose, we consider two metrics evaluated over a maximum conversational horizon of $\maxl = 200$ turns: the \textbf{Event Rate (\ttuer)} $\mu_{\mathrm{event}}=\mathbb{P}(T\leq\maxl)$, which measures the proportion of trajectories in which a key event occurred within $\maxl$ steps, and the \textbf{Restricted Mean Time-to-Event (\ttrmttu)} $\mu_{\mathrm{RMTTE}}=\mathbb{E}[\min(T,\maxl)]$, which measures the average truncated time-to-event, assigning value $\maxl$ to trajectories for which no event occurs within the evaluation horizon. For safety setups, the \ttuer is the unsafe event rate, and the \ttrmttu is the restricted mean time-to-unsafe, whereas in agentic utility, we refer to these metrics as the successful event rate and the restricted mean time-to-success.

Estimating these metrics under limited computing resources is challenging, as when the budget allocator halts a conversation early, the outcome is right-censored and ignoring it introduces bias to the estimator.
Treating these early stopped samples as samples for which no event occurs decreases the perceived event rate. Conversely, running every conversation for the full $\maxl$ iterations to find the true Oracle outcome is computationally infeasible. However, by employing a stochastic budget allocator that tracks the probability of reaching the censoring time, such as \ttmethod, we can construct unbiased estimates of the true population metrics using only a fraction of the compute.

For metric estimation, we set the acquisition horizon to $\maxl$ and use the target:
\[
A_i:=\mathbb{I}\{T_i\leq\maxl\}.
\]
Accordingly, the version of \ttmethod for metric estimation replaces the LPB-specific target
in its soft-prefix objective with the event-rate target. The coefficients used in the optimization objective in Phase-I
are
\[
a_i(t)
=
\mathbb{I}\{t\leq \min(T_i,\maxl)\}
\widehat{\mathbb{P}}\!\left(
T_i=t
\,\middle|\,
T_i\geq t,H_i(t)
\right).
\]

We estimate the target population metrics by reweighting observed events using their known
inclusion probabilities, similarly to the miscoverage estimator
in~\eqref{eq:miscoverage_estimator}. Let $\rho_i(t):=\prod_{s=1}^{t}P_i(s)$ denote the probability that the allocation policy acquires trajectory
$i$ through turn $t$ and define $\pi_i:=\rho_i\!\left(\min\{T_i,\maxl\}\right).$
For samples $i\in\calIone \cup \calIcrc$ we have $\pi_i = 1$. For the Static method, $\pi_i$ is its single static probability $p_i$.

We estimate the event rate using the Horvitz-Thompson estimator~\citep{Horvitz01121952},
\begin{equation}
\widehat\mu_{\mathrm{event}}
=
\frac{1}{|\calI|}
\sum_{i\in\calI}
\frac{
\mathbb{I}\{T_i\leq C_i\}
\mathbb{I}\{\tilde T_i\leq\maxl\}
}{\pi_i}.
\label{eq:metric_estimator}
\end{equation}

For the restricted mean time-to-event, we use the identity:
\[
\min(T_i,\maxl)
=
\maxl
-
\mathbb{I}\{T_i\leq\maxl\}(\maxl-T_i),
\]
which gives the complement estimator
\begin{equation}
\widehat\mu_{\mathrm{RMTTE}}
=
\maxl-
\frac{1}{|\calI|}
\sum_{i\in\calI}
\frac{
\mathbb{I}\{T_i\leq C_i\}
\mathbb{I}\{\tilde T_i\leq\maxl\}
}{\pi_i}
(\maxl-\tilde T_i).
\label{eq:metric_rmtte_estimator}
\end{equation}

Under the assumptions of Theorem~\ref{thm:main_alg_validity}, we get that
\[
\mathbb{E}\!\left[
\frac{\mathbb{I}\{T_i\leq C_i\}}{\pi_i}
\mathbb{I}\{\tilde T_i\leq\maxl\}
\,\middle|\,
T_i,H_i
\right]
=
\mathbb{I}\{T_i\leq\maxl\}.
\]
Consequently, following the arguments of
Theorem~\ref{thm:general_validity},
\[
\mathbb{E}[\widehat\mu_{\mathrm{event}}]
=
\mu_{\mathrm{event}},
\qquad
\mathbb{E}[\widehat\mu_{\mathrm{RMTTE}}]
=
\mu_{\mathrm{RMTTE}}.
\]

In the following experiments, the Uncalibrated baseline computes both metrics directly from the initial prediction of the models of the distribution of the time-to-event.
All experiments set the target budget per sample to $\bar B=20$. For \ttmethod, $N_1=50$ and $|\calIcrc|=50$; the no-CRC variant uses all 100 samples in $\calIone$. Each figure reports the two
estimators and their empirical variances over random calibration-test data partitions,
realized budget, and the metric-target mean weighted error $A_i/\pi_i$, where $A_i=\mathbb{I}\{T_i\leq\maxl\}$ is the event-observation indicator. The budget and weighted-error panels exclude Uncalibrated and Oracle because these are not comparable allocation quantities. The red dashed estimator reference is the value computed on the fully observed calibration and test sets.

\begin{figure}[ht]
    \centering
    \includegraphics[width=\linewidth]{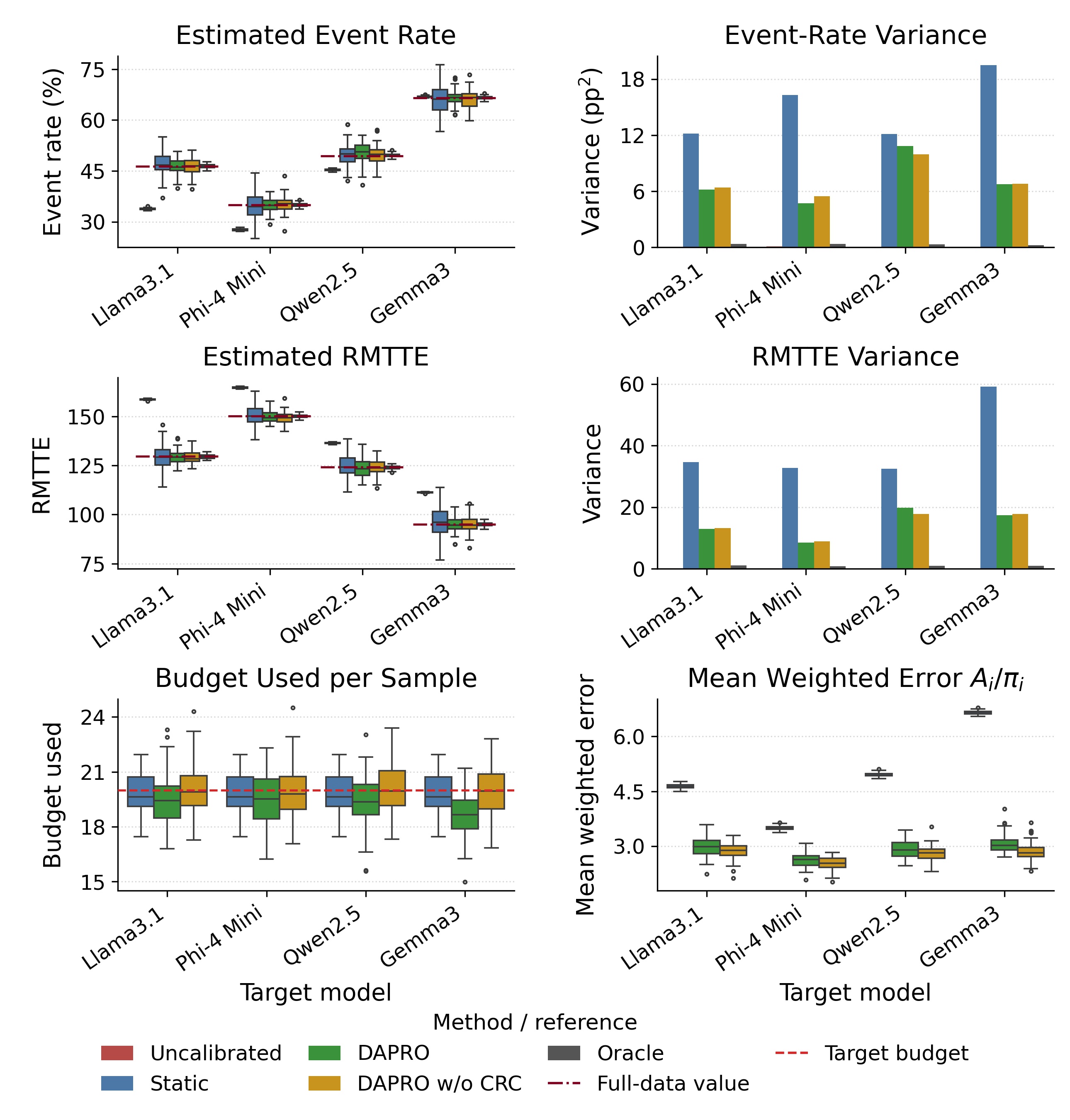}
    \caption{Population-level metric estimation on the Toxicity dataset with the Qwen 2.5 14B Instruct model serving as the attacker. The top two rows display the estimated Event Rate and Restricted Mean Time-to-Event, alongside their respective variances. The red dashed lines are the oracle quantities. The bottom row displays the utilized budget per sample, with a red dashed line indicating the target level $\bar B=20$, and the total number of observed unsafe events. While the uncalibrated baseline underestimates true metric values, and baseline calibration suffers from high variance, \ttmethod accurately estimates the true Oracle metrics with zero bias and lower variance. Furthermore, \ttmethod achieves the highest budget efficiency, observing the most events under the same resource constraints. Results are aggregated over 50 random calibration-test splits.}
    \label{fig:metrics_toxicity}
\end{figure}

Figures \ref{fig:metrics_toxicity}--\ref{fig:metrics_autoif} present the estimated \ttuer and \ttrmttu, alongside their variance, across all five dataset--judge configurations. Since all calibration methods correctly address the right-censoring distribution shift by applying re-weighting, they successfully attain the true Oracle value (marked by the red dashed line) in expectation. As all calibration methods have zero bias, the estimator's measurement of reliability is its variance. 
Across all these five benchmarks, \ttmethod with CRC has lower empirical variance than Static for both the event-rate and restricted-mean estimators. The non-CRC version is often slightly more efficient because it avoids the finite-sample CRC correction, but only the CRC variant is guaranteed to control the budget in expectation. 
Overall, this experiment demonstrates that while proper statistical re-weighting is necessary to construct an unbiased estimator, efficient budget allocation is essential to obtain low-variance, reliable safety and utility metrics for large-scale tasks under limited compute resources.

\begin{figure}[ht]
    \centering
    \includegraphics[width=\linewidth]{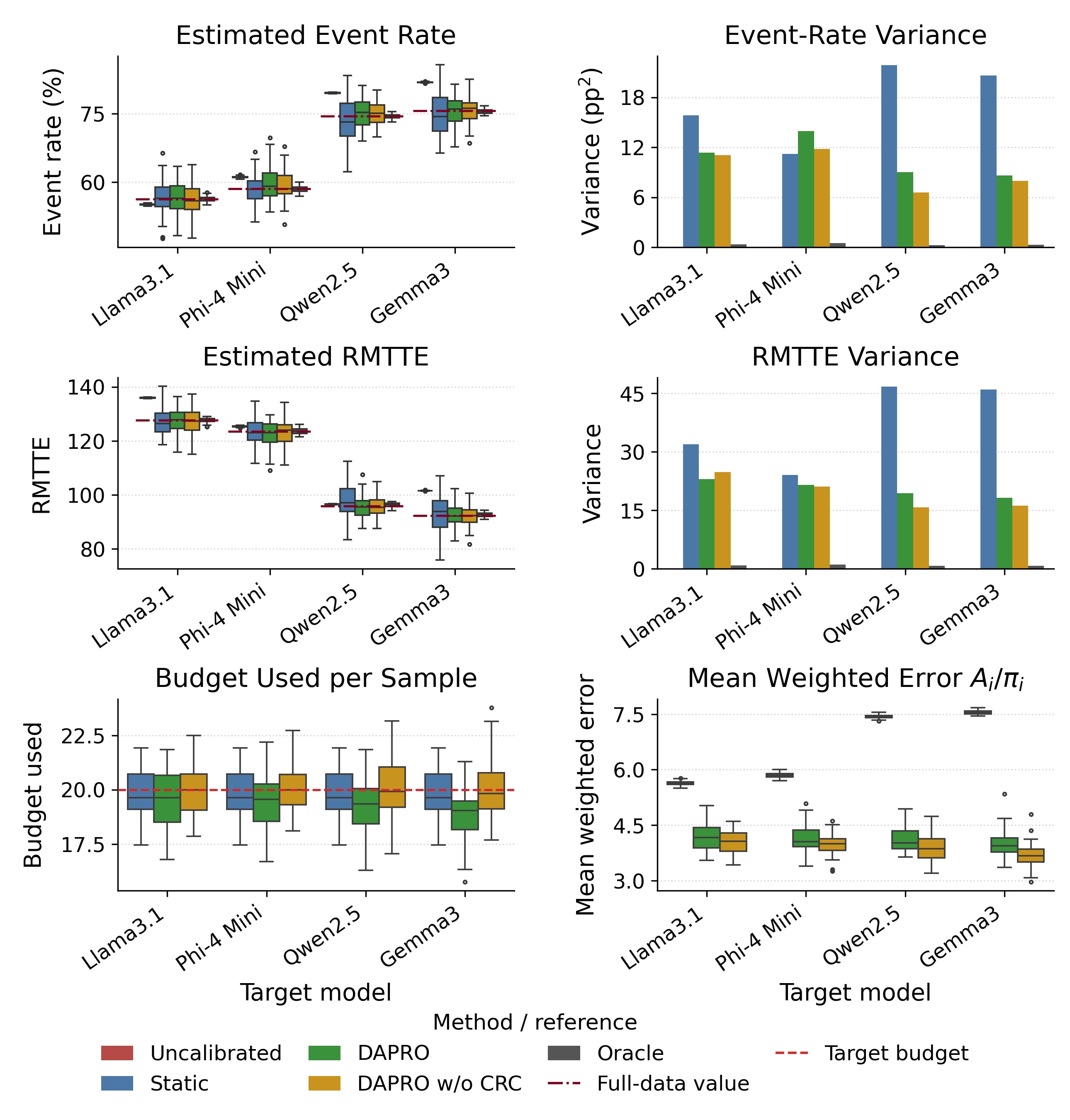}
    \caption{Population-level event-rate and restricted-mean time-to-event
    estimation on \textbf{Red-Team with a Qwen judge}.  The configuration is
    $\bar B=20$, $N_1=50$, and $|\calIcrc|=50$, evaluated over 50 random
    splits.}
    \label{fig:metrics_red_team_qwen}
\end{figure}

\begin{figure}[ht]
    \centering
    \includegraphics[width=\linewidth]{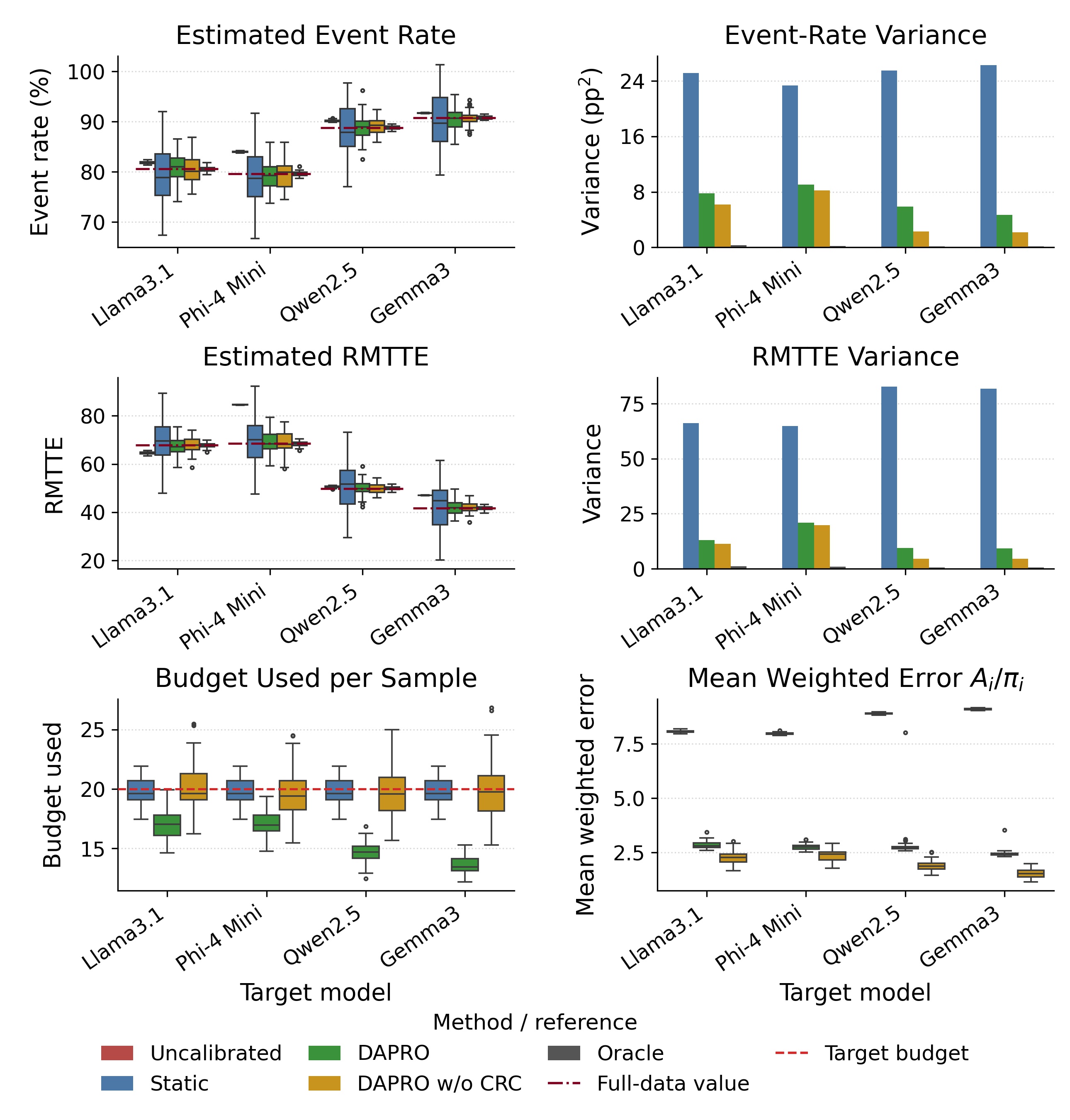}
    \caption{Population-level event-rate and restricted mean time-to-event
    estimation on \textbf{Red-Team with Llama Guard}.  The configuration is
    $\bar B=20$, $N_1=50$, and $|\calIcrc|=50$, evaluated over 50 random
    splits.}
    \label{fig:metrics_red_team_llama_guard}
\end{figure}

\begin{figure}[ht]
    \centering
    \includegraphics[width=\linewidth]{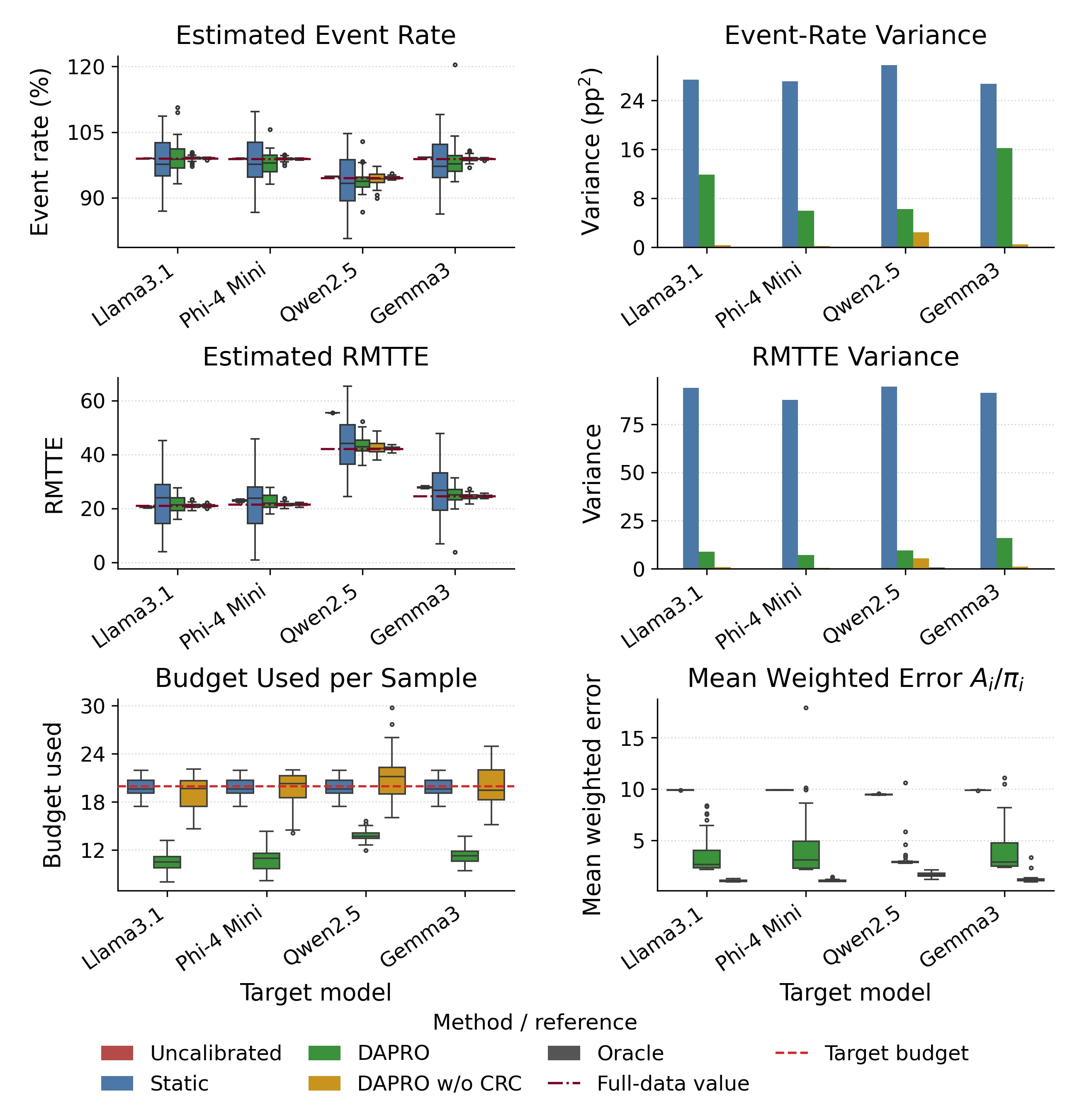}
    \caption{Population-level event-rate and restricted-mean time-to-event
    estimation on the \textbf{Hallucination} benchmark.  The configuration is
    $\bar B=20$, $N_1=50$, and $|\calIcrc|=50$, evaluated over 50 random
    splits.}
    \label{fig:metrics_hallucination}
\end{figure}

\begin{figure}[ht]
    \centering
    \includegraphics[width=\linewidth]{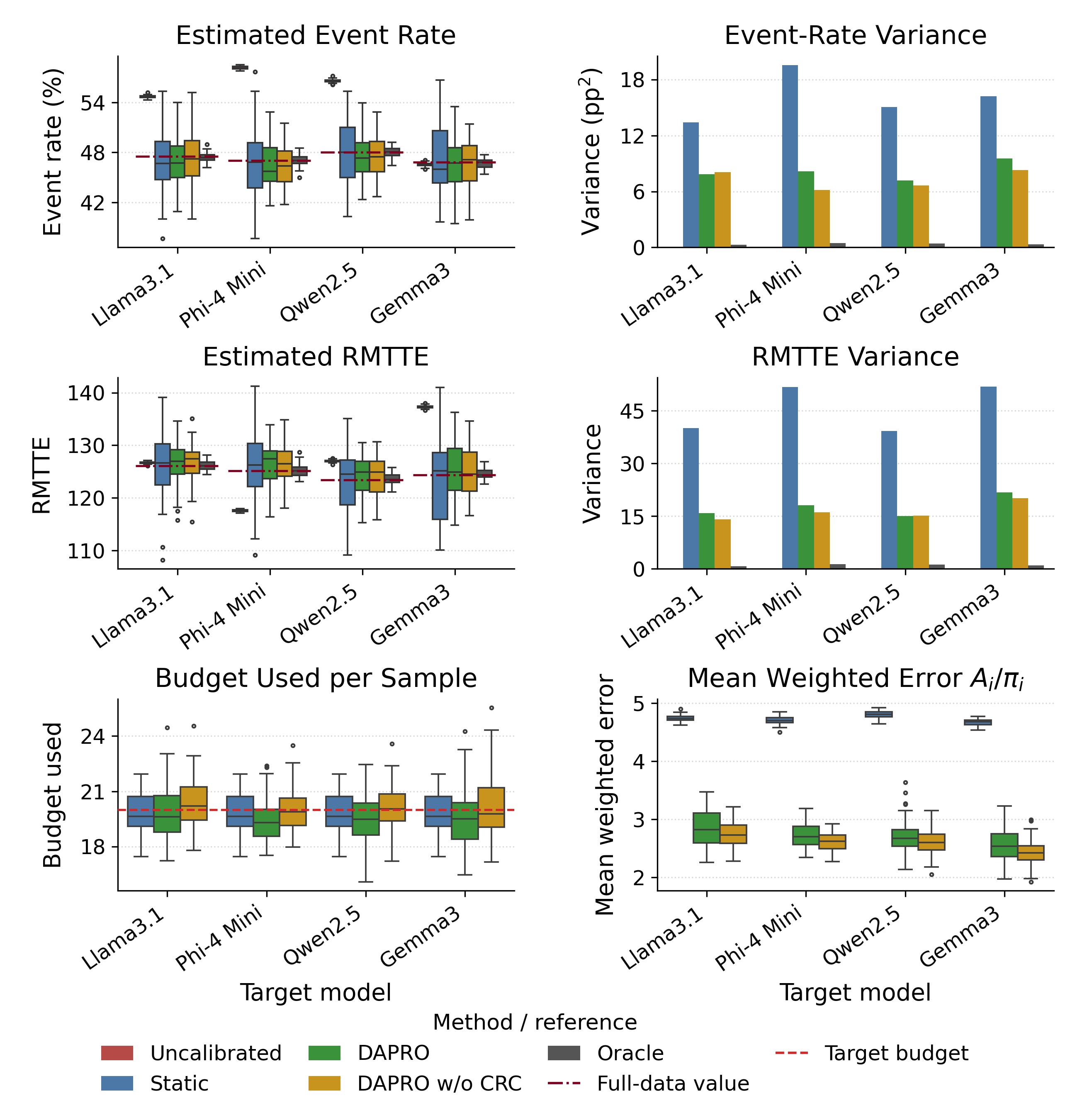}
    \caption{Population-level event-rate and restricted-mean time-to-event
    estimation on \textbf{AutoIF}.  The configuration is $\bar B=20$,
    $N_1=50$, and $|\calIcrc|=50$, evaluated over 50 random splits.}
    \label{fig:metrics_autoif}
\end{figure}

\subsection{The coverage gap}
\label{sec:toxicity_lpb_budget_coverage_gap}

We study the coverage gap we introduced in~\ref{sec:theoretical_guarantees} and compare it to the prior coverage gap of~\citet{davidov2026calibrated}. We employ \ttmethod (with CRC) to construct LPBs for the Toxicity data at target $1-\alpha=90\%$ coverage rate. We vary the target budget per sample over $\{5,10,20,30,40,50\}$ and evaluate each setting using $50$ random seeds, with confidence level $\delta=0.05$.

Recall that the prior coverage gap scales with the maximum weight across all calibration data, while the coverage gap presented in this work scales with the mean weights of trajectories contributing to the miscoverage estimator.
Specifically, for this experiment, we evaluate our coverage gap using the exact definition of $\bar{\nu}$ from Section~\ref{sec:theoretical_guarantees}:
\begin{equation}
    \label{nu_with_tau_prior}
\bar{\nu} = 
    \mathbb{E}[w(i) \mathbb{I}\{T_i<\fhat_{\tauprior}(X_i)\} \mid \{(X_i, H_i, T_i)\}_{i \in \calIone \cup \calIcrc}, \Dtrain] \leq \bar{\nu},
     \forall i \in \calItwo.
 \end{equation}
Above $\tauprior=0.56$.

Figure~\ref{fig:toxicity_lpb_coverage_gap} shows that the prior gap ranges
from approximately $894$ to $533$ percentage points and is therefore
uninformative for every evaluated budget. In contrast, our gap decreases from
approximately $24.1$ to $5.0$ percentage points and remains useful
even at the smallest budget. Across the evaluated range, our gap is
approximately $37$--$110\times$ smaller than the prior gap. We remark that our coverage gap can be even tighter when setting a smaller quantile level than the prior level $\tauprior$ in the definition of $\hat{\nu}$ in~\eqref{nu_with_tau_prior}

This experiment demonstrates that the proposed coverage gap is substantially more informative than the prior one.

\begin{figure}[ht]
    \centering
    \includegraphics[width=0.5\linewidth]{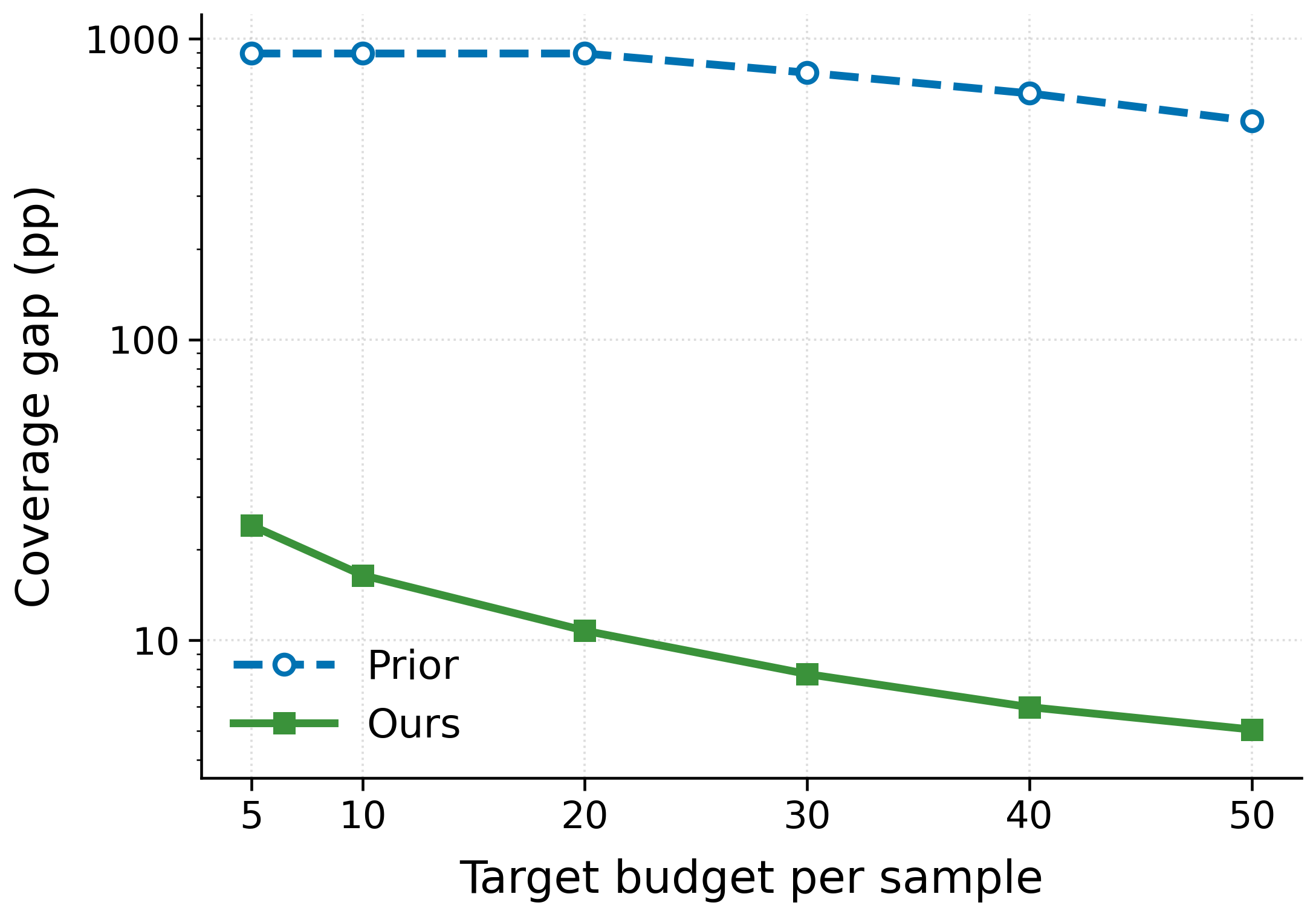}
    \caption{Coverage gap: prior (blue) and proposed (green) of \ttmethod on the Toxicity dataset.}
    \label{fig:toxicity_lpb_coverage_gap}
\end{figure}

\subsection{Ablation study for \ttmethod}
\label{sec:ablation_exps}

In this section, we analyze the robustness of our proposed \ttmethod to modifications of its components.
All ablations use the Toxicity dataset with Qwen 2.5 14B Instruct as both
attacker and target and Detoxify as the judge, except for the attacker-shift study. The default setup has nominal coverage level set to $1-\alpha = 90\%$, $\bar B=20$ budget per sample and total Phase~I set size $|\calIone|+|\calIcrc|=50$; hence \ttmethod uses $N_1=25$ and $|\calIcrc|=25$. Every plot includes Static, \ttmethod, and \ttmethod without CRC. For LPB ablations, the weighted-error diagnostic is always $A_i(\fhat_{\hat\tau})/\pi_i$. For metric-estimation ablations, the weighted-error diagnostic is $A_i(\maxl)/\pi_i$.
For the optimization-process, score-noise, score-function, and
representation-capacity ablations, we run them over the metric-estimation task, while the calibration-size, budget, optimization-
objective, CRC-cap, and attacker-shift ablations are demonstrated over the LPB construction task.

Overall, this study reveals that \ttmethod is effective and valid under varying experimental conditions. Even under extreme conditions, such as a limited budget or severely corrupted scores, our framework consistently achieves the desired coverage level, and its budget consumption is close to the target level.

\subsubsection{The effect of the Phase~I split size (LPB)}

We employ \ttmethod with various values of the first calibration set split size ($N_1$, $\calIcrc$) and display its performance in Figure~\ref{fig:ablation_n1}. \ttmethod consistently achieves the nominal coverage rate, 90\%, and satisfies the target budget per sample of $20$ across all tested values of ($N_1$+ $\calIcrc$). This figure reveals that the budget used by \ttmethod converges to the target level as Phase~I data size increases. However, \ttmethod attains a higher mean weight for higher values of ($N_1$+ $\calIcrc$). Yet, \ttmethod achieves a lower coverage difference and lower coverage variance compared to the static baseline allocation of~\citep{davidov2026calibrated} for all examined values of $N_1$.

\begin{figure}[ht]
    \centering
    \includegraphics[width=0.98\linewidth]{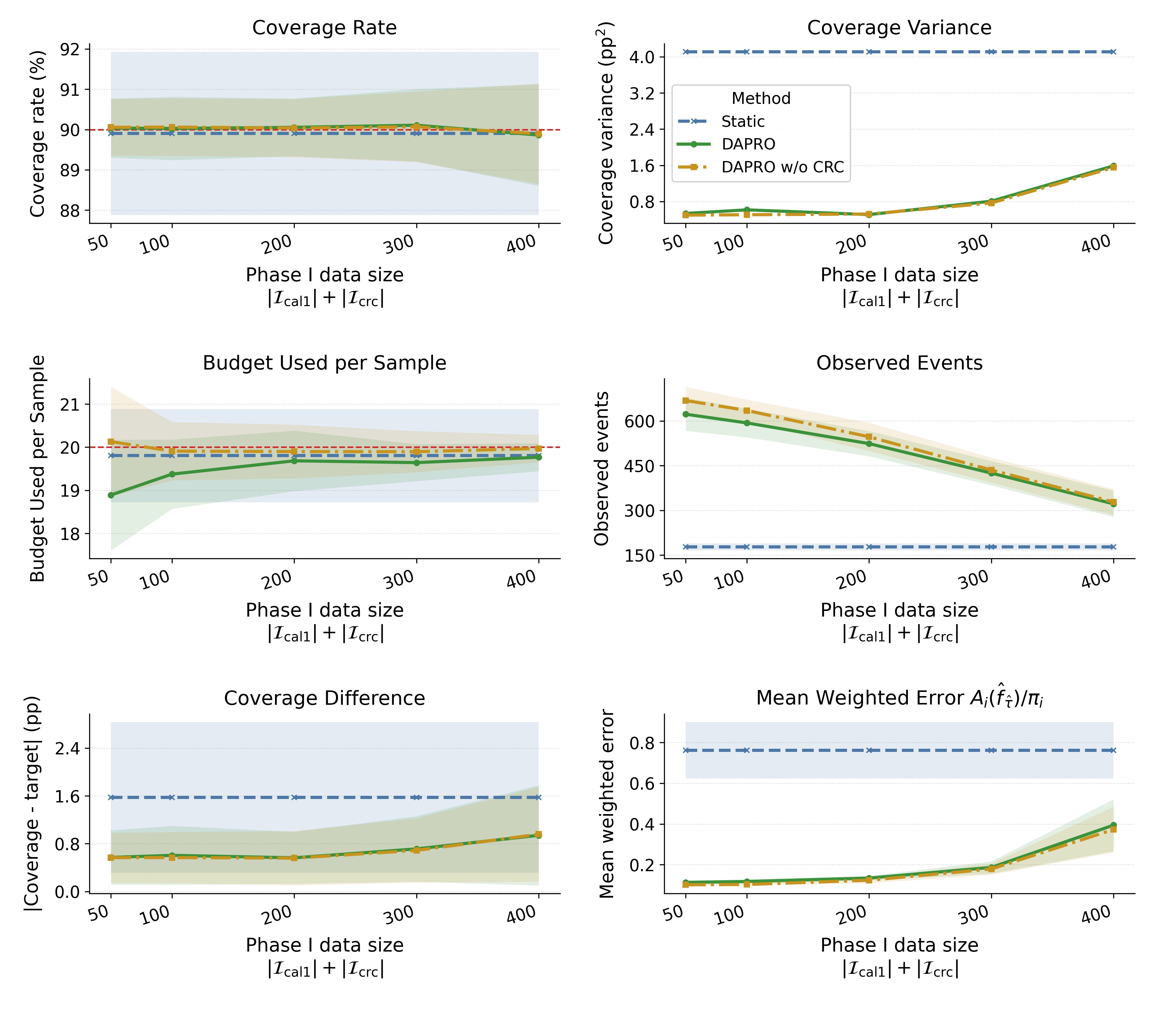}
    \caption{LPB Phase~I calibration set size ablation. The x-axis is
    $|\calIone|+|\calIcrc|$.  Lines show means and shaded
    regions show one standard deviation over 50 random partitions.}
    \label{fig:ablation_n1}
\end{figure}

\subsubsection{The effect of the available budget (LPB)}

We employ \ttmethod and the static baseline of~\citep{davidov2026calibrated} with varying levels of budget per sample and present their performance in Figure~\ref{fig:ablation_budget}. To adapt our method and its data split to the varying budget per sample, we adapt its Phase~I calibration set size to the nominal average budget per sample. For CRC runs, the total is split
as evenly as possible between $\calIone$ and $\calIcrc$.
This figure demonstrates that under a low budget constraint, the static baseline suffers from a high coverage variance that highly deviates from the target level. It requires more budget to eventually converge towards coverage values attained by our method with a less available budget. In stark contrast, \ttmethod maintains a valid empirical coverage ($\approx 90\%$) across all budget levels. Furthermore, as the nominal budget increases, our approach attains lower values of mean weight. This indicates more efficient resource utilization of \ttmethod.

\begin{figure}[ht]
    \centering
    \includegraphics[width=0.98\linewidth]{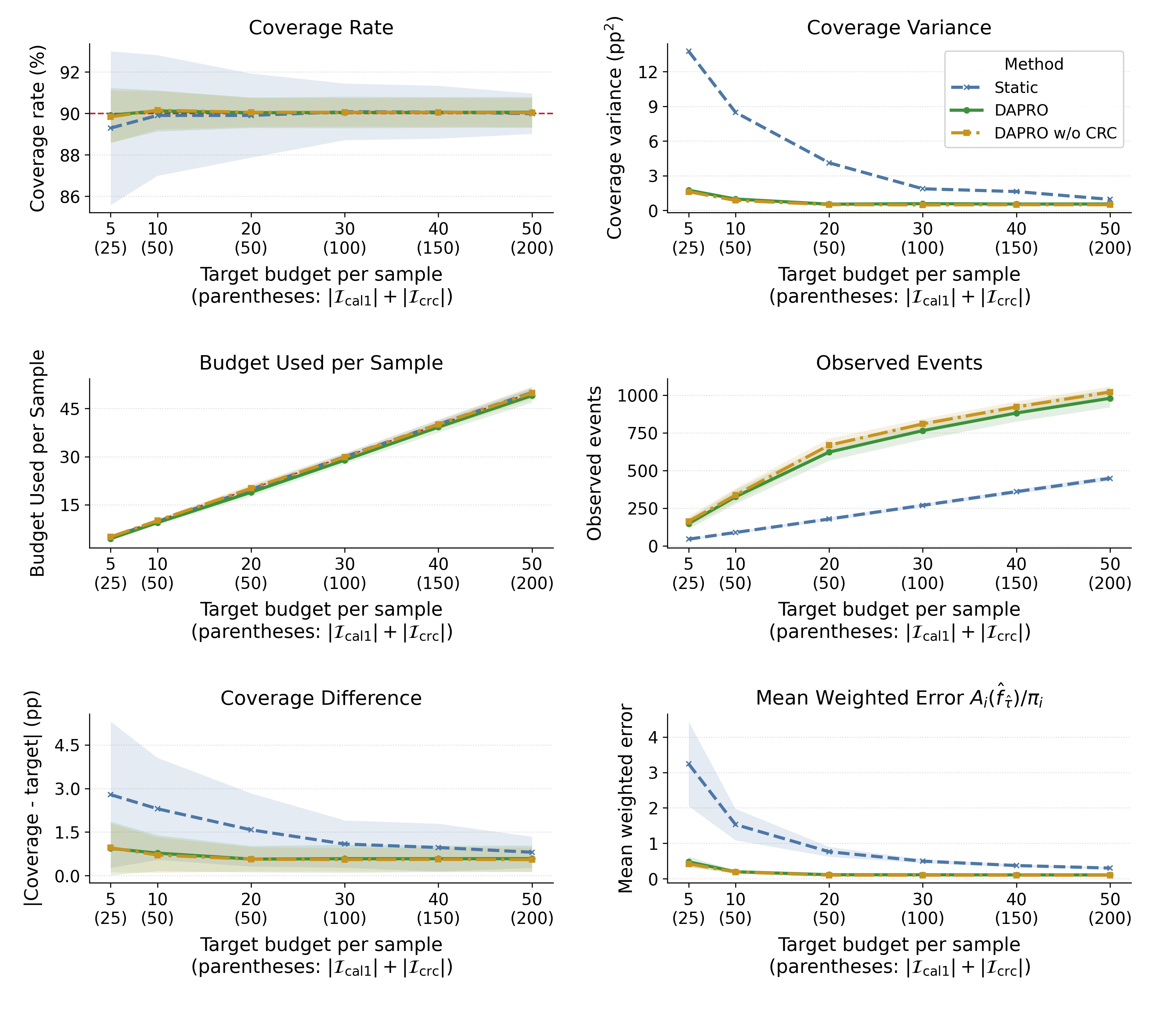}
    \caption{LPB budget ablation.  Parenthesized x-axis values give the total
    adaptive-calibration size $|\calIone|+|\calIcrc|$.}
    \label{fig:ablation_budget}
\end{figure}

\subsubsection{Optimization objective (LPB)}

In this section, we study the effect of the optimization objective used by \ttmethod.
To define the optimization objectives, let
\[
F_i := \fhat_{\alpha}(X_i),
\qquad
b_i := \min\{T_i,\qprior(X_i)\},
\qquad
\rho_i(t;M)
:=
\prod_{s=1}^{t}M_s^{\mathrm{raw}}(S_i(s)),
\]
where $\rho_i(t;M)$ is the probability that policy $M$ acquires trajectory
$i$ through turn $t$. All variants solve the same optimization
problem
\begin{equation}
\begin{aligned}
\label{eq:objective_ablation_general}
&\min_{M^{\mathrm{raw}}}
\quad
\frac{1}{N_1}
\sum_{i\in\calIone}
\sum_{t=1}^{b_i}
\frac{a_i(t)}{\rho_i(t;M^{\mathrm{raw}})}\\
& \text{s.t.}\quad  \frac{1}{N_1}\sum_{i\in\calIone}\bfunc(M^{\mathrm{raw}}(S_i))\leq\Bbar,\quad S_i(t)\leq S_j(t) \Longrightarrow M^{\mathrm{raw}}_t(S_i(t))\leq M^{\mathrm{raw}}_t(S_j(t))
      \quad\text{for every active }i,j,t.
\end{aligned}
\end{equation}
The variants differ only in their coefficients $a_i(t)$.

\paragraph{Soft-prefix objective.}
At each observed prefix $t$, let
\[
\widehat h_i(t)
:=
\widehat{\mathbb P}\!\left(
T_i=t
\,\middle|\,
T_i\geq t,H_i(t)
\right).
\]
The coefficient assigned to turn $t$ is
\begin{equation}
\label{eq:soft_prefix_ablation}
a_i^{\mathrm{SP}}(t)
:=
\mathbb I\{t\leq b_i\}\,
\widehat h_i(t)\,
\frac{\mathbb I\{t<F_i\}+\delta}{1+\delta},
\end{equation}
where $\delta\geq0$ is a small regularizer. Thus, the
model-estimated probability of the target event is distributed over the
causal prefixes at which continuation decisions are made.

\paragraph{Hard-terminal objective.}
Let
\[
A_i:=\mathbb I\{T_i<F_i\}.
\]
The hard-terminal objective places the regularized realized target indicator
only at the final observed point:
\begin{equation}
\label{eq:hard_terminal_ablation}
a_i^{\mathrm{HT}}(t)
:=
\frac{A_i+\delta}{1+\delta}\,
\mathbb I\{t=b_i\}.
\end{equation}
Consequently, its objective is
\[
\frac{1}{N_1}
\sum_{i\in\calIone}
\frac{(A_i+\delta)/(1+\delta)}
{\rho_i(b_i;M)}.
\]

\paragraph{Hard-prefix objective.}
A realized hard event contributes a one-hot coefficient at its event prefix:
\[
\mathbb I\{T_i=t<F_i\}.
\]
Since $T_i<F_i\leq\qprior(X_i)$ implies $b_i=T_i$,
\[
\sum_{t=1}^{b_i}
\frac{\mathbb I\{T_i=t<F_i\}}{\rho_i(t;M)}
=
\frac{\mathbb I\{T_i<F_i\}}{\rho_i(b_i;M)}.
\]
After applying the same endpoint regularization, this is exactly~\eqref{eq:hard_terminal_ablation}. The hard-prefix and hard-terminal objectives therefore produce the same optimization problem, so we omit the redundant hard-prefix curve.

\paragraph{Soft-terminal objective.}
For convenience, we write $H_i(0)=H_i(1)$. 
Finally, define the model's target-event probability at the initial
pre-acquisition state by
\[
\widehat A_i^{(0)}
:=
\widehat{\mathbb P}\!\left(T_i<F_i\mid H_i(0)\right)
=
\sum_{u<F_i}
\widehat{\mathbb P}\!\left(T_i=u\mid H_i(0)\right).
\]
The soft-terminal objective attaches this single probability to the last
observed point:
\begin{equation}
\label{eq:soft_terminal_ablation}
a_i^{\mathrm{ST}}(t)
:=
\frac{\widehat A_i^{(0)}+\delta}{1+\delta}\,
\mathbb I\{t=b_i\}.
\end{equation}
Unlike the soft-prefix objective, it does not update or distribute the target
signal across the sequence of continuation decisions.

Figure~\ref{fig:ablation_objective} compares the soft-prefix optimization objective used by \ttmethod with hard-terminal and soft-terminal objectives. The hard-prefix
case is omitted because its realized objective is algebraically equivalent to the hard-terminal case. Soft prefix is the most reliable choice: it observes the most events and attains the lowest weighted error and coverage variance. Hard terminal is competitive but less efficient, while soft terminal performs substantially worse because a single terminal probability gives a weaker learning signal for the sequence of continuation decisions. This supports using the prefix decomposition even though the final variance target is
terminal.

\begin{figure}[ht]
    \centering
    \includegraphics[width=0.98\linewidth]{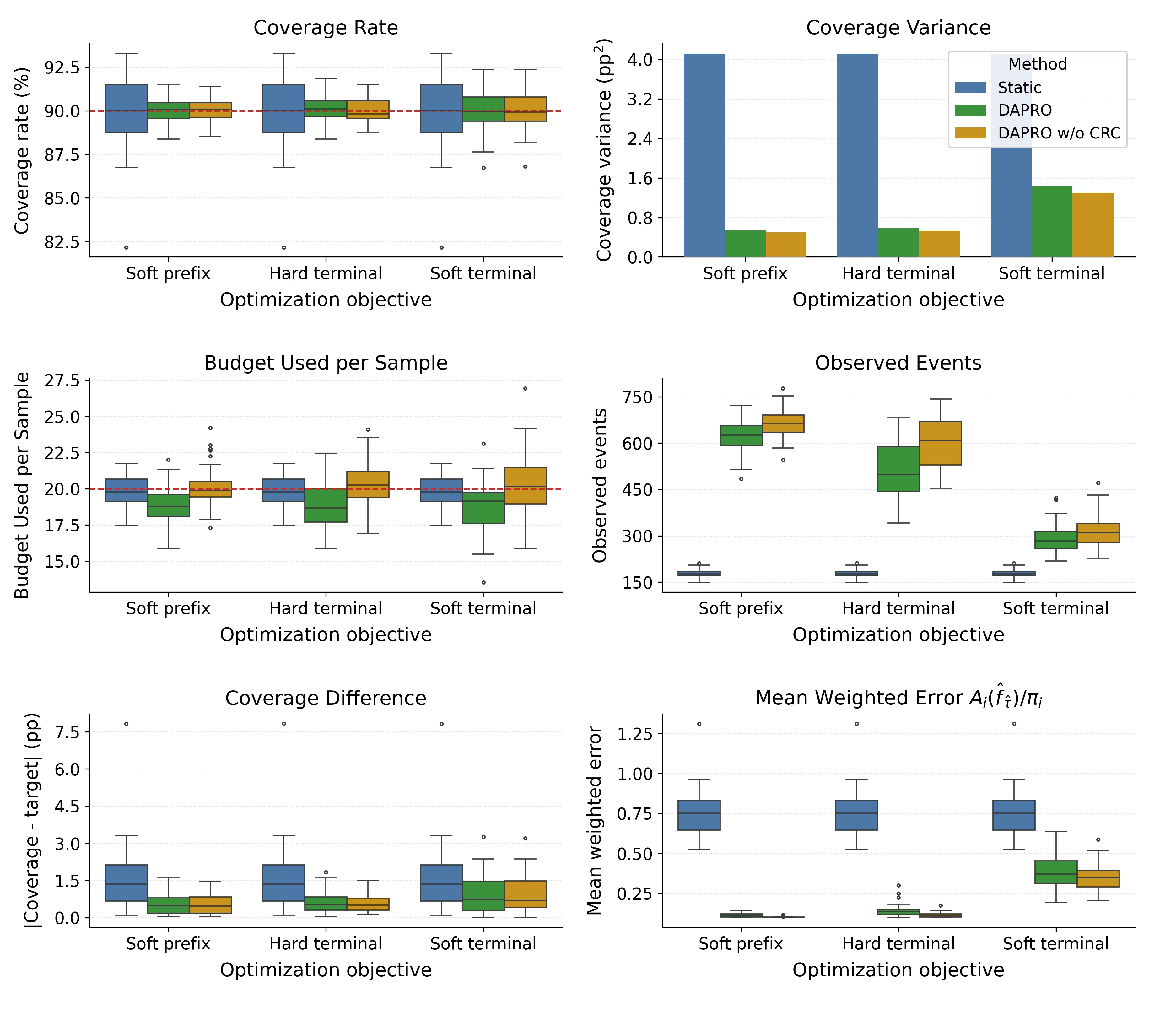}
    \caption{LPB optimization-objective ablation at $\bar B=20$,
    $N_1=25$, and $|\calIcrc|=25$.}
    \label{fig:ablation_objective}
\end{figure}

\subsubsection{Optimization process (metric estimation)}
\label{sec:optimization_process_ablation}
To study the effect of the Phase-I optimization, we compare \ttmethod with a uniform budget allocation policy, referred to as Uniform+CRC, which calibrates its continuation probabilities with CRC, similarly to our \ttmethod. In this Uniform+CRC policy,
the uncalibrated continuation probabilities are not optimized and are drawn uniformly. This naive approach devotes all 50 Phase~I samples to the CRC calibration. 
The remaining samples are used for the dynamic budget allocation in Phase~II, as in \ttmethod.
Figure~\ref{fig:ablation_optimization} shows that the Uniform+CRC method satisfies the expected budget constraint and observes more events, although its event-rate variance and its mean weighted error are higher than the static baseline and \ttmethod. The large gap demonstrates that dynamic budget allocation alone is insufficient: the allocation must optimize the continuation probabilities to minimize the variance. It also shows that the number of observed events is not, by itself, indicative of the event-rate variance.

\begin{figure}[ht]
    \centering
    \includegraphics[width=0.98\linewidth]{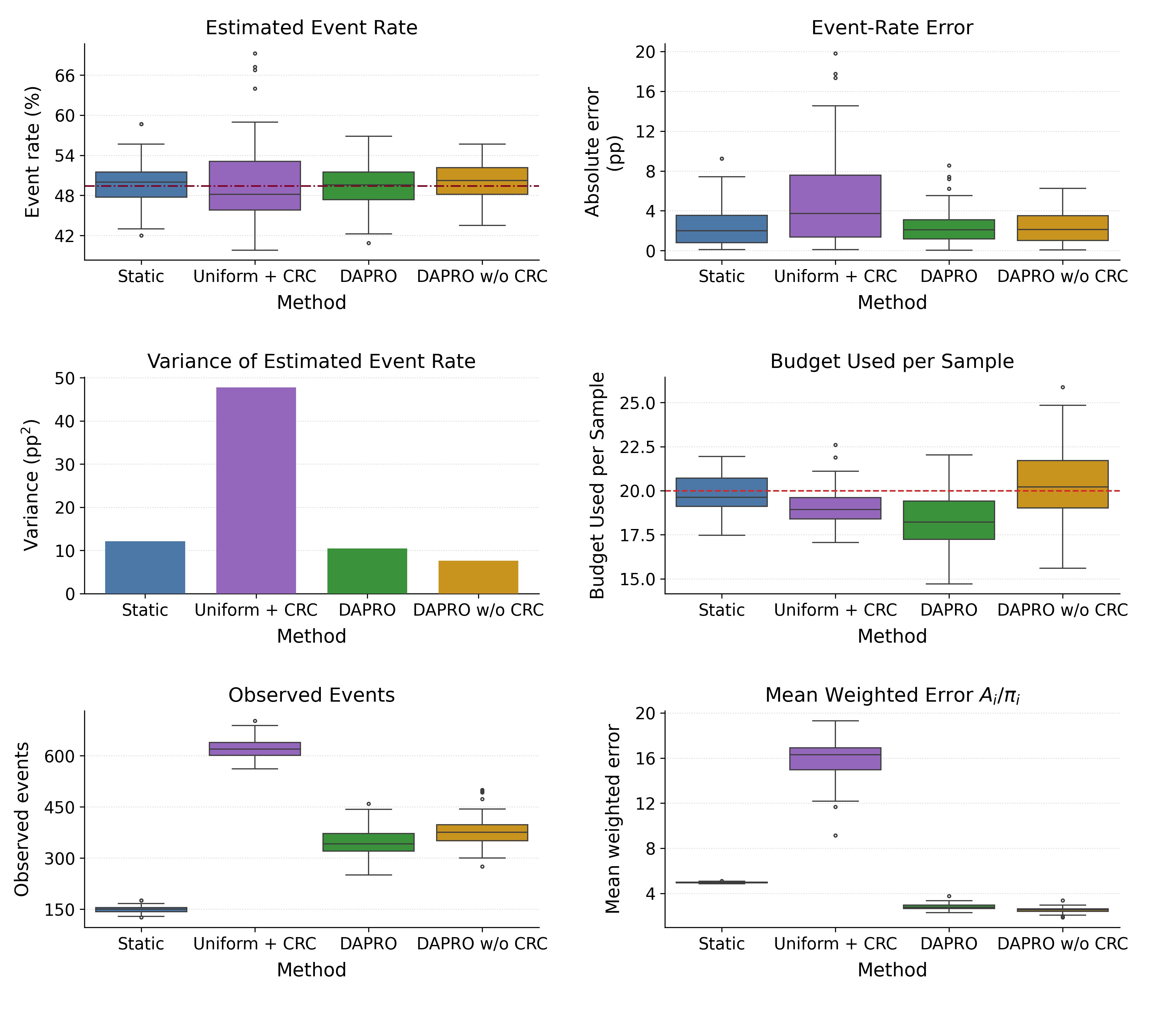}
    \caption{Metric-estimation optimization process ablation at
    $\bar B=20$. Standard \ttmethod uses $N_1=25$ and
    $|\calIcrc|=25$; Uniform+CRC uses all 50 samples for CRC.}
    \label{fig:ablation_optimization}
\end{figure}

\subsubsection{Score corruption (metric estimation)}

We corrupt the learned score according to
\begin{equation}
    S_i^{(\lambda)}(t)
    =(1-\lambda)\,\mathcal{S}(H_i(t))+\lambda U_i(t),
    \qquad U_i(t)\sim\operatorname{Unif}(0,1).
\end{equation}
Thus, larger $\lambda$ means more noise, and $\lambda=1$ is fully random.
Figure~\ref{fig:ablation_score_noise} shows that the CRC policy is stable for
$\lambda\leq0.75$, but at $\lambda=1$ its event-rate variance increases, the event count decreases, and the weighted error
increases. The no-CRC variant has its event count decrease and weighted error increase at high noise. The main
conclusion is therefore not that the score is irrelevant: moderate corruption
is absorbed by the coarse $K=2$ projection and CRC calibration, whereas removing
all signal from the score harms the effectiveness of the policy.

\begin{figure}[ht]
    \centering
    \includegraphics[width=0.98\linewidth]{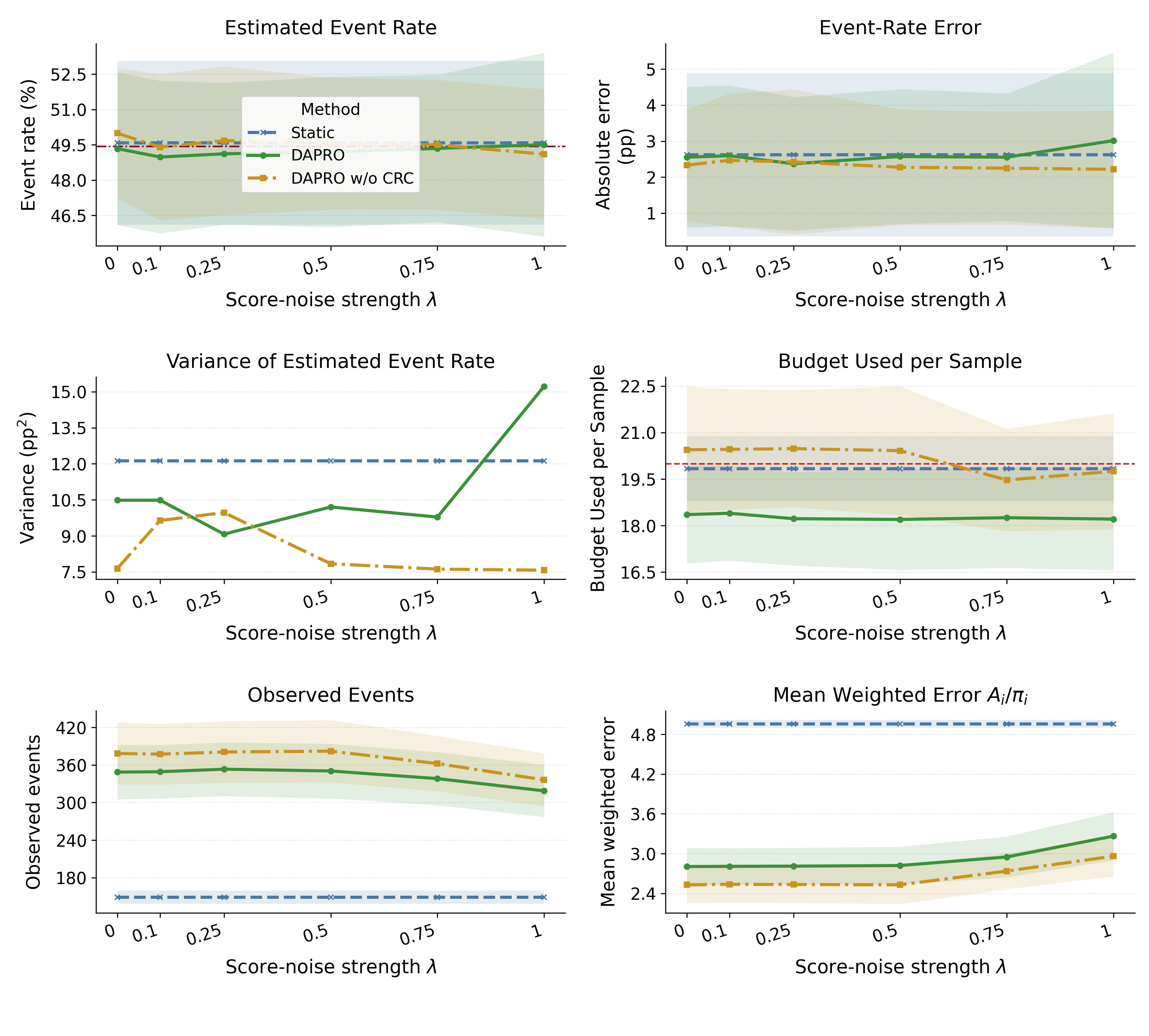}
    \caption{Metric-estimation score-noise ablation at $\bar B=20$,
    $N_1=25$, and $|\calIcrc|=25$.}
    \label{fig:ablation_score_noise}
\end{figure}

\subsubsection{Score function (metric estimation)}

Figure~\ref{fig:ablation_score} compares several options for the score function: the estimated current event
probability, the estimated remaining-time quantile, a random score, and the
oracle remaining time while setting the score-to-probability projection model at $K=2$.
For \ttmethod with CRC, the estimated remaining-time score yields the lowest
event-rate error and a slightly lower variance than the current-event score;
the current-event score nevertheless has the lowest weighted error and reveals
the most events. Random scoring has the largest variance for \ttmethod with CRC. Oracle
remaining time does not dominate because remaining-time ordering is not the
oracle ordering for the event-rate variance objective. Since rankings also
change slightly for \ttmethod without CRC, no alternative score consistently performs the best; the
current-event probability is the most robust default across diagnostics.

\begin{figure}[ht]
    \centering
    \includegraphics[width=0.98\linewidth]{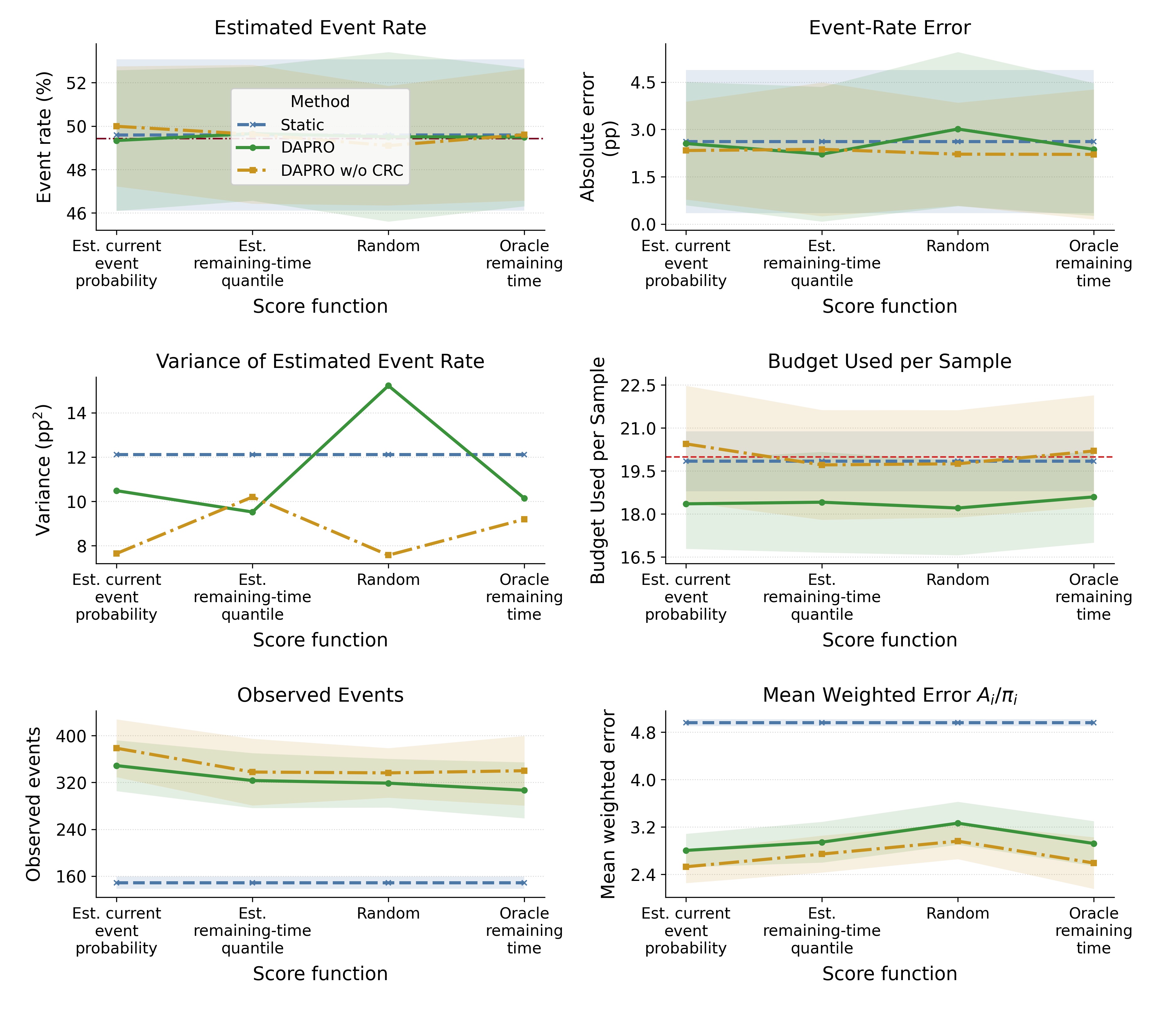}
    \caption{Metric-estimation score-function ablation at $\bar B=20$,
    $N_1=25$, $|\calIcrc|=25$, and $K=2$.}
    \label{fig:ablation_score}
\end{figure}

\subsubsection{Representation capacity (metric estimation)}

We study the effect of the representation of the score-to-probability mapper.
Figures~\ref{fig:ablation_representation}--
\ref{fig:ablation_representation_diagnostics} compare $K\in\{1,2,4,8\}$ score bins with a continuous monotone score map.  Increasing capacity from $K=1$ improves the fitted Phase-I and held-out Phase-II objectives, confirming that the score contains information discarded by a constant policy. With CRC, $K=8$ has the lowest event-rate variance, while the
continuous and $K=4$ maps are close. Without CRC, $K=2$ is the most stable choice. Hence, richer maps can improve the CRC variant, but the
improvement is not universal; $K=2$ remains a robust low-variance,
low-overfitting default employed by the main experiments.

\begin{figure}[ht]
    \centering
    \includegraphics[width=0.98\linewidth]{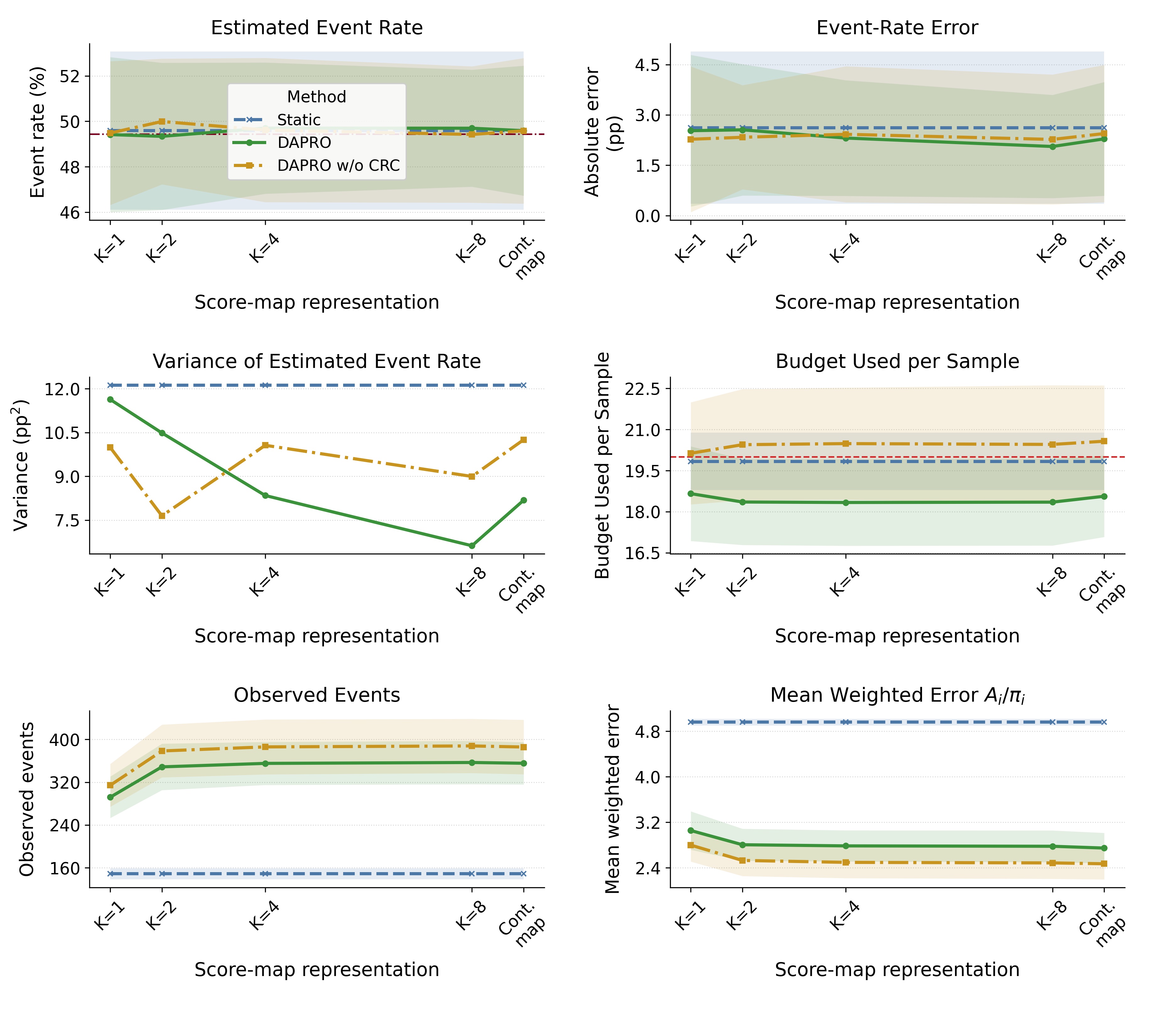}
    \caption{Metric-estimation representation-capacity ablation at
    $\bar B=20$, $N_1=25$, and $|\calIcrc|=25$.}
    \label{fig:ablation_representation}
\end{figure}

\begin{figure}[ht]
    \centering
    \includegraphics[width=0.98\linewidth]{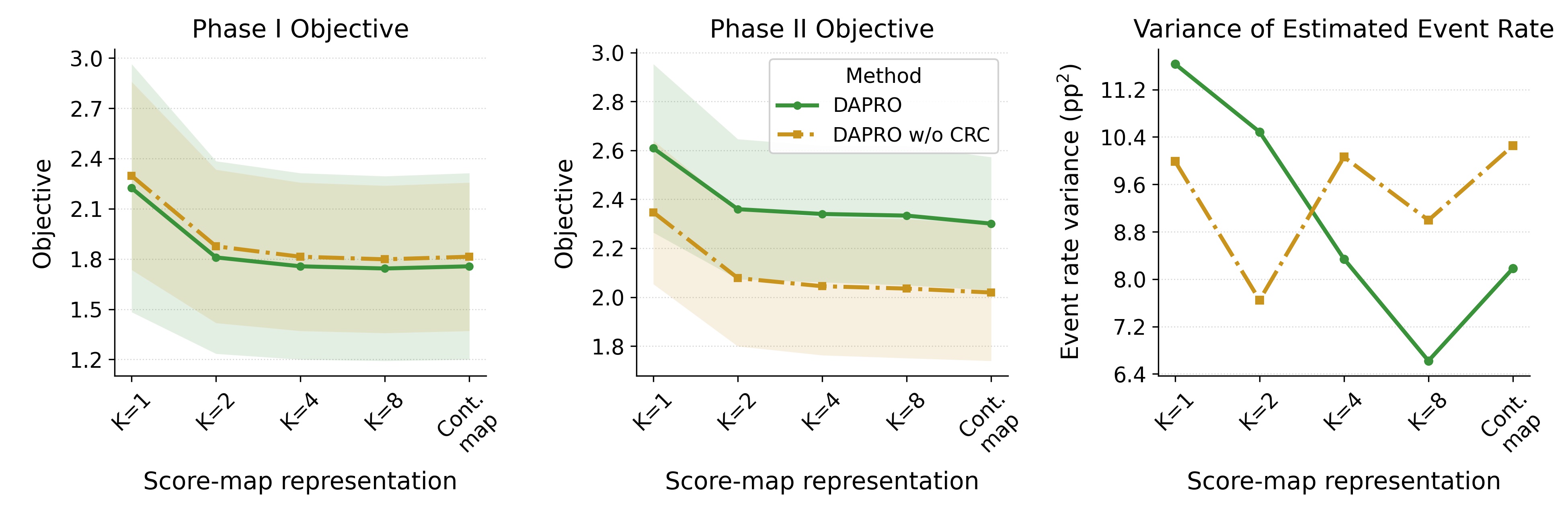}
    \caption{representation-capacity ablation for metric-estimation task: The optimization objective over the Phase I samples, the same objective computed over the Phase II data, and the event-rate variance.}
    \label{fig:ablation_representation_diagnostics}
\end{figure}

\subsubsection{CRC row-cost cap (LPB)}

Figure~\ref{fig:ablation_cmax} varies $C_{\max}$ from $0.1\bar B$ to
$10\bar B=\maxl$. The no-CRC policy is unchanged, as expected, because this
cap corresponds only to the CRC component. A very small cap prevents the policy from
following informative long trajectories and severely underuses the budget. A very large cap makes the finite-sample CRC correction loose and therefore also leads to a conservative correction of the CRC.  Values near $\bar B$--$2\bar B$ give the best tradeoff;
we use $C_{\max}=2\bar B$ because it retains low variance while allowing an
individual trajectory to receive more than the average budget.

\begin{figure}[ht]
    \centering
    \includegraphics[width=0.98\linewidth]{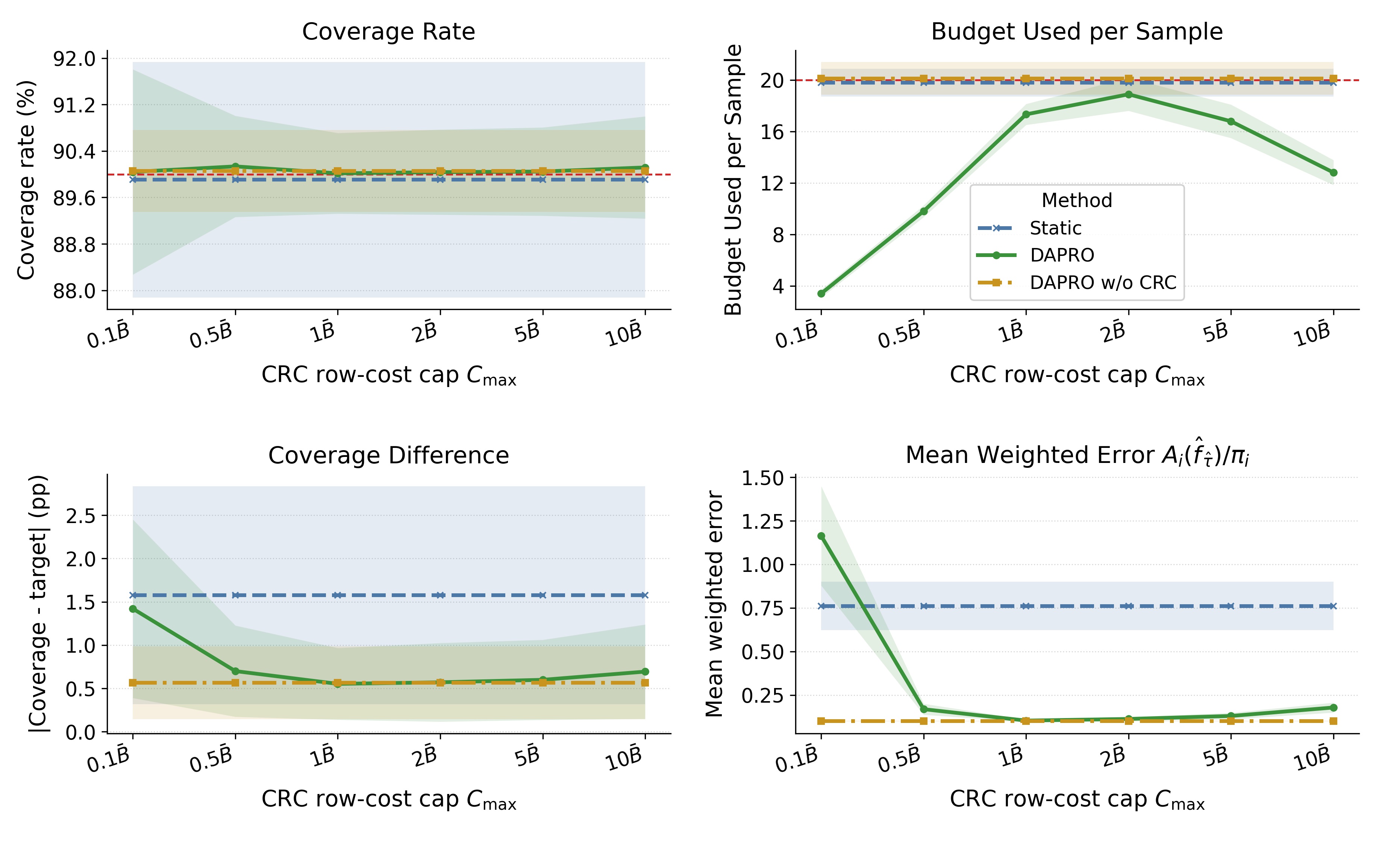}
    \caption{LPB CRC row-cost-cap ablation at $\bar B=20$,
    $N_1=25$, and $|\calIcrc|=25$.}
    \label{fig:ablation_cmax}
\end{figure}

\subsubsection{The miscoverage-rate estimator}
\label{sec:miscoverage_estimator_ablation}

In this section, we study the effect of the miscoverage rate estimator in~\eqref{eq:miscoverage_estimator} on the performance of \ttmethod and the Static method of~\citet{davidov2026calibrated} for UPB construction.
For a candidate upper bound $f$, define
\begin{equation}
    A_i(f)
    :=
    \mathbb{I}\{f<\maxl+1\}
    \mathbb{I}\{T_i>f\}.
    \label{eq:upb_miscoverage_target2}
\end{equation}
Thus, the value $f=\maxl+1$ represents an infinite UPB and has zero miscoverage.

For brevity, let $O_i(t):=\mathbb{I}\{C_i\geq t\}$, and let
$\rho_i(t)$ be the probability, under the fixed allocation policy, that trajectory $i$ is acquired through turn $t$. For the sequential \ttmethod policy,
\[
    \rho_i(t)=\prod_{s=1}^{t}P_i(s).
\]
We compare the following three unbiased estimators
based on~\citep{Horvitz01121952,robins1994estimation,wang2001augmented}.
\begin{itemize}
    \item \textbf{Ordinary HT.}
    The observed miscoverage indicator is reweighted by the
    probability of reaching the candidate bound:
    \begin{equation}
        \widetilde A_i^{\mathrm{HT}}(\tau)
        :=
        \frac{O_i(\fhat_\tau(X_i))}
             {\rho_i(\fhat_\tau(X_i))}
        A_i(\fhat_\tau(X_i)).
        \label{eq:upb_ordinary_ht}
    \end{equation}
    \item \textbf{Terminal AHT.}
    Let
    \[
        m_{i,0}(f)
        :=
        \widehat{\mathbb{P}}(T_i>f\mid H_i(0))
    \]
    be the estimated miscoverage probability. The value of
    $A_i(f)$ is known only after observing the trajectory through $
        r_i(f):=\min\{T_i,f\}.$
    Terminal AHT augments the initial prediction with an
    inverse-probability-weighted residual:
    \begin{equation}
        \widetilde A_i^{\mathrm{term}}(\tau)
        :=
        m_{i,0}(\fhat_\tau(X_i))
        +
        \frac{O_i(r_i(\fhat_\tau(X_i)))}
             {\rho_i(r_i(\fhat_\tau(X_i)))}
        \left[
            A_i(\fhat_\tau(X_i))-m_{i,0}(\fhat_\tau(X_i))
        \right].
        \label{eq:upb_terminal_aht}
    \end{equation}
    \item \textbf{Sequential AHT.}
    Let
    \[
        m_{i,t}(f)
        :=
        \widehat{\mathbb{P}}(T_i>f\mid H_i(t)).
    \]
    Once the event is observed or turn $f$ is reached, the prediction is replaced by the realized value
    $A_i(f)$. Sequential AHT uses every acquired prediction update:
    \begin{equation}
        \widetilde A_i^{\mathrm{seq}}(\tau)
        :=
        m_{i,0}(\fhat_\tau(X_i))
        +
        \sum_{t=1}^{\min (T_i, \fhat_\tau(X_i))}
        \frac{O_i(t)}{\rho_i(t)}
        \left[
            m_{i,t}(\fhat_\tau(X_i))
            -
            m_{i,t-1}(\fhat_\tau(X_i))
        \right].
        \label{eq:upb_sequential_aht2}
    \end{equation}
\end{itemize}
For $\fhat_i(\tau)=\maxl+1$, all three contributions are defined to be zero.
For any estimator
$E\in\{\mathrm{HT},\mathrm{term},\mathrm{seq}\}$, the estimated
miscoverage rate is
\begin{equation}
    \widehat\alpha_E(\tau)
    :=
    \frac{1}{|\calI|}
    \sum_{i\in\calI}
    \widetilde A_i^E(\tau).
    \label{eq:upb_miscoverage_estimator2}
\end{equation}

Under the assumptions of Theorem~\ref{thm:general_validity}, all three
estimators are unbiased for the miscoverage rate of every fixed candidate $\tau$, regardless of the model's accuracy.

We evaluate Static with ordinary HT and terminal AHT, and \ttmethod with ordinary, terminal, and sequential AHT, both with and without CRC. Static is not shown with sequential AHT because its single static probability of acquisition makes sequential AHT algebraically equivalent to terminal AHT. To isolate the effect of the estimator, all variants of a given allocation method use the same fitted policy, random seed, and realized acquisition path.

We use Toxicity with Qwen2.5 14B as both attacker and target and Detoxify as the judge. The target coverage is $80\%$, $\bar B=20$, $\tau_{\mathrm{prior}}=0.98$, and the Phase-I data size is
$|\mathcal I_{\mathrm{cal1}}|+|\mathcal I_{\mathrm{crc}}|=200$.
The CRC variant uses 100 samples in each Phase-I split, while the no-CRC variant uses all 200 samples for policy fitting. Results are computed over 50 random calibration-test partitions.
Figure~\ref{fig:upb_estimator_ablation_toxicity} reports the performance metrics of each method with each miscoverage estimator, indicating that Sequential AHT provides the lowest estimator variance for \ttmethod while preserving coverage near the $80\%$ target and producing UPBs of comparable size. One can observe from this figure that Sequential AHT does improve the coverage variance of \ttmethod compared to the alternative variations of the miscoverage estimator. Moreover, this figure shows that the Static method attains a lower coverage variance with the terminal AHT compared to the Ordinary HT.
We conclude from this experiment that the augmented HT estimators, and especially the Sequential AHT, lead to a lower coverage variance.

\begin{figure}[ht]
    \centering
    \includegraphics[width=0.98\linewidth]
    {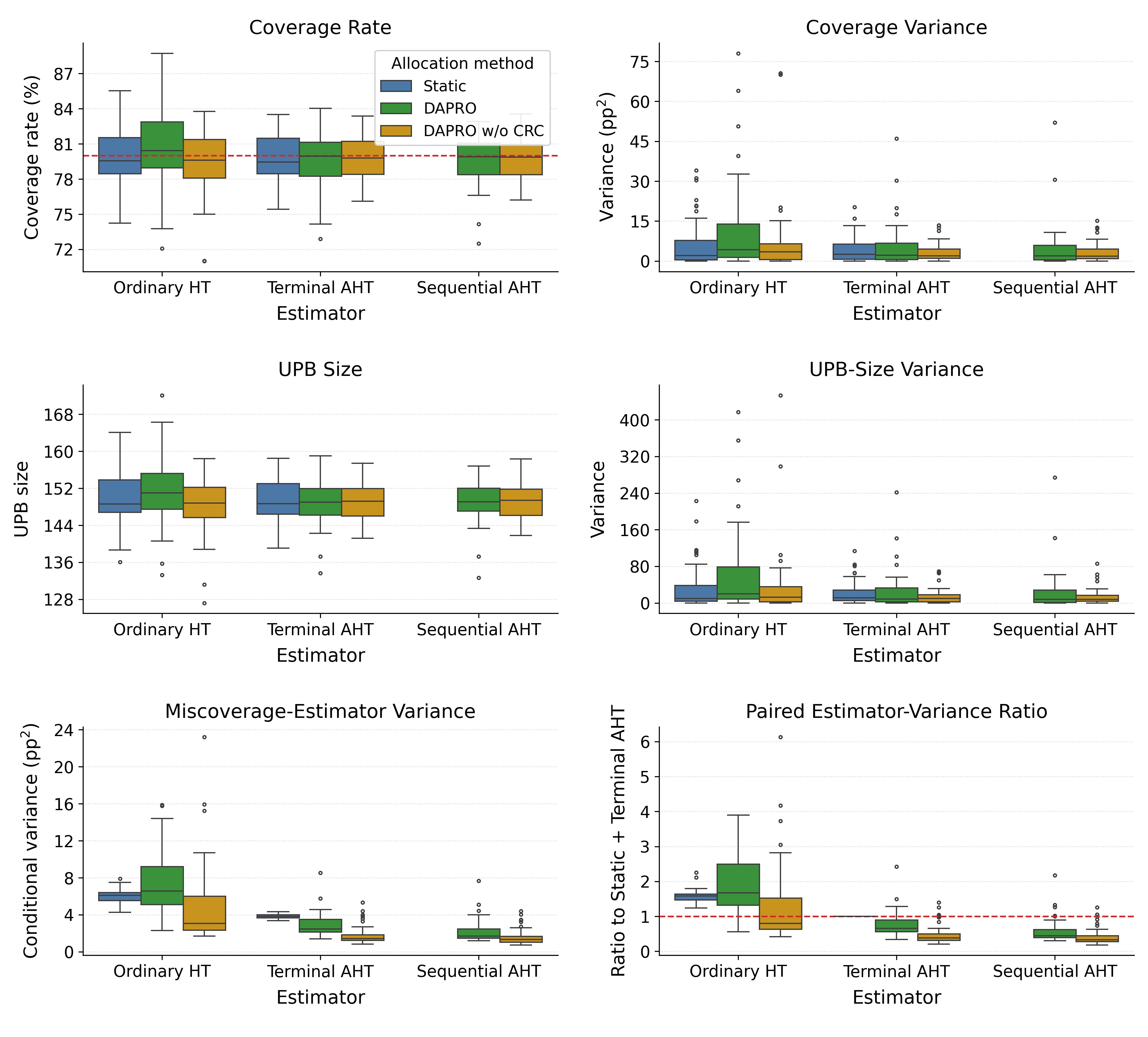}
    \caption{\textbf{UPB estimator ablation on Toxicity.}
    Coverage and its variance, UPB size and its variance,
    conditional miscoverage-estimator variance, and paired variance ratios relative to Static with terminal AHT.}
    \label{fig:upb_estimator_ablation_toxicity}
\end{figure}

\subsubsection{The effect of the training budget}
\label{sec:training_budget}
In this section, we analyze the effect of the budget allocated for the training phase on the resulting predictions of our framework. 
We train the predictive model under varying training budgets per sample, ${5,10,20,50,100,200}$, where a budget of $200$ corresponds to fully observing all training trajectories.
For each budget level, we initially allocate the training budget randomly by selecting a subset of samples whose trajectories are fully observed. Then, we redistribute, in a uniformly random order, the remaining budget resulting from early stopping caused by observing an event early or reaching the $\maxl$, across samples that did not receive budget in the initial allocation.

We note that we deliberately use this simple allocation policy for the
training phase, in which the covariates and estimated event times are not leveraged for the allocation. Thus, this experiment studies
the robustness of our framework to limited training compute rather than
proposing an optimal training-data acquisition strategy. A more intelligent adaptive acquisition policy that relies on such sample-specific details could further improve the performance of our methods under the same training budget.

To isolate the effect of training budget, we use the raw predictive-model outputs without temperature scaling. Our temperature calibrator becomes unstable when the limited training budget results in very few observed validation trajectories. For example, at a training budget of five turns per sample, the scaling is fitted using only approximately 17 validation trajectories and increases the NLL over the test set from \(1.47\) to \(2.37\).

We employ the Static approach of~\citet{davidov2026calibrated} and our proposed \ttmethod to construct a calibrated LPB for the Toxicity~\citep{gehman2020realtoxicityprompts} dataset using different predictive models. In all experiments, we aim to obtain a $90\%$ coverage rate with the restriction of $\bar{B} = 20$ budget per sample for the calibration data.

Figure~\ref{fig:training_budget} presents the performance of the LPBs constructed by methods using the varying predictive models.
Across all examined training budget levels, the variance of \ttmethod both with and without CRC is substantially lower than the Static method, while attaining the target coverage level.
Notably, this advantage of \ttmethod persists even at
$\bar B_{\mathrm{train}}=5$, indicating that \ttmethod can produce informative and stable LPBs even when the predictive model is trained under small budgets.
We remark that the figure shows small non-monotonic differences across budget levels. We believe this is attributed to the variability of the training phase, as each training-budget level is represented by a single fitted predictive model. We therefore interpret the overall trend rather than individual adjacent changes.
Overall, this experiment indicates that our framework remains useful even when only limited budget is available for training the predictive model.

\begin{figure}[ht]
    \centering
    \includegraphics[width=0.98\linewidth]{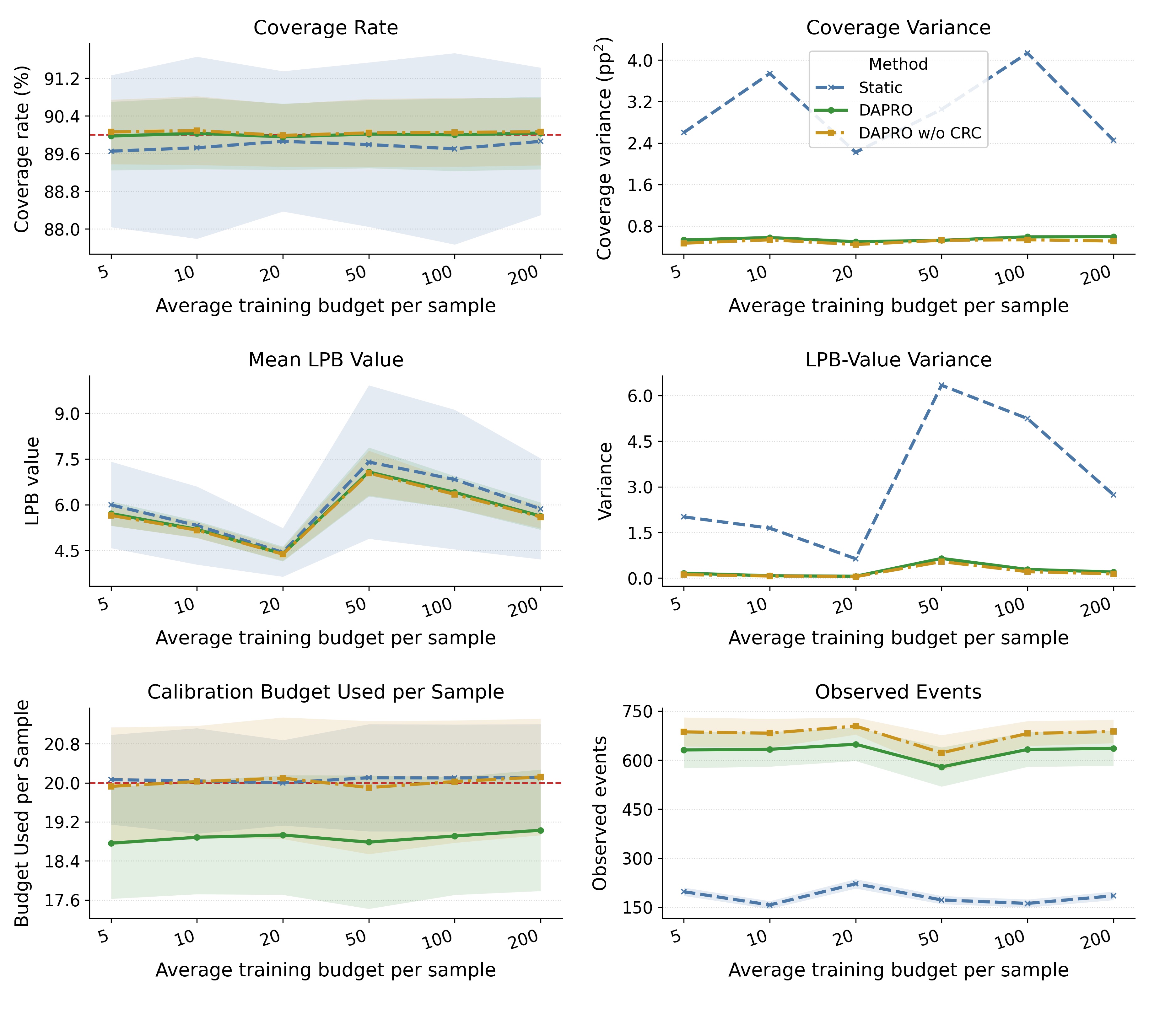}
    \caption{Training-budget ablation on the Toxicity dataset. The x-axis is the training budget per sample. The target coverage is $90\%$, and the calibration budget is fixed at $\bar{B}=20$ per sample.}
    \label{fig:training_budget}
\end{figure}

\subsubsection{Robustness to attacker shift}
\label{sec:attacker_shift}

Finally, we study the robustness of \ttmethod to changes in the attacking model. We learn and calibrate the policy on trajectories where the attacker is Qwen2.5 14B Instruct and evaluate and test it on a dataset where the attacker is Gemma 3 12B IT. We use Llama-3.1 8B Instruct as the target model.
We test both Red-Team and Toxicity at $\bar B=10$, $N_1=25$, and
$|\calIcrc|=25$. In both experiments, the judge for the Toxicity dataset is Detoxify; in the Red-Team experiments, the LLM judge is the same as the attacker model, so that experiment
changes both attacker and judge. Thus, the experiment includes a true distribution shift, and
the i.i.d. assumption between calibration and test does not hold. Thus, our \ttmethod is not theoretically guaranteed to attain the nominal coverage level.

Figure~\ref{fig:attacker_shift} shows that \ttmethod and the static baseline attain a conservative coverage: $94$--$97\%$ rather than $90\%$. \ttmethod exhibits lower coverage variance in both setups. It also observes substantially
more events and lowers the selected-target weighted error while remaining at
or below the target budget. Thus, the learned allocation remains statistically
stable under this shift, even though the LPBs in the shifted environment are more conservative.

\begin{figure}[ht]
    \centering
    \includegraphics[width=0.98\linewidth]{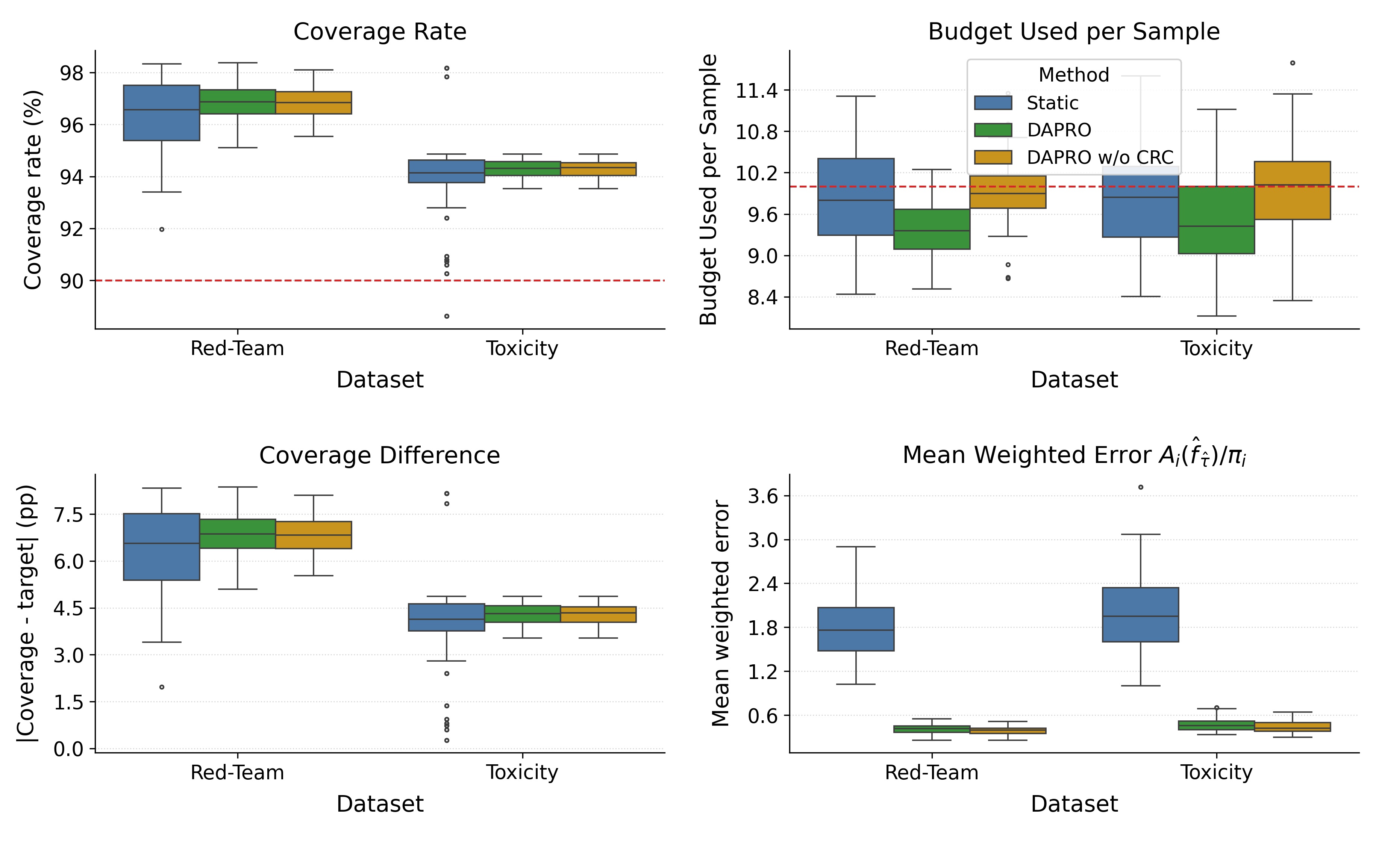}
    \caption{LPB attacker-shift ablation across the Red-Team and Toxicity
    datasets. The x-axis identifies the dataset. In both cases, the policy
    is calibrated on data where the attacker is Qwen2.5 14B Instruct and evaluated on data with Gemma 3 12B IT serving as attacker. The target coverage is $90\%$ and $\bar B=10$.}
    \label{fig:attacker_shift}
\end{figure}

\end{document}

%% file: math_commands.tex
\usepackage{amsmath,amsfonts,bm}

\def\eqref#1{equation~\ref{#1}}
\def\1{\bm{1}}

\def\eps{{\epsilon}}

\DeclareMathAlphabet{\mathsfit}{\encodingdefault}{\sfdefault}{m}{sl}
\SetMathAlphabet{\mathsfit}{bold}{\encodingdefault}{\sfdefault}{bx}{n}

\newcommand{\E}{\mathbb{E}}

\newcommand{\R}{\mathbb{R}}